\documentclass[a4paper, 12pt, oneside, openright]{memoir}
\usepackage[T1]{fontenc}
\usepackage[english]{babel}
\usepackage[pdftex]{graphicx}   
\usepackage{epsfig}             
\usepackage{amsmath}            
\usepackage{amssymb}            
\usepackage{amsthm}
\usepackage{mathtools}
\usepackage[numbers]{natbib}
\usepackage{multicol}
\usepackage[bookmarks=true]{hyperref}
\usepackage[x11names]{xcolor}
\usepackage{empheq}
\usepackage{bm}
\usepackage{siunitx}
\usepackage{enumitem}
\usepackage{nicematrix}
\usepackage{nicefrac, xfrac}
\usepackage[many]{tcolorbox}
\usepackage[noabbrev]{cleveref}
\usepackage{cancel}
\usepackage[normalem]{ulem}
\usepackage[frozencache,cachedir=.]{minted}

\usepackage{algorithm2e}
\NoCaptionOfAlgo

\def\author{Alessandro Fornasier}
\def\place{Klagenfurt}
\def\submitdate{April 2024}

\def\includetblr{}

\usepackage{tabularray}
\UseTblrLibrary{booktabs}
\usepackage{array}

\usepackage{acro}
\input{acro/acro.sty}
\acsetup{make-links}

\usepackage[T1]{fontenc}

\usepackage{MnSymbol}
\usepackage{mathdots}

\providecommand{\R}{\mathbb{R}}

\providecommand{\SO}{\mathbf{SO}}
\providecommand{\SL}{\mathbf{SL}}

\providecommand{\SE}{\mathbf{SE}}
\providecommand{\HG}{\mathbf{HG}}
\providecommand{\SOT}{\mathbf{SOT}}
\providecommand{\SIM}{\mathbf{SIM}}

\providecommand{\grpG}{\mathbf{G}}
\providecommand{\grpH}{\mathbf{H}}

\providecommand{\grpA}{\mathbf{G_{DP}}}
\providecommand{\grpB}{\mathbf{G_{SD}}}
\providecommand{\grpC}{\mathbf{G_{TF}}}
\providecommand{\grpD}{\mathbf{G_{TG}}}
\providecommand{\grpE}{\mathbf{G_{O}}}
\providecommand{\grpF}{\mathbf{G_{ES}}}
\providecommand{\IN}{\mathbf{IN}}

\providecommand{\gothso}{\mathfrak{so}}
\providecommand{\gothse}{\mathfrak{se}}

\providecommand{\gothhg}{\mathfrak{hg}}

\providecommand{\gothin}{\mathfrak{in}}

\providecommand{\gothg}{\mathfrak{g}}

\providecommand{\gothGrpA}{\mathfrak{g}_{\mathbf{DP}}}

\providecommand{\gothGrpD}{\mathfrak{g}_{\mathbf{TG}}}
\providecommand{\gothGrpE}{\mathfrak{g}_{\mathbf{O}}}
\providecommand{\gothGrpF}{\mathfrak{g}_{\mathbf{ES}}}

\providecommand{\so}{\mathfrak{so}}
\providecommand{\se}{\mathfrak{se}}
\providecommand{\sot}{\mathfrak{sot}}

\providecommand{\hg}{\mathfrak{hg}}

\providecommand{\sdpgrpG}{\mathbf{G}_{\mathfrak{g}}^{\ltimes}}

\providecommand{\calC}{\mathcal{C}}

\providecommand{\calG}{\mathcal{G}}

\providecommand{\calI}{\mathcal{I}}

\providecommand{\calM}{\mathcal{M}}
\providecommand{\calN}{\mathcal{N}}
\providecommand{\calO}{\mathcal{O}}

\providecommand{\calHG}{\mathcal{HG}}

\providecommand{\torSO}{\mathcal{SO}}

\providecommand{\torSE}{\mathcal{SE}}

\providecommand{\torIN}{\mathcal{IN}}

\providecommand{\vecL}{\mathbb{L}}

\providecommand{\PD}{\mathbb{S}_+} 

\DeclareMathOperator{\Ad}{Ad}
\DeclareMathOperator{\ad}{ad}

\providecommand{\tT}{\mathrm{T}} 
\providecommand{\td}{\mathrm{d}}
\providecommand{\tD}{\mathrm{D}}

\providecommand{\ddt}{\frac{\td}{\td t}}

\newcommand{\Fr}[2]{\left. \mathrm{D}_{#1} \right|_{#2}}

\usepackage{accents}
\makeatletter
\providecommand{\scirc}{%
    \hbox{\fontfamily{\rmdefault}\fontsize{0.4\dimexpr(\f@size pt)}{0}\selectfont{\raisebox{-0.52ex}[0ex][-0.52ex]{$\circ$}}}}

\makeatother

\mathchardef\mhyphen="2D

\providecommand{\etal}{\textit{et al.}~}

\newcommand\Tstrut{\rule{0pt}{3.5ex}}
\newcommand\Bstrut{\rule[-1.5ex]{0pt}{0pt}}

\DeclareSIUnit{\sqrts}{\ensuremath{\sqrt{\text{\second}}}}

\newcommand{\ts}{\textsuperscript}

\theoremstyle{plain}
\newtheorem{theorem}{Theorem}[section]

\newtheorem{lemma}[theorem]{Lemma}

\theoremstyle{definition}

\theoremstyle{remark}
\newtheorem{remark}[theorem]{Remark}

\tcbset{
    box/.style={
        enhanced,
        breakable=true,
        colback=black!2.5,
        colframe=black!65,
        fonttitle=\bfseries,
    }
}

\newtcbtheorem[number within=chapter, number freestyle={(E\noexpand\thechapter.\noexpand\arabic{\tcbcounter})}, crefname={example}{examples}]{boxexample}{Example}{box}{th}

\newtcbtheorem[number within=chapter, number freestyle={(I\noexpand\thechapter.\noexpand\arabic{\tcbcounter})}, crefname={intuition}{intuitions}]{boxintuition}{Intuition}{box}{th}

\newtcbtheorem[number within=chapter, number freestyle={(Q\noexpand\thechapter.\noexpand\arabic{\tcbcounter})}, crefname={question}{quetions}]{boxquestion}{Question}{box}{th}

\newtcbtheorem[number within=chapter, number freestyle={(A\noexpand\thechapter.\noexpand\arabic{\tcbcounter})}, crefname={algorithm}{algorithms}]{boxalgorithm}{Algorithm}{box}{th}

\newtcbtheorem[number within=chapter, number freestyle={(S\noexpand\thechapter.\noexpand\arabic{\tcbcounter})}, crefname={snippet}{snippets}]{boxsnippet}{Snippet}{box}{th}

\newcommand{\norm}[1]{\lVert #1 \rVert}

\newcommand{\AtoB}[2]{\;:\;#1\;\rightarrow\;#2}

\newcommand{\chrulefill}{%
  \leaders\hrule height \dimexpr\fontdimen22\scriptscriptfont2+0.2pt\relax
                 depth -\dimexpr\fontdimen22\scriptscriptfont2-0.2pt\relax
          \hfill}
\newcommand{\xoverline}[2]{%
  \mathord{
    \vbox{\offinterlineskip
      \halign{##\cr
        \chrulefill$\,\scriptscriptstyle#1\,$\chrulefill\cr
        \noalign{\kern.2ex}
        $#2$\cr
      }%
    }%
  }%
}

\newcommand{\set}[2]{\left\{ #1 \;\middle|\; #2 \right\}}
\newcommand{\st}{\;\left.\middle|\right.\;}
\newcommand{\stq}{\quad\left.\middle|\right.\quad}

\providecommand{\tT}{\mathrm{T}} 
\providecommand{\td}{\mathrm{d}}
\providecommand{\tD}{\mathrm{D}}

\providecommand{\ddt}{\frac{\td}{\td t}}

\providecommand{\Fr}[2]{\mathrm{D}_{#1}\big|_{#2}}

\newcommand{\eye}{\mathbf{I}}

\newcommand{\Adsym}[2]{\mathrm{Ad}_{#1}\left[#2\right]}
\newcommand{\adsym}[2]{\mathrm{ad}_{#1}\left[#2\right]}

\newcommand{\AdMsym}[1]{\mathbf{Ad}^\vee_{#1}}
\newcommand{\adMsym}[1]{\mathbf{ad}^\vee_{{#1}}}

\providecommand{\xizero}{\mathring{\xi}}
\providecommand{\uzero}{\mathring{u}}
\providecommand{\yzero}{\mathring{y}}

\newcommand{\Vector}[3]{\prescript{#1}{}{\bm{#2}}_{#3}}
\newcommand{\Vectorfull}[4]{\prescript{#1}{#2}{\bm{#3}}_{#4}}
\newcommand{\dotVector}[3]{\prescript{#1}{}{\dot{\bm{#2}}}_{#3}}
\newcommand{\hatVector}[3]{\prescript{#1}{}{\hat{\bm{#2}}}_{#3}}

\newcommand{\Pose}[2]{\prescript{#1}{}{\mathbf{T}}_{#2}}
\newcommand{\PoseS}[2]{\prescript{#1}{}{\mathbf{S}}_{#2}}
\newcommand{\PoseP}[2]{\prescript{#1}{}{\mathbf{P}}_{#2}}
\newcommand{\dotPose}[2]{\prescript{#1}{}{\dot{\mathbf{T}}_{#2}}}

\newcommand{\hatPose}[2]{\prescript{#1}{}{\hat{\mathbf{T}}_{#2}}}

\newcommand{\hatPoseP}[2]{\prescript{#1}{}{\hat{\mathbf{P}}_{#2}}}

\newcommand{\VAtt}[2]{\prescript{#1}{}{\mathbf{V}}_{#2}}

\newcommand{\Rot}[2]{\prescript{#1}{}{\mathbf{R}}_{#2}}
\newcommand{\dotRot}[2]{\prescript{#1}{}{\dot{\mathbf{R}}_{#2}}}
\newcommand{\hatRot}[2]{\prescript{#1}{}{\hat{\mathbf{R}}_{#2}}}

\newcommand{\frameofref}[1]{\{$#1$\}}

\providecommand{\bw}{\Vector{}{b}{\bm{\omega}}}
\providecommand{\ba}{\Vector{}{b}{a}}
\providecommand{\bnu}{\Vector{}{b}{\bm{\nu}}}
\providecommand{\angv}{\Vector{}{\omega}{}}
\providecommand{\acc}{\Vector{}{a}{}}
\providecommand{\velnu}{\Vector{}{\nu}{}}
\providecommand{\ethree}{\Vector{}{e}{3}}

\SetAlFnt{\small}

\newcommand*{\titlepage}{%
\thispagestyle{empty}
\begingroup%
\centering
{\huge\bfseries\scshape Equivariant Symmetries for Aided Inertial Navigation}\par
\vspace{1.5\baselineskip}
{\large\mdseries Dott. Mag. \author}\par
\vspace{2.0\baselineskip}
{\large\mdseries\scshape DISSERTATION}\par
\vspace{0.5\baselineskip}
{\small submitted in fulfillment of the requirements for the degree of\\[0.25\baselineskip]
{\itshape Doctor of Technical Sciences}}\par
\vspace{1.0\baselineskip}
\includegraphics[width=0.33\linewidth]{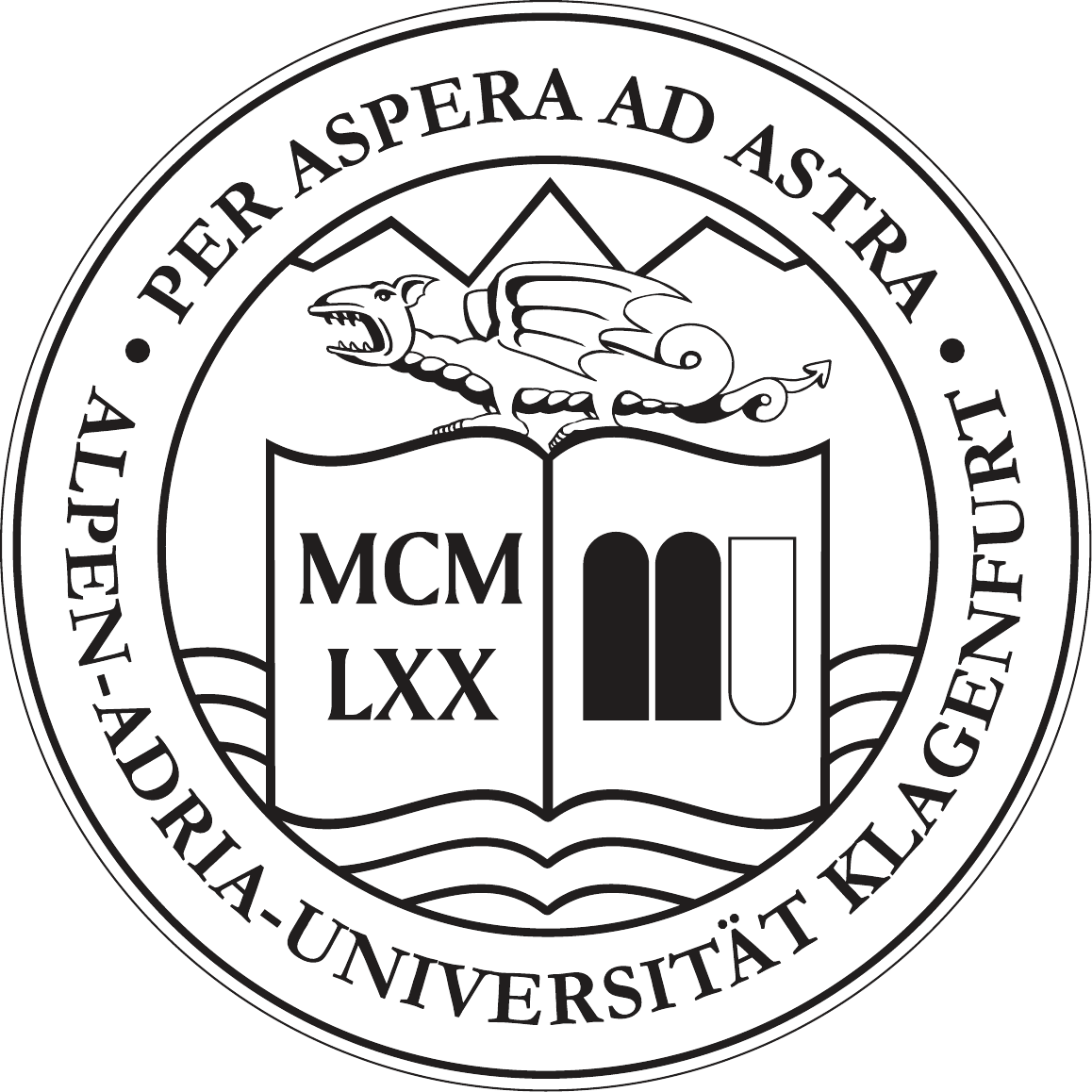}\\[1.0\baselineskip]
{University of Klagenfurt\\[0.25\baselineskip]
Faculty of Technical Sciences}\par
\vspace{2.0\baselineskip}
\begin{tblr}
{
    rows = {m},
    column{1} = {valign=t, halign=c, wd=0.45\textwidth, font=\small},
    column{2} = {valign=t, halign=c, wd=0.45\textwidth, font=\small},
}
\SetCell[c=2]{c}{Supervisors}\\[0.25\baselineskip]
{Univ.-Prof. Dr. Stephan Weiss\\University of Klagenfurt\\Department of Smart System Technologies} & 
{Prof. Robert Mahony Ph.D.\\Australian National University\\College of Engineering and Computer Science}\\[1.0\baselineskip]
\SetCell[c=2]{c}{Expert Reviewers}\\[0.25\baselineskip]
{Prof. Timothy Barfoot Ph.D.\\University of Toronto\\Institute for Aerospace Studies} & 
{Prof. Stergios Roumeliotis Ph.D.\\University of Minnesota\\Department of Computer Science and Engineering}\\
\end{tblr}\\[1.5\baselineskip]
{\place, \submitdate}\par
\vfill
\endgroup}

\begin{document}

\setsecnumdepth{subsection}
\maxsecnumdepth{subsection}

\maxtocdepth{subsection}


\chapterstyle{madsen}
\renewcommand*{\chaptitlefont}{\normalfont\Huge\bfseries\raggedleft}
\renewcommand*{\chapnamefont}{\normalfont\LARGE\scshape\raggedleft}
\renewcommand*{\printchapternum}{%
    \makebox[0pt][l]{\hspace{0.4em}%
      \resizebox{!}{4ex}{%
        \chapnamefont\bfseries\thechapter}%
    }%
  }%

\makeatletter
\newcommand{\figrefcaption}[1]{\renewcommand\fnum@figure{\figurename~\thefigure~(from~#1)}}
\newcommand{\tabrefcaption}[1]{\renewcommand\fnum@table{\tablename~\thetable~(from~#1)}}
\makeatother


\frontmatter
\titlepage
\clearpage
\chapter*{Affidavit}

I hereby declare in lieu of an oath that
\begin{itemize}
    \item the submitted academic paper is entirely my own work and that no auxiliary materials have been used other than those indicated,
    \item I have fully disclosed all assistance received from third parties during the process of writing the thesis, including any significant advice from supervisors,
    \item any contents taken from the works of third parties or my own works that have been included either literally or in spirit have been appropriately marked and the respective source of the information has been clearly identified with precise bibliographical references (e.g. in footnotes),
    \item I have fully and truthfully declared the use of generative models (Artificial Intelligence, e.g. ChatGPT, Grammarly Go, Midjourney) including the product version,
    \item to date, I have not submitted this paper to an examining authority either in Austria or abroad and that
    \item when passing on copies of the academic thesis (e.g. in printed or digital form), I will ensure that each copy is fully consistent with the submitted digital version.
\end{itemize}
I am aware that a declaration contrary to the facts will have legal consequences.\\[0.25\baselineskip]
\begin{tblr}{
    rows = {m},
    width = \linewidth,
    colspec = {X[1,l] X[1,r]}
}
\author~m.p. & \place, \submitdate
\end{tblr}
\clearpage
\chapter*{Disclosure}

I hereby declare the use of AI-assisted technologies, including Grammarly and GPT 3.5, as online spell-checking tools during the preparation of this manuscript. Note that none of the aforementioned technologies were directly used to generate complete paragraphs within this thesis.
\clearpage
\chapter*{Acknowledgments}

First and foremost, I am extremely grateful to my supervisors, Professor Stephan Weiss and Professor Robert Mahony, for their invaluable supervision, continuous inspiration, and fruitful discussions. They have guided me throughout my Ph.D. journey, fueling my curiosity and teaching me how to conduct meaningful research, ultimately shaping me as a researcher.

The Control of Networked Systems (CNS) has provided an incredibly inspiring environment and an unforgettable experience. I express my gratitude to all my colleagues, with special thanks to Eren Allak, who guided me during my internship at CNS before I became a Ph.D. candidate, sparking my interest in research. Additionally, I extend special appreciation to my office colleagues Christian Brommer, Martin Scheiber, and Giulio Delama for their continuous engagement in productive conversations, boundless enthusiasm, and friendship.

I am incredibly pleased to have met and collaborated with my colleagues from the Australian National University, Yonhon Ng, Yixiao Ge, and Pieter van Goor. Their knowledge and contributions were crucial to the work presented in this dissertation.

I express my deep gratitude to my parents, Ottorino and Dianella, my brother Simone, and my girlfriend Naira, for their invaluable life lessons, unwavering support, and profound belief in my abilities. I owe much of who I am to them.

Last but not least, I would like to express my deepest appreciation to Stergios Roumeliotis and Timothy Barfoot for reviewing this dissertation.

\clearpage
\chapter*{Abstract}
\emph{Respecting the geometry of the underlying system and exploiting its symmetry} have been driving concepts in deriving modern geometric filters for \acfp{ins}. Despite their success, the explicit treatment of \ac{imu} biases remains challenging, unveiling a gap in the current theory of filter design. In response to this gap, this dissertation builds upon the recent theory of equivariant systems to address and overcome the limitations in existing methodologies. The goal is to \emph{identify new symmetries of \aclp{ins} that include a geometric treatment of \ac{imu} biases and exploit them to design filtering algorithms that outperform state-of-the-art solutions in terms of accuracy, convergence rate, robustness, and consistency}.

This dissertation leverages the semi-direct product rule and introduces the tangent group for \aclp{ins} as the first equivariant symmetry that properly accounts for \ac{imu} biases. Based on that, we show that it is possible to derive an \ac{eqf} algorithm with autonomous navigation error dynamics. The resulting filter demonstrates superior to state-of-the-art solutions. 

Through a comprehensive analysis of various symmetries of \aclp{ins}, we formalized the concept that \emph{every filter can be derived as an \ac{eqf} with a specific choice of symmetry}. This underlines the fundamental role of symmetry in determining filter performance.

To address modern challenges in \aclp{ins} and multi-sensor fusion, we propose a general methodology to extend the presented symmetries to account for sensor calibration states, other variables of interest, and different measurement types. The established theoretical framework provides a robust foundation encompassing a wide range of inertial navigation problems, offering a versatile and unified solution for mobile robots to tackle complex scenarios.

Specifically, in the context of \ac{vins}, we present a novel Lie group symmetry and combine the idea of the multi state constraint filter methodology with the established theoretical framework to develop the \acf{msceqf}. We show that the proposed equivariant approach to the \ac{vio} problem exhibits beyond state-of-the-art robustness while achieving comparable accuracy and being consistent without artificial manipulation of the linearization point. The \ac{msceqf} was tested on common real-world datasets and was successfully deployed for closed-loop control of a resource-constrained aerial platform. 

This dissertation advances the understanding of equivariant symmetries in the context of \aclp{ins} and serves as a basis for the next generation of equivariant estimators, marking a significant leap toward more reliable navigation solutions.

\clearpage
\tableofcontents*
\printacronyms
\clearpage
\listoffigures*
\clearpage
\listoftables*
\clearpage
\chapter*{Publications}
\addcontentsline{toc}{chapter}{Publications}

Throughout my doctoral studies, I conducted research under the supervision of Professors Stephan Weiss and Robert Mahony. This research resulted in the publication of multiple papers in internationally recognized peer-reviewed journals and conference proceedings. Papers including results directly presented in this thesis are marked with a filled triangle ($\blacktriangle$), and papers outcomes of collaborations and projects are marked with an empty triangle ($\triangle$). Equal contribution is marked with an asterisk superscript in the names ($*$).

\section*{Journal papers}
\begin{itemize}[label=$\blacktriangle$, leftmargin=*]
\item A. Fornasier, P. van Goor, E. Allak, R. Mahony, and S. Weiss, “MSCEqF: A Multi State Constraint Equivariant Filter for Vision-aided Inertial Navigation”, in 2023 IEEE Robotics and Automation Letters (RA-L), 2023. \href{https://ieeexplore.ieee.org/document/10325586}{[IEEExplore]} \href{https://github.com/aau-cns/MSCEqF}{[Code]}
\item A. Fornasier, Y. Ng, C. Brommer, C. B\"ohm, R. Mahony, and S. Weiss, “Overcoming Bias: Equivariant Filter Design for Biased Attitude Estimation with Online Calibration”, in 2022 IEEE Robotics and Automation Letters (RA-L), 2022. \href{https://ieeexplore.ieee.org/abstract/document/9905914}{[IEEExplore]} \href{https://github.com/aau-cns/ABC-EqF}{[Code]}
\end{itemize}
\begin{itemize}[label=$\triangle$, leftmargin=*]
\item M. Scheiber$^{*}$, A. Fornasier$^{*}$, R. Jung, C. B\"ohm, R. Dhakate, C. Stewart, S. Weiss, and C. Brommer$^{*}$, “CNS Flight Stack for Reproducible, Customizable, and Fully Autonomous Applications”, in 2022 IEEE Robotics and Automation Letters (RA-L), 2022. \href{https://ieeexplore.ieee.org/abstract/document/9849131}{[IEEExplore]} \href{https://github.com/aau-cns/flight_stack}{[Code]}
\item C. Brommer, A. Fornasier, M. Scheiber, J. Delaune, R. Brockers, J. Steinbrener, and S. Weiss, “The INSANE dataset: Large number of sensors for challenging UAV flights in Mars analog, outdoor, and out-indoor transition scenarios”, in the International Journal of Robotics Research (IJRR), 2024. \href{https://journals.sagepub.com/doi/full/10.1177/02783649241227245}{[Sagepub]}
\end{itemize}

\section*{Conference papers}
\begin{itemize}[label=$\blacktriangle$, leftmargin=*]
\item A. Fornasier, G. Yixiao, P. van Goor, M. Scheiber, A. Tridgell, R. Mahony, and S. Weiss, “An Equivariant Approach to Robust State Estimation for the ArduPilot Autopilot System”, in 2024 IEEE International Conference on Robotics and Automation (ICRA).
\end{itemize}
\begin{itemize}[label=$\blacktriangle$, leftmargin=*]
\item A. Fornasier, Y. Ng, R. Mahony, and S. Weiss, “Equivariant filter design for inertial navigation systems with input measurement biases”, in 2022 IEEE International Conference on Robotics and Automation (ICRA), 2022. \href{https://ieeexplore.ieee.org/document/9811778}{[IEEExplorer]}
\item M. Scheiber$^{*}$, A. Fornasier$^{*}$, C. Brommer, and S. Weiss, “Revisiting multi-GNSS Navigation for UAVs -- An Equivariant Filtering Approach”, in 2023 IEEE International Conference on Advanced Robotics (ICAR), 2023. 
\end{itemize}
\begin{itemize}[label=$\triangle$, leftmargin=*]
\sloppy
\item E. Allak, A. Fornasier, and S. Weiss, “Consistent Covariance Pre-Integration for Invariant Filters with Delayed Measurements”, in 2020 IEEE International Conference on Intelligent Robots and Systems (IROS),
2020. \href{https://ieeexplore.ieee.org/abstract/document/9340833}{[IEEExplore]}
\item B. Starbuck$^{*}$, A. Fornasier$^{*}$, S. Weiss, and C. Pradalier, “Consistent State Estimation on Manifolds for Autonomous Metal Structure Inspection”, in 2021 IEEE International Conference on Robotics and Automation (ICRA), 2021. \href{https://ieeexplore.ieee.org/abstract/document/9561837}{[IEEExplore]}
\item A. Fornasier, M. Scheiber, A. Hardt-Stremayr, R. Jung, and S. Weiss, “VINSEval: Evaluation Framework for Unified Testing of Consistency and Robustness of Visual-Inertial Navigation System Algorithms”, in 2021 IEEE International Conference on Robotics and Automation (ICRA), 2021. \href{https://ieeexplore.ieee.org/abstract/document/9561382}{[IEEExplore]}
\item J. Blueml, A. Fornasier, and S. Weiss, “Bias Compensated UWB Anchor Initialization using Information-Theoretic Supported Triangulation Points”, in 2021 IEEE International Conference on Robotics and Automation (ICRA), 2021. \href{https://ieeexplore.ieee.org/abstract/document/9561455}{[IEEExplore]}
\item J. Steinbrener, C. Brommer, T. Jantos, A. Fornasier, and S. Weiss, “Improved State Propagation through AI-based Pre-processing and Down-sampling of High-Speed Inertial Data”, in 2022 IEEE International Conference on Robotics and Automation (ICRA), 2022. \href{https://ieeexplore.ieee.org/abstract/document/9811989}{[IEEExplore]}
\item J. Michalczyk, C. Sch\"offmann, A. Fornasier, J. Steinbrener, and S. Weiss, “Radar-Inertial State-Estimation for UAV Motion in Highly Agile Manoeuvres”, in 2022 IEEE International Conference on Unmanned Aircraft Systems (ICUAS), 2022.  \href{https://ieeexplore.ieee.org/abstract/document/9836130}{[IEEExplore]}
\item G. Delama, F. Shamsfakhr, S. Weiss, D. Fontanelli, and A. Fornasier, “UVIO: An UWB-Aided Visual-Inertial Odometry Framework with Bias-Compensated Anchors Initialization”, in 2023 IEEE International Conference on Intelligent Robots and Systems (IROS), 2023.
\end{itemize}

\section*{Preprints}
\begin{itemize}[label=$\blacktriangle$, leftmargin=*]
\item A. Fornasier, Y. Ge, P. van Goor, R. Mahony, and S. weiss, “Equivariant Symmetries for Inertial Navigation Systems”, submitted to Automatica, 2023. \href{https://arxiv.org/abs/2309.03765}{[ArXiv]}
\end{itemize}

\mainmatter
\chapter{Introduction}
Navigation systems are key components of modern autonomous robotics platforms and vehicles. These systems rely on effective \emph{control} algorithms that define the movement and behavior of the robotic platform. A key prerequisite to any control algorithm is an accurate estimation of the system's \emph{state} variables, including orientation, position, and velocity.

\emph{State estimation} is the problem of combining multiple sources of information to determine the state of a kinematic system. Exploiting the modeled dynamics and optimally fusing the measurements from different sensors are the fundamental principles behind many estimators.

Understanding the core concepts of state estimation is essential to comprehend the significant advancements achieved in the field of \acp{ins}. This understanding becomes relevant when considering the historical milestones marked by the introduction of the \ac{kf}~\cite{Kalman1960AProblems, Kalman1961NewTheory} for linear systems in the early sixties, and its counterpart for nonlinear systems, the \ac{ekf}~\cite{Smith1962ApplicationVehicle, McElhoe1966AnVenus}, which have had a long history of successful practical application~\cite{Grewal2010ApplicationsPresent}.

The limitations of using Euler angles as the parameterization for the \ac{ekf} were clear from the start, and this motivated the consideration of alternative attitude parameterisations~\cite{markley2003attitude} that respects the geometry of the problem. An important contribution was made in the early eighties in the development of the so-called \ac{mekf}~\cite{Lefferts1982KalmanEstimation} that exploited the group structure of the quaternion representation of global attitude. The \ac{mekf} framework became the industry-standard methodology for \ac{ins} filter design for 30 years.

The lack of strong stability properties and the requirement for extensive tuning of \ac{ekf}-like estimators, together with a general interest in geometric methods, have motivated the study of nonlinear geometric observers with convergence guarantees for deterministic inertial navigation problems. This research challenge led to the development of robust nonlinear observers~\cite{mahony2008nonlinear, batista2012ges}, that explicitly used the structure of $\SO(3)$, the Lie group of 3D rotations, to provide reliable and simple attitude estimation for real-world commercial grade vehicles. 
The resulting insight led to a range of new results including observers on the Special Euclidean group ${\SE(3)}$ for navigation problems with velocity input~\cite{Bras2013NonlinearMotion, Vasconcelos2010AMeasurements, Hua2011ObserverSE3}, and on the Special linear group ${\SL(3)}$ for homography estimation~\cite{Hamel2011HomographyCorrespondence, Mahony2012NonlinearGroup, Hua2020NonlinearStabilization}. 

Nonlinear geometric observers, however, do not generalize well and are often restricted to specific applications or limited classes of systems. 

Bonnabel \etal~introduced the concept of symmetry-preserving observer~\cite{bonnabel2008symmetry} and, as a first attempt to exploit the symmetry of these systems in the context of stochastic filtering, developed the left \ac{iekf} for attitude estimation~\cite{bonnabel2007left, Bonnabel2009InvariantProblem}. This work was consolidated into the general theory of the \acl{iekf} for a class of group affine systems~\cite{7523335, barrau:tel-01247723} by Barrau and Bonnabel, and they showed that the \ac{iekf} is locally asymptotic stable~\cite{7523335} in a deterministic setting. 

Since then, it has been clear that \emph{respecting the geometry of the problem and exploiting its symmetry} are fundamental concepts for stochastic estimators with improved linearization error and convergence properties. Recently, Mahony, Trumpf and Hamel introduced the theory of \emph{equivariant systems} and \emph{equivariant observer} design~\cite{Mahony2013ObserversSymmetry, Trumpf2019OnSymmetry, Mahony2020EquivariantDesign, Mahony2021EquivariantGroups, Mahony2022ObserverEquivariance}, which led to the development of the \ac{eqf} by van Goor \etal~\cite{VanGoor2020EquivariantSpaces, Mahony2021EquivariantGroups, vanGoor2022EquivariantEqF}; a general filter design for systems on homogeneous space with symmetries.

This dissertation builds upon the theory of equivariant systems and significantly contributes to the field of \aclp{eqf} for \aclp{ins}.

\section{Motivation}

In the context of \aclp{ins}, the introduction of the extended special Euclidean group $\SE_2(3)$~\cite{barrau2015ekf, mahony2017geometric}, and the understanding that inertial navigation problems are \emph{group affine}~\cite{7523335, barrau:tel-01247723} when directly posed on the group, marked a major contribution in the field. This opened the door to the study of symmetries of \aclp{ins} and their use in geometric estimation algorithms for inertial navigation problems.

Unfortunately, modeling \aclp{ins} directly on $\SE_2(3)$ fails to account for \ac{imu} biases, and extending the state space to model \ac{imu} biases as elements of an Euclidean space in a direct product structure, breaks the group affine property of the unbiased \ac{ins} on $\SE_2(3)$~\cite{barrau:tel-01247723}.
Nevertheless,
\ac{imu} biases are crucial components to account for in modern \ac{ins}, hence modeling inertial navigation problem on $\SE_2(3)$ with \ac{imu} biases as vector elements and applying the \ac{iekf} design methodology resulted in a filter, termed an Imperfect-\ac{iekf}~\cite{barrau:tel-01247723}, that has been demonstrated across multiple applications~\cite{barrau:tel-01247723, doi:10.1177/0278364919894385, Wu2017AnConsistency, Brossard2018InvariantSLAM, Pavlasek2021InvariantEstimationd, Liu2023InGVIO:Odometry} and is accepted as the state-of-the-art filter for \ac{ins}.
In recent work, Barrau and Bonnabel proposed the \ac{tfgiekf}~\cite{Barrau2022TheProblems}, a variation of the original \ac{iekf} for two frames systems, that provides a more structured treatment of the symmetry of \ac{imu} biases.

Despite the effort made, the theory of filter design for biased \aclp{ins} remains \emph{incomplete}. It is without any doubt that a proper geometric description of the \ac{imu} biases in \aclp{ins} is highly desired and would lead to the next generation of \ac{ins} filtering algorithms. As mentioned above, this dissertation builds upon the recent theory of equivariant systems and seeks to achieve the aforementioned goal; that is, in its essence, \emph{finding new symmetries for \aclp{ins} that properly include a geometric treatment of the \ac{imu} biases and exploit them to design filtering algorithms that respect the geometry of the problem, hence yielding lower linearization error, and outperform the state-of-the-art in terms of accuracy, robustness, and consistency}

\section{Contributions}
This dissertation contributes to the field of \acl{eqf} design for \aclp{ins} through the following core contributions:
\begin{itemize}[label=$\bullet$, leftmargin=*]
\sloppy\item \emph{Introduction of the Tangent group for biased \aclp{ins}, as the first Lie group symmetry that properly accounts for \ac{imu} biases.}
\item \emph{Development of the first \acl{eqf} algorithm for biased \aclp{ins} with autonomous navigation error dynamics.}
\item \emph{Introduction of multiple symmetries for biased \aclp{ins}, based on variation of the Tangent group. Theoretical assessment and discussion to identify the optimal symmetry for filter design through linearization error analysis.}
\item \emph{Formalization of the concept that every filter is equivariant. Every modern geometric \ac{ins} filter can be derived as an \ac{eqf} with a specific choice of symmetry, and the choice of symmetry structure is, in fact, the only difference between different versions of modern geometric \ac{ins} filters.}
\item \emph{Introduction of general a methodology to cast global-referenced measurements into body-referenced measurements that are compatible with the symmetry, leading to a unified framework suitable for various measurement types}.
\item \emph{Extension of the semi-direct product symmetry group to account for generic (translation, rotation, and roto-translation) sensor calibration states and extra state variables of interest}.
\item \emph{Design, implementation, and testing of a robust, self-calibrating, multi-sensor fusion algorithm based on the \acl{eqf}. A case study for the ArduPilot autopilot system.}
\item \emph{Theoretical derivation, design, implementation, and testing of the first \ac{msceqf} for vision-aided inertial navigation.}
\end{itemize}
Bridging the gap between theoretical results and novel state estimation algorithms for inertial navigation problems has always been part of our goal. In this regard, we made source-available implementations of the algorithms discussed in this dissertation. In particular, an educational implementation of \ac{eqf} for biased attitude system (\href{https://github.com/aau-cns/ABC-EqF}{https://github.com/aau-cns/ABC-EqF}), a header-only C++ library for symmetry groups (\href{https://github.com/aau-cns/Lie-plusplus}{https://github.com/aau-cns/Lie-plusplus}) to bootstrap the development of \aclp{eqf}, and finally, a real-time C++ implementation of the \ac{msceqf} for vision-aided inertial navigation, with wrappers for the most used middlewares in robotics (\href{https://github.com/aau-cns/msceqf}{https://github.com/aau-cns/msceqf}).

\section{Outline}
Besides this introductory chapter, the dissertation is divided into the following chapters.

\textbf{Chapter 2 and Chapter 3}. \emph{These two chapters introduce the fundamentals of this dissertation}.
In particular, \cref{math_chp} introduces the notation used throughout this work together with concepts in the theory of Lie groups that are needed for the understanding of the work presented in this dissertation. Finally, \cref{math_chp} concludes with a discussion of all the Lie groups that are important in the context of \aclp{ins}. In \cref{eq_chp}, the recent results in the theory of Equivariant systems are reviewed. The concept of symmetry for a kinematic system and how a symmetry is exploited to lift an estimation problem from the original state space to the Lie group are discussed. Furthermore, in \cref{eq_chp}, the recently introduced \acl{eqf} algorithm, upon which the results in this dissertation are based, is discussed.

\textbf{Chapter 4 and Chapter 5}. \emph{These two chapters introduce the principal theoretical results of this dissertation -- the symmetries of \aclp{ins} and their exploitation in \aclp{eqf} design.} \Cref{bas_chp} starts introducing the first and easiest example of biased \acl{ins}, that is the attitude kinematics with gyroscope bias. For this system, a novel symmetry based on the tangent group of $\SO(3)$ given by $\SO(3) \ltimes \so(3)$ is derived. \Cref{bas_chp} continues describing how the tangent symmetry is extended to account for the sensor's calibration states and how the overall symmetry is exploited to design an \acl{eqf} capable of handling a generic number of body-frame and reference-frame direction type measurements and sensor's self-calibration. Finally, \cref{bas_chp} concludes with an experimental evaluation of the proposed \ac{eqf}, and a comparison with the state-of-the-art \ac{iekf} on both simulated data and real-world field data.

\Cref{bins_chp} provides a generalization of the results in \cref{bas_chp}, presenting the general theory of symmetries and \acl{eqf} design for second-order \aclp{ins}. Specifically, we discuss different known symmetries for the \ac{ins} problem and present novel symmetries based on the tangent group of $\SE_2(3)$. \Cref{bins_chp} provides a detailed analysis of the linearization error of \aclp{eqf} designed exploiting the presented symmetries. 
This analysis not only serves as a performance indicator but also highlights the main results of \cref{bins_chp}. 
To begin, exploiting the tangent group of $\SE_2(3)$ results in a filter with exact linearization of the navigation states. Secondly, every modern filter (\ac{mekf}, Imperfect-\ac{iekf}, \ac{tfgiekf}) can be derived as an \ac{eqf} for a specific choice of symmetry. Therefore, the only difference among modern geometric filters lies in the choice of symmetry. \Cref{bins_chp} continues discussing the problem of position-based navigation as a direct application of the presented theoretical results. This application serves as a convenient framework to introduce a new interesting result, that is, how global-referenced measurements of state variables are reformulated as fixed body-referenced measurements that are compatible with the presented symmetries. Finally, \cref{bins_chp} concludes with a Monte-Carlo simulation of the position-based navigation problem, confirming the theoretical findings of the linearization error analysis. 

\textbf{Chapter 6}. \emph{This chapter acts as a bridge between theoretical results and practical inertial navigation applications.} \Cref{ms_bins_chp} discusses how the symmetries presented in \cref{bins_chp} are extended to account for generic sensor extrinsic parameters and extra state of interest often present in modern \aclp{ins}. Continuing, \cref{ms_bins_chp} presents a comprehensive case study — the design of a robust, self-calibrating, multi-sensor fusion, \acl{eqf} algorithm for the ArduPilot autopilot system. This case study combines all the results discussed in this dissertation and exemplifies the real-world application of the developed theoretical framework. 

\textbf{Chapter 7} \emph{This chapter presents a novel symmetry for \aclp{vins} and introduces the \acf{msceqf}, a novel \acl{vio} algorithm}. In \cref{msceqf_chp}, the results of previous chapters are combined, introducing a novel Lie group symmetry for \aclp{vins} that accounts for camera intrinsic and extrinsic parameters. Moreover, the idea of a multi state constraint filter is combined with the \acl{eqf} framework to design a novel system for the \ac{vio} problem. To handle zero motion scenarios, a novel equivariant formulation of the \acl{zvu} is derived. Additionally, the consistency properties are analyzed, and the proposed symmetry is shown to be compatible with the invariance to reference frame transformations of the \ac{vio} problem. Thus, the resulting filter is naturally consistent and does not require any observability-enforcing techniques. 

The results presented in \cref{msceqf_chp} demonstrate that the \ac{msceqf} exhibits beyond state-of-the-art robustness against expected and unexpected errors while achieving accuracy comparable to state-of-the-art \ac{vio} filter algorithms on commonly used real-world datasets. Moreover, the real-time capabilities for closed-loop control or a resource-constrained aerial platform have been experimentally demonstrated.

\Cref{msceqf_chp} concludes proposing an extension of the Lie group symmetry for \aclp{vins} to include fixed visual features and introduces a hybrid version of the \ac{msceqf}.

\chapter[Notation and Mathematical Preliminaries][Notation and Preliminaries]{Notation and Mathematical Preliminaries}\label{math_chp}

This chapter introduces the notation adopted throughout this dissertation, as well as a series of mathematical concepts that are fundamental for understanding the results presented in this dissertation.

\section{Vectors, matrices, and transformations of space}

\subsection{Vectors and matrices}
Vector spaces are denoted with blackboard letters such as ${\mathbb{X}}$. Vectors expressing elements of some $n$-dimensional vector space $\mathbb{X} \subset \R^n$ are denoted by lowercase bold letters such as $\Vector{}{x}{}$. Matrices representing linear operators on vector spaces are denoted by capital bold letters such as $\mathbf{X}$.

Vectors describing physical quantities $\Vector{}{x}{}$ of a target object \frameofref{B}, observed from an \emph{observer frame of reference} \frameofref{O} and expressed in a \emph{coordinate frame of reference} \frameofref{A}, are denoted by $\Vectorfull{A}{O}{x}{B}$. Note that when the observer frame and the coordinate frame coincide, the left subscript is dropped ${\Vectorfull{A}{A}{x}{B} = \Vector{A}{x}{B}}$.

In particular, a vector encoding the translation between a frame of reference \frameofref{A} and a frame of reference \frameofref{B} expressed in the frame of reference \frameofref{A} is denoted $\Vector{A}{p}{B}$. Similarly, a vector encoding the relative velocity of a moving frame of reference \frameofref{B} expressed in a frame of reference \frameofref{A} is denoted by $\Vector{A}{v}{B}$. 

\subsection{Rotation matrices}
Rotation matrices encoding the orientation of a frame of reference \frameofref{B} with respect to a reference \frameofref{A} are denoted by $\Rot{A}{B}$. In particular, $\Rot{A}{B}$ is the matrix whose columns are the orthonormal vectors defining the axes of \frameofref{B} expressed in \frameofref{A}:
\begin{equation*}
    \Rot{A}{B} = 
    \begin{bmatrix}
        \Vectorfull{A}{B}{e}{B_1} & \Vectorfull{A}{B}{e}{B_2} & \Vectorfull{A}{B}{e}{B_3}
    \end{bmatrix} \in \R^{3 \times 3} .
\end{equation*}

\subsection{Poses, extended poses and Galilean matrices}
Pose matrices encode simultaneously the orientation and the translation of a frame of reference \frameofref{B} with respect to a reference \frameofref{A}, and are denoted by $\PoseP{A}{B}$. 
\begin{equation*}
    \PoseP{A}{B} = 
    \begin{bmatrix}
        \Rot{A}{B} & \Vector{A}{p}{B}\\
        \Vector{}{0}{1 \times 3} & 1
    \end{bmatrix} \in \R^{4 \times 4} .
\end{equation*}

Galilean matrices encode simultaneously the orientation and the relative velocity of a frame of reference \frameofref{B} with respect to a reference \frameofref{A}, and are denoted by $\VAtt{A}{B}$. In particular, $\VAtt{A}{B}$ is written
\begin{equation*}
    \VAtt{A}{B} = 
    \begin{bmatrix}
        \Rot{A}{B} & \Vector{A}{v}{B}\\
        \Vector{}{0}{1 \times 3} & 1
    \end{bmatrix} \in \R^{4 \times 4} .
\end{equation*}

Extended pose matrices encode simultaneously the orientation and the translation and the relative velocity of a frame of reference \frameofref{B} with respect to a reference \frameofref{A}, and are denoted by $\Pose{A}{B}$. In particular, $\Pose{A}{B}$ is written
\begin{equation*}
    \Pose{A}{B} = 
    \begin{bmatrix}
        \Rot{A}{B} & \Vector{A}{v}{B} & \Vector{A}{p}{B}\\
        \Vector{}{0}{1 \times 3} & 1 & 0\\
        \Vector{}{0}{1 \times 3} & 0 & 1
    \end{bmatrix} \in \R^{5 \times 5} .
\end{equation*}

\subsection{Transformations of space}
Although often intertwined in the robotics literature, the notion of space transformation refers to a distinct concept than rotations, poses, extended poses, and Galilean matrices of a rigid body, previously introduced. \emph{Transformation of spaces refers to geometric transformations of a Euclidean space that preserves relative distances and relative angles. Space transformations are not physical variables, but they are maps acting on physical variables. Specific transformations form a Lie group and can be expressed in terms of some algebraic parametrization}.

\begin{boxexample}{Rotation transformation}{}
    A rigid body rotation is defined as the map ${R \AtoB{\R^3}{\R^3}}$. In order to be able to express the rotation transformation as the action of rotating a physical variable $\Vector{}{x}{}$, we need to define an algebraic parametrization, which is the matrix form $\Rot{}{} \in \R^{3 \times 3}$. Therefore, $R\left(\Vector{}{x}{}\right) \coloneqq \Rot{}{}\Vector{}{x}{}$.
\end{boxexample}
\begin{boxexample}{Rigid body transformation}{}
    A rigid body transformation is defined as the map ${P \AtoB{\R^3}{\R^3}}$. In order to be able to express the roto-translation transformation as the action of rotating and translating a physical variable $\Vector{}{x}{}$, we need to define an algebraic parametrization, which is given by the matrix elements $\left(\Rot{}{}, \Vector{}{p}{}\right)$, or by the matrix form $\PoseP{}{} \in \R^{4 \times 4}$, if the physical vector $\Vector{}{x}{}$ is expressed in homogenous coordinates\cite{Hartley2004MultipleVision}. Therefore, $P\left(\Vector{}{x}{}\right) \coloneqq \Rot{}{}\Vector{}{x}{} + \Vector{}{p}{}$.
\end{boxexample}

Although the previously introduced rotations, poses, extended poses, and Galilean matrices are not space transformations, the algebraic representation of transformations can be used to place coordinates on rotations, poses, extended poses, and Galilean matrices. 

\begin{boxexample}{Coordinates of pose}{}
    Imagine fixing a reference frame \frameofref{A}, finding the transformation that transforms \frameofref{B} to \frameofref{A} and defining an algebraic representation $\left(\Rot{A}{B}, \Vector{A}{p}{B}\right)$. This representation provides coordinates for the pose encoding the orientation and the translation of \frameofref{B} with respect to \frameofref{A}.
\end{boxexample}

\emph{This way of construction coordinates generalizes to any group that acts via a free and transitive action on a manifold}.

\section{Smooth manifolds}
A smooth (differentiable) manifold is a $n$-dimensional topological space that locally resembles some Euclidean space $\R^n$~\cite{JohnMLee2012IntroductionManifolds}. Smooth manifolds are denoted with calligraphic letters such as $\mathcal{M}$. Elements of smooth manifolds are denoted by lowercase regular letters such as $\xi$.

Let $\xi \in \calM$ be an element of a smooth manifold $\calM$, the tangent space of $\calM$ at $\xi$ is denoted $\tT_{\xi}\calM$. The tangent bundle of $\calM$, denoted $\tT\calM$, is defined as the disjoint union of the tangent spaces at all points of $\calM$.
A smooth vector field on $\calM$ is a smooth map ${f \AtoB{\calM}{\tT\calM}}$ such that ${f \in \tT_{\xi}\calM}$. The set of all smooth vector fields on $\calM$ is denoted ${\mathfrak{X}\left(\calM\right)}$.
Let ${h \AtoB{\calM}{\calN}}$ be a continuous and differentiable map between manifolds. The differential is written ${\td h \AtoB{\tT\calM}{\tT\calN}}$. Given ${\xi \in \calM}$ and ${\eta_{\xi} \in \tT_{\xi}\calM}$, the differential $\td h$ is evaluated pointwise by the \emph{Férchet derivative} as
\begin{equation*}
    \td h\left(\xi\right)\left[\eta_{\xi}\right] = \Fr{\zeta}{\xi}h\left(\zeta\right)\left[\eta_{\xi}\right] \in \tT_{h\left(\xi\right)}\calM .
\end{equation*}
In different terms, $\td h\left(\xi\right)\left[\eta_{\xi}\right]$ refers to the \emph{differential of h with respect to the argument $\zeta$, evaluated at $\xi$, in the direction $\eta_{\xi}$}.

We refer the reader to~\cite{JohnMLee2012IntroductionManifolds} for an in-depth introduction to smooth manifolds.

\section{Lie theory}

The laws of motion that govern physical systems often vary in a structured manner according to specific transformations. The formalism to describe those transformations is provided by the theory of Lie groups, pioneered by the mathematician Sophus Lie at the end of the 19th century. This theory constitutes an extremely important tool for modern robotics and physics. In the following subsections, we will elaborate on the major concepts.

\subsection{Lie groups}
A Lie group $\grpG$ is a smooth manifold endowed with a smooth group multiplication that satisfies the group axioms~\cite{Chirikjian2011StochasticApplications}. For any $X, Y \in \grpG$, the group multiplication is denoted $X \circ Y$, the group inverse $X^{-1}$, and the identity element $I$. For a given Lie group $\grpG$, the $\grpG$-Torsor~\cite{Mahony2013ObserversSymmetry}, denoted $\calG$, is defined as the set of elements of $\grpG$ (the underlying smooth manifold), but without the group structure, see \cref{math_torsor_example}.

Matrix Lie groups are a particular case of Lie groups, whose elements can be represented by invertible square matrices and where the group operation $\circ$ is the matrix multiplication. In the context of this dissertation, we limit our consideration to matrix Lie groups and products of matrix Lie groups.

Elements of Lie groups and matrix Lie groups are always represented with capital non-bold letters such as $X$.

\begin{boxexample}[label={math_torsor_example}]{$\grpG$-Torsor}{}
    In the context of kinematic systems with symmetries, the concept of $\grpG$-Torsor plays an important role. With this example, we try to clarify the difference between what a $\grpG$-Torsor represents and what a Lie group $\grpG$ represents. Consider the rotational kinematics of a moving platform, its state space representation is written as follows:
    \begin{equation*}
        \dotRot{G}{B} = \Rot{G}{B}\Vector{B}{\omega}{}^{\wedge}.
    \end{equation*}
    In this case, $\Rot{G}{B}$ is a rotation matrix that encodes the orientation of the platform's body frame \frameofref{B} with respect to some global frame of reference \frameofref{G}, and therefore is an element of $\torSO(3)$, the $\SO(3)$-Torsor. For such a system, moreover, we can define a Lie group that acts on the state space of the system; thus, we can define $R$ as an element of the Lie group $\SO(3)$. While $R$ represents a transformation that acts on the state space and hence is an element of a Lie group, $\Rot{G}{B}$ represents the physical orientation of \frameofref{B} with respect to \frameofref{A}, and hence is an element of the smooth manifold $\torSO(3)$. More details about Lie groups, Lie algebras, and Lie group actions follow.
\end{boxexample}

\subsection{Lie algebras}
For a given Lie group $\grpG$, the Lie algebra, denoted $\gothg$, is a vector space corresponding to the tangent space at the identity of the group, together with a bilinear non-associative map ${[\cdot, \cdot] \AtoB{\gothg}{\gothg}}$ called Lie bracket. 
The Lie algebra $\gothg$ of a $n$-dimensional Lie group $\grpG$, is isomorphic to a vector space $\R^{n}$.
Define the \emph{wedge} map and its inverse, the \emph{vee} map, as linear mappings between the vector space and the Lie algebra:
\begin{align*}
    &\left(\cdot\right)^{\wedge} \AtoB{\R^{n}}{\gothg} \stq \Vector{}{u}{} \mapsto \Vector{}{u}{}^{\wedge},\\
    &\left(\cdot\right)^{\vee} \AtoB{\gothg}{\R^{n}} \stq \Vector{}{u}{}^{\wedge} \mapsto \Vector{}{u}{}. 
\end{align*} 

Elements of Lie algebras are generally represented with lowercase letters such as $u$. Bold letters with a wedge superscript, such as $\Vector{}{u}{}^{\wedge}$, are used only to make explicit the physical nature of the vector variable.

\subsection{Group translations}
In general, for any $X, Y \in \grpG$, define the \emph{left translation} and the \emph{right translation} as
\begin{align*}
    &\textrm{L}_{X} \AtoB{\grpG}{\grpG}, && \textrm{L}_{X}\left(Y\right) = X \circ Y ,\\
    &\textrm{R}_{X} \AtoB{\grpG}{\grpG}, && \textrm{R}_{X}\left(Y\right) = Y \circ X . 
\end{align*}
In the context of matrix Lie groups, the left translation and the right translation are written
\begin{align*}
    &\textrm{L}_{X} \AtoB{\grpG}{\grpG}, && \textrm{L}_{X}\left(Y\right) = XY ,\\
    &\textrm{R}_{X} \AtoB{\grpG}{\grpG}, && \textrm{R}_{X}\left(Y\right)= YX . 
\end{align*}
Follows that the differential of the left and right translations are derived as follows:
\begin{align*}
    &\td \textrm{L}_{X} \AtoB{\gothg}{\gothg}, && 
    \begin{aligned}[t]
        \td\textrm{L}_{X}\left[\eta\right] &= \Fr{Y}{I}\textrm{L}_{X}(Y)\left[\eta\right]\\
        &= \lim_{t\to0}\frac{1}{t}\left[X(I - \eta t) - X\right] = X \eta ,
    \end{aligned}\\
    &\td \textrm{R}_{X} \AtoB{\gothg}{\gothg}, && 
    \begin{aligned}[t]
        \td\textrm{R}_{X}\left[\eta\right] &= \Fr{Y}{I}\textrm{R}_{X}(Y)\left[\eta\right]\\
        &= \lim_{t\to0}\frac{1}{t}\left[(I - \eta t)X - X\right] = \eta X .
    \end{aligned}
\end{align*} 

\subsection{Adjoint representations}
Let $X \in \grpG$ and ${\eta \in \gothg}$, the \emph{Adjoint map} for the group $\grpG$, $\Ad_X \AtoB{\gothg}{\gothg}$ is defined by
\begin{equation*}
    \Adsym{X}{\eta} = \td \textrm{L}_{X} \td \textrm{R}_{X^{-1}}\left[\eta\right] ,
\end{equation*} 
where $\mathbf{X}$ is the matrix representation of the group element $X$.
In the context of matrix Lie groups, we can define the \emph{Adjoint matrix}. Let ${\Vector{}{u}{} \in \R^n \st \Vector{}{u}{}^{\wedge} \in \gothg}$, the Adjoint matrix ${\AdMsym{X} \AtoB{\R^{n}}{\R^{n}}}$ is defined by
\begin{equation*}
    \AdMsym{X}\Vector{}{u}{} = \left(\Adsym{X}{\Vector{}{u}{}^{\wedge}}\right)^{\vee} = \left(\mathbf{X}\Vector{}{u}{}^{\wedge}\mathbf{X}^{-1}\right)^{\vee} .
\end{equation*}
In addition to the Adjoint map for the group $\grpG$, the \emph{adjoint map} for the Lie algebra $\gothg$ is defined as the differential at the identity of the Adjoint map for the group $\grpG$.
Let $X \in \grpG$ and ${\eta, \kappa \in \gothg}$, the adjoint map for the Lie algebra $\ad_{\eta} \AtoB{\gothg}{\gothg}$ is defined by 
\begin{equation*}
    \adsym{\eta}{\kappa} = \left[\eta, \kappa\right] ,
\end{equation*}
which is equivalent to the Lie bracket previously defined. 
In the context of matrix Lie groups, we can define the \emph{adjoint matrix} $\adMsym{\Vector{}{u}{}} \AtoB{\R^{n}}{\R^{n}}$ as 
\begin{equation*}
    \adMsym{\Vector{}{u}{}}\Vector{}{v}{} = \left(\Vector{}{u}{}^{\wedge}\Vector{}{v}{}^{\wedge} - \Vector{}{v}{}^{\wedge}\Vector{}{u}{}^{\wedge}\right)^{\vee} = \left[\Vector{}{u}{}^{\wedge}, \Vector{}{v}{}^{\wedge}\right]^{\vee} .
\end{equation*}
In general, the adjoints commute as follows:
\begin{subequations}\label{math_adj_comm}
    \begin{align}
        \Adsym{X}{\adsym{\eta}{\kappa}} &= \adsym{\left(\Adsym{X}{\eta}\right)}{\Adsym{X}{\kappa}} ,\\
         \AdMsym{X}\adMsym{\Vector{}{u}{}}\Vector{}{v}{} &= \adMsym{\AdMsym{X}{\Vector{}{u}{}}}\AdMsym{X}\Vector{}{v}{}.
    \end{align}
\end{subequations}

\subsection{Exponential map and auxiliary operators}
Let $\eta \in \gothg$, the \emph{exponential map}, ${\exp \AtoB{\gothg}{\grpG}}$ is defined by
\begin{equation*}
    \exp(\eta) = X \in \grpG .
\end{equation*}
In the context of matrix Lie group, the exponential map is the matrix exponential. Let ${\Vector{}{u}{} \in \R^n \st \Vector{}{u}{}^{\wedge} \in \gothg}$, then the exponential map for matrix Lie groups is defined by
\begin{equation*}
    \exp(\Vector{}{u}{}^{\wedge}) = \mathbf{X} \in \grpG .
\end{equation*}
The exponential map is analytic, and hence the inverse map, the \emph{logarithmic map}, ${\log \AtoB{\grpG}{\gothg}}$ is defined by
\begin{equation*}
    \log(X) = \eta \in \gothg .
\end{equation*}
In the context of matrix Lie group, the logarithmic map is the matrix logarithm:
\begin{equation*}
    \log(\mathbf{X}) = \Vector{}{u}{}^{\wedge} \in \gothg .
\end{equation*}

For a n-dimensional matrix Lie group, the exponential map is written by the ordinary series expansion: 
\begin{equation*}
    \exp(\Vector{}{u}{}^{\wedge}) \coloneqq \sum_{k=0}^{\infty}\frac{1}{k!}\left(\Vector{}{u}{}^{\wedge}\right)^k.
\end{equation*}
Define the auxiliary linear operator~\cite{Bloesch2013StateIMU} ${\mathbf{\Gamma}_m \AtoB{\R^n}{\R^n}}$:
\begin{equation}\label{math_Gamma_aux}
    \mathbf{\Gamma}_m(\Vector{}{u}{}) \coloneqq \sum_{k=0}^{\infty}\frac{1}{(k+m)!}(\adMsym{\Vector{}{u}{}})^k.
\end{equation}
In particular, $\mathbf{\Gamma}_0$ relates to the matrix exponential and corresponds to the Adjoint matrix for the Lie group $\grpG$:
\begin{equation*}
    \mathbf{\Gamma}_0(\Vector{}{u}{}) \coloneqq \exp\left(\adMsym{\Vector{}{u}{}}\right) = \AdMsym{\exp\left(\Vector{}{u}{}^{\wedge}\right)}{}.
\end{equation*}
$\mathbf{\Gamma}_1 $ is referred as to the \emph{left Jacobian}:
\begin{align*}
    \mathbf{\Gamma}_1(\Vector{}{u}{}) \coloneqq \mathbf{J}_{L}(\Vector{}{u}{}).
\end{align*}
The left Jacobian relates to the first derivative of the exponential map. In particular, by algebraic manipulation of the series in \cref{math_Gamma_aux}, one has ${\mathbf{\Gamma}_0(\Vector{}{u}{}) = \eye + \adMsym{\Vector{}{u}{}}\mathbf{\Gamma}_{1}(\Vector{}{u}{})}$.
The \emph{right Jacobian} is defined as
\begin{equation*}
    \mathbf{J}_{R}(\Vector{}{u}{}) \coloneqq \mathbf{J}_{L}(-\Vector{}{u}{}) = \mathbf{\Gamma}_1(-\Vector{}{u}{}).
\end{equation*}
Moreover, it can be shown that the following relation between the left and the right Jacobian holds.
\begin{equation*}
    \AdMsym{\exp\left(\Vector{}{u}{}^{\wedge}\right)}\mathbf{J}_{R}(\Vector{}{u}{}) = \mathbf{J}_{L}(\Vector{}{u}{})
\end{equation*}
Generally, every $\mathbf{\Gamma}_n$ for $n > 2$ relates to the higher-order derivatives of the exponential map.
Let $\Vector{}{u}{},\; \Vector{}{v}{} \in \R^n$, for the linear operator $\mathbf{\Gamma}_m$ the following relation holds:
\begin{equation*}
    \mathbf{\Gamma}_m(\mathbf{\Gamma}_0(\Vector{}{u}{})\Vector{}{v}{}) = \mathbf{\Gamma}_0(\Vector{}{u}{})\mathbf{\Gamma}_m(\Vector{}{v}{})\mathbf{\Gamma}_0(\Vector{}{u}{})^{-1} .
\end{equation*}
Note that every operator $\mathbf{\Gamma}_m$ relates to the higher order $\mathbf{\Gamma}_{m+1}$. By algebraic manipulation of the series, the following relation is obtained
\begin{equation*}
    \mathbf{\Gamma}_m(\Vector{}{u}{}) \coloneqq \frac{1}{m!}\eye_n + \adMsym{\Vector{}{u}{}}\mathbf{\Gamma}_{m+1}(\Vector{}{u}{}) .
\end{equation*}
Finally, let ${\Vector{}{u}{},\; \Vector{}{\delta}{u} \in \R^n \st \norm{\Vector{}{\delta}{u}} << \norm{\Vector{}{u}{}}}$, then by expanding the auxiliary operator $\mathbf{\Gamma}_m$ in Taylor series we obtain
\begin{equation*}
    \mathbf{\Gamma}_m(\Vector{}{u}{} +  \Vector{}{\delta}{u}) \simeq \exp((\mathbf{\Gamma}_{m+1}(\Vector{}{u}{})\Vector{}{\delta}{u})^{\wedge})\mathbf{\Gamma}_m(\Vector{}{u}{}) .
\end{equation*}
With $m = 0$, this formula is expanded in the following useful identities:
\begin{align*}
    &\exp\left(\left(\Vector{}{u}{} + \Vector{}{\delta}{u}\right)^{\wedge}\right) \simeq \exp\left(\Vector{}{u}{}^{\wedge}\right)\exp\left(\left(\mathbf{J}_{R}(\Vector{}{u}{})\Vector{}{\delta}{u}\right)^{\wedge}\right) \simeq \exp\left(\left(\mathbf{J}_{L}(\Vector{}{u}{})\Vector{}{\delta}{u}\right)^{\wedge}\right)\exp\left(\Vector{}{u}{}^{\wedge}\right),\\
    &\exp\left(\Vector{}{u}{}^{\wedge}\right)\exp\left(\Vector{}{\delta}{u}^{\wedge}\right) \simeq \exp\left(\left(\Vector{}{u}{} + \mathbf{J}_{R}^{-1}(\Vector{}{u}{})\Vector{}{\delta}{u}\right)^{\wedge}\right),\\
    &\exp\left(\Vector{}{\delta}{u}^{\wedge}\right)\exp\left(\Vector{}{u}{}^{\wedge}\right) \simeq \exp\left(\left(\Vector{}{u}{} + \mathbf{J}_{L}^{-1}(\Vector{}{u}{})\Vector{}{\delta}{u}\right)^{\wedge}\right).
\end{align*}

\subsection{Group actions and homogeneous spaces}\label{math_action}
Lie groups are often referred to as transformation groups. A Lie group is said to \emph{act} on a differentiable manifold to apply a specific transformation. In particular, a \emph{right (left) group action} of a Lie group $\grpG$ on a differentiable manifold $\calM$ is a smooth map ${\phi \AtoB{\grpG\times \calM}{\calM}}$ that satisfies 
\begin{align*}
    &\phi\left(I, \xi\right) = \xi ,\\
    &\phi\left(X, \phi\left(Y, \xi\right)\right) = \phi\left(Y \circ X, \xi\right) \quad \text{(right group action)},\\
    &\phi\left(X, \phi\left(Y, \xi\right)\right) = \phi\left(X \circ Y, \xi\right) \quad \text{(left group action)},\\
\end{align*} 
for any $X,Y \in \grpG$ and $\xi \in \calM$.
A right (left) group action $\phi$ induces a family of diffeomorphisms (differentiable maps between manifolds that have a differentiable inverse) ${\phi_X \AtoB{\calM}{\calM}}$ and smooth nonlinear projections ${\phi_{\xi} \AtoB{\grpG}{\calM}}$. The group action $\phi$ is said to be transitive if the induced projection $\phi_{\xi}$ is surjective, and then, in this case, $\calM$ is termed a \emph{homogeneous space} of $\grpG$.

Any right (left) group action ${\phi \AtoB{\grpG\times \calM}{\calM}}$ induces a right (left) group action ${\Phi \AtoB{\grpG \times \mathfrak{X}\left(\calM\right)}{\mathfrak{X}\left(\calM\right)}}$ on the set of all smooth vector fields over $\calM$:
\begin{equation}\label{math_induced_action}
    \Phi\left(X,f\right) = \td\phi_{X} \circ f \circ \phi_{X^{-1}} ,
\end{equation}
for each $f \in \mathfrak{X}\left(\calM\right)$, $\phi$-invariant vector field and $X \in \grpG$.

\subsection{Semi-direct product groups}\label{math_sdp_sec}
A semi-direct product group $\mathbf{G} \ltimes \mathbf{H}$ can be seen as a generalization of the direct product group $\grpG \times \grpH$ where the underlying set is given by the cartesian product of two groups $\grpG$ and $\grpH$, but, contrary to the direct product, in the semi-direct product, the group multiplication is defined with the group $\grpG$ that acts on a group $\grpH$ by a group homomorphism. The semi-direct product group $\mathbf{G} \ltimes \mathbf{H}$ with the trivial automorphism corresponds to the direct product group $\grpG \times \grpH$. Specifically, let $X = (A,a), Y = (B,b) \in \grpG \ltimes \grpH$ the semi-direct product rule is defined by $X \circ Y = (A \circ B, a \circ \varphi_{A}(b))$, with $\varphi \AtoB{\grpG}{\text{Aut}(\grpH)}$ a group homomorphism.

In this dissertation, we will discuss the use of semi-direct product groups in the form of $\sdpgrpG := \mathbf{G} \ltimes \gothg$ where $\gothg$ is the Lie algebra of $\grpG$ or a Lie sub algebra of $\grpG$. In particular, when $\gothg$ is the Lie algebra of $\grpG$, the group $\sdpgrpG := \mathbf{G} \ltimes \gothg$ is termed \emph{tangent group of $\grpG$}.

Let $\grpG$ be a $n$-dimensional Lie group, let ${A,B \in \grpG}$, and ${a,b \in \gothg}$. Define ${X = \left(A, a\right)}$ and ${Y = \left(B, b\right)}$ as elements of the group $\sdpgrpG$. The identity element is defined as ${\left(I, 0\right)}$. The inverse element is defined as 
\begin{equation*}
    X^{-1} = \left(A^{-1}, -\Adsym{A^{-1}}{a}\right).
\end{equation*} 
Group multiplication is defined as the semi-direct product, and hence, left and right translation are respectively defined by
\begin{align}
    \textrm{L}_{X}\left(Y\right) = \left(A \circ B, a + \Adsym{A}{b}\right) ,\label{math_sdp_rule}\\
    \textrm{R}_{X}\left(Y\right) = \left(B \circ A, b + \Adsym{B}{a}\right) .
\end{align}
The differentials of the left and right translations in the direction ${\eta = \left(\gamma, \zeta\right) \in \gothg \oplus \gothg}$ for the semi-direct product are written
\begin{align*}
    &\td \textrm{L}_{X}[\eta] = (A\gamma, \Adsym{A}{\zeta}),\\
    &\td \textrm{R}_{X}[\eta] = (\gamma A, \zeta + \adsym{\gamma}{a}).
\end{align*}
The Adjoint map $\textrm{Ad}_{X}$ can be then computed as follows:
\begin{align*}
    \Adsym{X}{\left(\gamma, \zeta\right)} &= \td \textrm{L}_{X} \td \textrm{R}_{X^{-1}}\left[\left(\gamma, \zeta\right)\right]\\
    &= \left(\Adsym{A}{\gamma}, \Adsym{A}{\zeta} - \Adsym{A}{\adsym{\gamma}{\Adsym{A^{-1}}{a}}}\right)\\
    &= \left(\Adsym{A}{\gamma}, \Adsym{A}{\zeta} - \adsym{\Adsym{A}{\gamma}}{a}\right)\\
    &= \left(\Adsym{A}{\gamma}, \Adsym{A}{\zeta} + \adsym{a}{\Adsym{A}{\gamma}}\right),
\end{align*}
where we have exploited the adjoint commutation formula in \cref{math_adj_comm}. Moreover, let ${\kappa = \left(\alpha,\beta\right) \in \gothg \oplus \gothg}$, the adjoint map $\textrm{ad}_{\eta}$ can be computed as
\begin{align*}
    \adsym{\eta}{\kappa} &= \left(\adsym{\gamma}{\alpha}, \adsym{\gamma}{\beta} + \adsym{\zeta}{\alpha}\right).
\end{align*}

In the context of matrix Lie groups, the same rules apply. Moreover, let $\grpG$ be a $n$-dimensional matrix Lie group, let $\Vector{}{\gamma}{}, \Vector{}{\zeta}{} \in \R^{n}$, the explicit matrix representation of the semi-direct product group $\sdpgrpG$ and its Lie algebra are derived from the semi-direct product rule, and are written as
\begin{align*}
    &\Vector{}{x}{}^{\wedge} = \begin{bmatrix}
        \adMsym{\Vector{}{\gamma}{}} & \Vector{}{\zeta}{}\\
        \mathbf{0}_{1 \times n} & 0
    \end{bmatrix} && 
    X = \begin{bmatrix}
        \AdMsym{A} & \Vector{}{a}{}\\
        \mathbf{0}_{1 \times n} & 1
    \end{bmatrix}
\end{align*}
Specifically, the closed-form analytic solutions of the semi-direct product group exponential and logarithm can be written as
\begin{align*}
    X &= \exp_{\sdpgrpG}\left(\Vector{}{x}{}^{\wedge}\right) = \begin{bmatrix}
        \AdMsym{\exp_{\grpG}\left(\Vector{}{\gamma}{}^{\wedge}\right)} & \mathbf{\Gamma}_1\left(\Vector{}{\gamma}{}\right)\Vector{}{\zeta}{}\\
        \mathbf{0}_{1 \times n} & 1
    \end{bmatrix} \in \sdpgrpG \subset \R^{(n + 1)\times (n + 1)}\\
    \Vector{}{x}{}^{\wedge} &= \log_{\sdpgrpG}\left(X\right) = \begin{bmatrix}
        \adMsym{\log_{\grpG}\left(A\right)} & \mathbf{\Gamma}_{1}^{-1}\left(\log_{\grpG}\left(A\right)^{\vee}\right)\Vector{}{a}{}\\
        \mathbf{0}_{1 \times n} & 0
    \end{bmatrix} \in \sdpgrpG \subset \R^{(n + 1)\times (n + 1)}
\end{align*}
where ${\mathbf{\Gamma}_1}$ is the left Jacobian for $\grpG$.

\section{Important matrix Lie groups}

The following subsections dive into the algebraic properties of matrix Lie groups, playing a crucial role in robotic systems, specifically in the context of \aclp{ins}.

\subsection{Special orthogonal group $\SO(3)$}
The \emph{special orthogonal group} $\SO(3)$ is the Lie group of 3D rotations in space. Let $A \in \SO(3)$ represent an element of the Lie group. Let ${\Vector{}{\omega}{}, \Vector{}{v}{} \in \R^3}$, then the special orthogonal group and its Lie algebra are defined as follows:
\begin{align*}
    \SO(3) &= \set{\mathbf{A} \in \R^{3\times3}}{
    \mathbf{A} \mathbf{A}^\top = \mathbf{I}_3, \; \det(\mathbf{A} ) = 1}, \\
    \so(3) &= \set{\Vector{}{\omega}{}^{\wedge} \in \R^{3\times3}}{
    \Vector{}{\omega}{}^{\wedge} = - {\Vector{}{\omega}{}^{\wedge}}^\top, \; \Vector{}{\omega}{}^{\wedge}\Vector{}{v}{} = \Vector{}{\omega}{} \times \Vector{}{v}{}},
\end{align*}
where $\times$ represents the cross-product. 
Let ${\Vector{}{\omega}{} = \left(\omega_x, \omega_y, \omega_z\right) \in \R^3}$, then $\Vector{}{\omega}{}^{\wedge} \in \R^{3 \times 3}$ is explicitly written 
\begin{equation*}\
    \Vector{}{\omega}{}^{\wedge} = 
    \begin{bmatrix}
        0 & -\omega_z & \omega_y\\
        \omega_z & 0 & -\omega_x\\
        -\omega_y & \omega_x & 0
    \end{bmatrix} .
\end{equation*}
Note that for $\Vector{}{\omega}{}^{\wedge} \in \so(3)$, the following relation holds:
\begin{equation*}
    \Vector{}{\omega}{}^{\wedge}\Vector{}{\omega}{}^{\wedge} = \Vector{}{\omega}{}\Vector{}{\omega}{}^{\top} - \Vector{}{\omega}{}^{\top}\Vector{}{\omega}{}\eye_3 .
\end{equation*}
Let $\mathbf{A} \in \SO(3)$ be the matrix representation of $A$, the inverse element is written $A^{-1} = \mathbf{A}^{\top}$. In matrix representation, the identity element is the $\eye_3$ matrix. Let ${\Vector{}{\omega}{} \in \R^3 \st \Vector{}{\omega}{}^{\wedge} \in \so(3)}$, the Adjoint map for $\SO(3)$ and the related Adjoint matrix are written
\begin{align*}
    &\Adsym{A}{\Vector{}{\omega}{}^{\wedge}} \coloneqq \mathbf{A}\Vector{}{\omega}{}^{\wedge}\mathbf{A}^{\top} = (\mathbf{A}\Vector{}{\omega}{})^{\wedge}, && \AdMsym{A} \coloneqq \mathbf{A}.
\end{align*}
Similarly, let ${\Vector{}{\omega}{} \in \R^3 \st \Vector{}{\omega}{}^{\wedge} \in \so(3)}$, and ${\Vector{}{v}{} \in \R^3 \st \Vector{}{v}{}^{\wedge} \in \so(3)}$ the adjoint map for $\so(3)$ and the adjoint matrix are written
\begin{align*}
    &\adsym{\Vector{}{\omega}{}^{\wedge}}{\Vector{}{v}{}^{\wedge}} \coloneqq \left[\Vector{}{\omega}{}^{\wedge}, \Vector{}{v}{}^{\wedge}\right], && \adMsym{\Vector{}{\omega}{}} \coloneqq \Vector{}{\omega}{}^{\wedge}.
\end{align*}
The auxiliary functions $\mathbf{\Gamma}_m$ previously introduced, are given in closed form for $\SO(3)$:
\begin{align*}
    &\mathbf{\Gamma}_0(\Vector{}{\omega}{}) = \exp(\Vector{}{\omega}{}^{\wedge}) = \eye_3 + \frac{\sin(\norm{\Vector{}{\omega}{}})}{\norm{\Vector{}{\omega}{}}}\Vector{}{\omega}{}^{\wedge} + \frac{1 - \cos(\norm{\Vector{}{\omega}{}})}{\norm{\Vector{}{\omega}{}}^{2}}(\Vector{}{\omega}{}^{\wedge})^2 ,\\
    &\mathbf{\Gamma}_1(\Vector{}{\omega}{}) = \mathbf{J}_{L}(\Vector{}{\omega}{}) = \eye_3 + \frac{1 - \cos(\norm{\Vector{}{\omega}{}})}{\norm{\Vector{}{\omega}{}}^{2}}\Vector{}{\omega}{}^{\wedge} + \frac{\norm{\Vector{}{\omega}{}} - \sin(\norm{\Vector{}{\omega}{}})}{\norm{\Vector{}{\omega}{}}^{3}}(\Vector{}{\omega}{}^{\wedge})^{2} ,\\
    &\mathbf{\Gamma}_2(\Vector{}{\omega}{}) = \frac{1}{2}\eye_3 + \frac{\norm{\Vector{}{\omega}{}} - \sin(\norm{\Vector{}{\omega}{}})}{\norm{\Vector{}{\omega}{}}^{3}}\Vector{}{\omega}{}^{\wedge} + \frac{\norm{\Vector{}{\omega}{}}^{2} + 2\cos(\norm{\Vector{}{\omega}{}}) - 2}{2\norm{\Vector{}{\omega}{}}^{4}}(\Vector{}{\omega}{}^{\wedge})^{2}
\end{align*}

\subsection{Scaled orthogonal transform group $\SOT(3)$}
The \emph{scaled orthogonal transform group} $\SOT(3)$~\cite{VanGoor2020AnEquivariance, Mahony2020EquivariantWild, vanGoor2021AnOdometry, vanGoor2023EqVIO:Odometry, VanGoor2023EquivariantAwareness} is the group of 3D rotation and scale in space. Let $X \in \SOT(3)$ represent an element of the Lie group. Let $\R^+$ be the set of strictly positive real numbers, then the scaled orthogonal transform group and its Lie algebra are defined as follows:
\begin{align*}
    \SOT(3) &= \set{
    a\mathbf{A} \in \R^{3\times3}}{
    \mathbf{A} \in \SO(3), \; a \in \R^+}, \\
    \sot(3) &= \set{
    \Vector{}{\omega}{}^{\wedge} + \alpha\eye_3 \in \R^{3\times3}}{
    \Vector{}{\omega}{}^{\wedge} \in \so(3), \; \alpha \in \R},
\end{align*}
Let ${X = (A, a) \in \SOT(3)}$, the inverse element is written $X^{-1} = (A^{-1}, \frac{1}{a})$. The identity element is $I = (I, 1)$, in matrix representation, is the $\eye_4$ matrix. Let $\mathbf{A}$ be the matrix representation of $\SO(3)$, and ${\Vector{}{x}{} = (\Vector{}{\omega}{}, \alpha) \in \R^4 \st \Vector{}{x}{}^{\wedge} \in \sot(3)}$, the Adjoint map for $\SOT(3)$ and the Adjoint matrix are written
\begin{align*}
    &\Adsym{X}{\Vector{}{x}{}^{\wedge}} \coloneqq \mathbf{A}\Vector{}{\omega}{}^{\wedge}\mathbf{A}^{\top} + \alpha\eye_3, && \AdMsym{X} \coloneqq 
    \begin{bmatrix}
        \mathbf{A} & 0\\
        \Vector{}{0}{1\times3} & 1
    \end{bmatrix}.
\end{align*}
Similarly, let ${\Vector{}{x}{} = (\Vector{}{\omega}{}, \alpha) \in \R^4 \st \Vector{}{x}{}^{\wedge} \in \sot(3)}$, and ${\Vector{}{y}{} = (\Vector{}{\nu}{}, \beta) \in \R^4 \st \Vector{}{y}{}^{\wedge} \in \sot(3)}$ the adjoint map for $\sot(3)$ and the adjoint matrix are written
\begin{align*}
\begin{split}
\adsym{\Vector{}{x}{}^{\wedge}}{\Vector{}{y}{}^{\wedge}} &\coloneqq \left[\Vector{}{x}{}^{\wedge}, \Vector{}{y}{}^{\wedge}\right]\\
&= (\Vector{}{\omega}{}^{\wedge}\Vector{}{\nu}{}^{\wedge} + \alpha\beta - \Vector{}{\nu}{}^{\wedge}\Vector{}{\omega}{}^{\wedge} - \alpha\beta)\\
&=\left[\Vector{}{\omega}{}^{\wedge}, \Vector{}{\nu}{}^{\wedge}\right] + 0\eye_3 = \left[\Vector{}{\omega}{}^{\wedge}, \Vector{}{\nu}{}^{\wedge}\right],
\end{split} && \adMsym{\Vector{}{x}{}} = 
    \begin{bmatrix}
        \Vector{}{\omega}{}^{\wedge} & 0\\
        \Vector{}{0}{1\times3} & 0
    \end{bmatrix}.
\end{align*}
The exponential map and the auxiliary functions $\mathbf{\Gamma}_m$, are given in closed form for $\SOT(3)$:
\begin{align*}
    \exp(\Vector{}{x}{}^{\wedge}) &= \exp(\alpha)\exp(\Vector{}{\omega}{}^{\wedge}) ,\\
    \mathbf{\Gamma}_m(\Vector{}{x}{}) &= 
    \begin{bmatrix}
        \mathbf{\Gamma}_m(\Vector{}{\omega}{}) & 0\\
        \Vector{}{0}{1\times3} & 1
    \end{bmatrix} .
\end{align*}

\subsection{Special euclidean group $\SE(3)$}
The \emph{special Euclidean group} $\SE(3)$ is the group of 3D rigid body transformation in space. Note that $\SE(3)$ is defined as a semi-direct product group, specifically, $\SE(3) \coloneq \SO(3) \ltimes \R^6$. Let $X \in \SE(3)$ represent an element of the Lie group. The special Euclidean group and its Lie algebra are defined as follows:
\begin{align*}
    \SE(3) &= \set{
    \begin{bmatrix}
        \mathbf{A} & \Vector{}{b}{} \\ \mathbf{0}_{1\times 3} & 1
    \end{bmatrix}
    \in \R^{4\times4}}{
    \mathbf{A} \in \SO(3), \; \Vector{}{b}{} \in \R^3}, \\
    \se(3) &= \set{
    \begin{bmatrix} \Vector{}{\omega}{}\\ \Vector{}{w}{} \end{bmatrix}^{\wedge} = 
    \begin{bmatrix}
        \Vector{}{\omega}{}^{\wedge} & \Vector{}{w}{} \\ \mathbf{0}_{1\times 3} & 0
    \end{bmatrix}
    \in \R^{4\times4}}{
    \Vector{}{\omega}{}^{\wedge} \in \so(3), \; \Vector{}{w}{} \in \R^3},
\end{align*}
Let ${X = (A, b) \in \SE(3)}$, the inverse element is written $X^{-1} = (A^{-1}, -A^{-1}b)$; which, in their matrix representation, are written
\begin{align*}
    \mathbf{X} = 
    \begin{bmatrix}
        \mathbf{A} & \Vector{}{b}{} \\ \mathbf{0}_{1\times 3} & 1
    \end{bmatrix}, &&
    \mathbf{X}^{-1} = 
    \begin{bmatrix}
        \mathbf{A}^{\top} & -\mathbf{A}^{\top}\Vector{}{b}{} \\ \mathbf{0}_{1\times 3} & 1
    \end{bmatrix}.
\end{align*}
The identity element is given by $I = (I, 0)$, and its matrix representation is given by the $\eye_4$ matrix. Specifically, $\mathbf{A} = \eye_3$ and $\Vector{}{b}{} = \Vector{}{0}{}$.
Let ${\Vector{}{x}{} \in \R^6 \st \Vector{}{x}{}^{\wedge} \in \se(3)}$, the Adjoint map for $\SE(3)$ and the Adjoint matrix are written
\begin{align*}
    &\Adsym{X}{\Vector{}{x}{}^{\wedge}} \coloneqq \mathbf{X}\Vector{}{x}{}^{\wedge}\mathbf{X}^{-1}, && \AdMsym{X} \coloneqq 
    \begin{bmatrix}
        \mathbf{A} & \Vector{}{0}{3\times3}\\
        \Vector{}{b}{}^{\wedge}\mathbf{A} & \mathbf{A}
    \end{bmatrix}.
\end{align*}
Similarly, let ${\Vector{}{x}{} = (\Vector{}{\omega}{}, \Vector{}{w}{}) \in \R^6 \st \Vector{}{x}{}^{\wedge} \in \se(3)}$, and ${\Vector{}{y}{} = (\Vector{}{\nu}{}, \Vector{}{v}{}) \in \R^6 \st \Vector{}{y}{}^{\wedge} \in \se(3)}$ the adjoint map for $\se(3)$ and the adjoint matrix are written
\begin{align*}
    &\adsym{\Vector{}{x}{}^{\wedge}}{\Vector{}{y}{}^{\wedge}} \coloneqq \left[\Vector{}{x}{}^{\wedge}, \Vector{}{y}{}^{\wedge}\right], && \adMsym{\Vector{}{x}{}} \coloneqq 
    \begin{bmatrix}
        \Vector{}{\omega}{}^{\wedge} & \Vector{}{0}{3\times3}\\
        \Vector{}{w}{}^{\wedge} & \Vector{}{\omega}{}^{\wedge}
    \end{bmatrix}.
\end{align*}
The exponential map and the left Jacobian $\mathbf{\Gamma}_1$, are given in closed form for $\SE(3)$:
\begin{align*}
    \exp(\Vector{}{x}{}^{\wedge}) &= 
    \begin{bmatrix}
        \exp(\Vector{}{\omega}{}^{\wedge}) & \mathbf{\Gamma}_1(\Vector{}{\omega}{})\Vector{}{b}{}\\
        \Vector{}{0}{1\times3} & 1
    \end{bmatrix} ,\\
    \mathbf{\Gamma}_1(\Vector{}{x}{}) = \mathbf{J}_L(\Vector{}{x}{}) &= 
    \begin{bmatrix}
        \mathbf{\Gamma}_1(\Vector{}{\omega}{}) & \Vector{}{0}{3\times3}\\
        \mathbf{Q}_1(\Vector{}{\omega}{}, \Vector{}{w}{}) & \mathbf{\Gamma}_1(\Vector{}{\omega}{})
    \end{bmatrix} ,
\end{align*}
with
\begin{align*}
    \mathbf{Q}_1(\Vector{}{\omega}{}, \Vector{}{w}{}) &\coloneqq \sum_{k=0}^{\infty}\sum_{p=0}^{\infty}(\Vector{}{\omega}{}^{\wedge})^{k}\Vector{}{w}{}^{\wedge}(\Vector{}{\omega}{}^{\wedge})^{p} =\\
    &= \frac{1}{2}\Vector{}{w}{}^{\wedge} + \left(\frac{\norm{\Vector{}{\omega}{}} - \sin(\norm{\Vector{}{\omega}{}})}{\norm{\Vector{}{\omega}{}}^{3}}\right)\left(\Vector{}{\omega}{}^{\wedge}\Vector{}{w}{}^{\wedge} + \Vector{}{w}{}^{\wedge}\Vector{}{\omega}{}^{\wedge} + \Vector{}{\omega}{}^{\wedge}\Vector{}{w}{}^{\wedge}\Vector{}{\omega}{}^{\wedge}\right) +\\
    &+ \left(\frac{\norm{\Vector{}{\omega}{}}^{2} + 2\cos(\norm{\Vector{}{\omega}{}}) - 2}{2\norm{\Vector{}{\omega}{}}^{4}}\right)\left(\Vector{}{\omega}{}^{\wedge}\Vector{}{\omega}{}^{\wedge}\Vector{}{w}{}^{\wedge} + \Vector{}{w}{}^{\wedge}\Vector{}{\omega}{}^{\wedge}\Vector{}{\omega}{}^{\wedge} - 3\Vector{}{\omega}{}^{\wedge}\Vector{}{w}{}^{\wedge}\Vector{}{\omega}{}^{\wedge}\right) +\\
    &+ \left(\frac{2\norm{\Vector{}{\omega}{}} - 3\sin(\norm{\Vector{}{\omega}{}}) + \norm{\Vector{}{\omega}{}}\cos(\norm{\Vector{}{\omega}{}})}{2\norm{\Vector{}{\omega}{}}^{5}}\right)\left(\Vector{}{\omega}{}^{\wedge}\Vector{}{w}{}^{\wedge}\Vector{}{\omega}{}^{\wedge}\Vector{}{\omega}{}^{\wedge} + \Vector{}{\omega}{}^{\wedge}\Vector{}{\omega}{}^{\wedge}\Vector{}{w}{}^{\wedge}\Vector{}{\omega}{}^{\wedge}\right)  .
\end{align*}
For a detailed derivation of the left Jacobian for $\SE(3)$, we refer the reader to~\cite{Barfoot2017StateRobotics}.

\subsection{Extended special Euclidean group $\SE_2(3)$}\label{math_se23_sec}
The \emph{extended special Euclidean group} $\SE_2(3)$~\cite{barrau:tel-01247723, 7523335, Barrau2020APreintegration, Brossard2021AssociatingEarth} is an extension of the special Euclidean group. Let $X \in \SE_2(3)$ represent an element of the Lie group. The extended special Euclidean group and its Lie algebra are defined as follows:
\begin{align*}
    \SE_2(3) &= \set{
    \begin{bmatrix}
        \mathbf{A} & \Vector{}{a}{} & \Vector{}{b}{} \\ 
        \mathbf{0}_{1\times 3} & 1 & 0\\
        \mathbf{0}_{1\times 3} & 0 & 1
    \end{bmatrix}
    \in \R^{5\times5}}{
    \mathbf{A} \in \SO(3), \; \Vector{}{a}{}, \Vector{}{b}{} \in \R^{3}}, \\
    \se_2(3) &= \set{
    \begin{bmatrix} \Vector{}{\omega}{}\\ \Vector{}{v}{}\\ \Vector{}{w}{} \end{bmatrix}^{\wedge} = 
    \begin{bmatrix}
        \Vector{}{\omega}{}^{\wedge} & \Vector{}{v}{} & \Vector{}{w}{} \\ \mathbf{0}_{1\times 3} & 0 & 0\\
        \mathbf{0}_{1\times 3} & 0 & 0
    \end{bmatrix}
    \in \R^{5\times5}}{
    \Vector{}{\omega}{}^{\wedge} \in \so(3), \; \Vector{}{v}{}, \Vector{}{w}{} \in \R^{3}}.
\end{align*}
$\SE_2(3)$ is an extension of $\SE(3)$ and hence all the identities previously introduced for $\SE(3)$ are extended as follows. 
Let ${X = (A, a, b) \in \SE_2(3)}$, the inverse element is written $X^{-1} = (A^{-1}, -A^{-1}a, -A^{-1}b)$; which, in their matrix representation, are written
\begin{align*}
    \mathbf{X} = 
    \begin{bmatrix}
        \mathbf{A} & \Vector{}{a}{} & \Vector{}{b}{} \\ 
        \mathbf{0}_{1\times 3} & 1 & 0\\
        \mathbf{0}_{1\times 3} & 0 & 1
    \end{bmatrix}, &&
    \mathbf{X}^{-1} = 
    \begin{bmatrix}
        \mathbf{A}^{\top} & -\mathbf{A}^{\top}\Vector{}{a}{} & -\mathbf{A}^{\top}\Vector{}{b}{} \\ 
        \mathbf{0}_{1\times 3} & 1 & 0\\
        \mathbf{0}_{1\times 3} & 0 & 1
    \end{bmatrix}.
\end{align*}
The identity element is given by $I = (I, 0, 0)$, and its matrix representation is given by the $\eye_5$ matrix. Specifically, $\mathbf{A} = \eye_3$ and ${\Vector{}{a}{} = \Vector{}{b}{} = \Vector{}{0}{}}$.
Let  ${\Vector{}{x}{} \in \R^9 \st \Vector{}{x}{}^{\wedge} \in \se_2(3)}$, the Adjoint map for $\SE_2(3)$ and the Adjoint matrix are written
\begin{align*}
    &\Adsym{X}{\Vector{}{x}{}^{\wedge}} \coloneqq \mathbf{X}\Vector{}{x}{}^{\wedge}\mathbf{X}^{-1}, && \AdMsym{X} \coloneqq 
    \begin{bmatrix}
        \mathbf{A} & \Vector{}{0}{3\times3} & \Vector{}{0}{3\times3}\\
        \Vector{}{a}{}^{\wedge}\mathbf{A} & \mathbf{A} & \Vector{}{0}{3\times3}\\
        \Vector{}{b}{}^{\wedge}\mathbf{A} & \Vector{}{0}{3\times3} & \mathbf{A}
    \end{bmatrix}.
\end{align*}
Let ${\Vector{}{x}{} = (\Vector{}{\omega}{}, \Vector{}{v}{}, \Vector{}{w}{}) \in \R^9 \st \Vector{}{x}{}^{\wedge} \in \se_2(3)}$, and ${\Vector{}{y}{} = (\Vector{}{\nu}{}, \Vector{}{u}{}, \Vector{}{q}{}) \in \R^9 \st \Vector{}{y}{}^{\wedge} \in \se_2(3)}$ the adjoint map for $\se_2(3)$ and the adjoint matrix are written
\begin{align*}
    &\adsym{\Vector{}{x}{}^{\wedge}}{\Vector{}{y}{}^{\wedge}} \coloneqq \left[\Vector{}{x}{}^{\wedge}, \Vector{}{y}{}^{\wedge}\right], && \adMsym{\Vector{}{x}{}} \coloneqq 
    \begin{bmatrix}
        \Vector{}{\omega}{}^{\wedge} & \Vector{}{0}{3\times3} & \Vector{}{0}{3\times3}\\
        \Vector{}{v}{}^{\wedge} & \Vector{}{\omega}{}^{\wedge} & \Vector{}{0}{3\times3}\\
        \Vector{}{w}{}^{\wedge} & \Vector{}{0}{3\times3} & \Vector{}{\omega}{}^{\wedge}
    \end{bmatrix}.
\end{align*}
The exponential map and the left Jacobian $\mathbf{\Gamma}_1$, are given in closed form for $\SE_2(3)$:
\begin{align*}
    \exp(\Vector{}{x}{}^{\wedge}) &= 
    \begin{bmatrix}
        \exp(\Vector{}{\omega}{}^{\wedge}) & \mathbf{\Gamma}_1(\Vector{}{\omega}{})\Vector{}{a}{} & \mathbf{\Gamma}_1(\Vector{}{\omega}{})\Vector{}{b}{}\\
        \Vector{}{0}{1\times3} & 1 & 0\\
        \Vector{}{0}{1\times3} & 0 & 1
    \end{bmatrix} ,\\
    \mathbf{\Gamma}_1(\Vector{}{x}{}) = \mathbf{J}_L(\Vector{}{x}{}) &= 
    \begin{bmatrix}
        \mathbf{\Gamma}_1(\Vector{}{\omega}{}) & \Vector{}{0}{3\times3} & \Vector{}{0}{3\times3}\\
        \mathbf{Q}(\Vector{}{\omega}{}, \Vector{}{v}{}) & \mathbf{\Gamma}_1(\Vector{}{\omega}{}) & \Vector{}{0}{3\times3}\\
        \mathbf{Q}(\Vector{}{\omega}{}, \Vector{}{w}{}) & \Vector{}{0}{3\times3} & \mathbf{\Gamma}_1(\Vector{}{\omega}{})
    \end{bmatrix} .
\end{align*}

The Special Euclidean group can be extended to ${\SE_n(3), \forall n \in \mathbb{N}}$~\cite{barrau2015ekf}. This extension is straightforward and does not require particular treatment.

\subsection{Homogenous galilean group $\HG(3)$}
The \emph{homogenous Galilean group} $\HG(3)$~\cite{Fornasier2023EquivariantSystems} is the group of 3D rotation and relative velocity transformations. $\HG(3)$ is isomorphic with $\SE(3)$ but acts on physical velocities rather than physical translations. Let $X \in \HG(3)$ represent an element of the Lie group. The homogeneous Galilean group and its Lie algebra are defined as follows:
\begin{align*}
    \HG(3) &= \set{
    \begin{bmatrix}
        \mathbf{A} & \Vector{}{a}{} \\ \mathbf{0}_{1\times 3} & 1
    \end{bmatrix}
    \in \R^{4\times4}}{
    \mathbf{A} \in \SO(3), \; \Vector{}{a}{} \in \R^3}, \\
    \hg(3) &= \set{
    \begin{bmatrix} \Vector{}{\omega}{}\\ \Vector{}{v}{} \end{bmatrix}^{\wedge} = 
    \begin{bmatrix}
        \Vector{}{\omega}{}^{\wedge} & \Vector{}{v}{} \\ \mathbf{0}_{1\times 3} & 0
    \end{bmatrix}
    \in \R^{4\times4}}{
    \Vector{}{\omega}{}^{\wedge} \in \so(3), \; \Vector{}{v}{} \in \R^3},
\end{align*}
$\HG(3)$ is isomorphic to $\SE(3)$. Let $X = (A, a)$, and $X^{-1} = (A^{-1}, -A^{-1}a)$, all the identities previously introduced for $\SE(3)$ hold.

\subsection{Inhomogeneous Galilean group $\grpG(3)$}\label{math_G3_sec}
The \emph{inhomogeneous Galilean group} $\grpG(3)$ is the group of 3D rotations, translations in space and time, and transformations between frames of reference that differ only by constant relative motion~\cite{DeMontigny2006GalileiContractions}. The inhomogeneous Galilean group is a subgroup of the $\SIM_2(3)$ Lie group, introduced in~\cite{vanGoor2023ConstructiveNavigation}. Let $X \in \grpG(3)$ represent an abstract element of the Lie group. The inhomogeneous Galilean group and its Lie algebra are defined as follows:
\begin{align*}
    \grpG(3) &= \set{
    \begin{bmatrix}
        \mathbf{A} & \Vector{}{a}{} & \Vector{}{b}{} \\ 
        \mathbf{0}_{1\times 3} & 1 & c\\
        \mathbf{0}_{1\times 3} & 0 & 1
    \end{bmatrix}
    \in \R^{5\times5}}{
    \mathbf{A} \in \SO(3), \; \Vector{}{a}{}, \Vector{}{b}{} \in \R^{3}, c \in \R}, \\
    \gothg(3) &= \set{
    \begin{bmatrix} \Vector{}{\omega}{}\\ \Vector{}{v}{}\\ \Vector{}{w}{}\\ \alpha \end{bmatrix}^{\wedge} = 
    \begin{bmatrix}
        \Vector{}{\omega}{}^{\wedge} & \Vector{}{v}{} & \Vector{}{w}{} \\ 
        \mathbf{0}_{1\times 3} & 0 & \alpha\\
        \mathbf{0}_{1\times 3} & 0 & 0
    \end{bmatrix}
    \in \R^{5\times5}}{
    \Vector{}{\omega}{}^{\wedge} \in \so(3), \; \Vector{}{v}{}, \Vector{}{w}{} \in \R^{3},\; \alpha \in \R}.
\end{align*}
\sloppy Let ${X = (A, a, b, c) \in \grpG(3)}$, the inverse element is written ${X^{-1} = (A^{-1}, -A^{-1}a, -A^{-1}(b - ca), -c)}$; which, in their matrix representation, are written
\begin{align*}
    \mathbf{X} =
    \begin{bmatrix}
        \mathbf{A} & \Vector{}{a}{} & \Vector{}{b}{} \\ 
        \mathbf{0}_{1\times 3} & 1 & c\\
        \mathbf{0}_{1\times 3} & 0 & 1
    \end{bmatrix}, &&
    \mathbf{X}^{-1} = 
    \begin{bmatrix}
        \mathbf{A}^{\top} & -\mathbf{A}^{\top}\Vector{}{a}{} & -\mathbf{A}^{\top}(\Vector{}{b}{} - c\Vector{}{a}{}) \\ 
        \mathbf{0}_{1\times 3} & 1 & -c\\
        \mathbf{0}_{1\times 3} & 0 & 1
    \end{bmatrix}.
\end{align*}
The identity element is given by $I = (I, 0, 0, 0)$, and its matrix representation is given by the $\eye_5$ matrix. Specifically, $\mathbf{A} = \eye_3$, ${\Vector{}{a}{} = \Vector{}{b}{} = \Vector{}{0}{}}$ and $c = 0$.
Let ${\Vector{}{x}{} \in \R^{10} \st \Vector{}{x}{}^{\wedge} \in \gothg(3)}$, the Adjoint map for $\grpG(3)$ and the Adjoint matrix are written
\begin{align*}
    &\Adsym{X}{\Vector{}{x}{}^{\wedge}} \coloneqq \mathbf{X}\Vector{}{x}{}^{\wedge}\mathbf{X}^{-1}, && \AdMsym{X} \coloneqq 
    \begin{bmatrix}
        \mathbf{A} & \Vector{}{0}{3\times3} & \Vector{}{0}{3\times3} & \Vector{}{0}{3\times1}\\
        \Vector{}{a}{}^{\wedge}\mathbf{A} & \mathbf{A} & \Vector{}{0}{3\times3} & \Vector{}{0}{3\times1}\\
        (\Vector{}{b}{} - c\Vector{}{a}{})^{\wedge}\mathbf{A} & -c\mathbf{A} & \mathbf{A} & \Vector{}{a}{}\\
        \Vector{}{0}{1\times3} & \Vector{}{0}{1\times3} & \Vector{}{0}{1\times3} & 1
    \end{bmatrix}.
\end{align*}
Let ${\Vector{}{x}{} = (\Vector{}{\omega}{}, \Vector{}{v}{}, \Vector{}{w}{}, \alpha) \in \R^{10} \st \Vector{}{x}{}^{\wedge} \in \gothg(3)}$, and ${\Vector{}{y}{} = (\Vector{}{\nu}{}, \Vector{}{u}{}, \Vector{}{q}{}, \beta) \in \R^{10} \st \Vector{}{y}{}^{\wedge} \in \gothg(3)}$ the adjoint map for $\gothg(3)$ and the adjoint matrix are written
\begin{align*}
    &\adsym{\Vector{}{x}{}^{\wedge}}{\Vector{}{y}{}^{\wedge}} \coloneqq \left[\Vector{}{x}{}^{\wedge}, \Vector{}{y}{}^{\wedge}\right], && \adMsym{\Vector{}{x}{}} \coloneqq 
    \begin{bmatrix}
        \Vector{}{\omega}{}^{\wedge} & \Vector{}{0}{3\times3} & \Vector{}{0}{3\times3} & \Vector{}{0}{3\times1}\\
        \Vector{}{v}{}^{\wedge} & \Vector{}{\omega}{}^{\wedge} & \Vector{}{0}{3\times3} & \Vector{}{0}{3\times1}\\
        \Vector{}{w}{}^{\wedge} & -\alpha\eye_3 & \Vector{}{\omega}{}^{\wedge} & \Vector{}{v}{}\\
        \Vector{}{0}{1\times3} & \Vector{}{0}{1\times3} & \Vector{}{0}{1\times3} & 0
    \end{bmatrix}.
\end{align*}
The exponential map is given in closed form for $\grpG(3)$:
\begin{align*}
    &\exp(\Vector{}{x}{}^{\wedge}) = 
    \begin{bmatrix}
        \exp(\Vector{}{\omega}{}^{\wedge}) & \mathbf{\Gamma}_1(\Vector{}{\omega}{})\Vector{}{v}{} & \mathbf{\Gamma}_1(\Vector{}{\omega}{})\Vector{}{w}{} + \alpha\mathbf{\Gamma}_2(\Vector{}{\omega}{})\Vector{}{v}{}\\
        \Vector{}{0}{1\times3} & 1 & \alpha\\
        \Vector{}{0}{1\times3} & 0 & 1
    \end{bmatrix} .\\
    &\log(\mathbf{X}) =
    \begin{bmatrix}
        \log(\mathbf{A}) & \mathbf{\Gamma}_1^{-1}(\log(\mathbf{A})^{\vee})\Vector{}{a}{} & \mathbf{\Gamma}_1^{-1}(\log(\mathbf{A})^{\vee})  \mathbf{\Xi}\\
        \Vector{}{0}{1\times3} & 0 & c\\
        \Vector{}{0}{1\times3} & 0 & 0
    \end{bmatrix} ,
\end{align*}
with
\begin{equation*}
    \mathbf{\Xi} = (\Vector{}{b}{} - c\mathbf{\Gamma}_2(\log(\mathbf{A})^{\vee})\mathbf{\Gamma}_1^{-1}(\log(\mathbf{A})^{\vee})\Vector{}{a}{}).
\end{equation*}
The left Jacobian $\mathbf{\Gamma}_1$ is given in closed form for $\grpG(3)$:
\begin{equation*}
\mathbf{\Gamma}_1(\Vector{}{x}{}) = \mathbf{J}_L(\Vector{}{x}{}) = 
\begin{bmatrix}
    \mathbf{\Gamma}_1(\Vector{}{\omega}{}) & \Vector{}{0}{3\times3} & \Vector{}{0}{3\times3} & \Vector{}{0}{3\times1}\\
    \mathbf{Q}_1(\Vector{}{\omega}{}, \Vector{}{v}{}) & \mathbf{\Gamma}_1(\Vector{}{\omega}{}) & \Vector{}{0}{3\times3} & \Vector{}{0}{3\times1}\\
    \mathbf{Q}_1(\Vector{}{\omega}{}, \Vector{}{w}{})-\alpha\mathbf{Q}_2(\Vector{}{\omega}{}, \Vector{}{v}{}) & -\alpha\mathbf{U}_1(\Vector{}{\omega}{}) & \mathbf{\Gamma}_1(\Vector{}{\omega}{}) & \mathbf{\Gamma}_2(\Vector{}{\omega}{})\Vector{}{v}{}\\
    \Vector{}{0}{1\times3} & \Vector{}{0}{1\times3} & \Vector{}{0}{1\times3} & 1\\
\end{bmatrix}
\end{equation*}
with
\begin{equation*}
    \begin{split}
    \mathbf{Q}_1(\Vector{}{\omega}{}, \Vector{}{z}{}) &\coloneqq \sum_{p=0}^{\infty}\sum_{r=0}^{\infty}\frac{1}{(p+r+2)!}(\Vector{}{\omega}{}^{\wedge})^{r}\Vector{}{z}{}^{\wedge}(\Vector{}{\omega}{}^{\wedge})^{p} = \\
    & = \frac{1}{2}\Vector{}{z}{}^{\wedge} + \left(\frac{\norm{\Vector{}{\omega}{}} - \sin(\norm{\Vector{}{\omega}{}})}{\norm{\Vector{}{\omega}{}}^{3}}\right)\left(\Vector{}{\omega}{}^{\wedge}\Vector{}{z}{}^{\wedge} + \Vector{}{z}{}^{\wedge}\Vector{}{\omega}{}^{\wedge} + \Vector{}{\omega}{}^{\wedge}\Vector{}{z}{}^{\wedge}\Vector{}{\omega}{}^{\wedge}\right) + \\
    & + \left(\frac{\norm{\Vector{}{\omega}{}}^{2} + 2\cos(\norm{\Vector{}{\omega}{}}) - 2}{2\norm{\Vector{}{\omega}{}}^{4}}\right)\left(\Vector{}{\omega}{}^{\wedge}\Vector{}{\omega}{}^{\wedge}\Vector{}{z}{}^{\wedge} + \Vector{}{z}{}^{\wedge}\Vector{}{\omega}{}^{\wedge}\Vector{}{\omega}{}^{\wedge} - 3\Vector{}{\omega}{}^{\wedge}\Vector{}{z}{}^{\wedge}\Vector{}{\omega}{}^{\wedge}\right) + \\
    & + \left(\frac{2\norm{\Vector{}{\omega}{}} - 3\sin(\norm{\Vector{}{\omega}{}}) + \norm{\Vector{}{\omega}{}}\cos(\norm{\Vector{}{\omega}{}})}{2\norm{\Vector{}{\omega}{}}^{5}}\right)\left(\Vector{}{\omega}{}^{\wedge}\Vector{}{z}{}^{\wedge}\Vector{}{\omega}{}^{\wedge}\Vector{}{\omega}{}^{\wedge} + \Vector{}{\omega}{}^{\wedge}\Vector{}{\omega}{}^{\wedge}\Vector{}{z}{}^{\wedge}\Vector{}{\omega}{}^{\wedge}\right) ,
    \end{split}
\end{equation*}
\begin{equation*} 
    \begin{split}   
    \mathbf{Q}_2(\Vector{}{\omega}{}, \Vector{}{z}{}) &\coloneqq \sum_{p=0}^{\infty}\sum_{r=0}^{\infty}\frac{p+1}{(p+r+3)!}(\Vector{}{\omega}{}^{\wedge})^{r}\Vector{}{z}{}^{\wedge}(\Vector{}{\omega}{}^{\wedge})^{p} = \\
    & = \frac{1}{6}\Vector{}{z}{}^{\wedge} + \left(\frac{\norm{\Vector{}{\omega}{}}^{2} + 2\cos(\norm{\Vector{}{\omega}{}}) - 2}{2\norm{\Vector{}{\omega}{}}^{4}}\right)\Vector{}{z}{}^{\wedge}\Vector{}{\omega}{}^{\wedge} + \\
    & + \left(\frac{\norm{\Vector{}{\omega}{}}^{3} - 6\norm{\Vector{}{\omega}{}} + 6\sin(\norm{\Vector{}{\omega}{}})}{6\norm{\Vector{}{\omega}{}}^{5}}\right)\Vector{}{z}{}^{\wedge}\Vector{}{\omega}{}^{\wedge}\Vector{}{\omega}{}^{\wedge} + \\
    & + \left(\frac{-2\cos(\norm{\Vector{}{\omega}{}}) - \norm{\Vector{}{\omega}{}}\sin(\norm{\Vector{}{\omega}{}}) + 2}{\norm{\Vector{}{\omega}{}}^{4}}\right)\Vector{}{\omega}{}^{\wedge}\Vector{}{z}{}^{\wedge} + \\
    & + \left(\frac{\norm{\Vector{}{\omega}{}}^{3} + 6\norm{\Vector{}{\omega}{}}\cos(\norm{\Vector{}{\omega}{}}) + 6\norm{\Vector{}{\omega}{}} - 12\sin(\norm{\Vector{}{\omega}{}})}{6\norm{\Vector{}{\omega}{}}^{5}}\right)\Vector{}{\omega}{}^{\wedge}\Vector{}{\omega}{}^{\wedge}\Vector{}{z}{}^{\wedge} + \\
    & + \left(\frac{-3\norm{\Vector{}{\omega}{}}\cos(\norm{\Vector{}{\omega}{}}) - (\norm{\Vector{}{\omega}{}}^{2} - 3)\sin(\norm{\Vector{}{\omega}{}})}{4\norm{\Vector{}{\omega}{}}^{5}}\right)\Vector{}{\omega}{}^{\wedge}\Vector{}{z}{}^{\wedge}\Vector{}{\omega}{}^{\wedge} + \\
    & + \left(\frac{\norm{\Vector{}{\omega}{}}\cos(\norm{\Vector{}{\omega}{}}) + 2\norm{\Vector{}{\omega}{}} - 3\sin(\norm{\Vector{}{\omega}{}})}{4\norm{\Vector{}{\omega}{}}^{5}}\right)\Vector{}{\omega}{}^{\wedge}\Vector{}{\omega}{}^{\wedge}\Vector{}{z}{}^{\wedge}\Vector{}{\omega}{}^{\wedge} + \\
    & + \left(\frac{(\norm{\Vector{}{\omega}{}}^{2} - 8)\cos(\norm{\Vector{}{\omega}{}}) - 5\norm{\Vector{}{\omega}{}}\sin(\norm{\Vector{}{\omega}{}}) + 8}{4\norm{\Vector{}{\omega}{}}^{6}}\right)\Vector{}{\omega}{}^{\wedge}\Vector{}{z}{}^{\wedge}\Vector{}{\omega}{}^{\wedge}\Vector{}{\omega}{}^{\wedge} + \\
    & + \left(\frac{2\norm{\Vector{}{\omega}{}}^{2} + 15\norm{\Vector{}{\omega}{}}\cos(\norm{\Vector{}{\omega}{}}) + 3(\norm{\Vector{}{\omega}{}}^{2} - 5)\sin(\norm{\Vector{}{\omega}{}})}{12\norm{\Vector{}{\omega}{}}^{7}}\right)\Vector{}{\omega}{}^{\wedge}\Vector{}{\omega}{}^{\wedge}\Vector{}{z}{}^{\wedge}\Vector{}{\omega}{}^{\wedge}\Vector{}{\omega}{}^{\wedge} ,
    \end{split}
\end{equation*}
\noindent\begin{align*}  
    \mathbf{U}_1(\Vector{}{\omega}{}) &\coloneqq \sum_{k=0}^{\infty}\frac{1}{(k+2)!}(\Vector{}{\omega}{}^{\wedge})^{k} = &\\
    & = \frac{1}{2}\eye_3 + \left(\frac{\sin(\norm{\Vector{}{\omega}{}}) - \norm{\Vector{}{\omega}{}}\cos(\norm{\Vector{}{\omega}{}})}{\norm{\Vector{}{\omega}{}}^{3}}\right)\Vector{}{\omega}{}^{\wedge} + &\\
    & + \left(\frac{\norm{\Vector{}{\omega}{}}^{2} - 2\norm{\Vector{}{\omega}{}}\sin(\norm{\Vector{}{\omega}{}}) - 2\cos(\norm{\Vector{}{\omega}{}}) + 2}{2\norm{\Vector{}{\omega}{}}^{4}}\right)\Vector{}{\omega}{}^{\wedge}\Vector{}{\omega}{}^{\wedge} .& 
\end{align*}

\subsection{Intrinsic group $\IN$}
The \emph{intrinsic group} $\IN$ is the group of camera intrinsic transformations. Let $X \in \IN$ represent an abstract element of the Lie group. The intrinsic group and its Lie algebra are defined as follows:
\begin{align*}
    \IN &= \set{
    \begin{bmatrix}
         a & s & x \\ 0 & b & y \\ 0 & 0 & 1
    \end{bmatrix} \in \R^{3\times 3}}
    {a,b > 0, \; s,x,y\in \R},\\
    \gothin &= \set{
    \begin{bmatrix}
         \alpha & \zeta & \gamma \\ 0 & \beta & \delta \\ 0 & 0 & 0
    \end{bmatrix} \in \R^{3\times 3}}
    {\alpha, \beta, \zeta, \delta, \gamma \in \R}
\end{align*}
Let ${X = (a,b,s,x,y) \in \IN}$, the inverse element is written ${X^{-1} = (\frac{1}{a}, \frac{1}{b}, -\frac{s}{ab}, -\frac{x}{a}+\frac{sy}{ab}, -\frac{y}{b})}$. In their matrix representation are written
\begin{align*}
    \mathbf{X} = 
    \begin{bmatrix}
         a & s & x \\ 0 & b & y \\ 0 & 0 & 1
    \end{bmatrix}, &&
    \mathbf{X}^{-1} = \begin{bmatrix}
        \frac{1}{a} & -\frac{s}{ab} & -\frac{x}{a}+\frac{sy}{ab}\\
        0 & \frac{1}{b} & -\frac{y}{b}\\
        0 & 0 & 1
    \end{bmatrix}.
\end{align*}
The identity element is given by $I = (1, 1, 0, 0, 0)$, and its matrix representation is given by the $\eye_4$ matrix. Specifically, $a = b = 1$, $s = x = y = 0$.
Let ${\Vector{}{x}{} = (\alpha, \beta, \zeta, \gamma, \delta) \in \R^5 \st \Vector{}{x}{}^{\wedge} \in \gothin}$, the Adjoint map for $\IN$ and the Adjoint matrix are written
\begin{align*}
    &\Adsym{X}{\Vector{}{x}{}^{\wedge}} \coloneqq \mathbf{X}\Vector{}{x}{}^{\wedge}\mathbf{X}^{-1}, && \AdMsym{X} \coloneqq 
    \begin{bmatrix}
        1 & 0 & 0 & 0 & 0 \\
        0 & 1 & 0 & 0 & 0 \\
        \frac{sy}{b}-x & -\frac{sy}{b} & -\frac{ay}{b}& a & s\\
        0 & -y & 0 & 0 & b
    \end{bmatrix}.
\end{align*}
Let ${\Vector{}{x}{} = (\alpha, \beta, \zeta, \gamma, \delta) \in \R^5 \st \Vector{}{x}{}^{\wedge} \in \gothin}$, and ${\Vector{}{y}{} \in \R^5 \st \Vector{}{y}{}^{\wedge} \in \gothin}$ the adjoint map for $\gothin$ and the adjoint matrix are written
\begin{align*}
    &\adsym{\Vector{}{x}{}^{\wedge}}{\Vector{}{y}{}^{\wedge}} \coloneqq \left[\Vector{}{x}{}^{\wedge}, \Vector{}{y}{}^{\wedge}\right], && \adMsym{\Vector{}{x}{}} \coloneqq 
    \begin{bmatrix}
        0 & 0 & 0 & 0 & 0 \\
        0 & 0 & 0 & 0 & 0 \\
        -\zeta & \zeta & \alpha-\beta & 0 & 0 \\
        -\gamma & 0 & -\delta & \alpha & \zeta \\
        0 & -\delta & 0 & 0 & \beta
    \end{bmatrix}.
\end{align*}
The matrix exponential and logarithm maps are given, respectively, by the matrix exponential and logarithm.

\section{Maps and operators}\label{math_maps_sec}

${\eye_n \in \R^{n \times n}}$ denotes the $n$-dimensional identity matrix, and ${\mathbf{0}_{n \times m} \in \R^{n \times m}}$ denotes the zero matrix with $n$ rows and $m$ columns.
For all $\Vector{}{v}{} = \left(x,y,z\right) \in \R^3$, define the maps:
\begin{align}
     &\pi_{\mathbb{Z}_1}\left(\cdot\right) \AtoB{\R^{3}}{\R^{3}},\qquad \Vector{}{v}{} \mapsto \pi_{\mathbb{Z}_1}\left(\Vector{}{v}{}\right) \coloneqq \frac{\Vector{}{v}{}}{z},\label{eq_piz1}\\
     &\pi_{\mathbb{S}^{2}}\left(\cdot\right) \AtoB{\R^{3}}{\R^{3}},\qquad \Vector{}{v}{} \mapsto \pi_{\mathbb{S}^{2}}\left(\Vector{}{v}{}\right) \coloneqq \frac{\Vector{}{v}{}}{\norm{\Vector{}{v}{}}},\label{eq_pis2}\\
     &\Xi\left(\cdot\right) \AtoB{\R^3}{\R^{3 \times 4}},\qquad
     \Vector{}{v}{} \mapsto \Xi\left(\Vector{}{v}{}\right) = 
     \begin{bmatrix}
         x & 0 & z & 0\\
         0 & y & 0 & z\\
         0 & 0 & 0 & 0
     \end{bmatrix}.
 \end{align}
For all $X = \left(A, a_{1}, \cdots, a_{n}\right) \in \SE_n(3)$, define the following useful map:
\begin{equation}
    \Gamma\left(\cdot\right) \AtoB{\SE_n(3)}{\SO(3)},\qquad X \mapsto \Gamma\left(X\right) \coloneqq A \in \SO(3).
\end{equation}
For all $X = \left(A, a, b\right) \in \SE_2(3)$, define the maps:
\begin{align}
    \chi\left(\cdot\right) \AtoB{\SE_2(3)}{\SE(3)},\qquad &X \mapsto \chi\left(X\right) = \left(A, a\right) \in \SE(3),\\
    \Theta\left(\cdot\right) \AtoB{\SE_2(3)}{\SE(3)},\qquad &X \mapsto \Theta\left(X\right) = \left(A, b\right) \in \SE(3) ,\\
    \Omega\left(\cdot\right) \AtoB{\SE_2(3)}{\se_2(3)},\qquad &X \mapsto \Omega\left(X\right) = \left(\mathbf{0}_{3 \times 1}, \mathbf{0}_{3 \times 1}, \Vector{}{a}{}\right)^\wedge \in \se_2(3) .
\end{align} 
For all $\Vector{}{x}{} = (\Vector{}{a}{}, \Vector{}{b}{}, \Vector{}{c}{}) \in \R^9 \st \Vector{}{a}{}, \Vector{}{b}{}, \Vector{}{c}{} \in \R^3$, define the linear maps:
 \begin{align}
    \Pi\left(\cdot\right) \AtoB{\se_2(3)}{\se(3)}, \quad &\Vector{}{x}{}^{\wedge} \mapsto \Pi\left(\Vector{}{x}{}^{\wedge}\right) = \left(\Vector{}{a}{}, \Vector{}{b}{}\right)^{\wedge} \in \se(3), \\
    \Upsilon\left(\cdot\right) \AtoB{\se_2(3)}{\se(3)}, \quad &\Vector{}{x}{}^{\wedge} \mapsto \Upsilon\left(\Vector{}{x}{}^{\wedge}\right) = \left(\Vector{}{a}{}, \Vector{}{c}{}\right)^{\wedge} \in \se(3),
\end{align}

\chapter{Equivariant System Theory}\label{eq_chp}

The theory of \emph{Equivariant systems} refers to the study of nonlinear systems that admits a \emph{symmetry}, associated with \emph{equivariance}~\cite{Mahony2020EquivariantDesign}. A core concept in this theory is that \emph{non-linearities in many natural, mechanical, and navigation systems do not arise from complex kinematics but rather from the geometric structure associated with the state space of the system}; in other words, there exist transformations that leaves unchanged or change in a structured manner the physical laws that govern their motion. Such structure is what is referred to as symmetry. 

This chapter reviews the recent results in the theory of equivariant systems and builds up the concepts that are at the core of the work presented in this dissertation. In particular, the first section of this chapter discusses the recent results by Mahony \etal~\cite{Mahony2020EquivariantDesign, Mahony2022ObserverEquivariance} and describes the meaning of symmetry and equivariance in the context of equivariant systems. The second section of this chapter discusses how symmetry and equivariance of a system are exploited for filter design. Specifically, it introduces the \emph{\acl{eqf}}; a general filtering approach for equivariant systems, established by van Goor \etal~\cite{VanGoor2020EquivariantSpaces, vanGoor2022EquivariantEqF, VanGoor2023EquivariantAwareness}.

\section{Symmetry, invariance and equivariance}\label{eq_sym_sec}

Let $\xi \in \calM$ be the state of a kinematic system posed on a smooth manifold $\calM$. Let $y \in \calN$ be the output of a kinematic system on a smooth manifold $\calN$. Let $u \in \vecL \subset \R^n$ be the input of a kinematic system. Consider a nonlinear kinematic system
\begin{align*}
    &\dot{\xi} = f(\xi, u) ,\\
    &y = h(\xi) .
\end{align*}
$f$ is an affine system function ${f \AtoB{\calM \times \vecL}{\tT\calM} \st \left(\xi, u\right) \mapsto f(\xi, u)}$, which, in some local coordinates, is written
\begin{equation}\label{eq_affine_form}
    f(\xi, u) = f_0(\xi) + \sum_{k=1}^{n}f_k(\xi)u_k = \sum_{k=0}^{n}f_k(\xi)u_k,
\end{equation}
for ${u = (u_0,u_1,\cdots,u_n) \in \vecL \st u_0 = 1}$. The output function is $h \AtoB{\calM}{\calN}$. The system function can be interpreted as an affine map of the input space ${f \AtoB{\vecL}{\mathfrak{X}(\calM)} \st u \mapsto f_u}$. This last interpretation will result in an extremely powerful characterization of \emph{invariance} and \emph{equivariance} properties of a system~\cite{Mahony2020EquivariantDesign, Mahony2022ObserverEquivariance}.

In the context of kinematic systems, the concept of symmetries was first introduced in the field of geometric control theory~\cite{Jurdjevic1996GeometricTheory, Bullo2005GeometricSystems}. A symmetry of a kinematic system can be seen as a set of transformations that either leave unchanged or change in a structured manner the physical laws that govern the motion of the system. A symmetry is defined as the action $\phi$ of a Lie group $\grpG$ (often referred to as \emph{symmetry group}) on the state space $\calM$ of a system. Formally $(\grpG, \phi) \st \phi \AtoB{\grpG \times \calM}{\calM}$. The group action $\phi$ can be either a right action or a left action. Either convention leads to the same mathematical framework~\cite{Mahony2022ObserverEquivariance}, and any left action is transformed to a right action by considering the inverse parametrization of the group~\cite{Mahony2020EquivariantDesign}; however, the results discussed in this dissertation make use of the right convention and, hence, of a right group action $\phi$. 

\begin{boxexample}{Direction kinematics symmetry}{}
For example, consider the kinematics of a direction on a sphere:
\begin{equation}\label{eq_dir_kin}
    \dotVector{}{d}{} = \Vector{}{d}{}^{\wedge}\Vector{}{\omega}{} = \Vector{}{d}{} \times \Vector{}{\omega}{},
\end{equation}
where $\xi = \Vector{}{d}{} \in S^2$ is the state of the system. The input of the system is ${u = \Vector{}{\omega}{} \in \vecL \subset \R^3}$.
A symmetry for such a system is given by the special orthogonal group $\SO(3)$ and by its action on the 2-sphere. Let ${A \in \SO(3)}$, the right-handed group action $\phi$ is written
\begin{equation}\label{eq_dir_kin_phi}
    \phi \AtoB{\SO(3) \times S^2}{S^2} \stq (A, \xi) \mapsto \phi(A, \xi) \coloneqq A^\top\Vector{}{d}{} .
\end{equation}
Note that $\phi(A, \xi) \coloneqq A^\top\Vector{}{d}{}$ represents a valid right group action, whereas $\phi(A, \xi) \coloneqq A\Vector{}{d}{}$ represents a left action.
\end{boxexample}

A symmetry that leaves the law of motion of a kinematic system unchanged encodes the \emph{invariance} of the system. Instead, a symmetry that changes the law of motion of a kinematic system in a structured manner encodes the \emph{equivaraince} of the system. To formally characterize invariance and equivariance, consider the induced group action ${\Phi \AtoB{\grpG \times \mathfrak{X}\left(\calM\right)}{\mathfrak{X}\left(\calM\right)}}$ on the set of all smooth vector fields over $\calM$ defined in \cref{math_induced_action}. A system $f$ is said to be \emph{invariant} if
\begin{equation}\label{eq_inv}
    \Phi(X, f_u) = \td\phi_{X} \circ f_u \circ \phi_{X^{-1}} = f_u.
\end{equation}
$\forall X \in \grpG$, and $u \in \vecL$. On the contrary, a system $f$ is said to be \emph{equivariant} if
\begin{equation}\label{eq_equi}
    \Phi(X, f_u) = \td\phi_{X} \circ f_u \circ \phi_{X^{-1}} = f_{\psi(X, u)},
\end{equation}
$\forall X \in \grpG$, $u \in \vecL$, and for a right-handed action of the group $\grpG$ on the input space $\vecL$, that is $\psi \AtoB{\grpG \times \vecL}{\vecL}$.

\begin{boxexample}{Direction kinematics equivariance}{}
For the direction kinematics introduced in the previous example, elements of the state space and elements of the input space change in a structured manner when transformed with actions of $\SO(3)$, specifically the induced group action $\Phi(X, f_u)$ is written
\begin{equation}
    \Phi(A, f_u) = \underbrace{A^{\top}}_{\td\phi_{X}}\underbrace{(\underbrace{A\Vector{}{d}{}}_{\phi_{X^{-1}}})^{\wedge}\Vector{}{\omega}{}}_{f_u \circ \phi_{X^{-1}}} = A^{\top}A\Vector{}{d}{}^{\wedge}A^{\top}\Vector{}{\omega}{} = \Vector{}{d}{}^{\wedge}(A^{\top}\Vector{}{\omega}{}),
\end{equation}
which is in the form of \cref{eq_equi} with the right-handed action $\psi$ defined by
\begin{equation}\label{eq_dir_kin_psi}
    \psi \AtoB{\SO(3) \times \vecL}{\vecL} \stq (A, u) \mapsto \psi(A, u) \coloneqq A^\top\Vector{}{\omega}{} .
\end{equation}
The direction kinematics in \cref{eq_dir_kin} is equivariant with respect to the action $\phi$ in \cref{eq_dir_kin_phi} and $\psi$ in \cref{eq_dir_kin_psi} of $\SO(3)$.
\end{boxexample}

Similar concepts apply to the configuration output $h$. In particular, a symmetry that leaves the output unchanged encodes the \emph{invariance of the output}, whereas a symmetry that changes the output in a structured manner encodes the \emph{equivariance of the output}. Formally the configuration output $h$ is said to be \emph{invariant} if
\begin{equation}
    h(\phi(X, \xi)) = h(\xi) ,
\end{equation}
$\forall X \in \grpG$, and $\xi \in \calM$. On the contrary, the configuration output $h$ is said to be \emph{equivariant} if
\begin{equation}\label{eq_out_equi}
    h(\phi(X, \xi)) = \rho(X, h(\xi)) ,
\end{equation}
$\forall X \in \grpG$, $\xi \in \calM$, and for a right-handed action of the group $\grpG$ on the output space $\calN$, that is $\rho \AtoB{\grpG \times \calN}{\calN}$.

\begin{boxexample}{Direction kinematics output equivariance}{}
Let $h\left(\xi\right) = \Vector{}{d}{} \in S^2$ represent the configuration output of the direction kinematics. Applying the state action $\phi$ previously introduced, yields
\begin{equation}\label{eq_dir_kin_rho}
h\left(\phi\left(X, \xi\right)\right) = A^\top\Vector{}{d}{} = \rho\left(X, h\left(\xi\right)\right).
\end{equation}
Hence, the configuration output of the direction kinematics is equivariant with respect to the action $\rho$ in \cref{eq_dir_kin_rho} of $\SO(3)$.
\end{boxexample}

In \cref{math_action}, we introduced the concept of homogeneous space. In particular, a smooth manifold $\calM$ is termed homogeneous space if the action $\phi$ of a Lie group $\grpG$ on $\calM$ is transitive. Thus, if the induced projection ${\phi_{\xi} \AtoB{\grpG}{\calM}}$ is surjective. In other words, if for every $\xi, \xizero \in \calM$, there exist at least one $X \in \grpG$ such that $\xi = \phi_{\xizero}(X)$. \emph{A homogeneous space is completely parametrized by the Lie group $\grpG$.}

This fact is of particular importance in the context of kinematic systems. If the state space of a system is a homogeneous space, the parametrization provided by the Lie group $\grpG$ allows the exploitation of the Lie group structure and, therefore, ``lift'' the system dynamics on the Lie group. That is, defining a \emph{lifted system} ${\dot{X} = X\Lambda}$ whose solutions $X(t)$ at time $t$, project to solutions $\xi(t) = \phi(X(t), \xizero)$. A graphical representation of the lifted system is provided in \cref{eq_symmetry}.

\begin{boxintuition*}{The concept of lifting the system dynamics}{}
    The concept of ``lifting'' the system dynamics on the Lie group is based on the idea of taking the original system evolving on the homogeneous space $\calM$ for which it is hard to work with and exploiting the symmetry $(\grpG, \phi)$ to define a new system evolving on the Lie group $\grpG$ for which is easy to work with.
    Consider the direction kinematics on the 2-sphere introduced in the previous examples. The 2-sphere $S^2$ does not possess structure; in fact, no composition rule is defined for elements of $S^2$; hence, working directly on the 2-sphere is hard. On the contrary, The Lie group structure of $\SO(3)$ makes it easy to work directly on it.
\end{boxintuition*}

According to~\cite{Mahony2020EquivariantDesign, Mahony2022ObserverEquivariance}, the lifted system can always be written
\begin{equation}
    \dot{X} = \td \textrm{L}_{X}\Lambda(\phi_{\xizero}(X), u),
\end{equation}
where $\Lambda \AtoB{\grpG \times \vecL}{\gothg}$ is the \emph{lift}, which provides the necessary structure that associates the system input $u$ with the Lie algebra $\gothg$ of the Lie group $\grpG$. The lift is the solution of the following equation:
\begin{equation}\label{eq_lift}
    \td_{\phi_{\xi}}\left[\Lambda(\xi, u)\right] = \Fr{X}{I}\phi_{\xi}\left(X\right)\left[\Lambda(\xi, u)\right] = f(\xi, u).
\end{equation}
Taking the time derivative of $\xi = \phi_{\xizero}(X)$ yields ${\dot{\xi} = f_u(\xi) = \td_{\phi_{\xizero}} \circ \dot{X}}$. It is clear now that the requirement of having a lifted system in the form of ${\dot{X} = \td \textrm{L}_{X}\Lambda(\phi_{\xizero}(X), u)}$ yields \cref{eq_lift}. Specifically, ${f_u(\xi) = \td_{\phi_{\xizero}} \circ \td \textrm{L}_{X}\Lambda(\phi_{\xizero}(X), u) = \td_{\phi_{\xi}}\Lambda(\xi, u)}$, with the last equality derived as follows:
\begin{align*}
    \td_{\phi_{\xizero}} \circ \td \textrm{L}_{X}\left[\Lambda\right] &= \Fr{X}{I}\phi_{\xizero}\left(X\right) \circ \Fr{Y}{I}\textrm{L}_{X}\left(Y\right)\left[\Lambda\right]\\
    &= \Fr{Y}{I}\phi\left(\textrm{L}_{X}\left(Y\right), \xizero\right)\left[\Lambda\right]\\
    &= \Fr{Y}{I}\phi\left(XY, \xizero\right)\left[\Lambda\right]\\
    &= \Fr{Y}{I}\phi\left(Y, \phi\left(X,\xizero\right)\right)\left[\Lambda\right]\\
    &= \Fr{Y}{I}\phi\left(Y, \xi\right)\left[\Lambda\right]\\
    &= \Fr{Y}{I}\phi_{\xi}\left(Y\right)\left[\Lambda\right] = \td_{\phi_{\xi}}\left[\Lambda\right].
\end{align*}

The existence of the lift is guaranteed by the transitivity of the group action $\phi$. However, if the system is equivariant, then the lift is also said to be equivariant, and it satisfies
\begin{equation}\label{eq_lift_equi}
    \Adsym{X^{-1}}{\Lambda(\xi, u)} = \Lambda(\phi_X(\xi), \psi_X(u)) .
\end{equation}
A graphical representation of equivariant symmetry, homogenous space, and lifted system is provided in \cref{eq_symmetry}.

\begin{boxexample}[label={eq_lift_example}]{Direction kinematics lift}{}
    For the direction kinematics,  the lift is the solution at the \cref{eq_lift}. In particular, the left-side of the equation is written
    \begin{equation}
    \begin{split}
        \td_{\phi_{\xi}}\left[\Lambda(\xi, u)\right] &= \Fr{X}{I}\phi_{\xi}\left(X\right)\left[\Lambda(\xi, u)\right]\\
        &= \lim_{t\to0}\left.\frac{1}{t}\left[(I - \lambda t)\Vector{}{d}{} - \Vector{}{d}{}\right]\right|_{\lambda = \Lambda(\xi, u)} \\
        &= -\Lambda(\xi, u)\Vector{}{d}{} ,
    \end{split}
    \end{equation}
    substituting now ${\td_{\phi_{\xi}}\Lambda(\xi, u)}$ in \cref{eq_lift} and solving for ${\Lambda(\xi, u)}$ yields
    \begin{equation}
    \begin{split}
        -\Lambda(\xi, u)\Vector{}{d}{} &= \Vector{}{d}{}^{\wedge}\Vector{}{\omega}{} \\
        &= -\Vector{}{\omega}{}^{\wedge}\Vector{}{d}{} \quad\Longrightarrow\quad \Lambda(\xi, u) = \Vector{}{\omega}{}^{\wedge} .
    \end{split}
    \end{equation}
    Moreover, the lift is clearly equivariant, that is
    \begin{equation}
    \begin{split}
        \Adsym{A^\top}{\Lambda(\xi, u)} &= A^\top\Lambda(\xi, u)A = A^\top\Vector{}{\omega}{}^{\wedge}A\\
        &= (A^\top\Vector{}{\omega}{})^{\wedge} = \Lambda(\phi_X(\xi), \psi_X(u)) .
    \end{split}
    \end{equation}
    Finally, the lifted system is written
    \begin{equation}
        \dot{A} = A\Vector{}{\omega}{}^{\wedge} .
    \end{equation}
\end{boxexample}

Natural questions that might arise at this point are: \emph{``How do we find a symmetry?''} and \emph{``How do we choose one symmetry if we have found more?''} 

Let's start with the latter question. If more symmetries are found, then choose the invariant symmetry (\cref{eq_inv}) first and the equivariant symmetry (\cref{eq_equi}) after. Clearly, invariant symmetries are much more rare, and most of the time, finding an equivariant symmetry already represents an achievement. For the former question, unfortunately, there exists no golden recipe yet. However, the whole process revolves around finding a transitive action $\phi$ of a Lie group. Explicit examples in the next chapters will teardown the thinking process into distinct stages for better understanding.

\begin{figure}[htp]
\centering
\includegraphics[width=\linewidth]{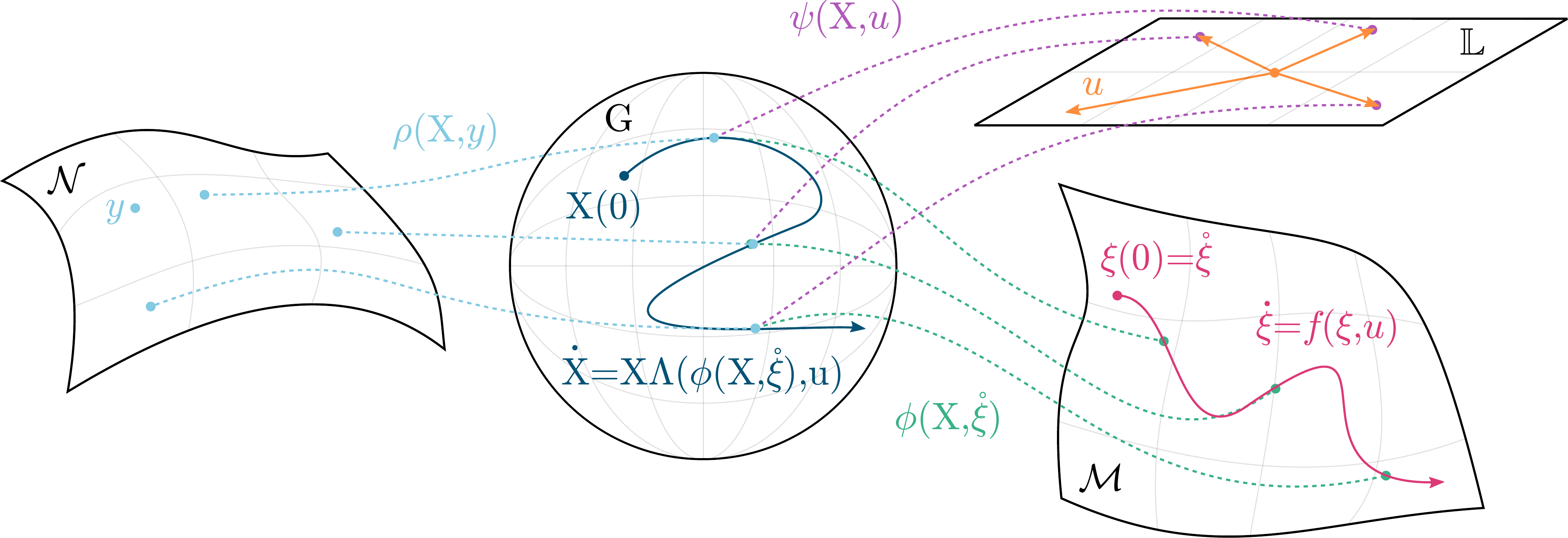}
\caption[Graphical representation of equivariant symmetry.]{Graphical representation of equivariant symmetry $(\grpG, \phi)$ of a homogeneous space $\calM$. A trajectory of a kinematic system $\xi(t) \in \calM$ is represented in red. The input of the system $u \in \vecL$ is represented in orange. The trajectory of the lifted system $X(t) \in \grpG$ is represented in blue. The dashed light green arrows represent the projection $\xi = \phi_{\xizero}(X)$ from $\grpG$ onto $\calM$ given by the transitive action $\phi$. The trajectory of the kinematic system on the homogeneous space, is completely parametrized by the trajectory of the lifted system onto the symmetry group. The dashed light purple arrows represent the action of the Lie group $\grpG$ on the input space $\vecL$ that defines equivariance. The dashed light blue lines represent the action of the Lie group $\grpG$ on the output space $\calN$ that defines the equivariance of the output.}
\label{eq_symmetry}
\end{figure}

\section{Equivariant filter design}\label{eq_eqf_design}
The concept of lifting the system dynamics onto the Lie group, introduced in the previous chapter, is of paramount importance in the theory of equivariant systems. The fact that the Lie group $\grpG$ parametrizes the state space $\calM$ of a kinematic system allows the formulation of a control or state estimation problem directly on the lifted system on the Lie group. Specifically for state estimation problems, this is equivalent to estimating an element of the symmetry group ${\hat{X} \in \grpG \st \hat{\xi} = \phi(\hat{X}, \xizero)}$, rather than $\hat{\xi} \in \calM$ directly. 
This is not just a mere ``embedding'' of the system on the Lie group, but abstracting the estimation problem on the Lie group and hence exploiting the symmetry of the kinematic systems allows the definition of a \emph{global error} ${e = \phi(\hat{X}^{-1}, \xi) \in \calM}$ termed \emph{equivariant error}~\cite{Mahony2020EquivariantDesign, Mahony2022ObserverEquivariance, VanGoor2020EquivariantSpaces, vanGoor2022EquivariantEqF, VanGoor2023EquivariantAwareness}.
\begin{align*}
    e &= \phi(\hat{X}^{-1}, \xi) && e \in \calM \text{: global error in the homogeneous space}\\
    &= \phi(\hat{X}^{-1}, \phi(X, \xizero))\\
    &= \phi(X\hat{X}^{-1}, \xizero)\\
    &= \phi(E, \xizero) && E = X\hat{X}^{-1} \in \grpG \text{: global error in the symmetry group}.
\end{align*}

The equivariant error is at the core of the \emph{\acl{eqf}} algorithm, introduced by van Goor \etal~\cite{VanGoor2020EquivariantSpaces, vanGoor2022EquivariantEqF, VanGoor2023EquivariantAwareness}, which represent a general filter design for systems on homogenous spaces. The \acl{eqf} is a state estimation algorithm where the \ac{ekf} design principles are applied to the linearized (global) error kinematics about an origin point $\xizero$. 
The equivariant error, in fact, is a measure between an element of the symmetry group $\hat{X} \in \grpG$, and an element of state space $\xi \in \calM$, defined in a neighborhood of a fixed point $\xizero$. This construction has motivated the choice to pose the filter state on the symmetry group.
\emph{The key advantage of such a global error definition is that the error kinematics is linearized about a single set of coordinates, whereas in the standard \ac{ekf} design, the error kinematics is linearized about linearization points taken along a time-varying trajectory and changing local coordinates}.

\begin{boxintuition*}{The difference between the classical and the equivariant error}{}
    The classical \ac{ekf} error for systems with Euclidean state space is defined as $e = \xi - \hat{\xi}$. For systems posed on manifolds, the classical error $e = \xi \boxminus \hat{\xi}$ can only be constructed by means of local coordinates centered around the time-varying state estimate $\hat{\xi}\left(t\right)$. In contrast, the equivariant error is globally defined by construction. Moreover, the equivariant error specializes to the classical error definition $e = \xi - \hat{\xi}$ for $\calM = \R^n$, $\grpG = \R^n$, $\xizero = \Vector{}{0}{}$ and the action $\phi$ given by vector addition.
\end{boxintuition*}

Let us walk through the main results that characterize the \acl{eqf} algorithm.
The \ac{eqf} is designed by linearizing the equivariant error and solving the Ricatti equation to propagate the error covariance. However, the equivariant error is defined as an element of the homogeneous space $\calM$. Hence, local coordinates for the error need to be defined. Similarly, the configuration output is, in general, an element of the output space $\calN$. Therefore, also for the configuration output, local coordinates need to be defined.
Consider a system whose state space is a $m$-dimensional homogenous space $\calM$, and output space a $n$-dimensional smooth manifold $\calN$. Let $\xizero \in \calM$ be the state origin. Define local coordinates of the state space $\varepsilon = \vartheta(e)$, thus a coordinate chart ${\vartheta \AtoB{\mathcal{U}_{\xizero} \subset \calM}{\R^{m}}}$. 
Let ${\yzero = h\left(\xizero\right)}$ be the output origin, then define local coordinates of the output space, thus a coordinate chart ${\delta \AtoB{\mathcal{U}_{\yzero} \subset \calN}{\R^{n}}}$.

\begin{boxexample}{Direction kinematics equivariant error}{}
The equivariant error and the global error in the symmetry group, for the direction kinematics previously introduced, are defined as follows:
\begin{align*}
    e &= \phi(\hat{X}^{-1}, \xi) = \hat{A}\Vector{}{d}{},\\
    E &= A\hat{A}^\top.
\end{align*}
\end{boxexample}

Let $\hat{X} \in \grpG$ denote the \acl{eqf} state state.
Let $\varepsilon$ denote the linearization of the error $e$ in the chart $\vartheta$. 
The linearized error kinematics, and the linearized output are defined according to~\cite{VanGoor2020EquivariantSpaces, vanGoor2022EquivariantEqF, VanGoor2023EquivariantAwareness}
\begin{align}
    &\dot{\varepsilon} \approx \mathbf{A}_{t}^{0}\varepsilon, \\ 
    &\mathbf{A}_{t}^{0} = \Fr{e}{\xizero}\vartheta\left(e\right)\Fr{E}{I}\phi_{\xizero}\left(E\right)\Fr{e}{\xizero}\Lambda\left(e, \mathring{u}\right)\Fr{\varepsilon}{\mathbf{0}}\vartheta^{-1}\left(\varepsilon\right) ,\label{eq_A0}\\
    &\delta\left(h\left(\xi\right)\right) - \delta\left(h\left(\hat{\xi}\right)\right) \approx \mathbf{C}^{0}\varepsilon , \\
    &\mathbf{C}^{0} = \Fr{y}{\yzero}\delta\left(y\right)\Fr{\xi}{\hat{\xi}}h\left(\xi\right)\Fr{e}{\xizero}\phi_{\hat{X}}(e)\Fr{\varepsilon}{\mathbf{0}}\vartheta^{-1}\left(\varepsilon\right) ,\label{eq_C0}
\end{align}
where \cref{eq_A0} is a chain of differentials and derives from differentiating the equivariant error in local coordinates ${\vartheta^{-1}(\phi(\hat{X}^{-1}, \xi))}$, whereas \cref{eq_C0} derives from differentiating ${\delta(h(\xi)) = \delta(h(\phi(\hat{X}, \vartheta^{-1}(\varepsilon))))}$. Detailed derivation of \cref{eq_A0,eq_C0} are found in~\cite[section V.B]{vanGoor2022EquivariantEqF}, and~\cite[section V.C]{vanGoor2022EquivariantEqF} respectively.

If an action $\rho$ of the symmetry group in the output space exists, then this can be exploited to define a different residual. Specifically
\begin{align}
    &\delta\left(h\left(e\right)\right) = \delta\left(\rho_{\hat{X}^{-1}}\left(h\left(\xi\right)\right)\right) \approx \mathbf{C}^{0}\varepsilon,\\
    &\mathbf{C}^{0}\varepsilon = \Fr{y}{\yzero}\delta\left(y\right)\Fr{e}{\xizero}h\left(e\right)\Fr{\varepsilon}{\mathbf{0}}\vartheta^{-1}\left(\varepsilon\right) .\label{eq_C0_rho}
\end{align}
Additionally, the equivariance of the output can be exploited together with the concept of normal coordinates to derive a linearized output with third-order error~\cite{VanGoor2020EquivariantSpaces, vanGoor2022EquivariantEqF, VanGoor2023EquivariantAwareness}.
Define normal coordinates to be ${e = \vartheta^{-1}(\varepsilon) \coloneqq \phi_{\xizero}(\exp_{\grpG}(\varepsilon^{\wedge})))}$, with ${\exp_{\grpG}}$ being the group exponential. A third-order linearization error of the output map can be achieved as follows:
\begin{align}
    &\delta\left(h\left(e\right)\right) = \delta\left(\rho_{\hat{X}^{-1}}\left(h\left(\xi\right)\right)\right) \approx \mathbf{C}^{\star}\varepsilon + \mathbf{O}(\varepsilon^3), \label{eq_eq_out_app}\\
    &\mathbf{C}^{\star}\varepsilon = \frac{1}{2}\Fr{y}{\yzero}\delta\left(y\right)\left(\Fr{E}{I}\rho_E\left(\yzero\right) + \Fr{E}{I}\rho_E\left(\rho_{\hat{X}^{-1}}\left(y\right)\right)\right)\varepsilon^{\wedge} ,\label{eq_C_star}
\end{align}
where $y$ is the output measurement. \Cref{eq_eq_out_app} defines the equivariant residual, and a detailed derivation of \cref{eq_C_star} is reported in \cref{appendix_B_chp}.
Moreover, if the system is not equivariant, thus, if no compatible action $\psi$ of the symmetry group on the input space is found, the state matrix can be computed alternatively according to
\begin{align}
    \begin{split}
    \mathbf{A}_{t}^{0} &= \Fr{e}{\xizero}\vartheta\left(e\right)
    \Fr{\xi}{\hat{\xi}}\phi_{\hat{X}^{-1}}\left(\xi\right)
    \Fr{E}{I}\phi_{\hat{\xi}}\left(E\right)\;\cdot\\
    &\quad\cdot\Fr{\xi}{\phi_{\hat{X}}\left(\xizero\right)}\Lambda\left(\xi, u\right)
    \Fr{e}{\xizero}\phi_{\hat{X}}\left(e\right)
    \Fr{\varepsilon}{\mathbf{0}}\vartheta^{-1}\left(\varepsilon\right),\label{eq_A0_alt}
    \end{split}\\
    \begin{split}
    &= \Fr{e}{\xizero}\vartheta\left(e\right)
    \Fr{E}{I}\phi_{\xizero}\left(E\right)\mathrm{Ad}_{\hat{X}}\;\cdot\\
    &\quad\cdot\Fr{\xi}{\phi_{\hat{X}}\left(\xizero\right)}\Lambda\left(\xi, u\right)
    \Fr{e}{\xizero}\phi_{\hat{X}}\left(e\right)
    \Fr{\varepsilon}{\mathbf{0}}\vartheta^{-1}\left(\varepsilon\right).\label{eq_A0_alt_normal}
    \end{split}
\end{align}

\begin{boxexample}{Direction kinematics local coordinates}{}
Since the direction kinematics exhibits output equivariance, a good choice of local coordinates for the error is represented by normal coordinates ${e = \vartheta^{-1}(\varepsilon) \coloneqq \phi_{\xizero}(\exp_{\grpG}(\varepsilon^{\wedge})))}$. 
Let $\xizero = \mathring{\bm{d}} \in S^2$ represent an arbitrary origin state, then the projection $\phi_{\xizero}$ is written
\begin{equation*}
    \phi_{\xizero}\left(X\right) = \xi = \Vector{}{d}{} = A^\top \mathring{\bm{d}}.
\end{equation*}
An inverse projection $\phi_{\xizero}^{\dagger}\left(\xi\right)$ is given by the rotation between the two direction $\mathring{\bm{d}}$ and $\Vector{}{d}{}$. Specifically, by employing the cross-product to obtain the rotation axis and atan2 to obtain the rotation angle, the rotation between the two directions is written
\begin{align*}
    \phi_{\xizero}^{\dagger}\left(\xi\right) &= A = \exp_{\SO\left(3\right)}\left(-\text{atan2}\left(\mathring{\bm{d}} \times \Vector{}{d}{}, \mathring{\bm{d}}^\top\Vector{}{d}{}\right)\frac{\mathring{\bm{d}} \times \Vector{}{d}{}}{\norm{\mathring{\bm{d}} \times \Vector{}{d}{}}}\right).
\end{align*}
Therefore, local coordinates $\varepsilon$ are written
\begin{equation*}
    \varepsilon = \vartheta\left(e\right) = \log_{\SO\left(3\right)}\left(\phi_{\xizero}^{\dagger}\left(e\right)\right) = -\text{atan2}\left(\mathring{\bm{d}} \times e, \mathring{\bm{d}}^\top e\right)\frac{\mathring{\bm{d}} \times e}{\norm{\mathring{\bm{d}} \times e}}
\end{equation*}

Similarly, let $\yzero = h\left(\xizero\right) = \xizero = \mathring{\bm{d}}$ be the output origin, then a choice of local coordinates for the output is given by
\begin{equation*}
    \delta\left(y\right) = \yzero \times y = \yzero^\wedge y.
\end{equation*}
\end{boxexample}

Define the initial condition of the equivariant filter state ${\hat{X}\left(0\right) = I}$. Let ${\bm{\Sigma} \in \PD\left(m\right) \subset \R^{m \times m}}$ be the Riccati matrix of the error in local coordinates, with initial condition ${\bm{\Sigma}\left(0\right) = \bm{\Sigma}_0}$. Let $\mathbf{A}_{t}^{0}$ be the state matrix defined in \cref{eq_A0,eq_A0_alt,eq_A0_alt_normal}, and $\mathbf{C}^0$ be the output matrix defined in \cref{eq_C0}. Let ${\mathbf{Q}\in \PD\left(m\right) \subset \R^{m \times m}}$ and ${\mathbf{R}\in \PD\left(n\right) \subset \R^{n \times n}}$ be respectively the state gain matrix and the output gain matrix. Chose a right inverse ${\Fr{E}{I}\phi_{\xi_0}\left(E\right)^{\dagger} \st \Fr{E}{I}\phi_{\xi_0}\left(E\right) \Fr{E}{I}\phi_{\xi_0}\left(E\right)^{\dagger} = I}$. Then, in a deterministic setting, the \acl{eqf} is given by the solution of the following equations:
\begin{align}
    &\dot{\hat{X}} = \td \textrm{L}_{\hat{X}}\Lambda(\phi_{\xizero}(\hat{X}), u) + \td \textrm{R}_{\hat{X}}\Delta ,\\
    &\Delta = \Fr{E}{I}\phi_{\xizero}\left(E\right)^{\dagger}\td\vartheta^{-1}\bm{\Sigma}{\mathbf{C}^{0}}^\top(\mathbf{C}^{0}\bm{\Sigma}{\mathbf{C}^{0}}^\top + \mathbf{R})^{-1}\delta(y - h(\phi(\hat{X}, \xizero))) ,\label{eq_deltaC0}\\
    &\dot{\bm{\Sigma}} = \mathbf{A}_{t}^{0}\bm{\Sigma} + \bm{\Sigma}{\mathbf{A}_{t}^{0}}^\top + \mathbf{B}_{t}\mathbf{Q}\mathbf{B}_{t}^{\top} - \mathbf{K}^{0}\mathbf{C}^{0}\bm{\Sigma} ,\label{eq_sigmadotC0}\\
    &\mathbf{K}^{0} = \bm{\Sigma}{\mathbf{C}^{0}}^\top(\mathbf{C}^{0}\bm{\Sigma}{\mathbf{C}^{0}}^\top + \mathbf{D}_t\mathbf{R}\mathbf{D}_t^{\top})^{-1} ,\label{eq_KC0}
\end{align}
with $\mathbf{B}_{t} = \eye_m$ and $\mathbf{D}_t = \eye_n$. 
If the output is equivariant, \cref{eq_deltaC0,eq_sigmadotC0,eq_KC0} change according to
\begin{align}
    &\Delta = \Fr{E}{I}\phi_{\xizero}\left(E\right)^{\dagger}\td\vartheta^{-1}\bm{\Sigma}{\mathbf{C}^{\star}}^\top(\mathbf{C}^{\star}\bm{\Sigma}{\mathbf{C}^{\star}}^\top + \mathbf{R})^{-1}\delta(\rho(\hat{X}^{-1}, y)) ,\label{eq_deltaCstar}\\
    &\dot{\bm{\Sigma}} = \mathbf{A}_{t}^{0}\bm{\Sigma} + \bm{\Sigma}{\mathbf{A}_{t}^{\star}}^\top + \mathbf{B}_{t}\mathbf{Q}\mathbf{B}_{t}^{\top} - \mathbf{K}^{\star}\mathbf{C}^{\star}\bm{\Sigma} ,\label{eq_sigmadotCstar}\\
    &\mathbf{K}^{\star} = \bm{\Sigma}{\mathbf{C}^{\star}}^\top(\mathbf{C}^{\star}\bm{\Sigma}{\mathbf{C}^{\star}}^\top + \mathbf{D}_t\mathbf{R}\mathbf{D}_t^{\top})^{-1} .\label{eq_KCstar}
\end{align}

Furthermore, the geometry of the state space $\calM$ yields a modification of \cref{eq_sigmadotC0,eq_sigmadotCstar} to account for the distortion of the covariance due to parallel transport on a non-flat manifold~\cite{Mahony2022ObserverEquivariance}.
Therefore, \cref{eq_sigmadotC0,eq_sigmadotCstar} are augmented with a \emph{curvature correction} term:
\begin{align}
   &\dot{\bm{\Sigma}} = \mathbf{A}_{t}^{0}\bm{\Sigma} + \bm{\Sigma}{\mathbf{A}_{t}^{0}}^\top + \mathbf{B}_{t}\mathbf{Q}\mathbf{B}_{t}^{\top} - \mathbf{K}^{0}\mathbf{C}^{0}\bm{\Sigma}-(\bm{\Gamma}_{\td\phi_{\xizero}\Delta}\bm{\Sigma} + \bm{\Sigma}\bm{\Gamma}_{\td\phi_{\xizero}\Delta}^{\top}) ,\\
   &\dot{\bm{\Sigma}} = \mathbf{A}_{t}^{0}\bm{\Sigma} + \bm{\Sigma}{\mathbf{A}_{t}^{\star}}^\top + \mathbf{B}_{t}\mathbf{Q}\mathbf{B}_{t}^{\top} - \mathbf{K}^{\star}\mathbf{C}^{\star}\bm{\Sigma}-(\bm{\Gamma}_{\td\phi_{\xizero}\Delta}\bm{\Sigma} + \bm{\Sigma}\bm{\Gamma}_{\td\phi_{\xizero}\Delta}^{\top}),
\end{align}
where $\bm{\Gamma}$ is the connection function for the homogeneous space $\calM$~\cite{Mahony2022ObserverEquivariance}.

In the context of stochastic filters, the uncertainty of the error $\varepsilon$ in local coordinates is represented by its covariance matrix $\bm{\Sigma}$. The matrices $\mathbf{Q}$ and $\mathbf{R}$ represent the covariance of the zero-mean Gaussian noise terms in the process and measurement, respectively. Therefore, considering noisy input measurements ${u_m = u + \eta}$, the matrix $\mathbf{B}_t$ is written
\begin{equation}
\begin{split}
    \mathbf{B}_{t} &= \Fr{e}{\xizero}\vartheta\left(e\right) \Fr{E}{I}\phi_{\xi_0}\left(E\right) \Adsym{\hat{X}}{\Fr{u}{u_m}\Lambda\left(\hat{\xi},u\right)} ,\\
    &= \Fr{e}{\xizero}\vartheta\left(e\right) \Fr{E}{I}\phi_{\xi_0}\left(E\right) \Fr{u}{\mathring{u}_m}\Lambda\left(\xizero,u\right) \psi_{\hat{X}^{-1}} .\label{eq_B}
\end{split}
\end{equation}

Furthermore, considering noisy output measurements $z = h(\xi) + n$, the matrix ${\mathbf{D}_{t}}$ is written, for standard and equivariant output, respectively
\begin{align}
    &\mathbf{D}_{t} = \Fr{y}{\yzero}\delta\left(y\right) ,\label{eq_D_std}\\
    &\mathbf{D}_{t} = \Fr{y}{\yzero}\delta\left(y\right)\Fr{y}{z}\rho_{\hat{X}^{-1}}\left(y\right) .\label{eq_D_equi}
\end{align}

\begin{boxexample}{\ac{eqf} matrices for the direction kinematics}{}
The \ac{eqf} matrices for the direction kinematics are computed according to \cref{eq_A0,eq_C_star,eq_B,eq_D_equi}.

\paragraph{State matrix:}
Given the choice of normal coordinates, the rightmost differential of \cref{eq_A0}, is written ${\Fr{\varepsilon}{\mathbf{0}}\vartheta^{-1}\left(\varepsilon\right)\left[\varepsilon\right] = \Fr{E}{I}\phi_{\xizero}\left(E\right) \circ \td\exp \left[\varepsilon\right]}$. The differential of the exponential is trivial, and it is written ${\td\exp \left[\varepsilon\right] = \varepsilon^{\wedge}}$. The differential of the projection $\phi_{\xizero}$, instead, is computed as follows:
\begin{align*}
    \Fr{\varepsilon}{\mathbf{0}}\vartheta^{-1}\left(\varepsilon\right)\left[\varepsilon\right] &= \Fr{E}{I}\phi_{\xizero}\left(E\right) \circ \td\exp \left[\varepsilon\right]\\
    &= \lim_{t\to0}\left.\frac{1}{t}\left[(I - \alpha t)\mathring{\bm{d}} - \mathring{\bm{d}}\right]\right|_{\alpha = \varepsilon^{\wedge}}\\
    &= -\varepsilon^\wedge \mathring{\bm{d}} = \mathring{\bm{d}}^{\wedge}\varepsilon.
\end{align*}
Continuing, we note that the lift of the direction kinematics ${\Lambda\left(\xi, u\right) = \Vector{}{\omega}{}^{\wedge}}$ does not depend on the state. Hence the differential ${\Fr{e}{\xizero}\Lambda\left(e, \mathring{u}\right)\left[\mathring{\bm{d}}^{\wedge}\varepsilon\right] = \Vector{}{0}{3 \times 3}}$, and as a consequence, ${\mathbf{A}_t^{0} = \Vector{}{0}{3 \times 3}}$.

\paragraph{Output matrix:}
The equivariance of the output, together with the choice of normal coordinates, allows computing the matrix $\mathbf{C}^{\star}$ in \cref{eq_C_star} achieving third-order linearization error. The differential of the output action $\rho$ is computed as follows:
\begin{equation*}
    \Fr{E}{I}\rho_E\left(y\right)\left[\varepsilon^{\wedge}\right] = \lim_{t\to0}\left.\frac{1}{t}\left[(I - \alpha t)y - y\right]\right|_{\alpha = \varepsilon^{\wedge}} = -\varepsilon^{\wedge}y = y^{\wedge}\varepsilon.
\end{equation*}
The differential of the local coordinates of the output is trivial, and it is written $\Fr{y}{\yzero}\delta\left(y\right)\left[\alpha\right] = \yzero^\wedge \alpha$.
Concatenating these differentials as in \cref{eq_C_star} yields
\begin{align*}
    \mathbf{C}^{\star} &= \frac{1}{2}\yzero\left(\yzero^{\wedge} + \left(\rho_{\hat{X}^{-1}}\left(y\right)\right)^{\wedge}\right)\\
    &= \frac{1}{2}\mathring{\bm{d}}^{\wedge}\left(\mathring{\bm{d}}^{\wedge} + \left(\hat{A}y\right)^{\wedge}\right)
\end{align*}

\paragraph{Input noise matrix:}
The lift for the direction kinematics ${\Lambda\left(\xi, u\right) = \Vector{}{\omega}{}^{\wedge}}$ is linear in the input. Hence ${\Fr{u}{u_m}\Lambda\left(\hat{\xi},u\right)\left[\eta\right] = \eta}$. Moreover, given the choice of normal coordinates of the error, the two leftmost differentials of \cref{eq_B} cancel out, yielding
\begin{equation*}
    \mathbf{B}_t = \AdMsym{\hat{X}} = \hat{A}.
\end{equation*}

\paragraph{Output noise matrix:}
The action $\rho$ of the symmetry group on the output space is linear in the output. Therefore, solving \cref{eq_D_equi} is straightforward and the matrix $\mathbf{D}_t$ is written
\begin{equation*}
    \mathbf{D}_t = \mathring{\bm{d}}^{\wedge}\hat{A}.
\end{equation*}

\paragraph{Curvature correction:}
Finally, the Cartan-Schouten (0)-connection $\mathbf{\Gamma}$ is chosen~\cite{Mahony2022ObserverEquivariance}
\begin{equation*}
    \mathbf{\Gamma} = \frac{1}{2}\adMsym{\Delta}.
\end{equation*}
\end{boxexample}

Finally, when the input $u$ and the output measurements $y$ are provided at different rates, the previous equations can be split in the classical propagation-update form, and the filter can be implemented in discrete-time as described in \cref{eq_eqf}.

\begin{boxalgorithm}[label={eq_eqf}]{Equivariant filter}{}
    \SetKwProg{Initialization}{Initialization}{}{end}
    \SetKwProg{Propagation}{Propagation}{}{end}
    \SetKwProg{Update}{Update}{}{end}
    \SetKwInput{kwParam}{Parameters}
    \begin{algorithm}[H]
    \DontPrintSemicolon
    \kwParam{\;
    System state origin (system's initial condition): $\xizero$, \;
    Initial error uncertainty (initial covariance of the error):  $\bm{\Sigma}_0$ \;
    Process noise covariance: $\mathbf{Q}$ \; 
    Measurement noise covariance: $\mathbf{R}$ \;
    }
    \;
    \Initialization{}{
        Initialize filter state to identity: $\hat{X} = I$\;
        Initialize filter covariance to given matrix: $\bm{\Sigma} = \bm{\Sigma}_0$
    }
    \;
    Upon receiving an input measurement $u$\;
    \Propagation{}{
        Compute state matrix $\mathbf{A}_t^0$\;
        Compute input matrix: $\mathbf{B}_t$\;
        Compute state transition matrix: $\mathbf{\Phi} = \exp\left(\mathbf{A}_t^0 \Delta t\right)$\;
        Compute Lift: $\Lambda(\phi_{\xi_0}(\hat{X}), u)$\;
        Propagate filter state: $\hat{X} \gets \hat{X}\exp\left(\Lambda(\phi_{\xi_0}(\hat{X}), u)\Delta t\right)$\;
        Propagate filter covariance: $\bm{\Sigma} \gets \mathbf{\Phi}\bm{\Sigma}\mathbf{\Phi}^T + \mathbf{B}_{t}\mathbf{Q}{\mathbf{B}_{t}}^\top\Delta t$\;
    }
    \;
    Upon receiving an output measurement $y$\;
    \Update{}{
        \uIf{Output is equivariant}{
            Compute C matrix $\mathbf{C} = \mathbf{C}^{\star}$\;
            Compute residual: $\Vector{}{r}{} = \delta(\rho(\hat{X}^{-1}, y))$\;
        }
        \Else{
            Compute C matrix $\mathbf{C} = \mathbf{C}^{0}$\;
            Compute residual: $\Vector{}{r}{} = \delta(y - h(\phi(\hat{X}, \xizero)))$\;
        }
        Compute output matrix: $\mathbf{D}_t$\;
        Compute gain: $\mathbf{K} = \bm{\Sigma}{\mathbf{C}}^\top(\mathbf{C}\bm{\Sigma}{\mathbf{C}}^\top + \mathbf{D}_t\mathbf{R}\mathbf{D}_t^{\top})^{-1}$\;
        Compute innovation: $\Delta = \Fr{E}{I}\phi_{\xizero}\left(E\right)^{\dagger}\td\vartheta^{-1}\bm{\Sigma}{\mathbf{C}}^\top\mathbf{K}\Vector{}{r}{}$\;
        Update filter state: $\hat{X} \gets \exp\left(\Delta\right)\hat{X}$\;
        Update filter covariance: $\bm{\Sigma} \gets \left(\eye - \mathbf{K}\mathbf{C}\right)\bm{\Sigma}$\;
        Compute curvature correction: $\bm{\Gamma}$\;
        Apply curvature correction: $\bm{\Sigma} \gets \exp(-\bm{\Gamma})\bm{\Sigma}\exp(-\bm{\Gamma})^\top$
    }   
    \end{algorithm}
\end{boxalgorithm}

\subsection{Applicability of the \ac{eqf} and its relation with the \ac{iekf}}

It is clear that the \acl{eqf} is an algorithm built on the principles of respecting the geometry of a system and exploiting its symmetry through the equivariant error definition. However, the \acl{iekf}, introduced by Barrau and Bonnabel~\cite{7523335}, also builds on similar principles. Natural questions that one might ask at this point are: \emph{``What is the relation between the \ac{eqf} and the \ac{iekf}, and what are their differences?''}

the former question is answered in~\cite[appendix B]{vanGoor2022EquivariantEqF}. In particular, \emph{the \ac{eqf} specializes to the \ac{iekf} for systems with group affine dynamics~\cite{7523335} on a Lie group, when the state origin $\xizero$ is chosen to be the identity, and when the coordinate chart of the state space $\vartheta$ is chosen to be the exponential map of the Lie group.} Therefore, the \ac{eqf} design methodology applies to any system posed on a homogeneous space and specializes to the \ac{iekf} for group-affine systems, as depicted in \cref{eq_systems}.

In \cref{bins_chp}, we will extend this result and show that the \ac{eqf} not only specialize to the \ac{iekf}; instead, every modern variant of \ac{ekf} can be interpreted as \aclp{eqf} applied to distinct choices of symmetries.

\begin{figure}[htp]
\centering
\includegraphics[width=\linewidth]{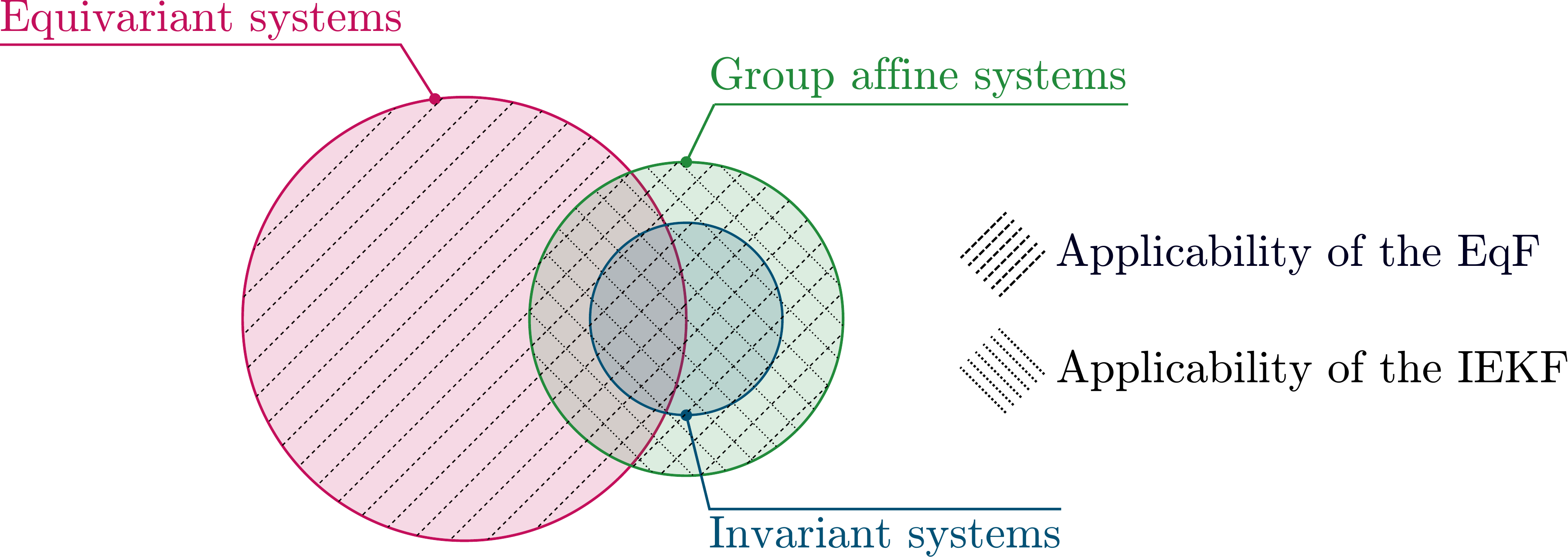}
\caption[Applicability of the \ac{eqf} design methodology.]{The \ac{eqf} design methodology applies to any system posed on a homogeneous space. That is, it applies to both invariant and equivariant systems. In particular, the \ac{eqf} specializes to the well-known \ac{iekf} when the \ac{eqf} design methodology is applied to group-affine and invariant systems.}
\label{eq_systems}
\end{figure}

It is clear that for group-affine systems, there are no differences between the \ac{eqf} and the \ac{iekf}. However, in general, the design methodology and the approach to the estimation problem differ. 

In the \ac{iekf} design methodology, we seek to model the state space of a kinematic system directly on a Lie group and then exploit the invariant error definition to design a filter. This is possible only for systems that possess a natural symmetry $(\grpG, \mathrm{R})$ given by the right translation $\mathrm{R}$ of a Lie group $\grpG$ on the state space of the system $\calG$, given by the $\grpG$-Torsor.

In the \ac{eqf} design methodology, the whole estimation problem is tackled from a different perspective; in particular, we do not seek to model the state space of a kinematic system directly on a Lie group, but we rather seek to find an equivariant symmetry for the system. We believe this change of perspective is crucial and overcomes the limitation of the \ac{iekf} design methodology. First of all, a kinematic system might possess symmetries even when modeling its state space directly on a Lie group is not possible. Furthermore, the same system could possess multiple symmetries with invariance or equivariance properties, which can be a better choice than the natural symmetry for state estimation problems. Specifically, this is the case of the tangent symmetry for biased \aclp{ins} discussed in the following chapters, which represents a major result of this dissertation.

\subsection{Why is the \acl{eqf} good?}
Interesting questions that might arise are \emph{``Why is the \ac{eqf} good?''} or \emph{``What makes the \ac{eqf} good?''} 
Although it should be clear that the \ac{eqf} does not represent a direct solution to an estimation problem but rather a symmetry-based design methodology that respects the geometry of a system; it might still not be fully clear where precisely the advantage of the \ac{eqf} comes from.

The upcoming chapters will clarify that \emph{it is the choice of symmetry the \ac{eqf} is built upon that dictates its performance}. Specifically, an \ac{eqf} based on a symmetry that provides invariance or equivariance allows the definition of a global error with lower linearization error than an \ac{eqf} based on a symmetry that does not provide invariance or equivariance. \emph{The improved linearization of the error dynamics and the output measurement is key in understanding what makes the \ac{eqf} good}.
\chapter[Tangent Group Symmetry for Attitude Systems][Tangent Symmetry for Attitude Systems]{Tangent Group Symmetry for Attitude Systems}\label{bas_chp}
\emph{The present chapter contains results that have been peer-reviewed and published in the IEEE Robotics and Automation Letters~\cite{Fornasier2022OvercomingCalibration}.}
\bigskip

\noindent This chapter introduces the first and easiest example of biased \acl{ins}, that is, the attitude kinematics of a rigid body freely moving in space, equipped with a biased gyroscope and receiving measurements of known directions. The goal of the estimation problem is to simultaneously estimate the attitude of the rigid body and the bias of the gyroscope. 

The results discussed in this chapter are particularly interesting; first of all, we introduce a new symmetry of the attitude kinematics based on the \emph{tangent group of $\SO(3)$}. We tear down the thinking process that has led to the definition of this symmetry group, and we show how the tangent group of $\SO(3)$ is exploited to design an \acl{eqf} for the problem at hand. Moreover, in \cref{bins_chp}, it will be clear that the idea behind the proposed symmetry group generalizes to second order \aclp{ins} and will be the key for modern geometric \ac{ins} filter design. Second, we show the derivation of the proposed \ac{eqf} in detail, and we demonstrate the performance of the proposed filter through a series of experiments, showing that the proposed \ac{eqf} is ideal for low-level attitude estimation of computationally constrained platforms. It is easily implementable, computationally cheap, and superior to the state-of-the-art. 

\begin{figure}[htp]
\centering
\includegraphics[width=\linewidth]{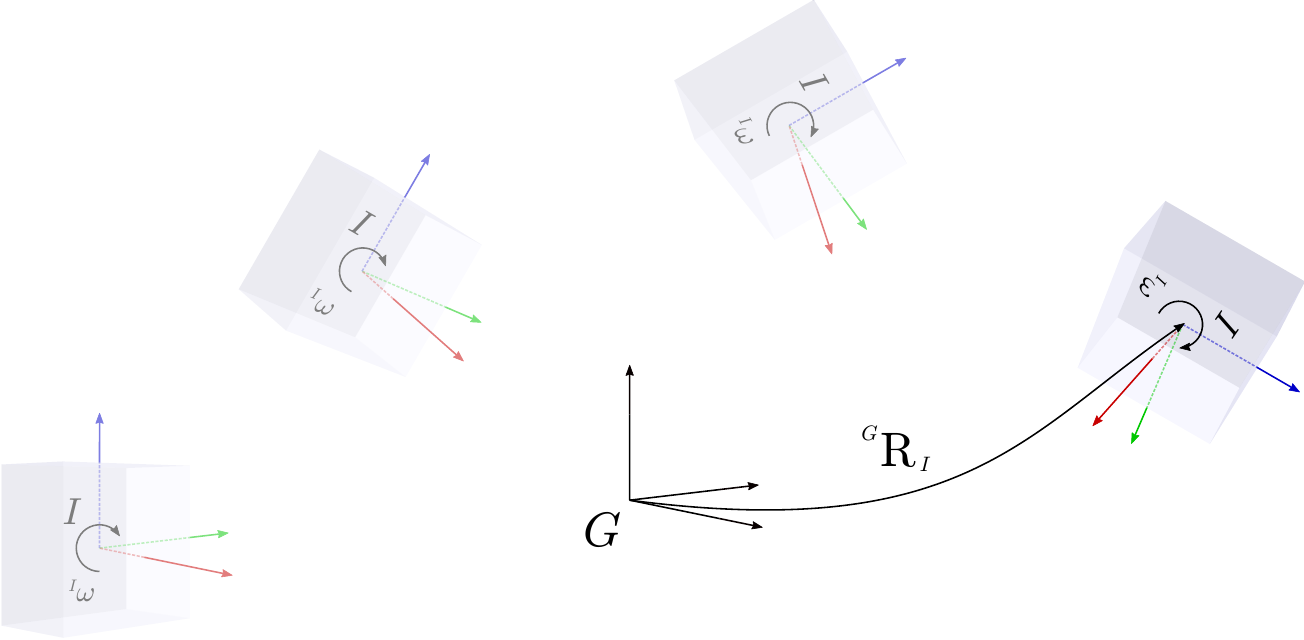}
\caption[Graphical representation of attitude system.]{Graphical representation of the time evolution of a rigid body freely rotating in space, representing the system of interest.}
\label{bas_system}
\end{figure}

\section{The biased attitude system}

Let \frameofref{G} denote the global inertial frame of reference, \frameofref{I} denote the gyroscope frame of reference, and \frameofref{S_i} denote the frame of reference of the $i\,$\ts{th} sensor providing direction measurements. 
Here, we focus on the problem of estimating the rigid body orientation ${\Rot{G}{I}}$ of a moving rigid platform, as well as the gyroscope bias ${\Vector{I}{b_{\bm{\omega}}}{}}$, and the extrinsic calibration ${\Rot{I}{S_i}}$ of the $i\,$\ts{th} direction sensors.
In non-rotating, flat earth assumption, the deterministic (noise-free) system takes the following general form
\begin{subequations}\label{bas_orig_bas}
    \begin{align}
        &\dotRot{G}{I} = \Rot{G}{I}\left(\Vector{I}{\bm{\bm{\omega}}}{} - \Vector{I}{b_{\bm{\omega}}}{}\right)^{\wedge} ,\\
        &\dotVector{I}{b_{\bm{\omega}}}{} = \mathbf{0} ,\\
        &\dotRot{I}{S_i} = \mathbf{0}^{\wedge} \quad \forall\;i=1,\dots,n.
    \end{align}
\end{subequations}
where $\Vector{I}{\bm{\omega}}{}$ is the body-fixed, biased angular velocity measurements provided by the gyroscope. For the sake of generality, note that $n \leq N$, where $N$ is the total number of direction sensors, since there could exist sensors that are already calibrated.

Let ${\gamma = \left(\Rot{G}{I},\, \Vector{I}{b}{\bm{\omega}}\right) \in \torSO\left(3\right) \times \R^{3}}$ denote the system core state, and define ${\mu = \left(\Vector{I}{\bm{\omega}}{},\, \mathbf{0}\right) \subseteq \R^6}$. Let ${\zeta = \left(\Rot{I}{S_1},\, \dots,\, \Rot{I}{S_n}\right) \in \torSO\left(3\right)^n}$ denote the sensors extrinsic calibration states, and define ${\upsilon = \left(\mathbf{0},\, \dots,\, \mathbf{0}\right) \subseteq \R^{3n}}$. Then, the full state of the system writes $\xi = {\left(\gamma,\,\zeta\right) \in \calM \coloneqq \torSO\left(3\right) \times \R^{3} \times \torSO\left(3\right)^n}$, while the system input writes ${u = \left(\mu,\,\upsilon\right) \in \mathbb{L} \subseteq \R^{6+3n}}$.
The full system kinematics in \cref{bas_orig_bas} can be written as $f_u(\xi)$, and in compact affine form as follows:
\begin{equation}\label{bas_bas}
\begin{split}
    \dot{\xi} = f_(\xi) &=\left(\Rot{G}{I}\left(\Vector{I}{\bm{\omega}}{} - \Vector{I}{b}{\bm{\omega}}\right)^{\wedge},\,\mathbf{0}^{\wedge},\,\dots,\,\mathbf{0}^{\wedge}\right) \\
    &= \left(-\Rot{G}{I}\Vector{I}{b}{\bm{\omega}}^{\wedge},\,\mathbf{0}^{\wedge},\,\dots,\,\mathbf{0}^{\wedge}\right) + \left(\Rot{G}{I}\Vector{I}{\bm{\omega}}{}^{\wedge},\,\mathbf{0}^{\wedge},\,\dots,\,\mathbf{0}^{\wedge}\right) ,
\end{split}
\end{equation}

For the estimation problem we consider the measurements of $N$ known directions ${\Vector{G}{d}{1},\, \dots,\, \Vector{G}{d}{N}}$, to be available to the system. The output space is then defined to be $\calN \coloneqq \mathbb{S}^{2N}$, and therefore, the configuration output ${h \AtoB{\calM}{\calN}}$ is written 
\begin{equation}\label{bas_confout}
    \begin{split}
        h\left(\xi\right) = &\left(\Rot{I}{S_1}^{\top}\Rot{G}{I}^{\top}\Vector{G}{d}{1},\,\dots\,\Rot{I}{S_n}^{\top}\Rot{G}{I}^{\top}\Vector{G}{d}{n},\,\dots\,,\right.\\
        &\left.\Rot{G}{I}^{\top}\Vector{G}{d}{n+1},\,\dots\,,\Rot{G}{I}^{\top}\Vector{G}{d}{N}\right) \in \calN .
    \end{split}
\end{equation}

\section{The tangent group $\SO(3) \ltimes \so(3)$ as symmetry group}

For the sake of clarity, in the two following sections, all the superscripts and subscripts associated with reference frames from the state variables and input variables are omitted to improve the readability of the definitions and theorems. 
Moreover, without loss of generality, our analysis is restricted to the case of $N=2$ direction sensors, one with a related calibration state and one assumed calibrated, thus $n=1$. This covers the general case for the filter derivation since the following theorems hold for every $n,\,N>0$.  Any other case can be easily derived, extending the result presented in the following sections.

Note that a nonlinear observability analysis~\cite{StephanM.Weiss2012VisionHelicopters} confirms that $N = 2$ non-parallel measurements is a sufficient condition to render the system fully observable even in the presence of no motion. 

Therefore, we consider the full system state ${\xi = \left(\left(\Rot{}{},\, \Vector{}{b}{}\right),\, \mathbf{C}\right) \in \calM}$, where $\Rot{}{}$ represents the attitude, ${\Vector{}{b}{}}$ the gyroscope bias, and ${\mathbf{C}}$ the calibration state. 
We also consider the output map ${h\left(\xi\right) = \left(\mathbf{C}^{\top}\Rot{}{}^{\top}\Vector{}{d}{1},\,\Rot{}{}^{\top}\Vector{}{d}{2}\right)} \in \calN$.

Before diving directly into the discussion of the proposed symmetry and the derivation of the \acl{eqf} algorithm, let us quickly recall the questions discussed in \cref{eq_chp} and tear down the thinking process to define new symmetries.
First of all, we already know one symmetry for the system of interest, which is the natural symmetry. $\grpG \coloneqq \SO(3) \times \R^3 \times \SO(3)$ is the most intuitive choice, and it is the symmetry that leads to the design of an Imperfect-\ac{iekf}~\cite{7523335, barrau:tel-01247723}. However, this symmetry is not necessarily the best choice. The action of the group on the state space is given by the right translation of the group, and although the natural symmetry leads to equivariance of the system, it does not lead to equivariance of the output. \emph{``Can we derive a symmetry that leads to equivariance of the output?''} Yes, in order to get output equivariance, we would need to choose a different action than the right translation of the group. Let $X = \left(\left(A,\,a\right),\,B\right)$ be an element of the symmetry group $\grpG \coloneqq \SO(3) \times \R^3 \times \SO(3)$, choosing
\begin{equation}\label{bas_phi_direct}
    \phi\left(X, \xi\right) \coloneqq \left(\Rot{}{}A,\, \Vector{}{b}{} - a,\, A^{\top}\mathbf{C}B\right) \in \calM ,
\end{equation}
the symmetry given by $(\grpG, \phi)$ is an equivariant symmetry for the system with equivariance of the output. At this point, we could ask \emph{``Can we find an even better symmetry?''} The answer is, again, yes. Even though there is no rigorous way to prove that one symmetry is better than another one, we can always try to argue based on intuition and verify the intuition via a linearization error analysis or experimental results. 

Consider the representation in \cref{bas_system_symmetries}. In black, the core state variables are represented; in orange, the elements of the symmetry group; and in blue, the new elements of the state space after the action of the symmetry group $\phi$ is applied. On the left side of \cref{bas_system_symmetries}, the symmetry $(\grpG, \phi)$ with $\phi$ in \cref{bas_phi_direct} is represented. Notice that the new bias element $\Vector{}{b}{} - a$ is expressed in the basis given by the gyroscope frame of reference \frameofref{I} before the action $\phi$ is applied. Although this is totally fine, it triggers the question: \emph{``Can we find a symmetry whose action leads to a change of basis of the gyroscope bias to the basis given by the new gyroscope frame of reference after the action $\phi$ is applied?''} as depicted on the right side of \cref{bas_system_symmetries}. The answer is again another yes. In general, the adjoint map $\mathrm{Ad}_X$ acts as a change of basis of the Lie algebra $\gothg$. These considerations have led to the definition of the tangent symmetry for inertial navigation problems.

\begin{figure}[htp]
\centering
\includegraphics[width=\linewidth]{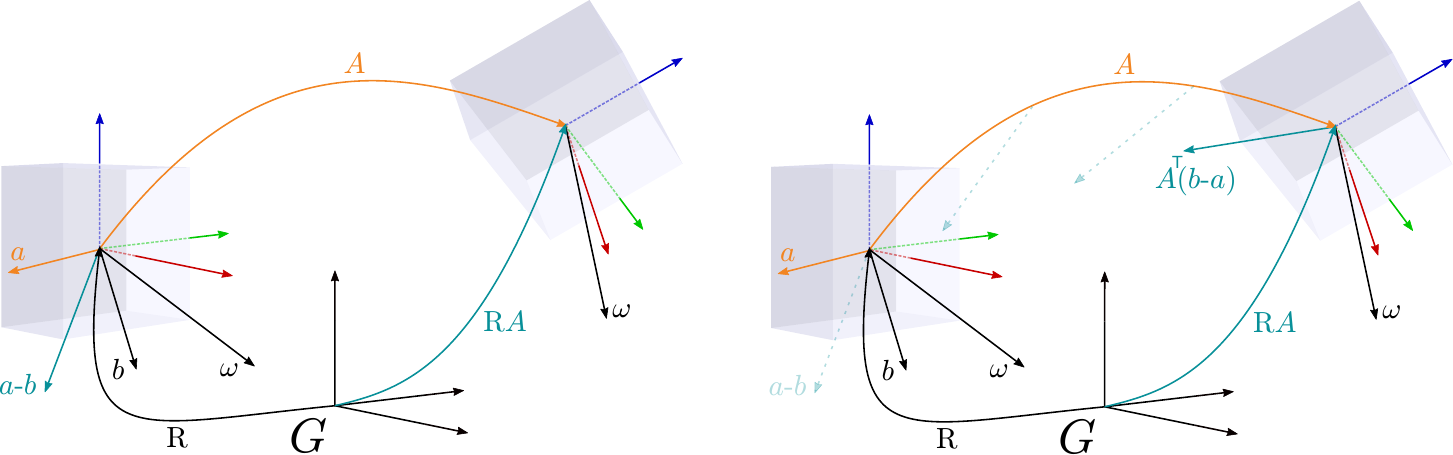}
\caption[Representation of the symmetry of the biased attitude system.]{Representation of the symmetry of the biased attitude system. In black, the core state variables are represented; in orange, the elements of the symmetry group; and in blue, the new elements of the state space after the action of the symmetry group $\phi$ is applied. The left side represents the action in \cref{bas_phi_direct}. The right side represents the action in \cref{bas_phi}.}
\label{bas_system_symmetries}
\end{figure}

\subsection{Equivariance of the system}
Let $X = \left(\left(A,\,a\right),\,B\right)$ be an element of the symmetry group ${\grpG \coloneqq \left(\SO\left(3\right) \ltimes \gothso\left(3\right)\right) \times \SO\left(3\right)}$. Group product and inverse follow those introduced in \cref{math_sdp_sec}, and in particular, the inverse element is written ${X^{-1} = \left(\left(A^{\top},\,-A^{\top}aA\right),\,B^{\top}\right)}$. Let ${X = \left(\left(A_X,\,a_X\right),\,B_X\right)}$ and ${Y = \left(\left(A_Y,\,a_Y\right),\,B_Y\right)}$ the product rule is defined by the semi-direct product in \cref{math_sdp_rule}, as ${XY = \left(\left(A_X A_Y,\,a_X + A_Xa_YA_X^{\top}\right),\,B_X B_Y\right)}$, where $\Adsym{A_X}{a_Y} = A_Xa_YA_X^{\top}$.

\begin{lemma}
Define ${\phi \AtoB{\grpG \times \calM}{\calM}}$ as
\begin{equation}\label{bas_phi}
    \phi\left(X, \xi\right) \coloneqq \left(\Rot{}{}A,\, A^{\top}\left(\Vector{}{b}{} - a^{\vee}\right),\, A^{\top}\mathbf{C}B\right) \in \calM .
\end{equation}
Then, $\phi$ is a transitive right group action of $\grpG$ on $\calM$.
\end{lemma}

\begin{proof}
Let ${X, Y \in \grpG}$ and $\xi \in \calM$. Then, 
\begin{align*}
    \phi\left(X,\phi\left(Y, \xi\right)\right) &= \phi\left(X,\, \left(\Rot{}{}A_Y,\, A_Y^{\top}\left(\Vector{}{b}{} - a_Y^{\vee}\right),\, A_Y^{\top}\mathbf{C}B_Y\right)\right)\\
    &= \left(\Rot{}{}A_YA_X,\, A_X^{\top}\left(A_Y^{\top}\left(\Vector{}{b}{} - a_Y^{\vee}\right) - a_X^{\vee}\right),\,A_X^{\top}A_Y^{\top}\mathbf{C}B_YB_X\right)\\
    &= \left(\Rot{}{}\left(A_YA_X\right),\, \left(A_YA_X\right)^{\top}\left(\Vector{}{b}{} - \left(a_Y^{\vee} + A_Ya_X^{\vee}\right)\right),\right.\\
    &\quad\left.\left(A_YA_X\right)^{\top}\mathbf{C}\left(B_YB_X\right)\right)\\
    &= \phi\left(YX,\, \xi\right) ,
\end{align*}
where we have used the property of the Adjoint map and the Adjoint matrix ${(Aa^{\vee})^{\wedge} = AaA^{\top}}$. This shows that $\phi$ is a valid right group action. Then, ${\forall \; \xi_1, \xi_2 \in \calM}$ we can always write the group element ${Z = \left(\left(\Rot{}{1}^{\top}\Rot{}{2},\, \Vector{}{b}{1} - \Rot{}{1}^{\top}\Rot{}{2}\Vector{}{b}{2}\right),\,\mathbf{C}_1^{\top}\Rot{}{1}^{\top}\Rot{}{2}\mathbf{C}_2\right)}$, such that
\begin{align*}
    \phi\left(Z, \xi_1\right) &= \left(\left(\Rot{}{1}\Rot{}{1}^{\top}\Rot{}{2},\, \left(\Rot{}{1}^{\top}\Rot{}{2}\right)^{\top}\left(\Vector{}{b}{1} - \Vector{}{b}{1} + \Rot{}{1}^{\top}\Rot{}{2}\Vector{}{b}{2}\right)\right),\right.\\
    &\quad\left.\left(\Rot{}{1}^{\top}\Rot{}{2}\right)^{\top}\mathbf{C}_1\mathbf{C}_1^{\top}\Rot{}{1}^{\top}\Rot{}{2}\mathbf{C}_2\right)\\
    &= \left(\left(\Rot{}{2},\, \Vector{}{b}{2}\right),\,\mathbf{C}_2\right)\\
    &= \xi_2 ,
\end{align*}
which demonstrates the transitive property of the group action.
\end{proof}

\begin{remark}
Note that in case of multiple calibration states, $(n>1)$, we would extend the symmetry group with multiple copies of $\SO(3)$, that is ${X = \left(\left(A,\,a,\right)\,B_1,\,B_2\,\dots\right)}$. As a consequence, the action $\phi$ is extended with ${\phi\left(X, \xi\right) \coloneqq \left(\left(\Rot{}{}A,\, A^{\top}\left(\Vector{}{b}{} - a^{\vee}\right)\right),\, A^{\top}\mathbf{C}_1 B_1,\, A^{\top}\mathbf{C}_2 B_2,\,\dots\right)}$ 
\end{remark}

\begin{lemma}
Define ${\psi \AtoB{\grpG \times \vecL}{\vecL}}$ as
\begin{equation}\label{bas_psi}
    \begin{split}
        \psi\left(X,u\right) &\coloneqq \left(A^{\top}\left(\bm{\omega} - a^{\vee}\right),\, \mathbf{0},\, \mathbf{0}\right) ,
    \end{split}
\end{equation}
Then, $\psi$ is a right group action of $\grpG$ on $\vecL$.
\end{lemma}

\begin{proof}
Let ${X, Y \in \grpG}$ and $u \in \vecL$. Then, 
\begin{align*}
    \psi\left(X,\psi\left(Y, u\right)\right) &= \psi\left(X,\left(A_Y^{\top}\left(\bm{\omega} - a_Y^{\vee}\right),\, \mathbf{0},\, \mathbf{0}\right)\right)\\
    &= \left(A_X^{\top}\left(A_Y^{\top}\left(\bm{\omega} - a_Y^{\vee}\right) - a_X^{\vee}\right),\, \mathbf{0},\, \mathbf{0}\right)\\
    &= \left(\left(A_YA_X\right)^{\top}\left(\bm{\omega} - \left(a_Y^{\vee} + A_Ya_X^{\vee}\right)\right),\, \mathbf{0},\, \mathbf{0}\right)\\
    &= \psi\left(YX, u\right) .
\end{align*}
Thus proving that $\psi$ is a valid right group action.
\end{proof}

\begin{remark}
In case of multiple calibration states $(n>1)$, we would simply repeat the last zero entry $n$ time in the action $\psi$. That is ${\psi\left(X,u\right) \coloneqq \left(A^{\top}\left(\bm{\omega} - a^{\vee}\right),\, \mathbf{0},\, \mathbf{0},\, \mathbf{0},\, \dots\right)}$
\end{remark}

\begin{theorem}
The biased attitude system in \cref{bas_bas} is equivariant under the actions $\phi$ in \cref{bas_phi} and $\psi$ in \cref{bas_psi} of the symmetry group $\grpG$. That is it satisfy \cref{eq_equi}:
\begin{equation*}
    f_{\psi_{X}\left(u\right)}\left(\xi\right) = \Phi_{X}f_{u}\left(\xi\right) .
\end{equation*}
\end{theorem}
\begin{proof}
Let ${X \in \grpG}$, ${\xi \in \calM}$ and ${u \in \vecL}$, then the inverse of the group action defined in \cref{bas_phi} is written
\begin{equation*}
    \phi\left(X^{-1},\,\xi\right) \coloneqq \left(\left(\Rot{}{}A^{\top},\, A\Vector{}{b}{} + a^{\vee}\right),\, A\mathbf{C}B^{\top}\right) \in \calM .
\end{equation*}
Therefore, computing the induced action defined in \cref{math_induced_action} yields
\begin{align*}
    \Phi_{X}f_{u}\left(\xi\right) &= \left(\left(\Rot{}{}\left(A^{\top}\left(\bm{\omega} - a^{\vee}\right)^{\wedge} - \Vector{}{b}{}^{\wedge}A^{\top}\right)A,\, \mathbf{0}\right),\,\mathbf{0}^{\wedge}\right) \\
    &= \left(\left(\Rot{}{}\left(\left(A^{\top}\left(\bm{\omega} - a^{\vee}\right)\right)^{\wedge} - \Vector{}{b}{}^{\wedge}\right),\, \mathbf{0}\right),\,\mathbf{0}^{\wedge}\right) \\
    &= f_{\psi_{X}\left(u\right)}\left(\xi\right) ,
\end{align*}
proving the equivariance of the system.
\end{proof}

\subsection{Equivariance of the output}
\begin{lemma}
Define ${\rho \AtoB{\grpG \times \calN}{\calN}}$ as
\begin{equation}\label{bas_rho}
    \rho\left(X,y\right) \coloneqq \left(B^{\top}y_{1},\,A^{\top}y_{2}\right) .
\end{equation}
Then, the configuration output defined in \cref{bas_confout}, for $N=2$ and $n=1$ is equivariant.
\end{lemma}

\begin{proof}
    Let ${X \in \grpG}$ and ${h\left(\xi\right) = \left(\mathbf{C}^{\top}\Rot{}{}^{\top}\Vector{}{d}{1},\,\Rot{}{}^{\top}\Vector{}{d}{2}\right) \in \calN}$ then,
    \begin{equation}
        \rho\left(X, h\left(\xi\right)\right) = \left(B^{\top}\mathbf{C}^{\top}\Rot{}{}^{\top}\Vector{}{d}{1},\,A^{\top}\Rot{}{}^{\top}\Vector{}{d}{2}\right) = h\left(\phi\left(X,\xi\right)\right) .
    \end{equation}
    This proves the output equivariance. 
\end{proof}

\begin{remark}
Here is where we have the major changes when explicitly considering the extrinsic calibration of a sensor. 
In particular, note that for any sensor that does not have an associated calibration states, the ${\rho}$ action is defined by ${A^{\top}y}$, where ${y}$ is the sensor measurement, whereas for any sensor that does have an associated calibration states, the ${\rho}$ action is the sensor measurement ${y}$ pre-multiplied by the element of the symmetry group relative to the sensor extrinsic calibration. 
More concretely, in the case of ${N=n=2}$ the ${\rho}$ action is ${\rho\left(X,y\right) \coloneqq \left(B_1^{\top}y_{1},\,B_2^{\top}y_{2}\right)}$, in the case of ${N=2, n=0}$ the ${\rho}$ action is ${\rho\left(X,y\right) \coloneqq \left(A^{\top}y_{1},\,A^{\top}y_{2}\right)}$, in the mixed case the ${\rho}$ action is as shown in \cref{bas_rho}.
\end{remark}

\subsection{Equivariant Lift and Lifted System}
\begin{theorem}
Define ${\Lambda \AtoB{\calM \times \vecL} \gothg}$ as
\begin{equation}\label{bas_lift}
    \Lambda\left(\xi, u\right) \coloneqq \left(\left(\left(\Vector{}{\bm{\omega}}{} - \Vector{}{b}{}\right)^{\wedge},\, - \left(\Vector{}{\omega}{}^{\wedge}\Vector{}{b}{}\right)^{\wedge}\right),\, \mathbf{C}^{\top}\left(\Vector{}{\bm{\omega}}{} - \Vector{}{b}{}\right)^{\wedge}\mathbf{C}\right) .
\end{equation}
Then, the map ${\Lambda\left(\xi, u\right)}$ is an equivariant lift for the system in \cref{bas_bas} with respect to the defined symmetry group.
\end{theorem}
\begin{proof}
Let ${\xi \in \calM}, u \in \vecL$, then solving \cref{eq_lift} as in \cref{eq_lift_example}
\begin{align*}
    \td\phi_{\xi}\left[\Lambda\left(\xi, u\right)\right] &= \left(\left(\Rot{}{}\left(\Vector{}{\bm{\omega}}{} - \Vector{}{b}{}\right)^{\wedge},\, - \left(\Vector{}{\bm{\omega}}{} - \Vector{}{b}{}\right)^{\wedge}\Vector{}{b}{} - \left(- \Vector{}{\omega}{}^{\wedge}\Vector{}{b}{}\right)\right),\right.\,\\
    &\quad\left. \mathbf{C}\mathbf{C}^{\top}\left(\Vector{}{\bm{\omega}}{} - \Vector{}{b}{}\right)^{\wedge}\mathbf{C} - \left(\Vector{}{\bm{\omega}}{} - \Vector{}{b}{}\right)^{\wedge}\mathbf{C}\right)\\
    &= \left(\left(\Rot{}{}\left(\Vector{}{\bm{\omega}}{} - \Vector{}{b}{}\right)^{\wedge},\, \mathbf{0}\right),\, \mathbf{0}^{\wedge}\right)\\
    &= f_0\left(\xi\right) + f_{u}\left(\xi\right) .
\end{align*}
To demonstrate the equivariance of the lift, we proceed by showing that the condition in \cref{eq_lift_equi} holds. Let ${\xi \in \calM}, u \in \vecL$ and ${X \in \grpG}$, then
\begin{align*}
    &\Adsym{X}{\Lambda\left(\phi_{X}\left(\xi\right),\psi_{X}\left(u\right)\right)} = X\Lambda\left(\phi_{X}\left(\xi\right),\psi_{X}\left(u\right)\right)X^{-1}\\
    &= \left(\left(\left(AA^{\top}\left(\Vector{}{\bm{\omega}}{} - \Vector{}{b}{}\right)\right)^{\wedge},\, -\left(A\left(A^{\top}\left(\Vector{}{\bm{\omega}}{} - a\right)\right)^{\wedge}A^{\top}\left(\Vector{}{b}{} - a\right) + \left(AA^{\top}\left(\Vector{}{\bm{\omega}}{} - \Vector{}{b}{}\right)\right)^{\wedge}a\right)^{\wedge}\right), \right.\\
    &\quad\left. B\left(B^{\top}\mathbf{C}^{\top}\left(\Vector{}{\bm{\omega}}{} - \Vector{}{b}{}\right)^{\wedge}\mathbf{C}B\right)B^{\top}\right) ,\\
    &=\left(\left(\left(\Vector{}{\bm{\omega}}{} - \Vector{}{b}{}\right)^{\wedge},\,-\left(\left(\Vector{}{\bm{\omega}}{} - a\right)^{\wedge}\left(\Vector{}{b}{} - a\right)+\left(\Vector{}{\bm{\omega}}{} - \Vector{}{b}{}\right)^{\wedge}a\right)^{\wedge}\right),\,\mathbf{C}^{\top}\left(\Vector{}{\bm{\omega}}{} - \Vector{}{b}{}\right)^{\wedge}\mathbf{C}\right) ,\\
    &=\left(\left(\left(\Vector{}{\bm{\omega}}{} - \Vector{}{b}{}\right)^{\wedge},\,-\left(\Vector{}{\bm{\omega}}{}^{\wedge}\Vector{}{b}{}\right)^{\wedge}\right),\,\mathbf{C}^{\top}\left(\Vector{}{\bm{\omega}}{} - \Vector{}{b}{}\right)^{\wedge}\mathbf{C}\right)\\
    &= \Lambda\left(\xi, u\right) ,
\end{align*}
proving the equivariance of the lift.
\end{proof}

The lift $\Lambda$ in \cref{bas_lift} associates the system input with the Lie algebra of the symmetry group and allows the construction of a lifted system on the symmetry group as described in \cref{eq_sym_sec}. Let $X \in \grpG$ be the state of the lifted system, and let $\xizero \in \calM$ be the selected origin, then the lifted system is written as
\begin{equation}
    \dot{X} = \td L_{X}\Lambda\left(\phi_{\xizero}\left(X\right), u\right) = X\Lambda\left(\phi_{\xizero}\left(X\right), u\right).
\end{equation}

\section{Equivariant filter design}
Let ${\Lambda}$ be the equivariant lift defined in \cref{bas_lift}, and ${\hat{X} \in \grpG}$ be the \acl{eqf} state, with initial condition ${\hat{X}\left(0\right) = I}$, the identity element of the symmetry group $\grpG$. Choose an arbitrary state origin  $\xizero$, note that we would generally derive the \ac{eqf} matrices with generic $\xizero$ and then set $\xizero$ with given initial conditions, however, for the sake of simplicity, in this chapter we derive the \ac{eqf} matrices with identity $\xizero$. 
Furthermore, as discussed in \cref{eq_chp}, the equivariant error evolves on the homogenous space $\calM$. Hence, a set of local coordinates on the homogeneous space needs to be defined. The choice of local coordinates is free; however, a natural choice is represented by exponential coordinates. Define
\begin{equation}\label{bas_local_coords}
    \varepsilon = \vartheta\left(e\right) = \vartheta\left(e_{R},\,e_{b},\, e_{C}\right) = \left(\log\left(e_{R}\right)^{\vee},\, e_{b},\, \log\left(e_{C}\right)^{\vee}\right) \in \R^{9} ,
\end{equation}
with ${\vartheta\left(\xizero\right) = \mathbf{0} \in \R^{9}}$. Note that since the output is equivariant, we could have chosen normal coordinates for improved linearization of the output map; however, for the sake of simplicity, we opted for exponential coordinates. The derivation for second-order \aclp{ins} with arbitrary $\xizero$ and normal coordinates is shown in \cref{appendix_A_chp}.

Let us now derive the \ac{eqf} matrices, starting with the state matrix $\mathbf{A}_{t}^{0}$ given by the solution of \cref{eq_A0}. The first step is to compute the differential of the coordinate chart. Given the choice of exponential coordinates, the solution is trivial, ${\Fr{\varepsilon}{\mathbf{0}}\vartheta^{-1}\left(\varepsilon\right)\left[\varepsilon\right] = \varepsilon^{\wedge}.}$
The second step involves deriving the differential of the lift in the direction ${\varepsilon^{\wedge}}$, that is ${\Fr{e}{\xizero}\Lambda\left(e, \mathring{u}\right)\Fr{\varepsilon}{\mathbf{0}}\vartheta^{-1}\left(\varepsilon\right)\left[\varepsilon\right] = \Fr{e}{\xizero}\Lambda\left(e, \mathring{u}\right)\left[\varepsilon^{\wedge}\right]}$. Let ${\mathring{u} \coloneqq \psi\left(\hat{X}^{-1}, u\right)  = \left(\mathring{\Vector{}{\omega}{}},\,\mathbf{0},\,\mathbf{0},\,\mathbf{0}\right)}$ be the origin input, the derivation of the differential of the lift can be broken down into two distinct derivations:
\begin{align*}
     \left(\mathring{\Vector{}{\omega}{}}^{\wedge} - \varepsilon_{b}^{\wedge},\, - \left(\mathring{\Vector{}{\omega}{}}^{\wedge}\varepsilon_{b}\right)^{\wedge}\right) - \left(\mathring{\Vector{}{\omega}{}}^{\wedge},\, \Vector{}{0}{}\right) &= \left(- \varepsilon_{b}^{\wedge},\,  - \left(\mathring{\Vector{}{\omega}{}}^{\wedge}\varepsilon_{b}\right)^{\wedge}\right) ,\\
     \left(I - \varepsilon_C^{\wedge}\right)\left(\mathring{\Vector{}{\omega}{}} - \varepsilon_{b}\right)^{\wedge}\left(I + \varepsilon_C^{\wedge}\right) - \mathring{\Vector{}{\omega}{}}^{\wedge} &= - \varepsilon_{b}^{\wedge} - \varepsilon_C^{\wedge}\mathring{\Vector{}{\omega}{}}^{\wedge} + \mathring{\Vector{}{\omega}{}}^{\wedge}\varepsilon_C^{\wedge} = \adsym{\mathring{\Vector{}{\omega}{}}^{\wedge}}{\varepsilon_C^{\wedge}} - \varepsilon_{b}^{\wedge} .
\end{align*}
Thus
\begin{align*}
    \Fr{e}{\xizero}\Lambda\left(e, \mathring{u}\right)\Fr{\varepsilon}{\mathbf{0}}\vartheta^{-1}\left(\varepsilon\right)\left[\varepsilon\right] &= \Fr{e}{\xizero}\Lambda\left(e, \mathring{u}\right)\left[\varepsilon^{\wedge}\right]\\
    &= \left(\left(- \varepsilon_{b}^{\wedge},\,  - \left(\mathring{\Vector{}{\omega}{}}^{\wedge}\varepsilon_{b}\right)^{\wedge}\right),\, \adsym{\mathring{\Vector{}{\omega}{}}^{\wedge}}{\varepsilon_C^{\wedge}} - \varepsilon_{b}^{\wedge}\right) ,\\
    &= \left(\left(- \varepsilon_{b}^{\wedge},\,  - \left(\mathring{\Vector{}{\omega}{}}^{\wedge}\varepsilon_{b}\right)^{\wedge}\right),\, \left(\mathring{\Vector{}{\omega}{}}^{\wedge}\varepsilon_C\right)^{\wedge} - \varepsilon_{b}^{\wedge}\right) .
\end{align*}
The next step in the derivation of the differential of the projection $\phi_{\xizero}$ in the direction ${\alpha = \left(\left(- \varepsilon_{b}^{\wedge},\,  - \left(\mathring{\Vector{}{\omega}{}}^{\wedge}\varepsilon_{b}\right)^{\wedge}\right),\, \adsym{\mathring{\Vector{}{\omega}{}}^{\wedge}}{\varepsilon_C^{\wedge}} - \varepsilon_{b}^{\wedge}\right)}$:
\begin{equation*}
    \Fr{E}{I}\phi_{\xizero}\left(E\right)\left[\alpha\right] = \left(\left(\alpha_R^{\wedge},\, -\alpha_b^{\wedge}\right),\, \alpha_C^{\wedge} - \alpha_R^{\wedge}\right).
\end{equation*}
Substituting ${\alpha = \left(\left(- \varepsilon_{b}^{\wedge},\,  - \left(\mathring{\Vector{}{\omega}{}}^{\wedge}\varepsilon_{b}\right)^{\wedge}\right),\, \adsym{\mathring{\Vector{}{\omega}{}}^{\wedge}}{\varepsilon_C^{\wedge}} - \varepsilon_{b}^{\wedge}\right)}$ yields
\begin{equation*}
    \Fr{E}{I}\phi_{\xizero}\left(E\right)\Fr{e}{\xizero}\Lambda\left(e, \mathring{u}\right)\Fr{\varepsilon}{\mathbf{0}}\vartheta^{-1}\left(\varepsilon\right)\left[\varepsilon\right] = \left(\left(- \varepsilon_{b}^{\wedge},\, \left(\mathring{\Vector{}{\omega}{}}^{\wedge}\varepsilon_{b}\right)^{\wedge}\right),\, \left(\mathring{\Vector{}{\omega}{}}^{\wedge}\varepsilon_C\right)^{\wedge}\right).
\end{equation*}
Finally, given the choice of exponential coordinates, the solution of the last differential is again trivial, $\Fr{e}{\xizero}\vartheta\left(e\right)\left[\alpha^{\wedge}\right] = \alpha$. Therefore, \cref{eq_A0} is solved by
\begin{equation}
    \mathbf{A}_{t}^{0} = \begin{bmatrix}
        \mathbf{0} & -\mathbf{I} & \mathbf{0}\\
        \mathbf{0} & \mathring{\Vector{}{\omega}{}}^{\wedge} & \mathbf{0}\\
        \mathbf{0} & \mathbf{0} & \mathring{\Vector{}{\omega}{}}^{\wedge}
    \end{bmatrix} \in \R^{9 \times 9} .\label{bas_A0}
\end{equation}

\begin{remark}\label{bas_err_linearization_remark}
There is an important aspect to be noticed in the linearization of the equivariant error in \cref{bas_A0}. \emph{The linearization of the attitude error $\dot{\varepsilon_{R}} = -\varepsilon_{b}$ is exact, and the attitude error dynamics is autonomous.} The nonlinearities are shifted to the bias and calibration error dynamics, which either are slowly time-varying or possess zero dynamics. This is a fundamental difference to the Imperfect-\ac{iekf}, in which the linearized attitude error dynamics is given by $\dot{\varepsilon_{R}} = -\hatRot{}{}\varepsilon_{b}$.
\end{remark}

Similar differentials are involved in the solution of \cref{eq_C0,eq_B,eq_D_equi}. In particular, let ${\yzero = h(\xizero) = (\Vector{}{d}{1}, \Vector{}{d}{2}) \in \calN}$ be the origin output then define local coordinates of the output as follows:
\begin{equation*}
    \delta(y) = (\Vector{}{d}{1}^{\wedge}y_{1}, \Vector{}{d}{2}^{\wedge}y_{2}).
\end{equation*}
Solving the differentials in \cref{eq_C0_rho} yields
\begin{equation}
    \mathbf{C}^{0} = \begin{bmatrix}
        (\Vector{}{d}{1}^{\wedge})^{2} & \mathbf{0} & (\Vector{}{d}{1}^{\wedge})^{2}\\
        (\Vector{}{d}{2}^{\wedge})^{2} & \mathbf{0} & \mathbf{0}
    \end{bmatrix} \in \R^{6 \times 9} ,\label{bas_C0}
\end{equation}
Specifically, the rightmost differential $\Fr{\varepsilon}{\mathbf{0}}\vartheta^{-1}\left(\varepsilon\right)$ results in the identity due to the choice of exponential coordinates, thus $\Fr{\varepsilon}{\mathbf{0}}\vartheta^{-1}\left(\varepsilon\right)\left[\varepsilon\right] = \varepsilon$. The second differential, $\Fr{e}{\xizero}h(e)$, results in $\Vector{}{d}{1}^{\wedge}(\varepsilon_R + \varepsilon_C)$ for the case of a sensor calibrated online, and $\Vector{}{d}{2}^{\wedge}\varepsilon_R$ for the case of an already calibrated sensor. Similarly, $\Fr{y}{\yzero}\delta(y)$ results in $\Vector{}{d}{1}^{\wedge}$ and $\Vector{}{d}{2}^{\wedge}$. Combining them, the matrix in \cref{bas_C0}, is derived.

Note that the major difference in the \ac{eqf} filter matrices for the case of an already calibrated sensor compared to an estimated online extrinsic calibration is in the $\mathbf{C}^{0}$ matrix. Specifically, the known direction $\Vector{}{d}{*}$ appears only in the first block column of the $\mathbf{C}^{0}$ matrix for the case of already calibrated sensors, whereas it appears on both the first block column and on the block column that corresponds to the extrinsic calibration of that specific sensor, in the case of an estimated online extrinsic calibration. 

The input and the output matrix, defined respectively in \cref{eq_B,eq_D_equi}, are written
\begin{align}
    &\mathbf{B}_{t} = 
    \begin{bmatrix}
        \hat{A} & \mathbf{0} & \mathbf{0}\\
        \mathbf{0} & \hat{A} & \mathbf{0}\\
        \mathbf{0} & \mathbf{0} & \hat{B}
    \end{bmatrix} \in \R^{9 \times 9} ,\\
    &\mathbf{D}_{t} = \begin{bmatrix}
        \hat{B} & \mathbf{0}\\
        \mathbf{0} & \hat{A}
    \end{bmatrix} \in \R^{6 \times 6}.
\end{align}

\subsection{Practical discrete-time implementation}
Consider a moving robotic platform, an \ac{uav}, for example, equipped with a gyroscope and receiving body-frame direction measurements from a magnetometer $y_{1}$, thus measuring the fixed direction to magnetic north $\Vector{}{d}{1}$, and a spatial direction measurements of from two \ac{gnss} receivers $y_{2}$, thus a fixed measurement of a time-varying direction $\Vector{}{d}{2}$, as depicted in \cref{bas_uav_system}.

\begin{figure}[htp]
\centering
\includegraphics[width=0.75\linewidth]{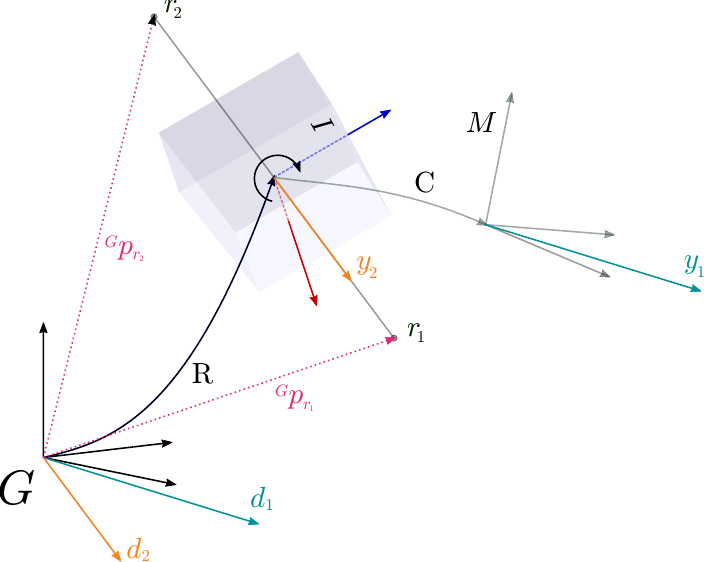}
\caption[Graphical representation of states and measurements of the biased attitude system.]{Graphical representation of the system of interest, and the corresponding body-frame direction measurement $y_{1}$ of a fixed direction $\Vector{}{d}{1}$, and a fixed spatial direction measurement $y_{2}$ of a time-varying direction $\Vector{}{d}{2}$.}
\label{bas_uav_system}
\end{figure}

It is particularly interesting to understand how a spatial direction measurement of a body-frame reference direction is constructed from the position measurements of two \ac{gnss} receivers. The platform is equipped with two \ac{gnss} receivers, $\Vector{}{r}{1}$ and $\Vector{}{r}{2}$, placed with sufficient baseline between each other as shown in \cref{bas_uav_system}. At the intersection of the baseline between the \ac{gnss} receivers and one of the three planes spanned by the body frame axes, we place a virtual \ac{gnss} frame \frameofref{g} with the y-axis aligned along the baseline. 
In our problem formulation, the virtual frame \frameofref{g} overlaps with the body frame \frameofref{I} (i.e., $\Rot{I}{g}=\eye_3$), and the baseline between the \ac{gnss} receivers is set to $1$\si[per-mode = symbol]{\meter}. 
We model the body-frame direction measurement as ${y_{2} = \left[0\,1\,0\right]^{\top}}$, resulting from the measurement model in \cref{bas_confout} with a time varying spatial reference direction $\Vector{}{d}{2}$, given by the rotation of the y-axis of the \frameofref{g} frame into the global inertial frame \frameofref{G}. The spatial reference direction $\Vector{}{d}{2}$ can be easily constructed from the raw \ac{gnss} measurements as follows:
\begin{equation}\label{bas_gnss_dir}
    \Vector{}{d}{2} = \frac{\left(\Vector{G}{p}{r_1}-\Vector{G}{p}{r_2}\right)}{\norm{\left(\Vector{G}{p}{r_1}-\Vector{G}{p}{r_2}\right)}}.
\end{equation}

The very last step to implement the proposed \ac{eqf} in discrete-time, as in \cref{eq_eqf}, is to derive the state transition matrix ${\mathbf{\Phi} \coloneqq \exp\left(\mathbf{A}_t^{0} \Delta t\right)}$. Generally, we would compute the matrix exponential numerically; however, we can derive a closed-form analytic solution in this case. Specifically, let $\Delta t$ be the time step in between gyro measurements, then, the \ac{eqf} state-transition matrix ${\mathbf{\Phi}\left(t+\Delta t, t\right) = \mathbf{\Phi}}$ is written as
\begin{equation*}\label{bas_state_transition_matrix}
    \mathbf{\Phi} = \begin{bmatrix}
    \eye & \mathbf{\Phi}_{12} & \mathbf{0}\\
    \mathbf{0} & \mathbf{\Phi}_{22} & \mathbf{0}\\
    \mathbf{0} & \mathbf{0} & \mathbf{\Phi}_{22}
    \end{bmatrix} \in \R^{9\times 9},
\end{equation*}
with
\begin{align*}
    \mathbf{\Phi}_{12} &= -\left(\Delta t \eye + \mathbf{\Psi}_{1}\Vector{}{\bm{\omega}}{0}^{\wedge} + \mathbf{\Psi}_{2} \Vector{}{\bm{\omega}}{0}^{\wedge}\Vector{}{\bm{\omega}}{0}^{\wedge}\right),\\
    \mathbf{\Phi}_{22} &= \eye + \mathbf{\Psi}_{3}\Vector{}{\bm{\omega}}{0}^{\wedge} + \mathbf{\Psi}_{1}\Vector{}{\bm{\omega}}{0}^{\wedge}\Vector{}{\bm{\omega}}{0}^{\wedge},\\
    \mathbf{\Psi}_{1} &= \frac{1-\cos\left(\norm{\Vector{}{\bm{\omega}}{0}} \Delta t\right)}{\norm{\Vector{}{\bm{\omega}}{0}}^2},\\
    \mathbf{\Psi}_{2} &= \frac{\norm{\Vector{}{\bm{\omega}}{0}} \Delta t - \sin\left(\norm{\Vector{}{\bm{\omega}}{0}} \Delta t{1}\right)}{\norm{\Vector{}{\bm{\omega}}{0}}^3},\\
    \mathbf{\Psi}_{3} &= \frac{\sin\left(\norm{\Vector{}{\bm{\omega}}{0}} \Delta t\right)}{\norm{\Vector{}{\bm{\omega}}{0}}}.
\end{align*}

Finally, for practical implementation aspects, it is useful to derive the state transition matrix $\mathbf{\Phi}$ in the limit of small angular velocity. therefore, taking the limit ${\lim_{\norm{\Vector{}{\bm{\omega}}{0}}\to0} \mathbf{\Phi}}$ and applying L'H\^{o}pital's rule yields
\begin{align*}
    \mathbf{\Phi}_{12} &\simeq -\Delta T\left(\eye + \frac{\Delta T}{2}\Vector{}{\bm{\omega}}{0}^{\wedge} + \frac{\Delta T^2}{6}\Vector{}{\bm{\omega}}{0}^{\wedge}\Vector{}{\bm{\omega}}{0}^{\wedge}\right),\\
    \mathbf{\Phi}_{22} &\simeq \eye + \Delta T\Vector{}{\bm{\omega}}{0}^{\wedge} + \frac{\Delta T^2}{2}\Vector{}{\bm{\omega}}{0}^{\wedge}\Vector{}{\bm{\omega}}{0}^{\wedge}.
\end{align*}

\section{Experiments and comparison with IEKF}
In this section, we compare the proposed \acl{eqf} design based on the proposed tangent symmetry with the Imperfect-\ac{iekf}~\cite{7523335, barrau:tel-01247723}, referred to as \ac{iekf}, in text, figures, and tables for improved readability. 

\subsection{Simulation-based experiment}
The comparison is done first on a simulation-based experiment of an \ac{uav} equipped with a gyroscope and receiving body-frame direction measurements from a magnetometer and spatial direction measurements from two \ac{gnss} receivers, as described at the end of the previous section. Ground-truth data is generated through the Gazebo/RotorS framework~\cite{furrer_rotors_2016} that simulates a \ac{uav}’s flight and sensor behavior by realistically modeling the \ac{uav}’s dynamics and sensor measurements.

The performance of the \ac{eqf} and \ac{iekf} were evaluated on a 100-run Monte-Carlo simulation consisting of different \SI{70}{\second} long Lissajous trajectories with different levels of excitation. In each run, the two filters were randomly initialized with a wrong attitude of ${10}$\si[per-mode = symbol]{\degree} error standard deviation per axis, identity calibration, and zero bias. The initial covariance is set large enough to cover the initialization error.
In order to replicate a realistic scenario, simulated gyroscope measurements provided to the filters at \SI{200}{\hertz} included Gaussian noise and a non-zero, time-varying bias modeled as a random walk process. 
The continuous-time standard deviation of the measurement noise are ${\sigma_{\Vector{}{w}{}} = 8.73\cdot10^{-4}}$\si[per-mode = symbol]{\radian\per\sqrt\second}, and ${\sigma_{\Vector{}{b}{w}} = 1.75\cdot10^{-5}}$\si[per-mode = symbol]{\radian\per\second\sqrt\second} for simulated gyroscope measurements and bias respectively. Zero mean white Gaussian noise was added to the two direction measurements $y_{1}$, and $y_{2}$ and they were provided to the filters at \SI{100}{\hertz}, and \SI{20}{\hertz} respectively. The discrete-time standard deviation of the (unit-less) direction measurement noise densities are ${\sigma_{y_{1}} = 0.2}$, and ${\sigma_{y_{2}} = 0.1}$. The matrices $\mathbf{Q}$ and $\mathbf{R}$ on both filters were set to reflect the measurement noise covariance.

\Cref{bas_tab_rmse} reports the averaged \acs{rmse} of the filter states over the 100 runs. In the table, $T$ denotes the transient phase considered the first \SI{35}{\second}, whereas $A$ denotes the asymptotic phase considered the last \SI{35}{\second} of each run. \Cref{bas_sim_error_plots} shows the evolution of the averaged filter states error and its sample standard deviation over the first \SI{25}{\second} of the \SI{70}{\second} trajectories for a better overview over the transient phase. In general, it is interesting to note the improved tracking and transient performance of the proposed \ac{eqf}, and the ability of the filter to converge quickly despite the heavily wrong initializations (the mean error norms for both filters start at roughly \SI{20}{\degree} and \SI{35}{\degree} for the orientation and calibration states respectively).
The \ac{eqf} derived from the proposed tangent symmetry outperforms the state-of-the-art \ac{iekf} in both transient response and asymptotic behavior. The non-ideal conditions in the real-world experiments that follow make the differences between the approaches even more apparent.

\begin{figure}[htp]
\centering
\includegraphics[width=\linewidth]{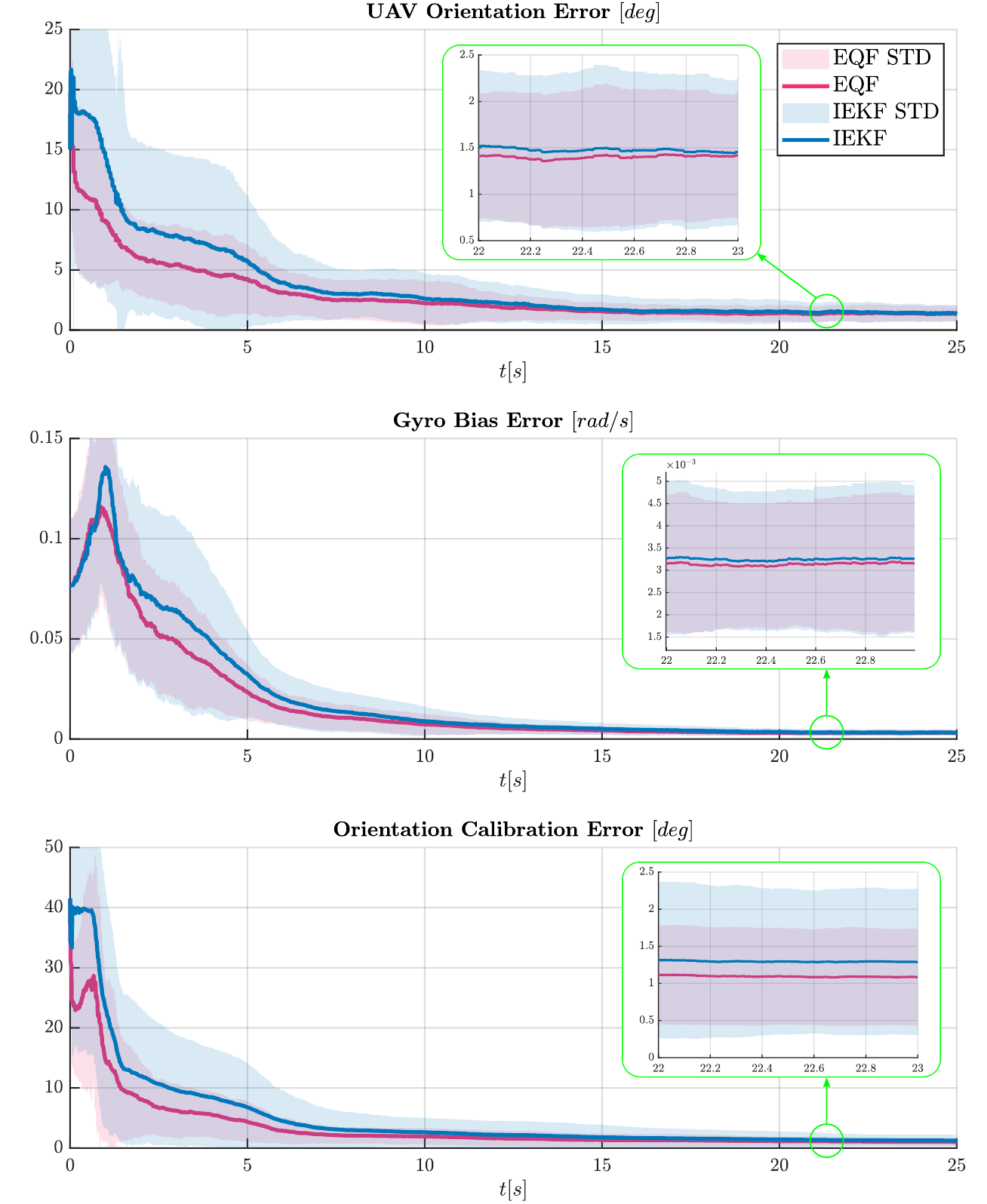}
\figrefcaption{\cite{Fornasier2022OvercomingCalibration}}
\caption[Averaged error plots of the proposed \ac{eqf} and the \ac{iekf} for the Monte-Carlo simulation.]{Averaged norm of the \ac{uav} attitude, magnetometer calibration, and gyroscope bias error over the 100 runs. The plots only depict the first \SI{25}{\second} of the \SI{70}{\second} trajectories showing the improved transient of the proposed \ac{eqf}. The analysis of the full trajectories is reflected in the numbers in \cref{bas_tab_rmse}.}
\label{bas_sim_error_plots}
\end{figure}
\ifdefined\includetblr
\begin{table}[htp]
    \setlength\tabcolsep{5pt}
    \centering
    \tabrefcaption{\cite{Fornasier2022OvercomingCalibration}}
    \captiontitlefont{\scshape\small}
    \captionnamefont{\scshape\small}
    \captiondelim{}
    \captionstyle{\centering\\}
    \caption{Monte-Carlo simulation transient and asymptotic averaged \acs{rmse}.}
    \begin{tblr}{
        rows = {m},
        column{1} = {c},
        column{2} = {c},
        column{3} = {c},
        column{4} = {c},
    }
    \toprule
    \acs{rmse} & Attitude [\si[per-mode = symbol]{\degree}] & bias [\si[per-mode = symbol]{\radian\per\second}] & Calibration [\si[per-mode = symbol]{\degree}] \\
    \midrule
    \ac{eqf} $(T)$ & $\mathbf{3.5331}$ & $\mathbf{0.0280}$ & $\mathbf{5.7892}$ \\
    \ac{iekf} $(T)$ & $4.9497$ & $0.0323$ & $8.0480$\\
    \midrule
    \ac{eqf} $(A)$ & $\mathbf{1.3870}$ & $\mathbf{0.0035}$ & $\mathbf{0.6989}$ \\
    \ac{iekf} $(A)$ & $1.3995$ & $\mathbf{0.0035}$ & $0.7798$\\
    \bottomrule
    \end{tblr}\label{bas_tab_rmse}
\end{table}
\fi

\subsection{Indoor real-world experiment}
In this second experiment, we compare against accurate ground-truth in a real scenario with non-ideal conditions such as multi-rate and unsynchronized sensor measurements and measurement dropouts. Specifically, an indoor dataset was recorded with an AscTec Hummingbird \ac{uav} flying an aggressive trajectory in a motion capture-equipped room for \SI{140}{\second}. Gyroscope measurements, as well as full 6-DOF pose measurements of the platform, were available at \SI{330}{\hertz}. The continuous-time standard deviation of the gyroscope measurement and bias noise are set to  ${\sigma_{\Vector{}{w}{}} = 0.013}$\si[per-mode = symbol]{\radian\per\sqrt\second}, and ${\sigma_{\Vector{}{b}{w}} = 0.0013}$\si[per-mode = symbol]{\radian\per\second\sqrt\second} respectively.
The pose measurements from the motion capture system were used to re-create the previously discussed scenario and, therefore, to manufacture measurements $y_{1}$ and $y_{2}$ of directions $\Vector{}{d}{1}$ and $\Vector{}{d}{2}$. To replicate the non-ideal conditions of real-world measurements, $y_{1}$, and $y_{2}$ were generated at \SI{100}{\hertz} and \SI{25}{\hertz} respectively. A dropout rate of $\sim10\%$ was actively induced for the magnetometer measurements. Moreover, the measurements were perturbed with zero mean white Gaussian noise with (unit-less) discrete-time standard deviations ${\sigma_{y_{1}} = 0.1}$, and ${\sigma_{y_{2}} = 0.01}$. 
The two filters were both initialized with a wrong attitude that corresponds to the Euler angles (ypr) ${\hatVector{G}{\theta}{}\simeq\left[70\;-40\;30\right]^{\top}}$ in degrees, and calibration corresponding to the Euler angles (ypr) ${\hatVector{I}{\phi}{}=\left[-90\;-60\;130\right]^{\top}}$ in degrees, and zero bias. The ground-truth initial attitude corresponds to the Euler angles (ypr) ${\Vector{G}{\theta}{}\simeq\left[90\;0\;0\right]^{\top}}$ in degrees, and the ground-truth extrinsic calibration corresponds to the Euler angles (ypr) ${\Vector{I}{\phi}{}=\left[30\;5\;25\right]^{\top}}$ in degrees. The matrices $\mathbf{Q}$ and $\mathbf{R}$ on both filters were set to reflect the measurement noise covariance.

The choice of initializing the filters with extremely wrong attitude and calibration states has been made to trigger the worst possible scenario and to highlight the estimator's transient behaviors. \Cref{bas_hummy_error_plots} shows the transients of the attitude error, extrinsic calibration error, and the bias norm for the first \SI{45}{\second} of the trajectory. The proposed \ac{eqf} clearly outperforms the \ac{iekf}, in particular, the attitude error norm of the \ac{eqf} stabilizes below \SI{10}{\degree} in $\sim$\SI{3}{\second}, and below \SI{5}{\degree} in $\sim$\SI{10}{\second}, while the \ac{iekf} requires respectively $\sim$\SI{15}{\second}, and $\sim$\SI{30}{\second}. In terms of calibration states error, the proposed \ac{eqf} shows a transient almost an order of magnitude faster, converging below \SI{5}{\degree} in just $\sim$\SI{5}{\second} compared to the $\sim$\SI{45}{\second} of the \ac{iekf}. Although no ground-truth information is available for the gyroscope bias, we compared the norm of the bias vector for the two filters. The rationale behind such comparison is that the norm of the bias vector is supposed to quickly decrease almost to zero since the AscTec platforms perform a gyro calibration upon system start. Results in \cref{bas_hummy_error_plots} show a faster decrement of that norm for the proposed \ac{eqf} compared to the \ac{iekf}, and therefore, a faster convergence of the bias state is assumed. As a final remark, in this flight, the \ac{uav} was taking off vertically for the first 2.5 seconds, thus being in a regime of low angular excitation. In such a situation, we can observe that the \ac{iekf} struggles to extract information for the state to converge, whereas the \ac{eqf} is able to do so.

\begin{figure}[htp]
\centering
\includegraphics[width=\linewidth]{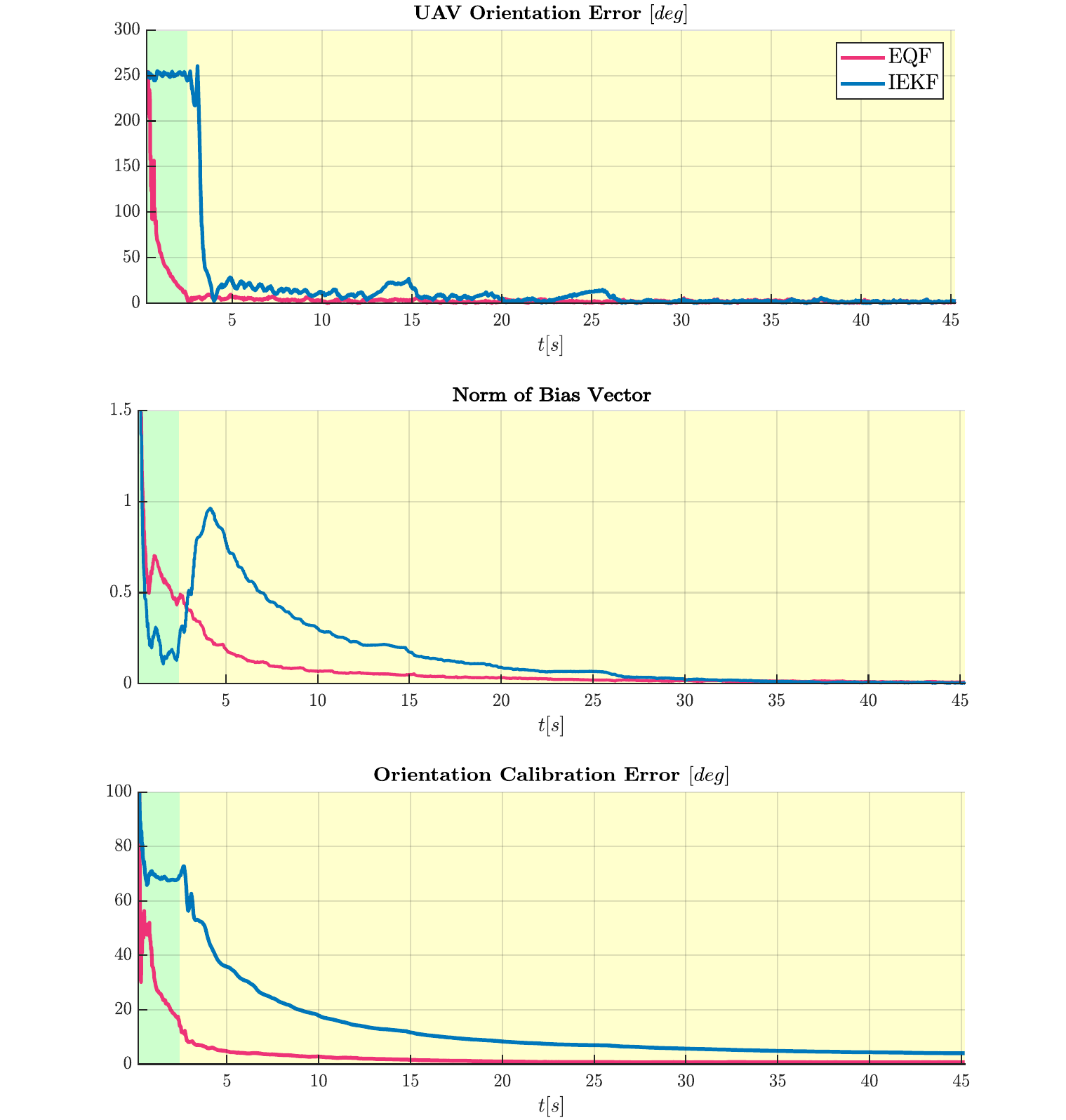}
\figrefcaption{\cite{Fornasier2022OvercomingCalibration}}
\caption[Error plots of the proposed \ac{eqf} and the \ac{iekf} for the indoor experiment.]{Norm of the \ac{uav} attitude, magnetometer calibration error, and norm of the estimated bias vector for the indoor experiment. The shaded green region represents the vertical take-off phase, while the yellow shaded region represents the mission phase. Note the better convergence of the \ac{eqf} with respect to the \ac{iekf} and the \ac{eqf}'s ability to attain a good convergence rate of the state estimate even in the presence of low angular excitation during the vertical take-off phase.}
\label{bas_hummy_error_plots}
\end{figure}

\subsection{Outdoor real-world experiment}
In this last experiment, we run and compare the proposed \ac{eqf} formulation and \ac{iekf} on an outdoor real-world dataset. 
Although no ground truth information is available for this outdoor experiment, we show an application of the filters in a field scenario with real sensor data from a magnetometer and two \ac{rtk} \ac{gnss} receivers.

The outdoor dataset was recorded by flying a Twins twinFOLD SCIENCE quadrotor for \SI{130}{\second}, equipped with a Holybro Pixhawk 4 FCU with \acs{imu}, magnetometer and two Ublox \ac{rtk} \ac{gnss} receivers as shown in \cref{bas_twins}. Gyroscope measurements were available at \SI{200}{\hertz}, magnetometer (calibrated for soft and hard-iron effects), and \ac{gnss} measurements were available at \SI{90}{\hertz} and \SI{8}{\hertz} respectively. The continuous-time standard deviation of the gyroscope measurement noise are ${\sigma_{\Vector{}{w}{}} = 1.75\cdot10^{-4}}$\si[per-mode = symbol]{\radian\per\sqrt\second}, and ${\sigma_{\Vector{}{b}{w}} = 8.73\cdot10^{-6}}$\si[per-mode = symbol]{\radian\per\second\sqrt\second} (obtained with Allan variance technique). Discrete-time standard deviations of the (unit-less) direction measurements were considered to be ${\sigma_{y_{1}} = 0.1}$, and ${\sigma_{y_{2}} = 0.01}$.
Again, the two filters' gain matrices, $\mathbf{Q}$ and $\mathbf{R}$, were set according to the measurement noise. 
In this experiment, we simulate a scenario of mid-air filter re-initialization from a completely wrong initial estimate (in the absence of prior information). Both filters were initialized with a wrong consecutive rotation (ypr) of $\sim$\SI{30}{\degree} along each axis for both attitude and calibration states and with initial gyro bias set to zero. 
Ground-truth information is not available for this experiment; however, the magnetometer extrinsic calibration is known to be almost identity and, as the baseline for the attitude, we refer to the converged estimates of an online available state-of-the-art modular multi-sensor fusion framework, MaRS~\cite{Brommer2021}, obtained by feeding \ac{imu}, \ac{gnss} positions and velocities as sensor measurements.

\begin{figure}[htp]
\centering
\includegraphics[width=\linewidth]{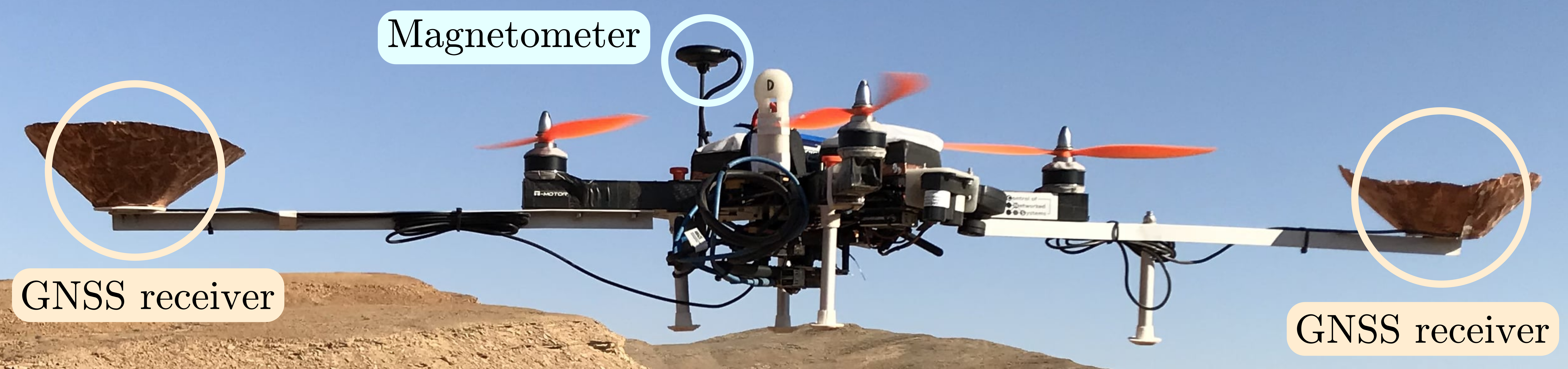}
\figrefcaption{\cite{Fornasier2022OvercomingCalibration}}
\caption[Platform used for the outdoor field experiment.]{Platform used for the outdoor field experiment, equipped with two \acs{rtk} \acs{gnss} receivers and a magnetometer integrated within the compass module.}
\label{bas_twins}
\end{figure}

Interesting aspects of this experiment are the poor quality of the magnetometer readings as well as the reduced amount of excitation the platform was subjected to.
Only three noticeable rotations were performed during the flight: a small combined rotation at about second~$5$, and two yaw-only rotations of \SI{180}{\degree} at about seconds~$25$, and~$45$. This low angular excitation is a common but challenging situation for the attitude estimation problem, including calibration and bias states, in \ac{uav} applications.

\cref{bas_twins_error_plots} shows the filters' attitude error, as well as the norm of the orientation calibration angular vector and the norm of the bias vector for the first \SI{45}{\second} of the trajectory. The attitude error is computed against the estimate of MaRS, previously initialized and converged (for $t\leq0$ with reference to \cref{bas_twins_error_plots}). Again, the rationale behind comparing the norm of the aforementioned vectors is that both the extrinsic calibration angular vector and the gyroscope bias vector are almost zero. Hence, the norm of the estimate of these vectors is expected to decrease. This provides a good measure of the transient performance of the filter. Both filters' performance showed in \cref{bas_twins_error_plots} are quite remarkable indeed. 

\begin{figure}[htp]
\centering
\includegraphics[width=0.95\linewidth]{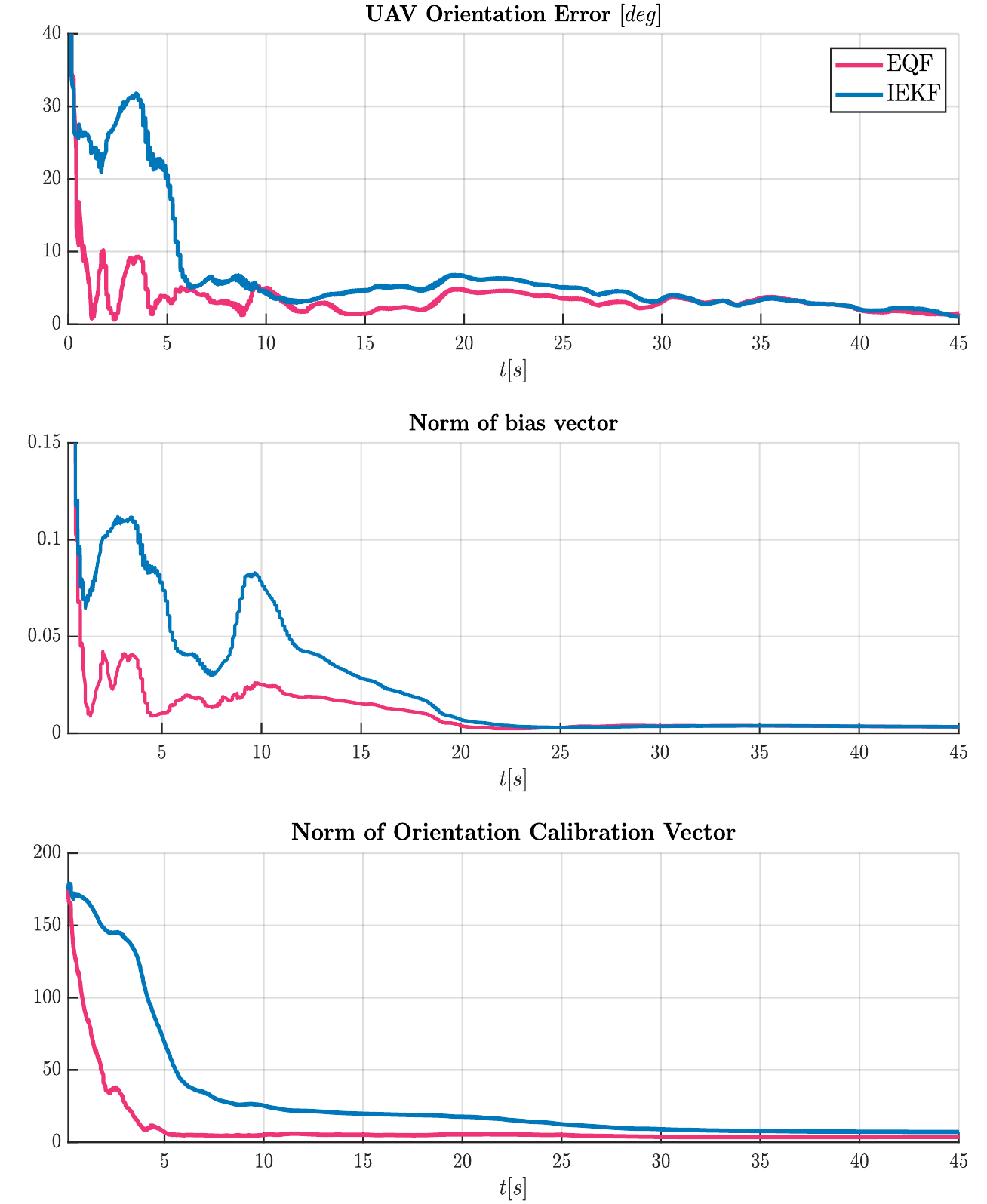}
\figrefcaption{\cite{Fornasier2022OvercomingCalibration}}
\caption[Error plots of the proposed \ac{eqf} and the \ac{iekf} for the outdoor experiments.]{Error plots of the proposed \ac{eqf} and the \ac{iekf} for the first \SI{45}{\second} after noticeably wrong initialization. The plots show the error in the estimated attitude of the \ac{uav}, and due to the absence of ground-truth information, the norm of the angular calibration, and the angular velocity bias vectors. The proposed \ac{eqf} shows noticeably better performance compared to the state-of-the-art \ac{iekf}}.
\label{bas_twins_error_plots}
\end{figure}

The lack of ground truth and the low excitation make a quantitative analysis of these results challenging. That said, \cref{bas_twins_error_plots} clearly shows qualitatively the ability of the proposed \ac{eqf} in real-world applications to quickly converge when initialized with a wrong estimate in quasi-stationary regime, providing a faster and more reliable estimate than the state-of-the-art.

\section{Chapter conclusion}
It is clear now that \emph{tangent group symmetries} play a crucial role in the context of \aclp{ins} and form the main results in this dissertation. Specifically, in this chapter, we presented a novel equivariant symmetry of the attitude kinematics based on the tangent group of $\SO(3)$ and developed an \acl{eqf} formulation based on the proposed symmetry. The proposed \ac{eqf} formulation can include an arbitrary number $N$ of generic direction measurements, being either body-frame fixed or spatial directions measurements covering the majority of real-world direction-sensor modalities. 

Statistically relevant results from simulated as well as real-world experiments showed that the proposed \ac{eqf} outperforms the state-of-the-art Imperfect-\ac{iekf} in both transient and asymptotic tracking performance. The presented results are of particular relevance when considering the practical applicability of the presented \ac{eqf} to real-world scenarios and the importance of having estimators that are able to converge quickly, without any prior knowledge, and in the presence of high initial errors. Furthermore, non-ideal effects present in real-world data are often the cause of poor filter performance. The proposed discrete-time \ac{eqf} implementation exhibited robust performance against such non-idealities of real-world sensor data. 

The lower performance of the Imperfect-\ac{iekf} in these experiments is assumed to stem from the natural symmetry not properly ``respecting'' the geometry of the gyro bias and modeling them as "tacked on" states. As highlighted in \cref{bas_err_linearization_remark}, it becomes evident that the choice of symmetry affects the linearized error dynamics, ultimately impacting the overall performance of the filter.

This chapter shows that exploiting the underlying symmetry of equivariant systems is of paramount importance for designing better estimators. The next chapter will investigate how these ideas apply to general second-order \aclp{ins}.
\chapter[Equivariant Symmetries for Inertial Navigation Systems][Equivariant Symmetries for INS]{Equivariant Symmetries for Inertial Navigation Systems}\label{bins_chp}
\emph{The present chapter contains results that have been peer-reviewed and published in the IEEE International Conference on Robotics and Automation (ICRA)~\cite{Fornasier2022EquivariantBiases}, as well as results submitted to Automatica and currently under revision~\cite{Fornasier2023EquivariantSystems}.}
\bigskip

\noindent This chapter extends the results introduced in the previous chapter to second-order inertial kinematic systems. In a sentence, \emph{this chapter establishes the general theory of equivariant \aclp{ins} providing a general framework for equivariant \ac{ins} filter design.}

We start this chapter by introducing different symmetries of second-order \aclp{ins} and discussing their properties in the context of \acl{eqf} design. Specifically, we discuss the role of the \emph{tangent group of $\SE_2(3)$} and how different symmetries can be derived as variations of it. 
We analyze the linearized error dynamics and derive an \acl{eqf} for each of the proposed symmetries. Moreover, we formalize the fundamental results that \emph{every filter can be derived as an \acl{eqf} with a specific choice of symmetry, and the choice of symmetry is, in fact, the only difference between different versions of \ac{ins} filters}. 

Finally, the theoretical results presented in this chapter are validated via extensive simulation of a position-based navigation scenario. Moreover, this application provides a convenient framework to introduce an additional result, that is, a general methodology to cast global-referenced measurements into body-referenced measurements that are compatible with the symmetries, making them suitable for many measurement types.

\begin{figure}[htp]
\centering
\includegraphics[width=0.9\linewidth]{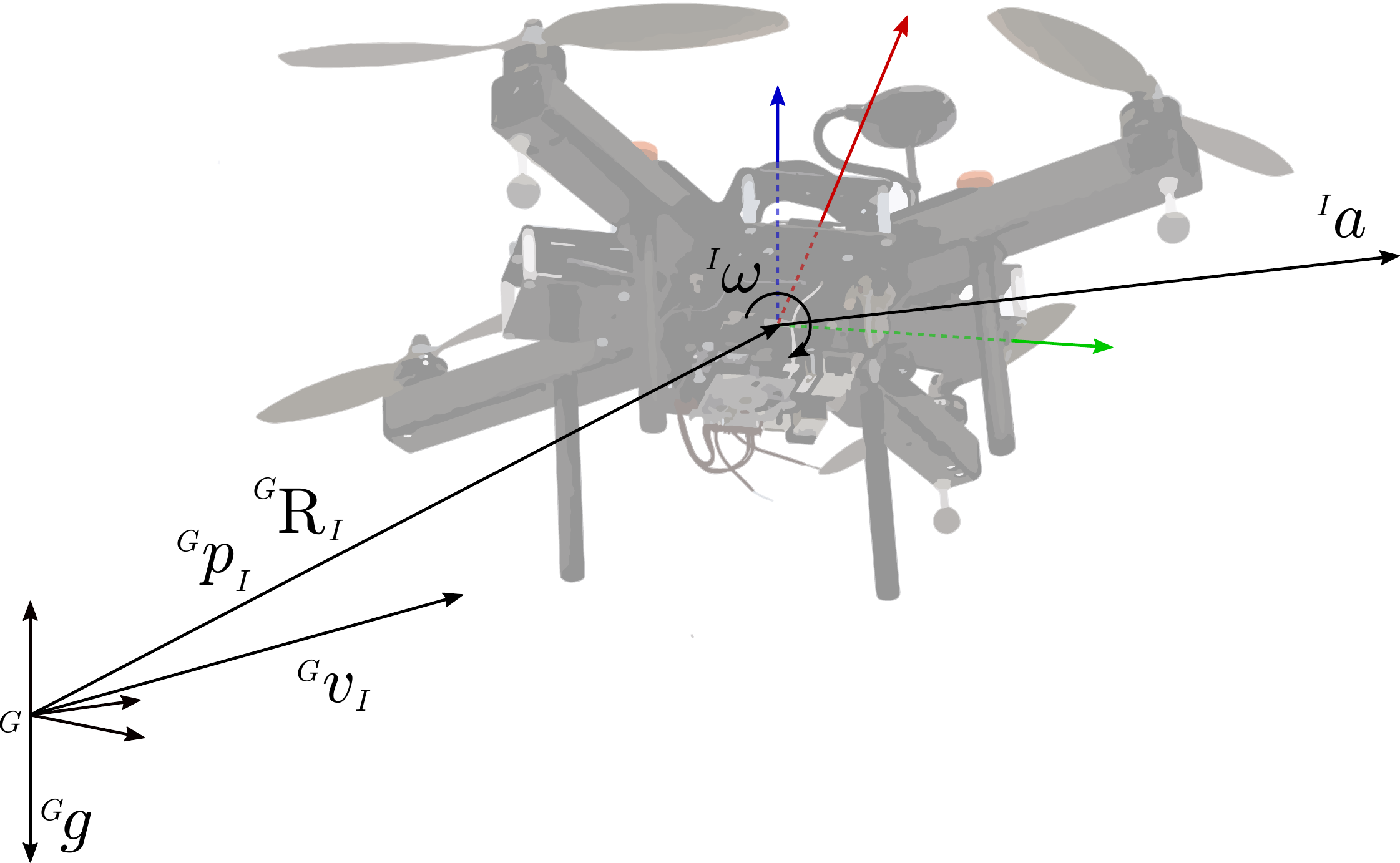}
\caption{Graphical representation of a \ac{uav} freely moving in space, representing an \acl{ins}.}
\label{bins_system}
\end{figure}

\section{The biased inertial navigation system}

Consider a mobile robot equipped with an \ac{imu} providing angular velocity and acceleration measurements, as well as other sensors providing partial direct or indirect state measurements (e.g., a \ac{gnss} receiver providing position measurements or a magnetometer providing direction measurements), as depicted in \cref{bins_system}.
Let \frameofref{G} denote the global inertial frame of reference and \frameofref{I} denote the \ac{imu} frame of reference. In non-rotating, flat earth assumption, the deterministic (noise-free) continuous-time biased inertial navigation system is
\begin{subequations}\label{bins_bins}
    \begin{align}
        &\dot{\Rot{G}{I}} = \Rot{G}{I}\left(\Vector{I}{\bm{\omega}}{} - \Vector{I}{b}{\bm{\omega}}\right)^{\wedge} ,\\
        &\dotVector{G}{v}{I} =  \Rot{G}{I}\left(\Vector{I}{a}{} - \Vector{I}{b}{a}\right) + \Vector{G}{g}{} ,\\
        &\dotVector{G}{p}{I} = \Vector{G}{v}{I} ,\label{bins_pdot}\\
        &\dotVector{I}{b_{\bm{\omega}}}{} = \Vector{I}{\tau}{\bm{\omega}} ,\\
        &\dotVector{I}{b_{a}}{} = \Vector{I}{\tau}{a} .
    \end{align}
\end{subequations}
Here, $\Rot{G}{I}$ denotes the rigid body orientation, and $\Vector{G}{p}{I}$ and $\Vector{G}{v}{I}$ denote the rigid body position and velocity expressed in the \frameofref{G} frame, respectively.
These variables are termed the \emph{navigation states}.
The gravity vector $\Vector{G}{g}{}$ is expressed in frame \frameofref{G}. 
The gyroscope measurement and accelerometer measurement are written $\Vector{I}{\bm{\omega}}{}$ and $\Vector{I}{\bm{a}}{}$ respectively. 
The two biases $\Vector{I}{b_{\bm{\omega}}}{}$ and $\Vector{I}{b_{\bm{a}}}{}$ are termed the \emph{bias states}. 
The inputs $\Vector{I}{\tau}{\bm{\omega}}$, $\Vector{I}{\tau}{a}$ are used to model the biases' dynamics and are zero when the biases are modeled as constant quantities.

The state space is $\calM = \torSO(3) \times \R^3 \times \R^3 \times \R^3 \times \R^3$ where $\torSO(3)$ is the $\SO(3)$-torsor modeling the rigid body orientation, and the 4 copies of $\R^3$ model velocity, position, and angular velocity and linear acceleration bias, respectively. An element of the state space and an element of the input space are written, respectively
\begin{align}
    &\xi =  \left(\Rot{G}{I}, \Vector{G}{v}{I}, \Vector{G}{p}{I}, \Vector{I}{b_{\bm{\omega}}}{}, \Vector{I}{b_{a}}{} \right) \in \calM ,\label{bins_state}\\
    & u = \left(\Vector{I}{\bm{\omega}}{}, \Vector{I}{a}{}, \Vector{I}{\tau}{\bm{\omega}}, \Vector{I}{\tau}{a}\right) \in \vecL \subset \R^{12} .\label{bins_input}
\end{align}

For the sake of clarity of the presentation, in the following sections, we drop subscripts and superscripts from state, input, and output variables and adopt the lean notation defined in \cref{bins_conversion_table}.

\ifdefined\includetblr
\begin{table}[htp]
    \centering
    \tabrefcaption{\cite{Fornasier2023EquivariantSystems}}
    \setlength\tabcolsep{5pt}
    \captiontitlefont{\scshape\small}
    \captionnamefont{\scshape\small}
    \captiondelim{}
    \captionstyle{\centering\\}
    \caption{Descriptive-Lean Notation Conversion Table.}
    \begin{tblr}
    {
        rows = {m},
        column{1} = {c},
        column{2} = {c},
        column{3} = {c},
    }
    \toprule
    Description & Descriptive notation & Lean notation \\
    \midrule
    Rigid body orientation & $\Rot{G}{I}$ & $\Rot{}{}$ \\
    Rigid body velocity & $\Vector{G}{v}{I}$ & $\Vector{}{v}{}$ \\
    Rigid body position & $\Vector{G}{p}{I}$ & $\Vector{}{p}{}$ \\
    Angular velocity measurement & $\Vector{I}{\bm{\omega}}{}$ & $\Vector{}{\bm{\omega}}{}$ \\
    Gyroscope bias & $\Vector{I}{b}{\bm{\omega}}$ & $\Vector{}{b}{\bm{\omega}}$ \\
    Acceleration measurement & $\Vector{I}{a}{}$ & $\Vector{}{a}{}$ \\
    Accelerometer bias & $\Vector{I}{b}{\bm{a}}$ & $\Vector{}{b}{\bm{a}}$ \\
    \bottomrule
    \end{tblr}
    \label{bins_conversion_table}
\end{table}
\fi

In the following sections, we analyze different symmetries of the biased \acl{ins} under the lens of equivariance; that is, we show how those symmetries relate to classical filter design when exploited within the \acl{eqf} framework and how every filter can be derived as an \acl{eqf} under an appropriate choice of symmetry.

\section{Special orthogonal group symmetry ${\grpE: \SO(3) \times \R^{12}}$}\label{bins_mekf_sym_sec}

Lie group theory was first applied to navigation systems to overcome the limitation and the singularities of using Euler angles as the parameterization of the attitude of a rigid body. Alternative attitude parameterizations were originally formulated on the quaternion group~\cite{Lefferts1982KalmanEstimation, markley2003attitude}; however, the modern approach directly models attitude on the special orthogonal group $\SO(3)$.

Consider the state space of second-order inertial kinmatic systems in \cref{bins_state}, define the symmetry group ${\grpE \coloneqq \SO(3) \times \R^{12}}$, and let ${X = \left(A, a, b, \alpha, \beta\right) \in \grpE}$, where ${A \in \SO(3)}$, ${a, b, \alpha, \beta \in \R^3}$.
Let ${X = \left(A_X, a_X, b_X, \alpha_X, \beta_X\right)}$, ${Y = \left(A_Y, a_Y, b_Y, \alpha_Y, \beta_Y\right)}$ be two elements of the symmetry group, then the group product is written  ${XY = \left(A_XA_Y, a_X + a_Y, b_X + b_Y, \alpha_X + \alpha_Y, \beta_X + \beta_Y\right)}$.
The inverse of an element $X$ is defined as ${X^{-1} = \left(A^\top, -a, -b, -\alpha, -\beta\right)}$.

\begin{lemma}
Define ${\phi \AtoB{\grpE \times \calM}{\calM}}$ as
\begin{equation}
    \phi\left(X, \xi\right) \coloneqq \left(\Rot{}{}A, \Vector{}{v}{} + a, \Vector{}{p}{} + b, \Vector{}{b}{\bm{\omega}} + \alpha, \Vector{}{b}{a} + \beta\right) \in \calM .
\end{equation}
Then, $\phi$ is a transitive right group action of $\grpE$ on $\calM$.
\end{lemma}

The existence of a transitive group action of the symmetry group $\grpE$ on the state space $\calM$ guarantees the existence of a lift~\cite{Mahony2020EquivariantDesign}.

\begin{theorem}
Define the map ${\Lambda \AtoB{\calM \times \vecL}{\gothGrpE}}$ by
\begin{equation*}
    \Lambda\left(\xi, u\right) \coloneqq \left(\Lambda_1\left(\xi, u\right), \cdots, \Lambda_5\left(\xi, u\right)\right).
\end{equation*}
where ${\Lambda_1 \AtoB{\calM \times \vecL}{\gothso(3)}}$, and ${\Lambda_2, \cdots, \Lambda_5 \AtoB{\calM \times \vecL}{\R^{3}}}$ are given by
\begin{align}
    &\Lambda_1\left(\xi, u\right) \coloneqq \left(\Vector{}{\bm{\omega}}{} - \Vector{}{b}{\bm{\omega}}\right)^{\wedge} ,\\
    &\Lambda_2\left(\xi, u\right) \coloneqq \Rot{}{}\left(\Vector{}{a}{} - \Vector{}{b}{a}\right) + \Vector{}{g}{} ,\\
    &\Lambda_3\left(\xi, u\right) \coloneqq \Vector{}{v}{I} ,\\
    &\Lambda_4\left(\xi, u\right) \coloneqq \Vector{}{\tau}{\bm{\omega}} ,\\
    &\Lambda_4\left(\xi, u\right) \coloneqq \Vector{}{\tau}{a} .
\end{align}
Then, ${\Lambda}$ is a lift for the system in \cref{bins_bins} with respect to the symmetry group ${\grpE \coloneqq \SO(3) \times \R^{12}}$.
\end{theorem}

Later in \cref{bins_mekf_linan}, the relation between the choice of symmetry for \acl{eqf} design and classical filter design becomes clear. In particular, it is shown that an \ac{eqf} designed exploiting this symmetry results in the well-known \ac{mekf}~\cite{Lefferts1982KalmanEstimation}.

\section{Extended special Euclidean group ${\grpF: \SE_2(3) \times \R^{6}}$}\label{bins_iekf_sym_sec}

Exploiting the extended special Euclidean group $\SE_2(3)$ to model the navigation states of \ac{ins} problems~\cite{7523335, barrau:tel-01247723, barrau2015ekf, Barrau2020APreintegration, Brossard2021AssociatingEarth} has been one of the major developments in \ac{ins} filtering in the past 10 years. 

Consider ${\xi = \left(\Pose{}{}, \Vector{}{b}{}\right) \in \calM \coloneqq \torSE_2(3) \times \R^{6}}$ to be the state space of the biased \ac{ins} in \cref{bins_bins}. ${\Pose{}{} = \left(\Rot{}{}, \Vector{}{v}{}, \Vector{}{p}{}\right) \in \torSE_2(3)}$ is the systems's extended pose~\cite{Brossard2021AssociatingEarth}, which includes the orientation, the velocity and the position of the rigid body. Note that $\Pose{}{}$ can be written in matrix form as described in \cref{math_se23_sec}. ${\Vector{}{b}{} = \left(\Vector{}{b}{\bm{\omega}}, \Vector{}{b}{\bm{a}}\right) \in \R^{6}}$ denotes the \ac{imu} biases.

Let ${u = \left(\Vector{}{w}{}, \bm{\tau}\right) \in \vecL \subseteq \R^{12}}$ denote the system input, where ${\Vector{}{w}{} = \left(\Vector{}{\bm{\omega}}{}, \Vector{}{a}{}\right) \in \R^{6}}$ denotes the input given by the \ac{imu} readings, and ${\bm{\tau} = \left(\bm{\tau_{\omega}}, \bm{\tau_{a}}\right) \in \R^{6}}$ denotes the input for the \ac{imu} biases.

Define the matrices $\mathbf{W}, \mathbf{B}, \mathbf{N}, \mathbf{G}$ to be
\begin{equation*}
    \begin{tblr}{ll}
        \mathbf{W} = \begin{bmatrix}
        \angv^{\wedge} & \acc & \mathbf{0}_{3\times 1}\\
        \mathbf{0}_{1\times 3} & 0 & 0\\
        \mathbf{0}_{1\times 3} & 0 & 0\\
        \end{bmatrix}, & 
        \mathbf{B} = \begin{bmatrix}
        \bw^{\wedge} & \ba & \mathbf{0}_{3\times 1}\\
        \mathbf{0}_{1\times 3} & 0 & 0\\
        \mathbf{0}_{1\times 3} & 0 & 0\\
        \end{bmatrix}, \\
        \mathbf{N} = \begin{bmatrix}
        \mathbf{0}_{3\times 3} & \mathbf{0}_{3\times 1} & \mathbf{0}_{3\times 1}\\
        \mathbf{0}_{1\times 3} & 0 & 1\\
        \mathbf{0}_{1\times 3} & 0 & 0\\
        \end{bmatrix}, &
        \mathbf{G}  = \begin{bmatrix}
        \mathbf{0}_{3\times 3} & \Vector{}{g}{} & \mathbf{0}_{3\times 1}\\
        \mathbf{0}_{1\times 3} & 0 & 0\\
        \mathbf{0}_{1\times 3} & 0 & 0\\
        \end{bmatrix}.
    \end{tblr}
\end{equation*}
Then, the system in \cref{bins_bins} may then be written as
\begin{subequations}\label{bins_bins_se23}
    \begin{align}
        &\dot{\Pose{}{}} = \Pose{}{}\left(\mathbf{W} - \mathbf{B} + \mathbf{N}\right) + \left(\mathbf{G} - \mathbf{N}\right)\Pose{}{} ,\\
        &\dotVector{}{b}{} =  \bm{\tau} .
    \end{align}
\end{subequations}

Define the symmetry group ${\grpF \coloneqq \SE_2(3) \times \R^{6}}$, and let ${X = \left(D, \delta\right) \in \grpF}$, where ${D = \left(A, a, b\right)\in \SE_2(3)}$, $A \in \SO(3)$, $a, b \in \R^3$, ${\delta \in \R^6}$.
Let ${X = \left(D_X, \delta_X\right)}$, ${Y = \left(D_Y, \delta_Y\right)}$ be two elements of the symmetry group, then the group product is written  ${XY = \left(D_XD_Y, \delta_X + \delta_Y\right)}$.
The inverse of an element $X$ is given by ${X^{-1} = \left(D^{-1}, -\delta\right)}$.

\begin{lemma}
Define ${\phi \AtoB{\grpF \times \calM}{\calM}}$ as
\begin{equation}
    \phi\left(X, \xi\right) \coloneqq \left(\Pose{}{}D, \Vector{}{b}{} + \delta\right) \in \calM .
\end{equation}
Then, $\phi$ is a transitive right group action of $\grpF$ on $\calM$.
\end{lemma}

\begin{theorem}
Define the map ${\Lambda \AtoB{\calM \times \vecL}{\gothGrpF}}$ by
\begin{equation*}
    \Lambda\left(\xi, u\right) \coloneqq \left(\Lambda_1\left(\xi, u\right), \Lambda_2\left(\xi, u\right)\right),
\end{equation*}
where ${\Lambda_1 \AtoB{\calM \times \vecL}{\gothse_2(3)}}$, and ${\Lambda_2 \AtoB{\calM \times \vecL}{\R^{6}}}$ are given by
\begin{align}
    &\Lambda_1\left(\xi, u\right) \coloneqq \left(\mathbf{W} - \mathbf{B} + \mathbf{N}\right) + \Pose{}{}^{-1}\left(\mathbf{G} - \mathbf{N}\right)\Pose{}{} ,\\
    &\Lambda_2\left(\xi, u\right) \coloneqq \bm{\tau} .
\end{align}
Then, ${\Lambda}$ is a lift for the system in \cref{bins_bins_se23} with respect to the symmetry group ${\grpF \coloneqq \SE_2(3) \times \R^{6}}$.
\end{theorem}

In \cref{bins_iekf_linan}, it is shown that applying the \ac{eqf} filter design methodology to this symmetry leads to the Imperfect-\ac{iekf}~\cite{7523335, barrau:tel-01247723}.

\section{Two-frames group: ${\grpC:\; \SO(3) \ltimes (\R^{6} \oplus \R^{6})}$}\label{bins_tfg_iekf_sym_sec}

Modeling \acl{ins} directly on the extended special Euclidean group $\SE_2(3)$ fails to account for \ac{imu} biases. The recently published Two-Frames group structure~\cite{Barrau2022TheProblems}, given by ${\SO(3) \ltimes (\R^{6} \oplus \R^{6})}$, represents one approach to address this theoretical issue. 

Consider the system in \cref{bins_bins_se23}. Define the symmetry group ${\grpC \coloneqq \SO(3) \ltimes (\R^{6} \oplus \R^{6})}$. The $\oplus$ operation means that ${\SO(3)}$ acts on two different vector spaces of 6 dimensions each, defined with respect to two different frames of references, hence the name two frames group. Let ${X = \left(D,\delta\right) \in \grpC}$, with ${D = \left(A, \left(a, b\right)\right) \in \SE_2(3) \coloneqq \SO(3) \ltimes \R^{6}}$ such that ${A \in \SO(3),\; \left(a,b\right) \in \R^{6}}$, and $\delta \in \R^{6} \oplus \R^{6}$. Let, $* \AtoB{\SO(3) \times \R^{3N}}{\R^{3N}}$ be the rotation action introduced in~\cite{Barrau2022TheProblems}, such that ${\forall\, A \in \SO(3)}$ and ${x = \left(x_1,\cdots,x_N\right) \in \R^{3N}}$, ${A*x = \left(Ax_1,\cdots,Ax_N\right)}$.
Define the group product  ${XY = \left(D_XD_Y, \delta_X + A * \delta_Y\right)}$
The inverse element of the symmetry group writes ${X^{-1} = \left(D^{-1},-A^{\top}*\delta\right) \in \grpC}$.

\begin{lemma}
Define ${\phi \AtoB{\grpC \times \calM}{\calM}}$ as
\begin{equation}
    \phi\left(X, \xi\right) \coloneqq \left(\Pose{}{}D, A^{\top}*\left(\Vector{}{b}{} - \delta\right)\right) \in \calM .
\end{equation}
Then, $\phi$ is a transitive right group action of $\grpC$ on $\calM$.
\end{lemma}

\begin{theorem}
Define ${\Lambda_1 \AtoB{\calM \times \vecL} \se_2(3)}$ as
\begin{equation}
    \Lambda_1\left(\xi, u\right) \coloneqq \left(\mathbf{W} - \mathbf{B} + \mathbf{N}\right) + \Pose{}{}^{-1}\left(\mathbf{G} - \mathbf{N}\right)\Pose{}{} ,
\end{equation}
define ${\Lambda_2 \AtoB{\calM \times \vecL} \R^{6}}$ as
\begin{equation}
    \Lambda_2\left(\xi, u\right) \coloneqq \left(\Vector{}{b}{\bm{\omega}}^{\wedge}\left(\bm{\omega} - \Vector{}{b}{\bm{\omega}}\right) - \bm{\tau_{\omega}} ,\; \Vector{}{b}{\bm{a}}^{\wedge}\left(\bm{\omega} - \Vector{}{b}{\bm{\omega}}\right) - \bm{\tau_a}\right) .
\end{equation}
Then, the map ${\Lambda\left(\xi, u\right) = \left(\Lambda_1\left(\xi, u\right), \Lambda_2\left(\xi, u\right), \Lambda_2\left(\xi, u\right)\right)}$ is a lift for the system in \cref{bins_bins_se23} with respect to the symmetry group ${\grpC \coloneqq \SO(3) \ltimes (\R^{6} \oplus \R^{6})}$.
\end{theorem}

In \cref{bins_tfgiekf_linan}, it is shown that designing and \ac{eqf} based on this symmetry leads to the recently published \ac{tfgiekf}~\cite{Barrau2022TheProblems}, showing that there exists a strong relation between classical filter design and the underlying choice of symmetry in \acl{eqf} design.

\section{Tangent group ${\grpD:\; \SE_2(3) \ltimes \gothse_2(3)}$}\label{bins_tg_sec}

This section extends the results introduced in the previous chapter to second-order biased \acl{ins}, introducing a symmetry based on the \emph{Tangent group of $\SE_2(3)$} that elegantly models the \ac{imu} bias terms, respect the geometry of the biased \acl{ins}, and exploit its equivariance.

In order to exploit the geometrical properties of the system, we need to introduce an additional virtual velocity input $\velnu$ associated with an additional state variable that represents a velocity bias $\bnu$.

Define ${\xi = \left(\Pose{}{}, \Vector{}{b}{}\right) \in \calM \coloneqq \torSE_2(3) \times \R^{9}}$ to be the state space of the system. ${\Pose{}{} \in \torSE_2(3)}$ represents the extended pose, whereas ${\Vector{}{b}{} = \left(\bw, \ba, \bnu\right) \in \R^{9}}$ represents the \ac{imu} biases, and an additional virtual bias $\bnu$.
Let ${u = \left(\Vector{}{w}{}, \bm{\tau}\right) \in \vecL \subseteq \R^{18}}$ denote the system input, where ${\Vector{}{w}{} = \left(\Vector{}{\bm{\omega}}{}, \Vector{}{a}{}, \Vector{}{\nu}{}\right) \in \R^{9}}$ denotes the input given by the \ac{imu} readings, and an additional virtual input $\Vector{}{\nu}{}$. ${\bm{\tau} = \left(\bm{\tau_{\omega}}, \bm{\tau_{a}}, \bm{\tau_{\nu}}\right) \in \R^{9}}$ denotes the input for the \ac{imu} biases.

Define the matrices $\mathbf{W}, \mathbf{B}, \mathbf{N}, \mathbf{G}$ to be
\begin{equation*}
    \begin{tblr}{ll}
        \mathbf{W} = \begin{bmatrix}
        \angv^{\wedge} & \acc & \Vector{}{\nu}{}\\
        \mathbf{0}_{1\times 3} & 0 & 0\\
        \mathbf{0}_{1\times 3} & 0 & 0\\
        \end{bmatrix}, & 
        \mathbf{B} = \begin{bmatrix}
        \bw^{\wedge} & \ba & \bnu\\
        \mathbf{0}_{1\times 3} & 0 & 0\\
        \mathbf{0}_{1\times 3} & 0 & 0\\
        \end{bmatrix}, \\
        \mathbf{N} = \begin{bmatrix}
        \mathbf{0}_{3\times 3} & \mathbf{0}_{3\times 1} & \mathbf{0}_{3\times 1}\\
        \mathbf{0}_{1\times 3} & 0 & 1\\
        \mathbf{0}_{1\times 3} & 0 & 0\\
        \end{bmatrix}, &
        \mathbf{G}  = \begin{bmatrix}
        \mathbf{0}_{3\times 3} & \Vector{}{g}{} & \mathbf{0}_{3\times 1}\\
        \mathbf{0}_{1\times 3} & 0 & 0\\
        \mathbf{0}_{1\times 3} & 0 & 0\\
        \end{bmatrix}.
    \end{tblr}
\end{equation*}
Then, the system in \cref{bins_bins} may then be written in compact form as in \cref{bins_bins_se23}, hence
\begin{align*}
    &\dot{\Pose{}{}} = \Pose{}{}\left(\mathbf{W} - \mathbf{B} + \mathbf{N}\right) + \left(\mathbf{G} - \mathbf{N}\right)\Pose{}{} ,\\
    &\dotVector{}{b}{} =  \bm{\tau} .
\end{align*}

Define the symmetry group ${\grpD \coloneqq \SE_2(3) \ltimes \gothse_2(3)}$ as follows. Let ${X = \left(D, \delta\right) \in \grpD}$, where ${D \in \SE_2(3)}$, ${\delta \in \gothse_2(3)}$. Let ${X = \left(D_X, \delta_X\right)}$, ${Y = \left(D_Y, \delta_Y\right)}$ be two elements of the symmetry group, then the group product is given by the semi-direct product in \cref{math_sdp_sec}, and is written  ${XY = \left(D_XD_Y, \delta_X + \Adsym{D_X}{\delta_Y}\right)}$. The inverse of an element $X$ is given by ${X^{-1} = \left(D^{-1}, -\Adsym{D^{-1}}{\delta}\right)}$.

\begin{lemma}\label{bins_phi_tg_proof}
Define ${\phi \AtoB{\grpD \times \calM}{\calM}}$ as
\begin{equation}\label{bins_phi_tg}
    \phi\left(X, \xi\right) \coloneqq \left(\Pose{}{}D, \AdMsym{D^{-1}}\left(\Vector{}{b}{} - \delta^{\vee}\right)\right) \in \calM .
\end{equation}
Then, $\phi$ is a transitive right group action of $\grpD$ on $\calM$.
\end{lemma}
\begin{proof}
Let ${X, Y \in \grpD}$ and $\xi \in \calM$. Then, 
\begin{align*}
    &\phi\left(X,\phi\left(Y, \xi\right)\right) = \phi\left(X,\; \left(\Pose{}{}D_{Y},\; \AdMsym{D_{Y}^{-1}}(\Vector{}{b}{} - \delta_{Y}^{\vee}\right)\right)\\
    &= \left(\Pose{}{}D_{Y}D_{X},\; \AdMsym{D_{X}^{-1}}\left(\left(\AdMsym{D_{Y}^{-1}}(\Vector{}{b}{} - \delta_{Y}^{\vee}\right) - \delta_{X}^{\vee}\right)\right)\\
    &= \left(\Pose{}{}D_{Y}D_{X},\; \AdMsym{\left(D_{Y}D_{X}\right)^{-1}}\left(\Vector{}{b}{} - \left(\delta_{X}^{\vee} + \AdMsym{D_{X}}\delta_{Y}^{\vee}\right)\right)\right)\\
    &= \phi\left(YX, \xi\right) .
\end{align*}
This shows that $\phi$ is a valid right group action. Then, ${\forall \; \xi_1, \xi_2 \in \calM}$ we can always write the group element ${Z = \left(\Pose{}{1}^{-1}\Pose{}{2}, \Vector{}{b}{1} - \AdMsym{\left(\Pose{}{1}^{-1}\Pose{}{2}\right)}\Vector{}{b}{2}\right)}$, such that
\begin{align*}
    \phi\left(Z, \xi_1\right) &= \left(\Pose{}{1}\Pose{}{1}^{-1}\Pose{}{2},\, \AdMsym{\left(\Pose{}{1}^{-1}\Pose{}{2}\right)^{-1}}\left(\Vector{}{b}{1} - \Vector{}{b}{1} + \AdMsym{\left(\Pose{}{1}^{-1}\Pose{}{2}\right)}\left(\Vector{}{b}{2}\right)\right)\right)\\
    &= \left(\Pose{}{2},\; \Vector{}{b}{2}^{\wedge}\right)\\
    &= \xi_2 ,
\end{align*}
which demonstrates the transitive property of the group action.
\end{proof}
From here, we derive a compatible action of the symmetry group ${\grpD}$ on the input space ${\vecL}$ and derive the lift equivariant ${\Lambda}$ via constructive design as described in \cref{eq_chp} and in~\cite{Mahony2020EquivariantDesign, VanGoor2020EquivariantSpaces}.
\begin{lemma}
Define ${\psi \AtoB{\grpD \times \vecL}{\vecL}}$ as
\begin{equation}\label{bins_psi_tg}
    \psi\left(X, u\right) \coloneqq \left(\AdMsym{D^{-1}}\left(\Vector{}{w}{} - \delta^{\vee}\right) + \Omega\left(D^{-1}\right)^{\vee}, \AdMsym{D^{-1}}\Vector{}{\tau}{}\right) \in \vecL ,
\end{equation}
with $\Omega(\cdot)$ defined in \cref{math_maps_sec} and ${\Omega\left(Z\right) = Z\mathbf{N} - \mathbf{N}Z = Z\mathbf{N} - \mathbf{N}, \forall Z \in \SE_2(3)}$. Then, $\psi$ is a right group action of $\grpD$ on $\vecL$.
\begin{proof}
Let ${X, Y \in \grpD}$ and $u \in \vecL$. Then, 
\begin{align*}
    &\psi\left(X,\psi\left(Y, u\right)\right)\\
    &\quad= \psi\left(X,\left(\AdMsym{D_{Y}^{-1}}\left(\Vector{}{w}{} - \delta_{Y}^{\vee}\right) + \Omega\left(D_{Y}^{-1}\right)^{\vee}, \AdMsym{D_{Y}^{-1}}\Vector{}{\tau}{}\right)\right)\\
    &\quad= \left(
    \AdMsym{D_{X}^{-1}}\left(\AdMsym{D_{Y}^{-1}}\left(\Vector{}{w}{} - \delta_{Y}^{\vee}\right) + \Omega\left(D_{Y}^{-1}\right)^{\vee} - \delta_{X}^{\vee}\right) + \Omega\left(D_{X}^{-1}\right)^{\vee}, \AdMsym{D_{X}^{-1}D_{Y}^{-1}}\Vector{}{\tau}{}
    \right)\\
    &\quad= \left(
    \AdMsym{\left(D_{Y}D_{X}\right)^{-1}}\left(\Vector{}{w}{} - \delta_{Y}^{\vee} - \AdMsym{D_{Y}}\delta_{X}^{\vee}\right) + \AdMsym{D_{X}^{-1}}\Omega\left(D_{Y}^{-1}\right)^{\vee} +\Omega\left(D_{X}^{-1}\right)^{\vee}, \AdMsym{D_{X}^{-1}D_{Y}^{-1}}\Vector{}{\tau}{}
    \right).
\end{align*}
$\Omega\left(\cdot\right)$, is right invariant. Thus, it can be shown that
\begin{align*}
    \Omega\left(\left(D_{Y}D_{X}\right)^{-1}\right) &= \Omega\left(D_{X}^{-1}D_{Y}^{-1}\right)\\
    &= D_{X}^{-1}\Omega\left(D_{Y}^{-1}\right) + \Omega\left(D_{X}^{-1}\right)\\
    &= \Adsym{D_{X}^{-1}}{\Omega\left(D_{Y}^{-1}\right)} + \Omega\left(D_{X}^{-1}\right)
\end{align*}
Therefore substituting ${\Omega\left(\left(D_{Y}D_{X}\right)^{-1}\right)^{\vee}}$ to ${\AdMsym{D_{X}^{-1}}\Omega\left(D_{Y}^{-1}\right)^{\vee} +\Omega\left(D_{X}^{-1}\right)^{\vee}}$ yields
\begin{align*}
    &\psi\left(X,\psi\left(Y, u\right)\right)\\
    &\quad= \left(
    \AdMsym{\left(D_{Y}D_{X}\right)^{-1}}\left(\Vector{}{w}{} - \left(\delta_{Y}^{\vee} - \AdMsym{D_{Y}}\delta_{X}^{\vee}\right)\right) + \Omega\left(\left(D_{Y}D_{X}\right)^{-1}\right)^{\vee}, \AdMsym{D_{X}^{-1}D_{Y}^{-1}}\Vector{}{\tau}{}
    \right)\\
    &\quad= \psi\left(YX, u\right) ,
\end{align*}
proving that $\psi$ is a valid right group action.
\end{proof}
\end{lemma}
\begin{theorem}
    The system in \cref{bins_bins_se23} is equivariant under the actions $\phi$ in \cref{bins_phi_tg} and $\psi$ in \cref{bins_psi_tg}. That is, it satisfies \cref{eq_equi}
    \begin{equation*}
        f_{\psi_{X}\left(u\right)}(\xi) = \Phi_{X}f_u(\xi).
    \end{equation*}
\end{theorem}
\begin{proof}
    Let ${X \in \grpD}$, ${\xi \in \calM}$ and ${u \in \vecL}$, then the inverse of the group action $\phi$ in \cref{bins_phi_tg} is written
    \begin{equation*}
        \phi(X^{-1}, \xi) \coloneqq (\Pose{}{}D^{-1}, \AdMsym{D}\Vector{}{b}{} + \delta^{\vee}).
    \end{equation*}
    Therefore, the induced action in \cref{math_induced_action} is written as follows:
    \begin{align*}
        \Phi_{X}f_u(\xi) &= \td\phi_X \circ \left(\Pose{}{}D^{-1}\left(\mathbf{W} - D\mathbf{B}D^{-1} - \delta + \mathbf{N}\right) + \left(\mathbf{G} - \mathbf{N}\right)\Pose{}{}D^{-1}, \Vector{}{\tau}{}\right)\\
        &= \left(\Pose{}{}D^{-1}\left(\mathbf{W} - D\mathbf{B}D^{-1} - \delta + \mathbf{N}\right)D + \left(\mathbf{G} - \mathbf{N}\right)\Pose{}{}D^{-1}D, \AdMsym{D^{-1}}\Vector{}{\tau}{}\right)\\
        &= \left(\Pose{}{}\left(D^{-1}\left(\mathbf{W} - \delta\right)D - \mathbf{B} + D^{-1}\mathbf{N}\right) + \left(\mathbf{G} - \mathbf{N}\right)\Pose{}{}, \AdMsym{D^{-1}}\Vector{}{\tau}{}\right)\\
        &= \left(\Pose{}{}\left(D^{-1}\left(\mathbf{W} - \delta\right)D - \mathbf{B} + D^{-1}\mathbf{N} - \mathbf{N} + \mathbf{N}\right) + \left(\mathbf{G} - \mathbf{N}\right)\Pose{}{}, \AdMsym{D^{-1}}\Vector{}{\tau}{}\right)\\
        &= \left(\Pose{}{}\left(D^{-1}\left(\mathbf{W} - \delta\right)D + \Omega(D^{-1}) - \mathbf{B} + \mathbf{N}\right) + \left(\mathbf{G} - \mathbf{N}\right)\Pose{}{}, \AdMsym{D^{-1}}\Vector{}{\tau}{}\right)\\
        &= f_{\psi_{X}\left(u\right)}(\xi) .
    \end{align*}
\end{proof}
The existence of a transitive group action of the symmetry group $\grpD$ on the state space $\calM$ and the equivariance of the system guarantees the existence of an equivariant lift~\cite{Mahony2020EquivariantDesign}.
\begin{theorem}
Define the map ${\Lambda \AtoB{\calM \times \vecL}{\gothGrpD}}$ by
\begin{equation*}
    \Lambda\left(\xi, u\right) \coloneqq \left(\Lambda_1\left(\xi, u\right), \Lambda_2\left(\xi, u\right)\right),
\end{equation*}
where ${\Lambda_1 \AtoB{\calM \times \vecL}{\gothse_2(3)}}$, and ${\Lambda_2 \AtoB{\calM \times \vecL}{\gothse_2(3)}}$ are given by
\begin{align}
    &\Lambda_1\left(\xi, u\right) \coloneqq \left(\mathbf{W} - \mathbf{B} + \mathbf{N}\right) + \Pose{}{}^{-1}\left(\mathbf{G} - \mathbf{N}\right)\Pose{}{} ,\\
    &\Lambda_2\left(\xi, u\right) \coloneqq \adsym{\Vector{}{b}{}^{\wedge}}{\Lambda_1\left(\xi, u\right)} - \bm{\tau}^{\wedge}.
\end{align}
Then, ${\Lambda}$ is an equivaraint lift for the system in \cref{bins_bins_se23} with respect to the symmetry group ${\grpD \coloneqq \SE_2(3) \ltimes \gothse_2(3)}$.
\end{theorem}
\begin{proof}
    Let $\xi \in \calM$, $u \in \vecL$, then
    \begin{align*}
        \td\phi_{\xi}\left[\Lambda\left(\xi, u\right)\right] &= \td\phi_{\xi}\left[\Lambda_1\left(\xi, u\right),\,\Lambda_2\left(\xi, u\right)\right]\\
        &= \left(\Pose{}{}\Lambda_1\left(\xi, u\right),\, \adsym{\Vector{}{b}{}}{\Lambda_1\left(\xi, u\right)} - \Lambda_2\left(\xi, u\right)\right)\\
        &= \left(\Pose{}{}\left(\mathbf{W} - \mathbf{B} + \mathbf{N}\right) + \left(\mathbf{G} - \mathbf{N}\right)\Pose{}{},  \bm{\tau}^{\wedge}\right)\\
        &= f_u(\xi).
    \end{align*}
    let $X \in \grpD$. Then, to demonstrate the equivariance of the lift, we show that the condition in \cref{eq_lift_equi} holds. In particular, let us split the proof for ${\Lambda_1\left(\xi, u\right)}$ and ${\Lambda_2\left(\xi, u\right)}$
    \begin{align*}
        &\Adsym{D}{\Lambda_1\left(\phi_{X}\left(\xi\right),\psi_{X}\left(u\right)\right)}\\
        &\quad= D\left(D^{-1}\left(\mathbf{W} - \delta\right)D + \Omega(D^{-1}) - D^{-1}\left(\mathbf{B} - \delta\right)D + \mathbf{N}\right)D^{-1} + \Pose{}{}^{-1}\left(\mathbf{G} - \mathbf{N}\right)\Pose{}{}\\
        &\quad= \left(\mathbf{W} - \mathbf{B} + D\left(\Omega(D^{-1})+ \mathbf{N}\right)D^{-1}\right) + \Pose{}{}^{-1}\left(\mathbf{G} - \mathbf{N}\right)\Pose{}{}\\
        &\quad= \mathbf{W} - \mathbf{B} + D\left(D^{-1}\mathbf{N} - \mathbf{N} + \mathbf{N}\right)D^{-1} + \Pose{}{}^{-1}\left(\mathbf{G} - \mathbf{N}\right)\Pose{}{} .
    \end{align*}
    The right invariant property of $\mathbf{N}$ yields
    \begin{equation*}
        \Adsym{D}{\Lambda_1\left(\phi_{X}\left(\xi\right),\psi_{X}\left(u\right)\right)} = \mathbf{W} - \mathbf{B} + \mathbf{N} + \Pose{}{}^{-1}\left(\mathbf{G} - \mathbf{N}\right)\Pose{}{} = \Lambda_1\left(\xi, u\right) ,
    \end{equation*}
    proving the equivariance of ${\Lambda_1\left(\xi, u\right)}$. For what concerns ${\Lambda_2\left(\xi, u\right)}$ we have
    \begin{align*}
        &\Adsym{D}{\Lambda_2\left(\phi_{X}\left(\xi\right),\psi_{X}\left(u\right)\right)} + \adsym{\delta}{\Adsym{D}{\Lambda_1\left(\phi_{X}\left(\xi\right),\psi_{X}\left(u\right)\right)}}\\
        &\quad= \Adsym{D}{\adsym{\Adsym{D^{-1}}{\Vector{}{b}{}^{\wedge} - \delta}}{\Lambda_1\left(\phi_{X}\left(\xi\right),\psi_{X}\left(u\right)\right)} - \Adsym{D^{-1}}{\bm{\tau}^{\wedge}}} + \adsym{\delta}{\Lambda_1\left(\xi,u\right)}\\
        &\quad= \Adsym{D}{\adsym{\Adsym{D^{-1}}{\Vector{}{b}{}^{\wedge} - \delta}}{\Adsym{D^{-1}}{\Lambda_1\left(\xi,u\right)}} - \Adsym{D^{-1}}{\bm{\tau}^{\wedge}}} + \adsym{\delta}{\Lambda_1\left(\xi,u\right)}.
    \end{align*}
    Exploiting the commutative property of the adjoints in \cref{math_adj_comm} yields
    \begin{align*}
        &\Adsym{D}{\Lambda_2\left(\phi_{X}\left(\xi\right),\psi_{X}\left(u\right)\right)} + \adsym{\delta}{\Lambda_1\left(\phi_{X}\left(\xi\right),\psi_{X}\left(u\right)\right)}\\
        &\quad= \Adsym{D}{\Adsym{D^{-1}}{\adsym{\left(\Vector{}{b}{}^{\wedge} - \delta\right)}{\Lambda_1\left(\xi,u\right)}}} - \bm{\tau}^{\wedge} + \adsym{\delta}{\Lambda_1\left(\xi,u\right)}\\
        &\quad= \adsym{\left(\Vector{}{b}{}^{\wedge} - \delta\right)}{\Lambda_1\left(\xi,u\right)} - \bm{\tau}^{\wedge} + \adsym{\delta}{\Lambda_1\left(\xi,u\right)}\\
        &\quad= \adsym{\Vector{}{b}{}^{\wedge}}{\Lambda_1\left(\xi,u\right)} - \bm{\tau}^{\wedge} - \adsym{\delta^{\vee}}{\Lambda_1\left(\xi,u\right)} + \adsym{\delta}{\Lambda_1\left(\xi,u\right)}\\
        &\quad= \adsym{\Vector{}{b}{}^{\wedge}}{\Lambda_1\left(\xi,u\right)} - \bm{\tau}^{\wedge}\\
        &\quad= \Lambda_2\left(\xi,u\right),
    \end{align*}
    proving the equivariance of ${\Lambda_2\left(\xi, u\right)}$, and hence the equivariance of the lift $\Lambda\left(\xi, u\right)$.
\end{proof}

\begin{remark}
    The approach presented in this section elegantly exploits the equivariance of the system defined in \cref{bins_bins_se23}; however, it requires the introduction of an additional state $\bnu$ in order to apply the tangent symmetry of $\SE_2(3)$.
    This additional state variable $\bnu$ represents the bias of a virtual velocity input $\velnu$.
    Although it will be clear in the later sections that exploiting the tangent symmetry of the \acl{ins} is beneficial, it is of interest to consider alternative symmetries that try to preserve the semi-direct product structure of the symmetry group, keeping a minimal parametrization of the state space; thus, without requiring additional state variables. 
\end{remark}

\section{Direct position group ${\grpA:\; \HG(3) \ltimes \gothhg(3) \times \R^{3}}$}\label{bins_dp_sec}

This section considers alternative symmetries, trying to preserve the semi-direct product structure of the symmetry group while keeping a minimal parametrization of the state space. Specifically, we tackle the following research question: \emph{``Can we find a symmetry for second-order \aclp{ins} that exploits the semi-direct product structure of the tangent symmetries without requiring additional state variables?''}. The most intuitive approach is to exploit the semi-direct product structure symmetry for the first-order kinematics of rotation and velocity while treating the position linearly. To this end, we introduce a new group termed $\HG(3)$, the \emph{homogeneous Galilean} group, which corresponds to extended pose transformations $\SE_2(3)$ where the spatial translation is zero. The homogeneous Galilean group is isomorphic to $\SE(3)$ in structure; however, since $\SE(3)$ is synonymous with pose transformations, we use the $\HG(3)$ notation to avoid confusion. 

Consider the second-order \acl{ins} in \cref{bins_bins}, in particular, rewrite \cref{bins_pdot} as ${\dotVector{}{p}{} = \Rot{}{}\bm{\nu} + \Vector{}{v}{}}$, where ${\bm{\nu}}$ is a virtual input.

Define ${\xi = \left(\Pose{}{}, \Vector{}{b}{}, \Vector{}{p}{}\right) \in \calM \coloneqq \calHG(3) \times \R^{6} \times \R^{3}}$ to be the state space of the system, where ${\Pose{}{} = \left(\Rot{}{}, \Vector{}{v}{}\right) \in \calHG(3)}$ includes the orientation and the velocity of the rigid body.
Let ${u = \left(\Vector{}{w}{}, \bm{\tau}, \bm{\nu}\right) \in \vecL \subseteq \R^{15}}$ denote the system input, with $\Vector{}{w}{} = (\Vector{}{\omega}{}, \Vector{}{a}{})$, and $\Vector{}{\tau}{} = (\Vector{}{\tau}{\omega},\Vector{}{\tau}{a})$.

Define the matrices
\begin{equation*}
    \begin{tblr}{lll}
        \mathbf{W} = \begin{bmatrix}
        \angv^{\wedge} & \acc\\
        \mathbf{0}_{1\times 3} & 0\\
        \end{bmatrix}, & 
        \mathbf{B} = \begin{bmatrix}
        \bw^{\wedge} & \ba\\
        \mathbf{0}_{1\times 3} & 0\\
        \end{bmatrix},  &
        \mathbf{G}  = \begin{bmatrix}
        \mathbf{0}_{3\times 3} & \Vector{}{g}{}\\
        \mathbf{0}_{1\times 3} & 0\\
        \end{bmatrix}.
    \end{tblr}
\end{equation*}
Then, the system in \cref{bins_bins} may then be written as
\begin{subequations}\label{bins_bins_hg}
    \begin{align}
        &\dot{\Pose{}{}} = \Pose{}{}\left(\mathbf{W} - \mathbf{B}\right) + \mathbf{G}\Pose{}{} ,\\
        &\dotVector{}{b}{} =  \bm{\tau} ,\\
        &\dotVector{}{p}{} = \Rot{}{}\bm{\nu} + \Vector{}{v}{} .\label{bins_bins_hg_p}
    \end{align}
\end{subequations}

Define the symmetry group ${\grpA \coloneqq \HG(3) \ltimes \gothse(3) \times \R^{3}}$, and let ${X = \left(B,\beta,c\right) \in \grpA}$ with ${B = \left(A, a\right) \in \HG(3)}$ such that ${A \in \SO(3),\; a \in \R^{3}}$. That is the symmetry that acts on rotation and velocity only with the semi-direct product induced by the $\SE_2(3)$ geometry. Let $X = \left(B_X,\beta_X,c_X\right), Y = \left(B_Y,\beta_Y,c_Y\right) \in \grpA$, the group product is written  ${XY = \left(B_XB_Y, \beta_X + \Adsym{B_X}{\beta_Y}, c_X + c_Y\right)}$. The inverse of an element $X \in \grpA$ is given by ${X^{-1} = \left(B^{-1},-\Adsym{B^{-1}}{\beta},-c\right) \in \grpA}$. 

\begin{lemma}
Define ${\phi \AtoB{\grpA \times \calM}{\calM}}$ as
\begin{equation}\label{bins_phi_hg}
    \phi\left(X, \xi\right) \coloneqq \left(\Pose{}{}B, \AdMsym{B^{-1}}\left(\Vector{}{b}{} - \beta^{\vee}\right), \Vector{}{p}{} + c\right) \in \calM .
\end{equation}
Then, $\phi$ is a transitive right group action of $\grpA$ on $\calM$.
\end{lemma}
We derive a compatible action of the symmetry group ${\grpA}$ on the input space ${\vecL}$.
\begin{lemma}
Define ${\psi \AtoB{\grpA \times \vecL}{\vecL}}$ as
\begin{equation}\label{bins_psi_hg}
    \psi\left(X, u\right) \coloneqq \left(\AdMsym{B^{-1}}\left(\Vector{}{w}{} - \beta^{\vee}\right), A^{\top}\left(\bm{\nu} - a\right), \AdMsym{B^{-1}}\Vector{}{\tau}{}\right) \in \vecL .
\end{equation}
Then, $\psi$ is a right group action of $\grpA$ on $\vecL$.
\end{lemma}
The system in \cref{bins_bins_hg} is equivariant under the actions $\phi$ in \cref{bins_phi_hg} and $\psi$ in \cref{bins_psi_hg}. Therefore, the existence of an equivariant lift is guaranteed. Proofs of transitivity of the action $\phi$ and equivariance of the system follow those in \cref{bins_tg_sec}.

\begin{theorem}
Define the map ${\Lambda \AtoB{\calM \times \vecL}{\gothGrpA}}$ by
\begin{align*}
    \Lambda\left(\xi, u\right) &\coloneqq \left(\Lambda_1\left(\xi, u\right), \Lambda_2\left(\xi, u\right), \Lambda_3\left(\xi, u\right)\right),
\end{align*}
where ${\Lambda_1 \AtoB{\calM \times \vecL} \gothhg(3)}$, ${\Lambda_2 \AtoB{\calM \times \vecL} \se(3)}$, and ${\Lambda_3 \AtoB{\calM \times \vecL} \R^{3}}$ are given by
\begin{align}
    \Lambda_1\left(\xi, u\right) &\coloneqq \left(\mathbf{W} - \mathbf{B}\right) + \Pose{}{}^{-1}\mathbf{G}\Pose{}{} , \\
    \Lambda_2\left(\xi, u\right) &\coloneqq \adsym{\Vector{}{b}{}^{\wedge}}{\Lambda_1\left(\xi, u\right)} - \bm{\tau}^{\wedge} , \\
    \Lambda_3\left(\xi, u\right) &\coloneqq \Rot{}{}\bm{\nu} + v .
\end{align}
Then, the ${\Lambda}$ is an equivariant lift for the system in \cref{bins_bins_hg} with respect to the symmetry group ${\grpA \coloneqq \HG(3) \ltimes \gothhg(3) \times \R^{3}}$.
\end{theorem}

The symmetry proposed in this section exploits the semi-direct product structure of $\grpA$ when acting on the rotation and velocity and keeps a linear structure for the position. Although this allows us to avoid the over-parametrization of the state with the additional velocity bias state, this is not fully satisfying. This construction comes at the cost of separating the position state from the geometric structure of $\SE_2(3)$ and modeling it as a direct product linear space.

\section{Semi-direct bias group: ${\grpB:\; \SE_2(3) \ltimes \gothse(3)}$}\label{bins_sdb_sec}

In the previous section, a symmetry overcoming the need for over-parametrization of the state with additional state variables is presented. This comes at the cost of separating the position state from the geometric structure of $\SE_2(3)$ and modeling it as a direct product linear space. The natural question that arises is \emph{``Can we find a symmetry for second-order \aclp{ins} that exploits the semi-direct product structure of the tangent symmetries without requiring additional state variables and without separating the position state from the geometric structure of $\SE_2(3)$?''}. In this section, we tackle exactly this question, and we propose a symmetry that maintains a minimal state representation, thus not requiring the introduction of an additional velocity bias state, while simultaneously keeping the position state within the geometric structure given by ${\SE_2(3)}$.

Consider the system in \cref{bins_bins_se23} and the respective system input. Define the symmetry group ${\grpB \coloneqq \SE_2(3) \ltimes \gothse(3)}$ with group product ${XY = \left(D_XD_Y, \delta_X + \Adsym{D_X}{\delta_Y}\right)}$ for $X = \left(D_X,\delta_X\right), Y = \left(D_Y,\delta_Y\right) \in \grpB$. Here, for ${X = \left(D,\delta\right) \in \grpB}$ one has ${D = \left(A, a, b\right) = \left(B, b\right)\in \SE_2(3)}$ such that ${A \in \SO(3),\; a,b \in \R^{3}}$, and ${B = \left(A, a\right) \in \HG(3)}$. 
The element $D \in \SE_2(3)$ in its matrix representation is written
\begin{equation*}
    D = \begin{bNiceArray}{w{c}{0.75cm}w{c}{0.45cm}:w{c}{0.45cm}}[margin]
        A & a & b\\
        \mathbf{0}_{1\times 3} & 1 & 0\Bstrut\\
        \hdottedline
        \mathbf{0}_{1\times 3} & 0 & 1\Tstrut
    \end{bNiceArray} =
    \begin{bNiceArray}{w{c}{0.75cm}w{c}{0.45cm}:w{c}{0.45cm}}[margin]
        \Block{2-2}{B} & & b\\
        & & 0\Bstrut\\
        \hdottedline
        \mathbf{0}_{1\times 3} & 0 & 1\Tstrut
    \end{bNiceArray} \in \SE_2(3) .
\end{equation*}
The inverse element is written ${{X^{-1} = \left(D^{-1},-\Adsym{B^{-1}}{\delta}\right) \in \grpB}}$.

\begin{lemma}
Define ${\phi \AtoB{\grpB \times \calM}{\calM}}$ as
\begin{equation}\label{bins_phi_sdb}
    \phi\left(X, \xi\right) \coloneqq \left(\Pose{}{}D, \AdMsym{B^{-1}}\left(\Vector{}{b}{} - \delta^{\vee}\right)\right) \in \calM .
\end{equation}
Then, $\phi$ is a transitive right group action of $\grpB$ on $\calM$.
\end{lemma}
The proof of transitivity of $\phi$ follows that in \cref{bins_tg_sec}.

\begin{theorem}
Define ${\Lambda_1 \AtoB{\calM \times \vecL} \se_2(3)}$ as
\begin{equation}
    \Lambda_1\left(\xi, u\right) \coloneqq \left(\mathbf{W} - \mathbf{B} + \mathbf{N}\right) + \Pose{}{}^{-1}\left(\mathbf{G} - \mathbf{N}\right)\Pose{}{} ,
\end{equation}
define ${\Lambda_2 \AtoB{\calM \times \vecL} \se(3)}$ as
\begin{equation}
    \Lambda_2\left(\xi, u\right) \coloneqq \adsym{\Vector{}{b}{}^{\wedge}}{\Pi\left(\Lambda_1\left(\xi, u\right)\right)} - \bm{\tau}^{\wedge} ,
\end{equation}
with ${\Pi(\cdot)}$ defined in \cref{math_maps_sec}. Then, the map ${\Lambda\left(\xi, u\right) = \left(\Lambda_1\left(\xi, u\right), \Lambda_2\left(\xi, u\right)\right)}$ is a lift for the system in \cref{bins_bins_se23} with respect to the symmetry group ${\grpB \coloneqq \SE_2(3) \ltimes \gothse(3)}$.
\end{theorem}

The symmetry proposed in this section is a variation of the symmetry defined in \cref{bins_tg_sec} that does not require over-parametrization of the state and additional state variables. In the next chapters, we compare the symmetries presented in this and in the previous sections in terms of performance and linearization error when exploited for \acl{eqf} design.

\section{Linearization error analysis and equivariant filter design}\label{bins_eqf_sec}
In the previous sections, we discussed and introduced different symmetries of \aclp{ins}. It is clear at this point that the choice of symmetry plays an important role in filter performance; however, a recurrent question in this dissertation is: \emph{``Which symmetry works best in the context of filter design?''}. In \cref{bas_chp}, we answered this question via experimental results. Here, we identify the filter's linearization error as an indicator of filter performance.   
In particular, in the upcoming section, we analyze the linearization error for each of the previously introduced symmetries in the context of equivariant filter design, showing how the different symmetries yield different error linearization. This analysis indirectly underscores an important result -- every modern filter type (\ac{mekf}, Imperfect-\ac{iekf}, \ac{tfgiekf}) is derived as an \acl{eqf} for a specific choice of symmetry. 

\subsection{\ac{mekf}}\label{bins_mekf_linan}
Recall the state space given by ${\calM \coloneqq \mathcal{SO}(3)\times \R^{12}}$ defined in \cref{bins_mekf_sym_sec}. Define ${\xi \coloneqq \left(\Rot{}{},\Vector{}{v}{},\Vector{}{p}{},\Vector{}{b}{\bm{\omega}}, \Vector{}{b}{a}\right)\in\calM}$.
Choose the state origin to be the state space identity, hence ${\xizero \coloneqq \left(\eye_3, \Vector{}{0}{3\times1},\Vector{}{0}{3\times1},\Vector{}{0}{3\times1},\Vector{}{0}{3\times1}\right)\in\calM}$.
The system's velocity input is given by ${u \coloneqq \left(\Vector{}{\omega}{}, \Vector{}{a}{},\Vector{}{\tau}{\bm{\omega}},\Vector{}{\tau}{a}\right)}$.

The symmetry group of \ac{mekf} is given by ${\grpE \coloneqq \SO(3)\times \R^{12}}$ defined in \cref{bins_mekf_sym_sec}. Define the state of the filter as the element of the symmetry group $\hat{X} = (\hat{A},\hat{a},\hat{b},\hat{\alpha},\hat{\beta})\in\grpE$, where $\hat{A}\in\SO(3)$ and $\hat{a},\hat{b},\hat{\alpha},\hat{\beta}\in\R^3$.
The state estimate is given by
\begin{equation}
        \hat{\xi} \coloneqq \phi(\hat{X},\xizero) = (\hat{A} ,\hat{a} ,\hat{b}, \hat{\alpha}, \hat{\beta}) = \left(\hatRot{}{},\hatVector{}{v}{},\hatVector{}{p}{},\hatVector{}{b}{\bm{\omega}}, \hatVector{}{b}{a}\right). 
\end{equation}
The state error is defined as 
\begin{equation}\label{bins_mekf_err}
    \begin{split}
        e \coloneqq \phi(\hat{X}^{-1},\xi) &= \left(\Rot{}{}\hat{A}^\top, \Vector{}{v}{} - \hat{a}, \Vector{}{p}{} - \hat{b}, \Vector{}{b}{\omega} - \hat{\alpha}, \Vector{}{b}{a} - \hat{\beta}\right)\\
        &= \left(\Rot{}{}\hatRot{}{}^\top, \Vector{}{v}{} - \hatVector{}{v}{}, \Vector{}{p}{} - \hatVector{}{p}{}, \Vector{}{b}{\omega} - \hatVector{}{b}{\omega}, \Vector{}{b}{a} - \hatVector{}{b}{a}\right).
    \end{split}
\end{equation}

\paragraph{Linearization error analysis.}
The error dynamics for each state is given by 
\begin{align*}
    \dot{e}_R &= \ddt (\Rot{}{} \hatRot{}{}^\top)\\
        &= \Rot{}{}(\Vector{}{\omega}{} - \Vector{}{b}{\bm{\omega}})^\wedge \hatRot{}{}^\top - \Rot{}{}\hatRot{}{}^\top\hatRot{}{}(\Vector{}{\omega}{} - \hatVector{}{b}{\bm{\omega}})^\wedge \hatRot{}{}^\top\\
        &=\Rot{}{}(\Vector{}{\omega}{} - \Vector{}{b}{\bm{\omega}}-\Vector{}{\omega}{}+\hatVector{}{b}{\bm{\omega}})^\wedge \hatRot{}{}^\top\\
        &=-e_R \hatRot{}{} e_{b_\omega}^\wedge \hatRot{}{}^\top\\
        &=-e_R \left(\hatRot{}{} e_{b_\omega}\right)^{\wedge},\\
    \dot{e}_v &= \ddt (\Vector{}{v}{} - \hatVector{}{v}{}) = \dot{\Vector{}{v}{}} - \dot{\hatVector{}{v}{}} \\
        &=\Rot{}{}(\Vector{}{a}{}-\Vector{}{b}{a})^\wedge + \Vector{}{g}{} - \hatRot{}{}(\Vector{}{a}{} - \hatVector{}{b}{a})^\wedge - \Vector{}{g}{}\\
        & = \Rot{}{}(\Vector{}{a}{}-\Vector{}{b}{a})^\wedge - \hatRot{}{}(\Vector{}{a}{} - \hatVector{}{b}{a})^\wedge\\
        & = e_R\hatRot{}{}(\Vector{}{a}{} - \Vector{}{b}{a}) - \hatRot{}{}(\Vector{}{a}{} -\hatVector{}{b}{a}),\\
    \dot{e}_p &= \ddt (\Vector{}{p}{} - \hatVector{}{p}{}) = \dot{\Vector{}{p}{}} - \dot{\hatVector{}{p}{}}\\
        &=\Vector{}{v}{} - \hatVector{}{v}{} = e_v,\\
    \dot{e}_{b} &= \ddt(\Vector{}{b}{}-\hatVector{}{b}{}) = \mathbf{0}.
\end{align*}
Choosing normal coordinates, the local coordinate chart ${\varepsilon = \log_{\grpE} \circ\; \phi_{\xizero}^{-1}(e)}$ for each state is written
\begin{align*}
    &\varepsilon_R  \coloneqq  \log_{\SO(3)}(e_R)^\vee \in\R^3,\\
    &\varepsilon_{v,p,{b_\omega},{b_a}} \coloneqq  e_{v,p,{b_\omega},{b_a}} \in\R^3.
\end{align*}
The linearization of the rotation error $\dot{e}_R= \tD \exp(\varepsilon_R^\wedge)[\dot{\varepsilon}_R^{\wedge}]$ is given by 
\begin{align*}
    e_R\frac{\eye-\exp(-\adMsym{\varepsilon_R})}{\adMsym{\varepsilon_R}}\dot{\varepsilon}_R^{\wedge} &= -e_R \left(\hatRot{}{} \varepsilon_{b_\omega}\right)^{\wedge}\\
    (\eye + \mathcal{O}(\varepsilon_R^\wedge))\dot{\varepsilon}_R^{\wedge} &= \left(\hatRot{}{} \varepsilon_{b_\omega}\right)^{\wedge}\\
    \dot{\varepsilon}_R &= \hatRot{}{}\varepsilon_{b_\omega} + O({\varepsilon_R}^2).
\end{align*}
The linearization of the velocity error $\dot{e}_v = \dot{\varepsilon}_v$ is given by 
\begin{align*}
    \dot{\varepsilon}_v & = e_R\hatRot{}{}(\Vector{}{a}{} - \Vector{}{b}{a}) - \hatRot{}{}(\Vector{}{a}{} -\hatVector{}{b}{a})\\
    &=(\eye+\varepsilon_R^\wedge + O({\varepsilon_R}^2))\hatRot{}{}(\Vector{}{a}{} - \Vector{}{b}{a}) - \hatRot{}{}(\Vector{}{a}{} - \hatVector{}{b}{a})\\
    &=\hatRot{}{}(\hatVector{}{b}{a} - \Vector{}{b}{a}) + \varepsilon_R^\wedge\hatRot{}{}(\Vector{}{a}{}-\Vector{}{b}{a}) + O({\varepsilon_R}^2)\\
    &=-\hatRot{}{}\varepsilon_{b_a} + \varepsilon_R^\wedge\hatRot{}{}(\Vector{}{a}{}-\varepsilon_{b_a}-\hatVector{}{b}{a}) + \mathcal{O}(\varepsilon^2)\\
    &=-\hatRot{}{}\varepsilon_{b_a} + \varepsilon_R^\wedge\hatRot{}{}(\Vector{}{a}{}-\hatVector{}{b}{a}) - \varepsilon_R^\wedge\hatRot{}{}\varepsilon_{b_a} + O({\varepsilon_R}^2)\\
    &= -(\hatRot{}{}(\Vector{}{a}{}-\hatVector{}{b}{a}))^\wedge\varepsilon_R -\hatRot{}{}\varepsilon_{b_a}  + \mathcal{O}(\varepsilon^2).
\end{align*}
The linearization of the position error $\dot{e}_p = \dot{\varepsilon}_p$ is given by
\begin{align*}
    \dot{\varepsilon}_p = \varepsilon_v.
\end{align*}
Finally, the linearization of the bias error $\dot{e}_{b} = \dot{\varepsilon}_{b}$ is given by $\dot{\varepsilon}_{b}=\Vector{}{0}{}$.

\paragraph{Filter state matrix.}
From the linearization error analysis above, it is trivial to see that the linearized error state matrix ${\mathbf{A}_{t}^{0} \st \dot{\varepsilon} \simeq \mathbf{A}_{t}^{0}\varepsilon}$ is written
\begin{equation}
    \mathbf{A}_{t}^{0} = \begin{bNiceArray}{ccc:cc}[margin]
        \Vector{}{0}{3 \times 3} & \Vector{}{0}{3 \times 3} & \Vector{}{0}{3 \times 3} & -\hatRot{}{} & \Vector{}{0}{3 \times 3}\\
        -\left(\hatRot{}{}\left(\Vector{}{a}{} - \hatVector{}{b}{a}\right)\right)^{\wedge} & \Vector{}{0}{3 \times 3} & \Vector{}{0}{3 \times 3} & \Vector{}{0}{3 \times 3} & -\hatRot{}{}\\
        \Vector{}{0}{3 \times 3} & \eye_3 & \Vector{}{0}{3 \times 3} & \Vector{}{0}{3 \times 3} & \Vector{}{0}{3 \times 3}\Bstrut\\
        \hdottedline
        \Block{1-3}{\mathbf{0}_{6 \times 9}} & & & \Block{1-2}{\Vector{}{0}{6 \times 6}}\Tstrut
    \end{bNiceArray} \in \R^{15 \times 15}. \label{bins_At0_mekf}
\end{equation}
It is straightforward to verify that the derived \acl{eqf} is equivalent to the well-known \ac{mekf}, and the \ac{eqf} state matrix in \cref{bins_At0_mekf} corresponds directly to the state matrix of the \ac{mekf}~\cite[Sec. 7]{Sola2017QuaternionFilter}

\subsection{Imperfect-\ac{iekf}}\label{bins_iekf_linan}
Recall the state space defined by ${\calM \coloneqq \mathcal{SE}_2(3)\times\R^6}$ in \cref{bins_iekf_sym_sec}. Define ${\xi \coloneqq (\Pose{}{},\Vector{}{b}{})\in\calM}$. One has ${\Pose{}{} = (\Rot{}{},\Vector{}{v}{},\Vector{}{p}{})\in\mathcal{SE}_2(3)}$ and ${\Vector{}{b}{}=(\Vector{}{b}{\bm{\omega}},\Vector{}{b}{a}) \in \R^6}$.
Choose the state origin to be ${\xizero = \left(\eye_5,\Vector{}{0}{6\times1}\right)\in\calM}$.
The system's velocity input is given by ${u \coloneqq \left(\Vector{}{\omega}{}, \Vector{}{a}{},\Vector{}{\tau}{\bm{\omega}},\Vector{}{\tau}{a}\right) = \left(\Vector{}{w}{}, \Vector{}{\tau}{}\right)}$.

The symmetry group of Imperfect-\ac{iekf} is given by ${\grpF \coloneqq \SE_2(3)\times\R^6}$. Define the filter state ${\hat{X} = (\hat{D},\hat{\delta})\in\grpF}$ with ${\hat{D} = (\hat{A},\hat{a},\hat{b})\in\SE_2(3)}$ and ${\hat{\delta} = (\hat{\delta_\omega}, \hat{\gamma_a})\in\R^6}$.
The state estimate is given by 
\begin{align}
    \hat{\xi} \coloneqq \phi(\hat{X},\xizero) = (\hat{A}, \hat{\delta}) = (\hatPose{}{}, \hatVector{}{b}{}). 
\end{align}
The state error is defined as 
\begin{equation}\label{bins_iekf_err}
\begin{split}
    e \coloneqq \phi(\hat{X}^{-1},\xi) &= (\Pose{}{}\hat{D}^{-1},\Vector{}{b}{}-\hat{\delta})\\
    &= (\Pose{}{}\hatPose{}{}^{-1},\Vector{}{b}{}-\hatVector{}{b}{}).
\end{split} 
\end{equation}

\paragraph{Linearization error analysis.}
The error dynamics given by 
\begin{align*}
    \dot{e}_R &= -e_R(\hatRot{}{}e_{b_\omega})^\wedge,\\
    \dot{e}_v &= \ddt(-\Rot{}{}\hatRot{}{}^\top\hatVector{}{v}{}+\Vector{}{v}{})\\
        &=-\dot{e}_R\hatVector{}{v}{} - e_R\dot{\hatVector{}{v}{}}+\dot{\Vector{}{v}{}}\\
        &=e_R(\hatRot{}{}e_{b_\omega})^\wedge \hatVector{}{v}{}-e_R \hatRot{}{}(\Vector{}{a}{}-\hatVector{}{b}{a}) -e_R \Vector{}{g}{} + \Rot{}{}(\Vector{}{a}{}-\Vector{}{b}{a})+\Vector{}{g}{}\\
        &=e_R(\hatRot{}{}e_{b_\omega})^\wedge \hatVector{}{v}{}-e_R\hatRot{}{}(\Vector{}{a}{}-\hatVector{}{b}{a}) -e_R \Vector{}{g}{}\\
        &\quad + e_R\hatRot{}{}(\Vector{}{a}{}-e_{b_a}+\hatVector{}{b}{a})+\Vector{}{g}{},\\
    \dot{e}_p &= \ddt (-\Rot{}{}\hatRot{}{}^\top\hatVector{}{p}{}+\Vector{}{p}{})\\
        &= -\dot{e}_R\hatVector{}{p}{} - e_R\dot{\hatVector{}{p}{}}+\dot{\Vector{}{p}{}}\\
        &= e_R(\hatRot{}{}e_{b_\omega})^\wedge \hatVector{}{p}{} - e_R\hatVector{}{v}{} + \Vector{}{v}{}\\
        &= e_R(\hatRot{}{}e_{b_\omega})^\wedge \hatVector{}{p}{} + e_v,\\
    \dot{e}_b &= \Vector{}{0}{},
\end{align*}
where the attitude error is derived as for the \ac{mekf}.
Choosing normal coordinates, the local coordinate chart ${\varepsilon = \log_{\grpG_\mathbf{ES}} \circ\; \phi_{\xizero}^{-1}(e)}$ for each state is written
\begin{align*}
    &\varepsilon_T  \coloneqq   \log_{\SE_2(3)}(\phi_{\xizero}^{-1}(e_T))^\vee =  \log_{\SE_2(3)}(e_T)^\vee \in\R^9,\\
    &\varepsilon_{b_\omega,b_a} \coloneqq  e_{b_\omega, b_a} \in\R^3.
\end{align*}
The linearization of the rotation error $\dot{e}_R$ follows that of the \ac{mekf} and yields
\begin{align*}
    \dot{\varepsilon}_R &= \hatRot{}{}\varepsilon_{b_\omega} + O({\varepsilon_R}^2).
\end{align*}
The linearization of the velocity error $\dot{e}_v = \dot{\varepsilon}_v+\mathcal{O}(\varepsilon^2)$ is given by 
\begin{align*}
    \dot{\varepsilon}_v &= -(\eye+\varepsilon_R^\wedge + O({\varepsilon_R}^2))(\hatRot{}{}\varepsilon_{b_\omega})^\wedge\hatVector{}{v}{}\\
    &\quad-(\eye+\varepsilon_R^\wedge + O({\varepsilon_R}^2))\hatRot{}{}(\Vector{}{a}{}-\hatVector{}{b}{a})\\
    &\quad -(\eye+\varepsilon_R^\wedge + O({\varepsilon_R}^2)) \Vector{}{g}{}\\
    &\quad+ (\eye+\varepsilon_R^\wedge + O({\varepsilon_R}^2))\hatRot{}{}(\Vector{}{a}{}-\varepsilon_{b_a}+\hatVector{}{b}{a})\\
    &\quad+\Vector{}{g}{}+\mathcal{O}(\varepsilon^2)\\
    &= -\hatVector{}{v}{}^\wedge \hatRot{}{} \varepsilon_{b_\omega} - \hatRot{}{}\varepsilon_{b_a}+\Vector{}{g}{}^\wedge\varepsilon_R + \mathcal{O}(\varepsilon^2).
\end{align*}
The linearization of the position error $\dot{e}_p = \dot{\varepsilon}_p+\mathcal{O}(\varepsilon^2)$ is given by 
\begin{align*}
    \dot{\varepsilon}_p &= (\eye+\varepsilon_R^\wedge + O({\varepsilon_R}^2))(\hatRot{}{}e_{b_\omega})^\wedge\hatVector{}{p}{} + \varepsilon_v+\mathcal{O}(\varepsilon^2)\\
    &=\varepsilon_v - \hatVector{}{p}{}^\wedge\varepsilon_{b_\omega}+\mathcal{O}(\varepsilon^2).
\end{align*}
To conclude, the linearization of the bias error $\dot{e}_{b} = \dot{\varepsilon}_{b}$ is given by $\dot{\varepsilon}_{b}=\Vector{}{0}{}$.

\paragraph{Filter state matrix.}
The linearized error state matrix ${\mathbf{A}_{t}^{0} \st \dot{\varepsilon} \simeq \mathbf{A}_{t}^{0}\varepsilon}$ yields
\begin{equation}
    \mathbf{A}_{t}^{0} = \begin{bNiceArray}{ccc:cc}[margin]
        \Vector{}{0}{3 \times 3} & \Vector{}{0}{3 \times 3} & \Vector{}{0}{3 \times 3} & -\hatRot{}{} & \Vector{}{0}{3 \times 3}\\
        \Vector{}{g}{}^{\wedge} & \Vector{}{0}{3 \times 3} & \Vector{}{0}{3 \times 3} & -\hatVector{}{v}{}^{\wedge}\hatRot{}{} & -\hatRot{}{}\\
        \Vector{}{0}{3 \times 3} & \eye_3 & \Vector{}{0}{3 \times 3} & -\hatVector{}{p}{}^{\wedge}\hatRot{}{} & \Vector{}{0}{3 \times 3}\Bstrut\\
        \hdottedline
        \Block{1-3}{\mathbf{0}_{6 \times 9}} & & & \Block{1-2}{\Vector{}{0}{6 \times 6}}\Tstrut
    \end{bNiceArray} \in \R^{15 \times 15}. \label{bins_At0_iekf}
\end{equation}
It is trivial to verify that the state matrix in \cref{bins_At0_iekf}, derived according to \acl{eqf} design principles, directly corresponds to the state matrix in the Imperfect-\ac{iekf} with a right-invariant error definition~\cite[Sec. 7]{doi:10.1177/0278364919894385}.

\subsection{TG-\ac{eqf}}\label{bins_tg_linan}
\sloppy Recall the state space defined by ${\calM \coloneqq \mathcal{SE}_2(3)\times\R^9}$ in \cref{bins_tfg_iekf_sym_sec}. Define ${\xi \coloneqq (\Pose{}{},\Vector{}{b}{})\in\calM}$. One has ${\Pose{}{} = (\Rot{}{},\Vector{}{v}{},\Vector{}{p}{})\in\mathcal{SE}_2(3)}$ and ${\Vector{}{b}{}=(\Vector{}{b}{\bm{\omega}},\Vector{}{b}{a},\Vector{}{b}{\bm{\nu}}) \in \R^9}$.
Choose the state origin to be ${\xizero = \left(\eye_5,\Vector{}{0}{9\times1}\right)\in\calM}$.
The system's velocity input is given by ${u \coloneqq \left(\Vector{}{\omega}{}, \Vector{}{a}{}, \Vector{}{\nu}{}, \Vector{}{\tau}{\bm{\omega}}, \Vector{}{\tau}{a}, \Vector{}{\tau}{\nu}\right)  = \left(\Vector{}{w}{}, \Vector{}{\tau}{}\right)}$.

The symmetry group of TG-\ac{eqf} is given by ${\grpD \coloneqq \SE_2(3)\ltimes \se_2(3)}$. Define the filter state ${\hat{X} = (\hat{D},\hat{\delta})\in\grpD}$ with ${\hat{D} = (\hat{A},\hat{a},\hat{b})\in\SE_2(3)}$ and ${\hat{\delta} = (\hat{\delta_\omega}, \hat{\gamma_a}, \hat{\delta_\nu})^\wedge\in\se_2(3)}$.
The state estimate is given by 
\begin{align}
    \hat{\xi} \coloneqq \phi(\hat{X},\xizero) = (\hat{D}, \AdMsym{\hat{D}^{-1}}(-\hat{\delta}^\vee)) = (\hatPose{}{}, \hatVector{}{b}{}). 
\end{align}
The state error is defined as 
\begin{equation}\label{bins_tg_err}
\begin{split}
    e \coloneqq \phi(\hat{X}^{-1},\xi) &= (\Pose{}{}\hat{D}^{-1},\AdMsym{\hat{D}}(\Vector{}{b}{}+\Adsym{\hat{D}^{-1}}{\hat{\delta}}^\vee))\\
    &= (\Pose{}{}\hat{D}^{-1},\AdMsym{\hat{D}}\Vector{}{b}{} + \hat{\delta}^{\vee})\\
    &= (\Pose{}{}\hatPose{}{}^{-1},\AdMsym{\hatPose{}{}}(\Vector{}{b}{} - \hatVector{}{b}{})).
\end{split}
\end{equation}

\paragraph{Linearization error analysis for the navigation states.}
The error dynamics for the navigation states ${{e}_T = \Pose{}{}\hatPose{}{}^{-1}}$ is given by
\begin{align*}
    \dot{e}_T &= \dot{\Pose{}{}}\hatPose{}{}^{-1} - \Pose{}{}\hatPose{}{}^{-1}\dot{\hatPose{}{}}\hatPose{}{}^{-1}\\
    &=\Pose{}{}(\mathbf{W}-\mathbf{B}+\mathbf{N})\hatPose{}{}^{-1} + (\mathbf{G}-\mathbf{N})\Pose{}{}\hatPose{}{}^{-1} \\
    &\quad- e_T\hatPose{}{}(\mathbf{W}-\hat{\mathbf{B}}+\mathbf{N})\hatPose{}{}^{-1} - e_T(\mathbf{G}-\mathbf{N})\hatPose{}{}\hatPose{}{}^{-1}\\
    &=e_T\hatPose{}{}(\mathbf{W}-\mathbf{B}+\mathbf{N})\hatPose{}{}^{-1} - e_T\hatPose{}{}(\mathbf{W}-\hat{\mathbf{B}}+\mathbf{N})\hatPose{}{}^{-1}\\
    &\quad+(\mathbf{G}-\mathbf{N})e_T - e_T(\mathbf{G}-\mathbf{N})\\
    &= e_T\Adsym{\hatPose{}{}}{\mathbf{B}-\hat{\mathbf{B}}} + (\mathbf{G}-\mathbf{N})e_T - e_T(\mathbf{G}-\mathbf{N}).
\end{align*}
The above dynamics can be separate to two parts: ${\dot{e}_T = \dot{e}_{T_W} + \dot{e}_{T_G}}$, where ${\dot{e}_{T_W} = e_T\Adsym{\hatPose{}{}}{\mathbf{B}-\hat{\mathbf{B}}}}$ and ${\dot{e}_{T_G} = (\mathbf{G}-\mathbf{N})e_T - e_T(\mathbf{G}-\mathbf{N})}$.
The linearization can be derived separately for each part.
Choosing normal coordinates, the local coordinate chart is given by ${\varepsilon = \log_{\grpD} \circ \phi_{\xizero}^{-1}(e)}$. Inverting the expression yields $e = \phi_{\xizero}(\exp_{\grpD}(\varepsilon^{\wedge}))$. Therefore, for each state, one has 
\begin{align*}
    e_T &= \exp_{\SE_2(3)}({\varepsilon_T}^\wedge),\\
    e_b &= \Adsym{{e_T}^{-1}}{(-\mathbf{J}_{L}(\varepsilon_T)\varepsilon_b)^\wedge},
\end{align*}
where the exponential map $\exp_{\grpD}$ is derived, as in \cref{math_sdp_sec}, from the semi-direct product structure, and $\mathbf{J}_{L}(\varepsilon_T)$ is the left Jacobian of $\SE_2(3)$ defined in \cref{math_se23_sec}.
Recall that by definition in \cref{bins_tg_err} one has ${e_b \coloneqq \AdMsym{\hatPose{}{}}(\Vector{}{b}{} - \hatVector{}{b}{}) = \left(\Adsym{\hatPose{}{}}{\mathbf{B}-\hat{\mathbf{B}}}\right)^{\vee}}$ with ${\Vector{}{b}{}^{\wedge} = \mathbf{B}}$, and ${\hatVector{}{b}{}^{\wedge} = \hat{\mathbf{B}}}$. Hence, for $\dot{e}_{T_W}$ one has 
\begin{align*}
    \dot{e}_{T_W} &= e_T\Adsym{\hatPose{}{}}{\mathbf{B}-\hat{\mathbf{B}}} = -e_Te_b^\wedge.
\end{align*}
Substituting the local coordinate yields
\begin{align}
    \tD\exp_{\SE_2(3)}({\varepsilon_T}^\wedge)[{\dot{\varepsilon}_{T_W}}^\wedge]&= -e_T \Adsym{{e_T}^{-1}}{(-\mathbf{J}_{L}(\varepsilon_T)\varepsilon_b)^\wedge}\nonumber\\
    e_T \frac{\eye-\exp(-\adMsym{\varepsilon_T})}{\adMsym{\varepsilon_T}} {\dot{\varepsilon}_{T_W}}^\wedge &= -e_T \Adsym{{e_T}^{-1}}{(-\mathbf{J}_{L}(\varepsilon_T)\varepsilon_b)^\wedge}\nonumber\\
    \AdMsym{e_T}\frac{\eye-\exp(-\adMsym{\varepsilon_T})}{\adMsym{\varepsilon_T}} {\dot{\varepsilon}_{T_W}}^\wedge &= (\mathbf{J}_{L}(\varepsilon_T)\varepsilon_b)^\wedge. \label{bins_tg_nav_err_1}
\end{align}
Because $\AdMsym{e_T} = \AdMsym{\exp({\varepsilon_T}^\wedge)}= \exp(\adMsym{\varepsilon_T})$, the term on the left side in \cref{bins_tg_nav_err_1} can be written as 
\begin{align}
    \AdMsym{e_T}\frac{\eye-\exp(-\adMsym{\varepsilon_T})}{\adMsym{\varepsilon_T}} 
    &= \exp(\adMsym{\varepsilon_T})\frac{\eye-\exp(-\adMsym{\varepsilon_T})}{\adMsym{\varepsilon_T}} \nonumber\\
    &= \frac{\exp(\adMsym{\varepsilon_T})-\eye}{\adMsym{\varepsilon_T}} \nonumber\\
    &= \sum_{k=0}^{\infty}\frac{1}{(k+1)!}(\adMsym{\varepsilon_T})^{k} = \mathbf{J}_{L}(\varepsilon_T) ,\label{bins_Jl}
\end{align}
where the last step is obtained by expanding the exponential and recalling the auxiliary linear operator in \cref{math_Gamma_aux}. Hence, for the linearization of $\dot{e}_{T_W}$, one has
\begin{align*}
    \dot{\varepsilon}_{T_W} = \varepsilon_b.
\end{align*}
For $\dot{e}_{T_G}=\tD\exp_{\SE_2(3)}({\varepsilon_T}^{\wedge})[{\dot{\varepsilon}_{T_G}}^{\wedge}]$, one has 
\begin{align}\label{bins_etg_tg}
    e_T \frac{\eye-\exp(-\adMsym{\varepsilon_T})}{\adMsym{\varepsilon_T}} {\dot{\varepsilon}_{T_G}}^\wedge & = \begin{bmatrix}
        \Vector{}{0}{3\times3} & \Vector{}{g}{}-e_R\Vector{}{g}{} & e_v \\
        \Vector{}{0}{1\times3} & 0 & 0 \\
        \Vector{}{0}{1\times3} & 0 & 0
    \end{bmatrix}.
\end{align}
Multiply both sides of \cref{bins_etg_tg} by ${e_T}^{-1}$ and then apply $\AdMsym{e_T}$ yields
\begin{align*}
\AdMsym{e_T} \frac{\eye-\exp(-\adMsym{\varepsilon_T})}{\adMsym{\varepsilon_T}} \dot{\varepsilon}_{T_G} & = e_T{e_T}^{-1}\begin{bmatrix}
    \Vector{}{0}{3\times3} & \Vector{}{g}{}-e_R\Vector{}{g}{} & e_v \\
    \Vector{}{0}{1\times3} & 0 & 0 \\
    \Vector{}{0}{1\times3} & 0 & 0    
\end{bmatrix}{e_T}^{-1}\\
&=\begin{bmatrix}
    \Vector{}{0}{3\times3} & \Vector{}{g}{}-e_R\Vector{}{g}{} & e_v \\
    \Vector{}{0}{1\times3} & 0 & 0 \\
    \Vector{}{0}{1\times3} & 0 & 0    
\end{bmatrix}.
\end{align*}
The left side of the previous equation is substituted by $\mathbf{J}_{L}(\varepsilon_T)\dot{\varepsilon}_{T_G}$, according to \cref{bins_Jl}:
\begin{align}\label{bins_tg_nav_err_2}
    \mathbf{J}_{L}(\varepsilon_T)\dot{\varepsilon}_{T_G} = \begin{bmatrix}
        \Vector{}{0}{3\times3} & (I-e_R)\Vector{}{g}{} & \mathbf{J}_{L}(\varepsilon_R)\varepsilon_v \\
        \Vector{}{0}{1\times3} & 0 & 0 \\
        \Vector{}{0}{1\times3} & 0 & 0    
    \end{bmatrix}.
\end{align}
Note that:
\begin{align*}
\eye-e_R &= \eye-\exp({\varepsilon_R}^\wedge)\\
    &= \eye - \sum_{k=0}^{\infty}\frac{1}{k!}{{\varepsilon_R}^\wedge}^k\\
    &= -(\sum_{k=0}^{\infty}\frac{1}{(k+1)!}{{\varepsilon_R}^\wedge}^k){\varepsilon_R}^\wedge\\
    &= -\mathbf{J}_{L}(\varepsilon_R){\varepsilon_R}^\wedge.
\end{align*}
One can then simplify \cref{bins_tg_nav_err_2} to 
\begin{align*}
    \dot{\varepsilon}_{T_G} = \begin{bmatrix}
        \Vector{}{0}{3\times3} & \Vector{}{g}{}^\wedge\varepsilon_R & \varepsilon_v \\
        \Vector{}{0}{1\times3} & 0 & 0 \\
        \Vector{}{0}{1\times3} & 0 & 0    
    \end{bmatrix}.
\end{align*}
Finaly, combining the results for $\dot{\varepsilon}_{T_W}$ and $\dot{\varepsilon}_{T_G}$, the linearized error dynamics of the navigation states is written
\begin{align*}
    \dot{\varepsilon}_{T} = (\varepsilon_{b_\omega}, \;\varepsilon_{b_a}+\Vector{}{g}{}^\wedge\varepsilon_R, \;\varepsilon_{b_\nu}+\varepsilon_v)^\wedge.
\end{align*}

\paragraph{Linearization error analysis for the bias states.}
The linearization of the bias error is derived from the formula in \cref{eq_A0} reported below:
\begin{align*}
    &\dot{\varepsilon} = \mathbf{A}_{t}^{0}\varepsilon + \mathcal{O}(\varepsilon^2) ,\\
    &\mathbf{A}_{t}^{0} = \Fr{e}{\xizero}\vartheta\left(e\right)\Fr{E}{I}\phi_{\xizero}\left(E\right)\Fr{e}{\xizero}\Lambda\left(e, {u^\circ}\right)\Fr{\varepsilon}{\mathbf{0}}\vartheta^{-1}\left(\varepsilon\right).
\end{align*}
Given the choice of normal coordinates as the local coordinate chart, the two leftmost differentials cancel out, yielding
\begin{align*}
    \dot{\varepsilon} = \Fr{e}{\xizero}\Lambda\left(e, \mathring{u}\right) \Fr{E}{I}\phi_{\xizero}(E)[\varepsilon] + \mathcal{O}(\varepsilon^2).
\end{align*}
Evaluating $\tD\phi_{\xizero}$ at $I$ with direction $[\varepsilon_T, \varepsilon_b]$ yields
\begin{align*}
    \tD\phi_{\xizero}(I)[\varepsilon_T,\varepsilon_b] = ({\varepsilon_T}^\wedge, -{\varepsilon_b}^\wedge).
\end{align*}
Evaluating $\tD\Lambda_{\mathring{u}}$ at $\xizero$ with direction $[{\varepsilon_T}^\wedge, -{\varepsilon_b}^\wedge]$ yields
\begin{align*}
    \tD\Lambda_{\mathring{u}}(\xizero)[{\varepsilon_T}^\wedge, -{\varepsilon_b}^\wedge] &= 
    (\sim, \adsym{-{\varepsilon_b}^\wedge}{\Lambda_1(\xizero, \mathring{u})})\\
    &=(\sim, \adMsym{\left(\mathring{\Vector{}{w}{}} + \mathbf{G}^{\vee}\right)}\varepsilon_b).
\end{align*}
Hence the linearization of bias error is given by
\begin{align*}
    \dot{\varepsilon}_b = \adMsym{\left(\mathring{\Vector{}{w}{}} + \mathbf{G}^{\vee}\right)}\varepsilon_b + \mathcal{O}(\varepsilon^2),
\end{align*}
with $\mathring{u} = (\mathring{\Vector{}{w}{}}, \mathring{\Vector{}{\tau}{}}) \coloneqq\psi_{\hat{X}^{-1}}(u)$ computed through the action in \cref{bins_psi_tg}.

\paragraph{Filter state matrix.}
The linearized error state matrix ${\mathbf{A}_{t}^{0} \st \dot{\varepsilon} \simeq \mathbf{A}_{t}^{0}\varepsilon}$ is defined according to

\begin{equation}
    \mathbf{A}_{t}^{0} = \begin{bNiceArray}{ccc:c}[margin]
        \Vector{}{0}{3 \times 3} & \Vector{}{0}{3 \times 3} & \Vector{}{0}{3 \times 3} &  \Block{3-1}{\eye_9}\\
        \Vector{}{g}{}^{\wedge} & \Vector{}{0}{3 \times 3} & \Vector{}{0}{3 \times 3} & \\
        \Vector{}{0}{3 \times 3} & \eye_3 & \Vector{}{0}{3 \times 3} & \Bstrut\\
        \hdottedline
        \Block{1-3}{\Vector{}{0}{9 \times 9}} &&& \adMsym{\left(\mathring{\Vector{}{w}{}} + \mathbf{G}^{\vee}\right)}\Tstrut
    \end{bNiceArray} \in \R^{18 \times 18}. \label{bins_At0_tg}
\end{equation}
Note that for a practical implementation of the presented \ac{eqf}, the virtual input ${\Vector{}{\nu}{}}$ is set to zero.

It is worth noticing that the \ac{eqf} built on the $\grpD$ symmetry group is the only filter with \emph{exact linearization of the navigation error dynamics}. The filter state matrix with arbitrary $\xizero$ is shown in \cref{appendix_A_chp}.

\subsection{DP-\ac{eqf}}\label{bins_dp_linan}
Recall the state space defined by ${\calM \coloneqq \mathcal{HG}(3)\times\R^3\times\R^6}$ in \cref{bins_dp_sec}. Define ${\xi \coloneqq (\Pose{}{},\Vector{}{p}{},\Vector{}{b}{})\in\calM}$. One has ${\Pose{}{} = (\Rot{}{},\Vector{}{v}{})\in\mathcal{HG}(3)}$ and ${\Vector{}{b}{}=(\Vector{}{b}{\bm{\omega}},\Vector{}{b}{a})\in\R^6}$.
Choose the state origin to be ${\xizero =  \left(\eye_4,\Vector{}{0}{6\times1},\Vector{}{0}{3\times1}\right)\in\calM}$.
The system's velocity input is given by ${u \coloneqq \left(\Vector{}{\omega}{}, \Vector{}{a}{},\Vector{}{\tau}{\bm{\omega}},\Vector{}{\tau}{a}, \Vector{}{\nu}{}\right)  = \left(\Vector{}{w}{}, \Vector{}{\tau}{}, \Vector{}{\nu}{}\right)}$.

\sloppy The symmetry group of DP-\ac{eqf} is given by ${\grpA \coloneqq \mathbf{HG}(3)\ltimes \gothhg(3) \times \R^3}$. Define the filter state ${\hat{X} = (\hat{B},\hat{\beta},\hat{c})\in\grpA}$ with ${\hat{B} = (\hat{A},\hat{a})\in\mathbf{HG}(3)}$ and ${\hat{\beta} = (\hat{\beta_\omega}, \hat{\beta_a})^\wedge \in \gothhg(3)}$.
The state estimate is given by 
\begin{align}
    \hat{\xi} \coloneqq \phi(\hat{X},\xizero) = (\hat{B}, \AdMsym{\hat{B}^{-1}}(-\hat{\beta}^{\vee}),\hat{c}) = (\hatPose{}{}, \hatVector{}{b}{}, \hatVector{}{p}{}). 
\end{align}
The state error is defined as 
\begin{align}\label{bins_dp_err}
e \coloneqq \phi(\hat{X}^{-1},\xi) &= (\Pose{}{}\hat{B}^{-1},\AdMsym{\hat{B}}\Vector{}{b}{} + \hat{\beta}^{\vee}, \Vector{}{p}{} - \hat{c})\\
&= (\Pose{}{}\hatPose{}{}^{-1},\AdMsym{\hatPose{}{}}(\Vector{}{b}{} - \hatVector{}{b}{}), \Vector{}{p}{} - \hatVector{}{p}{})
\end{align}

\paragraph{Linearization error analysis for the navigation states.}
The semi-direct product structure of the $\grpA$ group, which is associated with the rotation, velocity, and corresponding bias states, is equivalent to that of the $\grpD$ group. Therefore, the derivation of the error dynamics of the rotation, velocity, and bias states is similar to the one in TG-\ac{eqf}.
Specifically, when choosing normal coordinates, the linearized error dynamics for rotation and velocity states are written
\begin{align*}
    &\dot{\varepsilon}_R = \varepsilon_{b_\omega},\\
    &\dot{\varepsilon}_v = \varepsilon_{b_a}+\Vector{}{g}{}^\wedge\varepsilon_R.
\end{align*}
For the position error, one has 
\begin{align*}
    \dot{\varepsilon}_p &= \dot{e}_p = \dot{\Vector{}{p}{}} - \dot{\hatVector{}{p}{}} = \Rot{}{}\Vector{}{\nu}{} + \Vector{}{v}{} - \hatRot{}{}\Vector{}{\nu}{} - \hatVector{}{v}{}\\
    &=e_v+e_R\hatVector{}{v}{}-\hatVector{}{v}{}+e_R\hatRot{}{}\Vector{}{\nu}{}-\hatRot{}{}\Vector{}{\nu}{}\\
    &=\mathbf{J}_{L}(\varepsilon_R)\varepsilon_v+(\eye+\varepsilon_R^\wedge + O({\varepsilon_R}^2))\left(\hatRot{}{}\Vector{}{\nu}{}+\hatVector{}{v}{}\right)-\left(\hatRot{}{}\Vector{}{\nu}{}+\hatVector{}{v}{}\right)\\
    &=\mathbf{J}_{L}(\varepsilon_R)\varepsilon_v+(\eye+\varepsilon_R^\wedge + O({\varepsilon_R}^2))\mathring{\Vector{}{\nu}{}}-\mathring{\Vector{}{\nu}{}}\\
    &=\varepsilon_v - \mathring{\Vector{}{\nu}{}}^\wedge\varepsilon_R + \mathcal{O}(\varepsilon^2),
\end{align*}
where ${\uzero = (\mathring{\Vector{}{w}{}}, \mathring{\Vector{}{\nu}{}}, \mathring{\Vector{}{\tau}{}}) \coloneqq \psi_{\hat{X}^{-1}}(u)}$, with the action $\psi$ defined in \cref{bins_psi_hg}. In particular, ${\mathring{\Vector{}{\nu}{}} = \hatRot{}{}\Vector{}{\nu}{}+\hatVector{}{v}{}}$.

\paragraph{Linearization error analysis for the bias states.}
The derivation of bias error dynamics, when choosing normal coordinates, follows that of the TG-\ac{eqf} in the previous subsection:
\begin{align*}
    \dot{\varepsilon}_b = \adMsym{\left(\mathring{\Vector{}{w}{}} + \mathbf{G}^{\vee}\right)}\varepsilon_b + \mathcal{O}(\varepsilon^2).
\end{align*}

\paragraph{Filter state matrix.}
The linearized error state matrix ${\mathbf{A}_{t}^{0} \st \dot{\varepsilon} \simeq \mathbf{A}_{t}^{0}\varepsilon}$ is defined according to
\begin{equation}
    \mathbf{A}_{t}^{0} = \begin{bNiceArray}{cc:c:c}[margin]
        \Vector{}{0}{3 \times 3} & \Vector{}{0}{3 \times 3} &  \Block{2-1}{\eye_6} & \Vector{}{0}{6 \times 3} \\
        \Vector{}{g}{}^{\wedge} & \Vector{}{0}{3 \times 3} & & \Vector{}{0}{6 \times 3} \Bstrut\\
        \hdottedline
        \Block{1-2}{\Vector{}{0}{6 \times 3}} && \adMsym{\left(\mathring{\Vector{}{w}{}} + \mathbf{G}^{\vee}\right)} & \Vector{}{0}{6 \times 3} \Bstrut\\
        \hdottedline
        -\mathring{\Vector{}{\nu}{}}^{\wedge} & \eye_3 & \mathbf{0}_{3 \times 6} & \mathbf{0}_{3 \times 3} \Tstrut
    \end{bNiceArray} \in \R^{15 \times 15}. \label{bins_At0_HG_R3}
\end{equation}
For a practical implementation of the presented \ac{eqf}, the virtual input ${\Vector{}{\nu}{}}$ is set to zero.

\subsection{SD-\ac{eqf}}\label{bins_sd_linan}
Recall the state space defined by ${\calM \coloneqq \mathcal{SE}_2(3)\times\R^6}$ in \cref{bins_sdb_sec}. Define ${\xi \coloneqq (\Pose{}{},\Vector{}{b}{})\in\calM}$. One has ${\Pose{}{} = (\Rot{}{},\Vector{}{v}{},\Vector{}{p}{})\in\mathcal{SE}_2(3)}$ and ${\Vector{}{b}{}=(\Vector{}{b}{\bm{\omega}},\Vector{}{b}{a}) \in \R^6}$.
Choose the state origin to be ${\xizero = \left(\eye_5,\Vector{}{0}{6\times1}\right)\in\calM}$.
The system's velocity input is given by ${u \coloneqq \left(\Vector{}{\omega}{}, \Vector{}{a}{}, \Vector{}{\tau}{\bm{\omega}}, \Vector{}{\tau}{a}\right)  = \left(\Vector{}{w}{}, \Vector{}{\tau}{}\right)}$.

The symmetry group of SD-\ac{eqf} is given by ${\grpB \coloneqq \SE_2(3)\ltimes \se(3)}$. Define the filter state ${\hat{X} = (\hat{D},\hat{\delta})\in\grpB}$ with ${\hat{D} = (\hat{A},\hat{a},\hat{b})\in\SE_2(3)}$ and ${\hat{\delta} = (\hat{\delta_\omega}, \hat{\gamma_a})^\wedge\in\se(3)}$. The $\SE_2(3)$ component in $\hat{X}$ can also be expressed in $\hat{D} = (\hat{B},\hat{b})$ where $\hat{B} = (\hat{A},\hat{a})\in\mathbf{HG}(3)$.
The state estimate is given by 
\begin{align}
    \hat{\xi} \coloneqq \phi(\hat{X},\xizero) = (\hat{D}, \AdMsym{\hat{B}^{-1}}(-\hat{\delta}^\vee)) = (\hatPose{}{}, \hatVector{}{b}{}). 
\end{align}
The state error is defined as 
\begin{align}\label{eq:sd_err}
e \coloneqq \phi(\hat{X}^{-1},\xi) &= (\Pose{}{}\hat{D}^{-1},\AdMsym{\hat{B}}(\Vector{}{b}{}+\Adsym{\hat{B}^{-1}}{\hat{\delta}}^\vee))\\
&= (\Pose{}{}\hat{D}^{-1},\AdMsym{\hat{B}}\Vector{}{b}{} + \hat{\delta}^{\vee}).
\end{align}

\paragraph{Linearization error analysis for the navigation states.}
The semi-direct product structure of the $\grpB$ group, which is associated with the rotation, velocity, and corresponding bias states, is equivalent to that of the $\grpD$ group. Therefore, when choosing normal coordinates, the derivation of the error dynamics of the rotation, velocity, and bias states is similar to the one in TG-\ac{eqf}.
In particular, one has 
\begin{align*}
    &\dot{\varepsilon}_R = \varepsilon_{b_\omega},\\
    &\dot{\varepsilon}_v = \varepsilon_{b_a}+\Vector{}{g}{}^\wedge\varepsilon_R.
\end{align*}
The position error $\dot{e}_p=\dot{\varepsilon}_p+\mathcal{O}(\varepsilon^2)$, instead, yield the following linearized position error dynamics
\begin{align*}
    \dot{e}_p &= \ddt (-\Rot{}{}\hatRot{}{}^\top\hatVector{}{p}{}+\Vector{}{p}{})\\
        &= -\dot{e}_R\hatVector{}{p}{} - e_R\dot{\hatVector{}{p}{}}+\dot{\Vector{}{p}{}}\\
        &= e_Re_{b_\omega}^\wedge \hatVector{}{p}{} - e_R\hatVector{}{v}{} + \Vector{}{v}{}\\
    \dot{\varepsilon}_p &= ((\eye+\varepsilon_R^\wedge + O({\varepsilon}^2)))(\eye+\varepsilon_{b_\omega}^\wedge+\mathcal{O}(\varepsilon^2))\hatVector{}{p}{}+(\varepsilon_v+\mathcal{O}(\varepsilon^2))\\
    &=\varepsilon_v+\hatVector{}{p}{}^\wedge\varepsilon_{b_\omega}+\mathcal{O}(\varepsilon^2).
\end{align*}

\paragraph{Linearization error analysis for the bias states.}
The derivation of bias error dynamics for the SD-\ac{eqf}, with normal coordinates, is equivalent to that presented for the TG-\ac{eqf}, and yields
\begin{align*}
    \dot{\varepsilon}_b = \adMsym{\left(\AdMsym{\hat{B}}\Vector{}{w}{} + \hat{\delta}^{\vee} + \mathbf{G}^{\vee}\right)}\varepsilon_b + \mathcal{O}(\varepsilon^2).
\end{align*}
Finally, note the following relation
\begin{equation*}
    \AdMsym{\hat{B}}\Vector{}{w}{} + \hat{\delta}^{\vee} = \Pi\left(\mathring{\Vector{}{w}{}}^{\wedge}\right)^{\vee},
\end{equation*}
where $(\mathring{\Vector{}{w}{}}, \mathring{\Vector{}{\tau}{}}) \coloneqq\psi_{\hat{X}^{-1}}(u)$ is the origin input for the $\grpD$ symmetry computed with the action $\psi$ in \cref{bins_psi_tg}, and $\Pi(\cdot)$ defined in \cref{math_maps_sec}.

\paragraph{Filter state matrix.}
The linearized error state matrix ${\mathbf{A}_{t}^{0} \st \dot{\varepsilon} \simeq \mathbf{A}_{t}^{0}\varepsilon}$ is defined according to
\begin{equation}
    \mathbf{A}_{t}^{0} = \begin{bNiceArray}{ccc:w{c}{1.25cm}w{c}{1.25cm}}[margin]
        \Vector{}{0}{3 \times 3} & \Vector{}{0}{3 \times 3} & \Vector{}{0}{3 \times 3} & \Block{2-2}{\eye_6} \\
        \Vector{}{g}{}^{\wedge} & \Vector{}{0}{3 \times 3} & \Vector{}{0}{3 \times 3} & & \\
        \Vector{}{0}{3 \times 3} & \eye_3 & \Vector{}{0}{3 \times 3} & \hatVector{}{p}{}^{\wedge} & \mathbf{0}_{3 \times 3} \Bstrut\\
        \hdottedline
        \Block{1-3}{\mathbf{0}_{6 \times 9}} &&& \Block{1-2}{ \adMsym{\left(\AdMsym{\hat{B}}\Vector{}{w}{} + \hat{\delta}^{\vee} + \mathbf{G}^{\vee}\right)}} \Tstrut\\
    \end{bNiceArray} \in \R^{15 \times 15}. \label{bins_At0_SD}
\end{equation}

It is clear that there are similarities between the state matrices of the SD-\ac{eqf} and the TG-\ac{eqf}. In particular, when comparing ${\mathbf{A}_{t}^{0}}$ in \cref{bins_At0_tg} with the one in \cref{bins_At0_SD}, it is trivial to see the only difference between the two matrices is in the row of ${\mathbf{A}_{t}^{0}}$ relative to the position error. This is where the major difference between filters employing the symmetries $\grpD$ and $\grpB$ is found.

\subsection{\ac{tfgiekf}}\label{bins_tfgiekf_linan}
Recall the state space defined by ${\calM \coloneqq \mathcal{SE}_2(3)\times\R^6}$ in \cref{bins_tfg_iekf_sym_sec}. Define ${\xi \coloneqq (\Pose{}{},\Vector{}{b}{})\in\calM}$. 
One has ${\Pose{}{} = (\Rot{}{},\Vector{}{v}{},\Vector{}{p}{})\in\mathcal{SE}_2(3)}$ and ${\Vector{}{b}{}=(\Vector{}{b}{\bm{\omega}},\Vector{}{b}{a}) \in \R^6}$.
Choose the state origin to be ${\xizero = \left(\eye_5,\Vector{}{0}{6\times1}\right)\in\calM}$.
The system's velocity input is given by ${u \coloneqq \left(\Vector{}{\omega}{}, \Vector{}{a}{},\Vector{}{\tau}{\bm{\omega}},\Vector{}{\tau}{a}\right)  = \left(\Vector{}{w}{}, \Vector{}{\tau}{}\right)}$.

The symmetry group of \ac{tfgiekf} is given by ${\grpG_{\mathbf{TF}}:\SO(3)\ltimes(\R^6\oplus\R^6)}$. Define the filter state ${\hat{X} = (\hat{D},\hat{\delta})\in\grpG_{\mathbf{TF}}}$ with ${\hat{D} = (\hat{A},(\hat{a},\hat{b}))\in\SE_2(3)=\SO(3)\ltimes\R^6}$ and ${\hat{\delta} = (\hat{\delta_\omega}, \hat{\gamma_a})\in\R^6}$.
The state estimate is given by 
\begin{align}
    \hat{\xi} \coloneqq \phi(\hat{X},\xizero) = (\hat{D}, \hat{A}^{\top}*(-\hat{\delta})) = (\hatPose{}{}, \hatVector{}{b}{}). 
\end{align}
The state error is defined as 
\begin{align}\label{bins_tfgiekf_err}
e \coloneqq \phi(\hat{X}^{-1},\xi) &= (\Pose{}{}\hat{D}^{-1},\hat{A}*(\Vector{}{b}{}+\hat{A}^{\top}*\hat{\delta})) \\
&= (\Pose{}{}\hatPose{}{}^{-1}, \hatRot{}{}*(\Vector{}{b}{} - \hatVector{}{b}{})).
\end{align}

\paragraph{Linearization error analysis for the navigation states.}
The semi-direct product structure of the $\grpC$ group, which is associated with the rotation state, is equivalent to that of the $\grpD, \grpA$ and $\grpB$ groups. Therefore, the derivation of the rotation error dynamics follows that in TG-\ac{eqf}, DP-\ac{eqf} and SDB-\ac{eqf}. Therefore, when choosing normal coordinates, one has
\begin{align*}
    &\dot{\varepsilon}_R = \varepsilon_{b_\omega}.
\end{align*}
The velocity error $e_v = -\Rot{}{}\hatRot{}{}^\top\hatVector{}{v}{}+\Vector{}{v}{}$, instead, is written
\begin{align*}
    \dot{e}_v
        &=-\dot{e}_R\hatVector{}{v}{} - e_R\dot{\hatVector{}{v}{}}+\dot{v}\\
        &=e_R(e_{b_\omega})^\wedge\hatVector{}{v}{}-e_R\hatRot{}{}(\Vector{}{a}{}-\hatVector{}{b}{a}) -e_R \Vector{}{g}{} + \Rot{}{}(\Vector{}{a}{}-\Vector{}{b}{a})+\Vector{}{g}{}\\
        &=e_R(e_{b_\omega})^\wedge\hatVector{}{v}{}-e_R\hatRot{}{}(\Vector{}{a}{}-\hatVector{}{b}{a}) -(e_R-\eye) \Vector{}{g}{} \\
        &\quad+ e_R\hatRot{}{}(\Vector{}{a}{}-\hatRot{}{}^\top(e_{b_a}-\delta_{b_a}))\\
        &=e_Re_{b_\omega}^\wedge\hatVector{}{v}{} - e_Re_{b_a} - (e_R-\eye)\Vector{}{g}{};\\
    \dot{\varepsilon}_v &= \hatVector{}{v}{}^\wedge\varepsilon_{b_\omega}+\Vector{}{g}{}^\wedge\varepsilon_R+\varepsilon_{b_a}+\mathcal{O}(\varepsilon^2).
\end{align*}
The position error $e_p = -\Rot{}{}\hatRot{}{}^\top\hatVector{}{p}{}+\Vector{}{p}{}$ follows that of the SD-\ac{eqf}; specifically, it is given by
\begin{align*}
    \dot{\varepsilon}_p = \varepsilon_v+\hatVector{}{p}{}^\wedge\varepsilon_{b_\omega}+\mathcal{O}(\varepsilon^2).
\end{align*}

\paragraph{Linearization error analysis for the bias states.}
The error in bias state ${b_\omega}$ is given by $e_{b_\omega} = \hatRot{}{}*(\Vector{}{b}{\bm{\omega}}-\hatVector{}{b}{\bm{\omega}})$.
When choosing normal coordinates, the dynamics can be derived as follows:
\begin{align*}
    \dot{e}_{b_\omega} &= \hatRot{}{}(\omega-\hatVector{}{b}{\bm{\omega}})^\wedge * (\Vector{}{b}{\bm{\omega}}-\hatVector{}{b}{\bm{\omega}})\\
    &=\hatRot{}{}(\omega - \hatVector{}{b}{\bm{\omega}})^\wedge\hatRot{}{}^\top\hatRot{}{}* (\Vector{}{b}{\bm{\omega}}-\hatVector{}{b}{\bm{\omega}})\\
    &=(\hatRot{}{}(\omega-\hatVector{}{b}{\bm{\omega}}))^\wedge e_{b_\omega}.
\end{align*}
In local coordinates, the linearization is given by 
\begin{align*}
    \dot{\varepsilon}_{b_\omega}=(\hatRot{}{}(\omega-\hatVector{}{b}{\bm{\omega}}))^\wedge \varepsilon_{b_\omega} + \mathcal{O}(\varepsilon^2).
\end{align*}
The error in $b_a$ follows the same derivation, which is given by 
\begin{align*}
    \dot{\varepsilon}_{b_a}=(\hatRot{}{}(\omega-\hatVector{}{b}{\bm{\omega}}))^\wedge \varepsilon_{b_a} + \mathcal{O}(\varepsilon^2).
\end{align*}

\paragraph{Filter state matrix.}
The linearized error state matrix ${\mathbf{A}_{t}^{0} \st \dot{\varepsilon} \simeq \mathbf{A}_{t}^{0}\varepsilon}$ is defined according to
\begin{equation}
    \mathbf{A}_{t}^{0} = \begin{bNiceArray}{ccc:cc}[margin]
        \Block{3-3}{\prescript{}{4}{\mathbf{A}}} & & & \eye_3 & \mathbf{0}_{3 \times 3}\\
        & & & \hatVector{}{v}{}^{\wedge} & \eye_{3}\\
        & & & \hatVector{}{p}{}^{\wedge} & \mathbf{0}_{3 \times 3}\Bstrut\\
        \hdottedline
        \Block{1-3}{\mathbf{0}_{3 \times 9}} & & & \left(\hatRot{}{}\left(\Vector{}{\omega}{} - \hatVector{}{b}{\omega}\right)\right)^{\wedge} & \mathbf{0}_{3 \times 3}\Tstrut\\
        \Block{1-3}{\mathbf{0}_{3 \times 9}} & & & \mathbf{0}_{3 \times 3} & \left(\hatRot{}{}\left(\Vector{}{\omega}{} - \hatVector{}{b}{\omega}\right)\right)^{\wedge}
    \end{bNiceArray} \in \R^{15 \times 15}. \label{eq:At0_TFG}
\end{equation}

\subsection{Final considerations}

With the analysis of the linearization error conducted in this section, we can draw the following noteworthy considerations. \emph{First, we showed how \ac{mekf}, \ac{iekf}, and \ac{tfgiekf} are derived as \acl{eqf} for specific choices of symmetry. This yields the conclusion that the choice of symmetry is indeed the only difference between those filters, unveiling a new approach to filter design in which the choice of symmetry dictates the resulting filter, derived according to the \acl{eqf} design methodology. Second, we showed that modeling the coupling between bias and navigation states using a semi-direct product symmetry improves the linearization error of the filter's navigation states, shifting the nonlinearities toward the bias states, which possess slow dynamics. In particular, exploiting the tangent group of $\SE_2(3)$ yields a filter, the TG-\ac{eqf}, with exact linearization of the navigation state. The DP-\ac{eqf}, SD-\ac{eqf}, and \ac{tfgiekf} all have semi-direct geometric coupling between part of their navigation states and the bias states, leading to improved linearization where the coupling acts compatibly with the $\mathbf{T}\grpG$ structure.} A snapshot of this analysis is provided in \cref{bins_symmetries_overview}.

\ifdefined\includetblr
\begin{table}
    \setlength\tabcolsep{5pt}
    \centering
    \tabrefcaption{\cite{Fornasier2023EquivariantSystems}}
    \setlength\tabcolsep{5pt}
    \captiontitlefont{\scshape\small}
    \captionnamefont{\scshape\small}
    \captiondelim{}
    \captionstyle{\centering\\}
    \caption{Overview of the relation filter-symmetry and linearized filter error dynamics.}
    \begin{tblr}
    {
        rows = {m},
        column{1} = {5.0cm, c},
        column{3} = {7.0cm, c},
    }
        \toprule
        Filter and Symmetry group & ${\mathbf{A} \st \dot{\Vector{}{\varepsilon}{}} \simeq \mathbf{A}\Vector{}{\varepsilon}{}}$\\ 
        \midrule
        \SetCell[]{l}{
        \ac{mekf} \cite{Lefferts1982KalmanEstimation}\\
        {Special orthogonal group $\grpE: \SO(3) \times \R^{12}$} 
        } &
        \SetCell[]{l} {
        ${\dot{\Vector{}{\varepsilon}{R}} \simeq -\hatRot{}{}\Vector{}{\varepsilon}{b_{\omega}} + \calO\left(\Vector{}{\varepsilon}{}^2\right)}$,\\
        ${\dot{\Vector{}{\varepsilon}{v}} \simeq -\left(\hatRot{}{}\left(\Vector{}{a}{} - \hatVector{}{b}{a}\right)\right)^{\wedge}\Vector{}{\varepsilon}{R} - \hatRot{}{}\Vector{}{\varepsilon}{b_{a}} + \calO\left(\Vector{}{\varepsilon}{}^2\right)}$,\\
        ${\dot{\Vector{}{\varepsilon}{p}} \simeq \Vector{}{\varepsilon}{v}}$,\\
        ${\dot{\Vector{}{\varepsilon}{b}} = \Vector{}{0}{}}$.
        } \\
        \SetCell[]{l}{
        Imperfect-\ac{iekf} \cite{7523335}\\
        {Extended special Euclidean group $\grpF: \SE_{2}(3) \times \R^{6}$}
        } &
        \SetCell[]{l} {
        ${\dot{\Vector{}{\varepsilon}{R}} \simeq -\hatRot{}{}\Vector{}{\varepsilon}{b_{\omega}} + \calO\left(\Vector{}{\varepsilon}{}^2\right)}$,\\
        ${\dot{\Vector{}{\varepsilon}{v}} \simeq \Vector{}{g}{}^{\wedge}\Vector{}{\varepsilon}{R} - \hatVector{}{v}{}^{\wedge}\hatRot{}{}\Vector{}{\varepsilon}{b_{\omega}} - \hatRot{}{}\Vector{}{\varepsilon}{b_{a}} + \calO\left(\Vector{}{\varepsilon}{}^2\right)}$,\\
        ${\dot{\Vector{}{\varepsilon}{p}} \simeq \Vector{}{\varepsilon}{v} - \hatVector{}{p}{}^{\wedge}\hatRot{}{}\Vector{}{\varepsilon}{b_{\omega}} + \calO\left(\Vector{}{\varepsilon}{}^2\right)}$,\\
        ${\dot{\Vector{}{\varepsilon}{b}} = \Vector{}{0}{}}$.
        } \\
        \SetCell[]{l}{
        \ac{tfgiekf} \cite{Barrau2022TheProblems}\\
        Two-frames group $\grpC:\; \SO(3) \ltimes (\R^{6} \oplus \R^{6})$
        } & 
        \SetCell[]{l} {
        ${\dot{\Vector{}{\varepsilon}{R}} \simeq \Vector{}{\varepsilon}{b_{\omega}}}$,\\
        ${\dot{\Vector{}{\varepsilon}{v}} \simeq \Vector{}{g}{}^{\wedge}\Vector{}{\varepsilon}{R} + \hat{\Vector{}{v}{}}^{\wedge}\Vector{}{\varepsilon}{b_{\omega}} + \Vector{}{\varepsilon}{b_{a}} + \calO\left(\Vector{}{\varepsilon}{}^2\right)}$,\\
        ${\dot{\Vector{}{\varepsilon}{p}} \simeq \Vector{}{\varepsilon}{v} + \hat{\Vector{}{p}{}}^{\wedge}\Vector{}{\varepsilon}{b_{\omega}} + \calO\left(\Vector{}{\varepsilon}{}^2\right)}$,\\
        ${\dot{\Vector{}{\varepsilon}{b_{\omega}}} \simeq \left(\hatRot{}{}\left(\omega - \hatVector{}{b}{\omega}\right)\right)^{\wedge}\Vector{}{\varepsilon}{b_{\omega}} + \calO\left(\Vector{}{\varepsilon}{}^2\right)}$,\\
        ${\dot{\Vector{}{\varepsilon}{b_{a}}} \simeq \left(\hatRot{}{}\left(\omega - \hatVector{}{b}{\omega}\right)\right)^{\wedge}\Vector{}{\varepsilon}{b_{a}} + \calO\left(\Vector{}{\varepsilon}{}^2\right)}$.
        } \\
        \SetCell[]{l}{
        TG-\ac{eqf}\\
        Tangent group $\grpD:\; \SE_2(3) \ltimes \gothse_2(3)$ 
        } & 
        \SetCell[]{l} {
        ${\dot{\Vector{}{\varepsilon}{R}} \simeq \Vector{}{\varepsilon}{b_{\omega}}}$,\\
        ${\dot{\Vector{}{\varepsilon}{v}} \simeq \Vector{}{g}{}^{\wedge}\Vector{}{\varepsilon}{R} + \Vector{}{\varepsilon}{b_{a}}}$,\\
        ${\dot{\Vector{}{\varepsilon}{p}} \simeq \Vector{}{\varepsilon}{v} + \Vector{}{\varepsilon}{b_{\nu}}}$,\\
        ${\dot{\Vector{}{\varepsilon}{b}} \simeq \adMsym{\left(\mathring{\Vector{}{w}{}} + \mathbf{G}^{\vee}\right)}\Vector{}{\varepsilon}{b} + \calO\left(\Vector{}{\varepsilon}{}^2\right)}$.
        } \\
        \SetCell[]{l}{
        DP-\ac{eqf}\\
        Direct position group $\grpA:\; \HG(3) \ltimes \gothhg(3) \times \R^{3}$
        } &
        \SetCell[]{l} {
        ${\dot{\Vector{}{\varepsilon}{R}} \simeq \Vector{}{\varepsilon}{b_{\omega}}}$,\\
        ${\dot{\Vector{}{\varepsilon}{v}} \simeq \Vector{}{g}{}^{\wedge}\Vector{}{\varepsilon}{R} + \Vector{}{\varepsilon}{b_{a}}}$,\\
        ${\dot{\Vector{}{\varepsilon}{p}} \simeq \Vector{}{\varepsilon}{v} + \hatVector{}{v}{}^{\wedge}\Vector{}{\varepsilon}{R} + \calO\left(\Vector{}{\varepsilon}{}^2\right)}$,\\
        ${\dot{\Vector{}{\varepsilon}{b}} \simeq \adMsym{\left(\mathring{\Vector{}{w}{}} + \mathbf{G}^{\vee}\right)}\Vector{}{\varepsilon}{b} + \calO\left(\Vector{}{\varepsilon}{}^2\right)}$.
        } \\ 
        \SetCell[]{l}{
        SD-\ac{eqf}\\
        Semi-direct bias group $\grpB:\; \SE_2(3) \ltimes \gothse(3)$
        } &
        \SetCell[]{l} {
        ${\dot{\Vector{}{\varepsilon}{R}} \simeq \Vector{}{\varepsilon}{b_{\omega}}}$,\\
        ${\dot{\Vector{}{\varepsilon}{v}} \simeq \Vector{}{g}{}^{\wedge}\Vector{}{\varepsilon}{R} + \Vector{}{\varepsilon}{b_{a}}}$,\\
        ${\dot{\Vector{}{\varepsilon}{p}} \simeq \Vector{}{\varepsilon}{v} + \hatVector{}{p}{}^{\wedge}\Vector{}{\varepsilon}{b_{\omega}} + \calO\left(\Vector{}{\varepsilon}{}^2\right)}$,\\
        ${\dot{\Vector{}{\varepsilon}{b}} \simeq \adMsym{\left(\AdMsym{\hat{B}}\Vector{}{w}{} + \hat{\delta}^{\vee} + \mathbf{G}^{\vee}\right)}\Vector{}{\varepsilon}{b} + \calO\left(\Vector{}{\varepsilon}{}^2\right)}$.
        } \\
        \bottomrule
    \end{tblr}\label{bins_symmetries_overview}
\end{table}
\fi

\section{Position-based navigation}\label{bins_pos_nav_sec}

The previous analysis makes it evident that exploiting the tangent symmetry group ${\grpD = \SE_2(3) \ltimes \se_2(3)}$ for filter design results in an \ac{eqf} with exact linearization of the navigation states. In the upcoming section, we delve into the practical application of the filters and symmetries introduced earlier, specifically addressing the problem of position-based navigation. Finally, through extensive simulation results, we verify our theoretical findings, highlighting that any of the \ac{iekf}, \ac{tfgiekf}, TG-\ac{eqf}, DP-\ac{eqf}, and SD-\ac{eqf} are good candidates for high performance \ac{ins} filter design with the TG-\ac{eqf} demonstrating superior performance.

\subsection{Problem formulation and filters' output matrices}

Consider the \acl{ins} in \cref{bins_bins} and the global position measurement model
\begin{equation}\label{bins_position_output}
    h(\xi) = \Vector{}{p}{} \in \calN \subset \R^{3}.
\end{equation}
It is straightforward to verify that $\grpE, \grpA$ symmetries possess linear output for global position measurement since the position has a linear symmetry. However, this is not the case for ${\grpF, \grpD, \grpB, \grpC}$ symmetries; for global position measurements, they possess neither linear output nor output equivariance. This aspect opens up a deeper discussion that is often recurrent in the context of geometric filter design. Specifically, kinematic systems generally possess two output types, \emph{body-referenced measurement type} and \emph{global-referenced measurement type}. An example of \acl{ins} with both measurement types was already discussed in \cref{bas_chp}.

Inspired by the solution proposed in \cref{bas_chp}, where the global-referenced measurement was modeled as a fixed body-referenced measurement of a time-varying global-referenced direction, \emph{we derive here a general result that allows reformulating global-referenced measurements of navigation states into fixed body-referenced measurements by imposing nonlinear constraints~\cite{Julier2007OnConstraints} that are compatible with the symmetry, leading to a unified framework suitable for many measurement types.} The result is derived for global position measurements, but it is general and not limited to them. It applies to various global-referenced measurements of navigation states. Consider the system of interest in \cref{bins_bins} and the measurement model defined in \cref{bins_position_output}.
\begin{lemma}
    Let ${\Vector{}{\pi}{} \in \R^{3}}$ be a physical global-referenced measurement of a position ${\Vector{}{p}{} \in \R^{3}}$. Define a new measurement model $h\left(\xi\right)$, describing the body-referenced difference between the measurement and the position state as follows:
    \begin{equation}\label{bins_position_trick_output}
         h\left(\xi\right) = \Rot{}{}^\top\left(\Vector{}{\pi}{} - \Vector{}{p}{}\right) \in \R^{3} .
    \end{equation}
    Let ${y = h\left(\xi\right) \in \calN}$ be a measurement defined according to the model in \cref{bins_position_trick_output}, define ${\rho \AtoB{\grpG \times \R^{3}}{\R^{3}}}$ as
    \begin{equation}\label{bins_rho_position}
        \rho\left(X,y\right) \coloneqq A^\top\left(y - b\right).
    \end{equation}
    Then, the configuration output defined in \cref{bins_position_trick_output} is equivariant for any of the ${\grpF, \grpD, \grpB, \grpC}$ symmetries.
\end{lemma}
The noise-free value for $y$ is zero and the output innovation ${\delta\left(\rho_{\hat{X}^{-1}}\left({y}\right)\right) = \rho_{\hat{X}^{-1}}\left({y}\right) - \Vector{}{\pi}{} = \hatVector{}{p}{} - \Vector{}{\pi}{}}$ measures the mismatch of the observer state in reconstructing the true state up to noise in the raw measurement ${\Vector{}{\pi}{}}$. 

\emph{As a final remark, note that this is the same discussion as that in~\cite{7523335, Barrau2022TheProblems} about right/left-invariant observation types. The result presented here directly applies to \ac{iekf} design and allows to reformulate left-invariant observations as right-invariant observations.}

Back to the problem of position-based navigation, let us exploit the equivariance of the output in \cref{bins_rho_position} to derive an output linearization with third-order linearization error for the ${\grpF, \grpD, \grpB, \grpC}$ symmetries. according to \cref{eq_C0,eq_C_star} the output matrices for each of the filters discussed are written
\begin{align*}
    &\text{\ac{mekf}:} 
    && \mathbf{C}^{0} = 
    \begin{bmatrix}
        \mathbf{0}_{3 \times 6} & \eye_3 & \mathbf{0}_{3 \times 6}\\
    \end{bmatrix} \in \R^{3 \times 15},\\
    &\text{Imperfect-\ac{iekf}:} 
    && \mathbf{C}^{\star} = 
    \begin{bmatrix}
        \frac{1}{2}\left(y + \hatVector{}{p}{}\right)^{\wedge} & \mathbf{0}_{3 \times 3} & -\eye_3 & \mathbf{0}_{3 \times 6}\\
    \end{bmatrix} \in \R^{3 \times 15},\\
    &\text{\ac{tfgiekf}:} 
    && \mathbf{C}^{\star} = 
    \begin{bmatrix}
        \frac{1}{2}\left(y + \hatVector{}{p}{}\right)^{\wedge} & \mathbf{0}_{3 \times 3} & -\eye_3 & \mathbf{0}_{3 \times 6}\\
    \end{bmatrix} \in \R^{3 \times 15},\\
    &\text{TG-\ac{eqf}:} 
    && \mathbf{C}^{\star} = 
    \begin{bmatrix}
        \frac{1}{2}\left(y + \hatVector{}{p}{}\right)^{\wedge} & \mathbf{0}_{3 \times 3} & -\eye_3 & \mathbf{0}_{3 \times 9}
        \end{bmatrix} \in \R^{3 \times 18},\\
    &\text{DP-\ac{eqf}:} 
    && \mathbf{C}^{\star} = 
    \begin{bmatrix}
        \mathbf{0}_{3 \times 3} & \mathbf{0}_{3 \times 3} & \mathbf{0}_{3 \times 6} & \eye_3\\
    \end{bmatrix} \in \R^{3 \times 15},\\
    &\text{SD-\ac{eqf}:} 
    && \mathbf{C}^{\star} = 
    \begin{bmatrix}
        \frac{1}{2}\left(y + \hatVector{}{p}{}\right)^{\wedge} & \mathbf{0}_{3 \times 3} & -\eye_3 & \mathbf{0}_{3 \times 6}\\
    \end{bmatrix} \in \R^{3 \times 15}.
\end{align*}
Moreover, for the specific case of the TG-\ac{eqf}, an additional constraint can be imposed on the virtual bias ${\Vector{}{b}{\nu}}$; that is, an additional measurement in the form of ${h\left(\xi\right) = \Vector{}{b}{\nu} = \Vector{}{0}{} \in \R^{3}}$ can be considered, leading to the following output matrix
\begin{equation}
    \mathbf{C}^{0} = \begin{bmatrix}
        \mathbf{0}_{3 \times 3} & \mathbf{0}_{3 \times 3} & \mathbf{0}_{3 \times 3} & -\hatRot{}{}^\top & \mathbf{0}_{3 \times 3} & \hatRot{}{}^\top \hatVector{}{p}{}^{\wedge}
        \end{bmatrix} \in \R^{3 \times 18}. \label{bins_C0_bnu_tgeqf}
\end{equation}

\subsection{Monte-Carlo simulation}

To experimentally confirm the theoretical results of the linearization error analysis, we undertake an extensive experimental evaluation of the position-based navigation problem. Specifically, in this experiment, a Monte-Carlo simulation is conducted, including four hundred runs of a simulated \ac{uav} equipped with an \ac{imu}, receiving acceleration and angular velocity measurements at $200$\si[per-mode = symbol]{\hertz}, as well as global position measurements at $10$\si[per-mode = symbol]{\hertz}, simulating a \ac{gnss} receiver. In order to simulate realistic flight conditions, we selected the initial $80$\si[per-mode = symbol]{\second} from four sequences in the Euroc dataset's vicon room~\cite{Burri2016TheDatasets} as reference trajectories. For each sequence, we generated a hundred runs, incorporating synthetic \ac{imu} data and position measurements while varying the initial conditions for the position (distributed normally around zero with $1$\si[per-mode = symbol]{\meter} standard deviation per axis) and the attitude (distributed normally around zero with $20$\si[per-mode = symbol]{\degree} standard deviation per axis). The ground truth \ac{imu} biases are randomly generated every run following a Gaussian distribution with standard deviation of $0.01$\si[per-mode = symbol]{\radian\per\second\sqrt{\second}} for the gyro bias and $0.01$\si[per-mode = symbol]{\meter\per\second\squared\sqrt{\second}} for the accelerometer bias. 
To simulate realistic global position measurement, additive Gaussian noise with a standard deviation of $0.2$\si[per-mode = symbol]{\meter} per axis is added.

For a fair comparison, we were careful to use the same prior distributions and noise parameters for all filters. 
This includes accounting for the different scaling and transformations of noise due to the input and state parametrizations for the different geometries. 
Similarly, each filter shares the same input and output measurement noise covariance adapted to the particular symmetry of the filter. 
The validity of the noise models can be verified in the Average Energy plot in \cref{bins_mc_errors}, which plots the \acf{anees}~\cite{Li2012EvaluationTests}. 
Here, all filters initialize with unity \ac{anees}, demonstrating that the prior sampling and observer response corresponds to the stochastic prior used, and all filters converge towards unity \ac{anees} as expected from a filter driven by Gaussian noise. 
All the filters are initialized at the identity (zero attitude, zero position, zero velocity, and zero biases).

The primary plots in \cref{bins_mc_errors} are \acs{rmse} plots for the navigation states (on the top) and the bias states (on the bottom).  It is clear that the \ac{mekf} filter demonstrates worse performance than the modern geometric filters. 
There is little difference visible in the transient and asymptotic error response of the navigation states for the modern filters. 
The remaining attitude error is due to a yaw error, which is poorly observable in this scenario.
The position and velocity errors converge to the noise limits of the measurement signals. 
In contrast, there are clear differences visible in the transient response of the bias states.  
The filters are divided roughly into three categories: 
the three filters with semi-direct bias symmetry (TG-\ac{eqf}, DP-\ac{eqf} and SD-\ac{eqf}) appear to display the best transient response in both gyroscope and accelerometer bias. 
The TG-\ac{eqf} appears slightly better in the gyroscope bias. 
The \ac{iekf} and \ac{tfgiekf}, which have the $\SE_2(3)$ symmetry but do not use a semi-direct geometry for the bias geometry, have almost identical bias transient. 
The accelerometer bias, in particular, is clearly separated from the filters with the semi-direct group symmetry. 
Finally, the \ac{mekf} suffers from not modeling the $\SE_2(3)$ symmetry at all. 

The average energy plot provides an additional important analysis tool. 
This plot shows the \ac{anees}~\cite{Li2012EvaluationTests} defined as 
\[
{\text{ANEES} = \frac{1}{nM}\sum_{i = 1}^{M}\varepsilon_{i}^{\top}\mathbf{\Sigma}_{i}^{-1}\varepsilon_{i}},
\]
where $\varepsilon$ is the specific filter error state, $\mathbf{\Sigma}$ is the error covariance, ${M = 400}$ is the number of runs in the Monte-Carlo simulation, and $n$ is the dimension of the state space.  
The \ac{anees} provides a measure of the consistency of the filter estimate. 
An \ac{anees} of unity means that the observed error variance corresponds exactly to the estimated covariance of the information state. 
When \ac{anees} is larger than unity, it indicates that the filter is overconfident; that is, the observed error is larger than the estimate of the state covariance predicts. 
All ``pure'' extended Kalman filters tend to be overconfident since their derivation ignores linearization errors in the model. 
The closer to an \ac{anees} of unity that a filter manages is directly correlated to the consistency of the filter estimate and is usually linked to smaller linearization errors. 
To provide numeric results, we have averaged the \ac{anees} values over the transient and asymptotic sections of the filter response and presented them in \cref{bins_mc_anees}. 
Here, it is clear that the TG-\ac{eqf} is superior, the four filters \ac{iekf}, \ac{tfgiekf}, DP-\ac{eqf} and SD-\ac{eqf} are similar, and the \ac{mekf} is worst. 
The \ac{anees} of the \ac{mekf} diverges to over seven before converging, corresponding to an overconfidence of a factor of seven standard deviations in the state error. 
Such a level of overconfidence is dangerous in a real-world scenario and may indeed lead to divergence of the filter estimate in certain situations. 
Note that in practice, overconfidence in a filter is avoided by inflating the process noise model covariance to account for linearization error in the model. 
A more consistent filter requires a smaller covariance inflation and has correspondingly more confidence in its model than a filter that is less consistent. 

In conclusion, the TG-\ac{eqf} exhibits the best convergence rate, particularly in orientation and \ac{imu} biases, as well as the best consistency of all the filters. 
We believe that this performance can be traced back to the coupling of the \ac{imu} bias with the navigation states that are inherent in the semi-direct product structure of the symmetry group $\grpD$ and the exact linearization of the navigation error dynamics (\cref{bins_symmetries_overview}). 
Note that the bias states are poorly observable states and possess slow dynamics. 
Consequently, moving the linearization error into these states heuristically appears better than leaving the linearization error in the main navigation states that are much more dynamic. 

\ifdefined\includetblr
\begin{table}[htp]
    \centering
    \tabrefcaption{\cite{Fornasier2023EquivariantSystems}}
    \setlength\tabcolsep{5pt}
    \captiontitlefont{\scshape\small}
    \captionnamefont{\scshape\small}
    \captiondelim{}
    \captionstyle{\centering\\}
    \caption[Results of the different filters for the Monte-Carlo simulation of the position-based navigation problem.]{Results of the 400 runs Monte-Carlo simulation for the position-based navigation problem. \acs{anees} in the first half (transient (T)) and the second half (asymptotic (A)) of the trajectory length.}
    \begin{tblr}
    {
        rows = {m},
        column{1} = {c},
        column{2} = {c},
        column{3} = {c},
    }
    \toprule
    Filter & \acs{anees} (T) & \acs{anees} (A) \\
    \midrule
    \ac{mekf} & $3.11$ & $1.69$ \\
    Imperfect-\ac{iekf} & $1.36$ & $1.40$ \\
    \ac{tfgiekf} & $1.71$ & $1.43$ \\
    TG-\ac{eqf} & $\mathbf{1.20}$ & $\mathbf{1.22}$ \\
    DP-\ac{eqf} & $1.44$ & $1.42$ \\
    SD-\ac{eqf} & $1.32$ & $1.44$ \\
    \bottomrule
    \end{tblr}\label{bins_mc_anees}
\end{table}
\fi

\begin{figure}[htp]
\centering
\includegraphics[width=\linewidth]{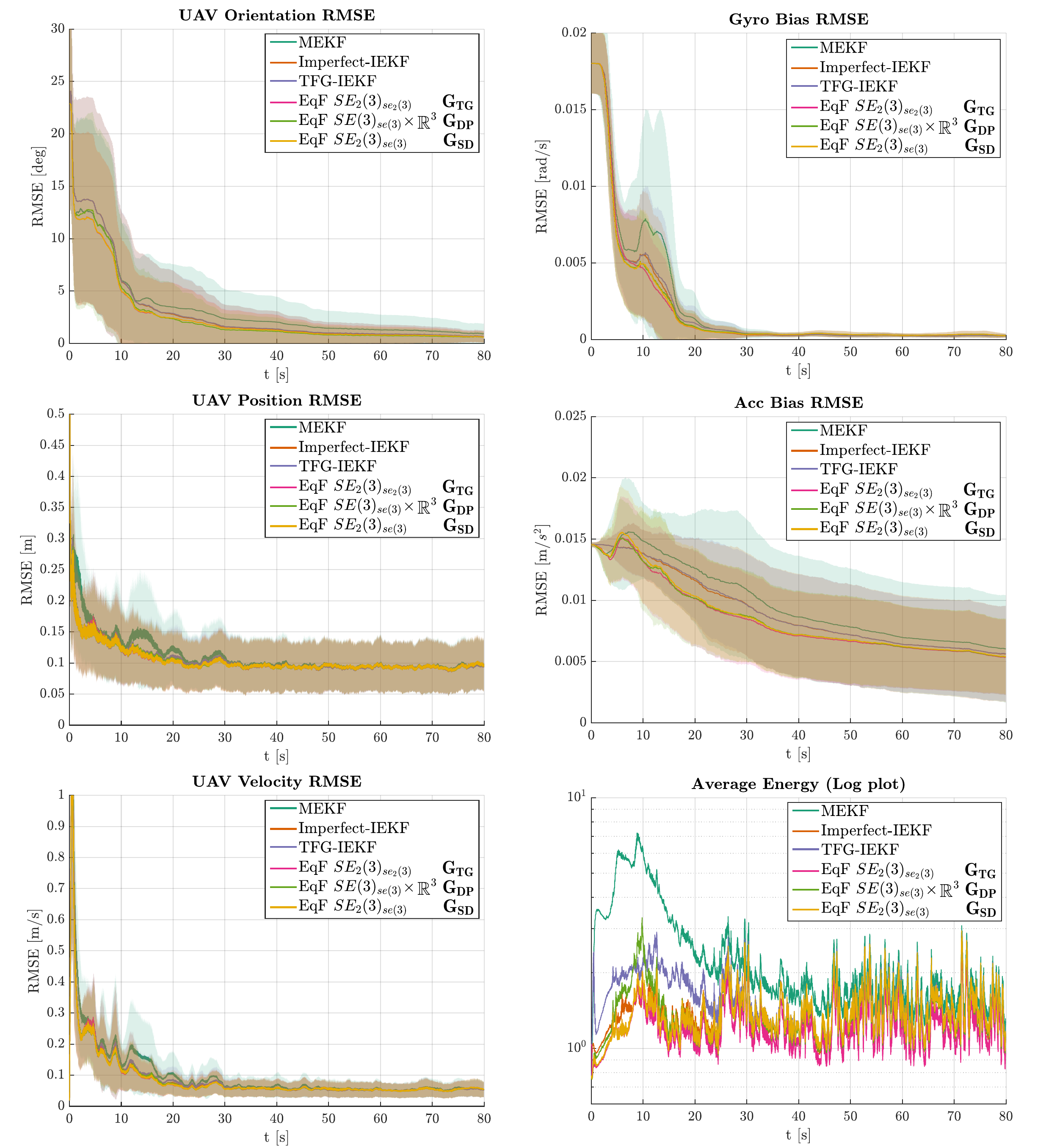}
\figrefcaption{\cite{Fornasier2023EquivariantSystems}}
\caption[Error plots of the different filters for the Monte-Carlo simulation of the position-based navigation problem.]{Results of the 400 runs Monte-Carlo simulation for the position-based navigation problem. Average \acs{rmse} and sample variance (shaded) of the filters' states, and full filter energy. {\color{DarkOrange2} Orange: Imperfect-\ac{iekf}}. {\color{SlateBlue4} Purple: \ac{tfgiekf}}. {\color{Magenta2} Magenta: TG-\ac{eqf}}. {\color{OliveDrab3} Green: DP-\ac{eqf}}. {\color{Goldenrod1} Yellow: SD-\ac{eqf}}.}
\label{bins_mc_errors}
\end{figure}

\section{Chapter conclusion}
With this chapter, we presented the general theory of symmetries and \acl{eqf} design for second-order \aclp{ins}. We discussed different symmetries for the \ac{ins} problem and carried out an analysis of the linearization error as an indicator of the performance of various well-known geometric filters, including \ac{mekf}, Imperfect-\ac{iekf}, and \ac{tfgiekf}. 

Noteworthy considerations, forming the main result of this chapter, were derived from the analysis of the linearization error. \emph{Modeling the coupling between the navigation states and the bias states exploiting the semi-direct product structure of the tangent group of $\SE_2(3)$ yields a filter with exact linearization of the navigation states. Moreover, the choice of symmetry represents the only difference among modern geometric filters for the \ac{ins} problem, and every filter is derived according to the \ac{eqf} design methodology for a specific choice of symmetry.}

The results presented in the linearization error analysis were confirmed through an extensive Monte-Carlo simulation in the context of position-based navigation. As a direct consequence of the problem formulation, we derived an additional result; in particular, \emph{we discussed how global-referenced measurements of the navigation states are reformulated as fixed body-frame measurements that are compatible with the symmetry, and hence yielding equivariance of the output.}

The next chapters will focus on extending and applying the methodology presented and discussed in this chapter to more complex multi-sensor based \aclp{ins}.

\chapter[Equivariant Multi-sensor Fusion for Robust Inertial Navigation Systems][Equivariant Multi-sensor Fusion]{Equivariant Multi-sensor Fusion for Robust Inertial Navigation Systems}\label{ms_bins_chp}

\emph{The present chapter contains results that have been peer-reviewed and published in the IEEE International Conference on Advanced Robotics (ICAR)~\cite{Scheiber2023RevisitingApproach}, as well as results submitted to IEEE International Conference on Robotics and Automation (ICRA).}
\bigskip

Results in the previous chapters have shown that respecting the geometry of the \acl{ins} and exploiting its symmetry is of paramount importance in designing high-performance \ac{ins} filters. The present chapter discusses how the results presented in \cref{bins_chp} are extended and applied to the multi-sensor fusion problem in the context of \acl{ins}. Specifically, this chapter starts by discussing how the symmetries presented in \cref{bins_chp} are extended to account for sensor extrinsic parameters and extra state variables commonly encountered in inertial navigation problems.

The symmetry extension presented here plays a pivotal role in designing robust state estimation algorithms capable of fusing multiple sources of information to achieve robust autonomous navigation. Specifically, in this chapter, we design the first equivariant self-calibrating multi-sensor fusion algorithm for robust state estimation based on \ac{gnss}, magnetic and inertial sensors. The algorithm is targeted for \ac{uav} navigation and is presented as a case study for implementation in the ArduPilot autopilot system, the world's most adopted autopilot system for autonomous navigation of unmanned vehicles.

\section{Inertial navigation system symmetries extension}\label{ms_bins_ext_sec}
Modern inertial navigation and multi-sensor fusion problems seek to estimate the navigation state of a robotic platform by means of multiple sources of information. The rotation, position, velocity, and \ac{imu} bias states are often extended with extra variable modeling parameters of interest. In particular, extra states include sensor calibration states such as \ac{gnss} receivers lever arms, magnetometers rotations, cameras, and lidars roto-translations, as well as other physical states of interest such as the geomagnetic direction and visual landmarks. In this section, we show how the symmetries presented in the previous chapter are extended to account for these extra variables, and we provide examples that cover common situations in inertial navigation and multi-sensor fusion problems.

Let $\xi_{I} \in \calM_{\calI}$ be the \ac{ins} states of an \acl{ins}, thus, rotation, velocity, position, and \ac{imu} biases. Let $\xi_{C} \in \calM_{\calC}$ be extra states of interest. The full system's state $\xi \in \calM$ is posed on the direct product ${\calM \coloneqq \calM_{\calI} \times \calM_{\calC}}$. Let $\grpG_{\mathbf{I}}$ be the symmetry group of the \ac{ins} states, and $\grpG_{\mathbf{C}}$ be the symmetry group of the extra states; the system's symmetry group $\grpG$ is given by the product ${\grpG \coloneqq \grpG_{\mathbf{I}} \cdot \grpG_{\mathbf{C}}}$, where $\cdot$ represents either a direct product or a semi-direct product.

This structure will be exploited to concatenate and compose symmetries to account for different extra states in self-calibrating inertial navigation systems and multi-sensor fusion problems. In particular, \cref{ms_bins_states_table} provides an overview of the base symmetry groups modeling typical states encountered in \aclp{ins}. 

\ifdefined\includetblr
\begin{table}[htp]
    \centering
    \setlength\tabcolsep{5pt}
    \captiontitlefont{\scshape\small}
    \captionnamefont{\scshape\small}
    \captiondelim{}
    \captionstyle{\centering\\}
    \caption{Symmetry extensions for common extra states commonly encountered in inertial navigation problems}
    \begin{tblr}
    {
        rows = {m},
        column{1} = {c},
        column{2} = {c},
        column{2} = {c},
        column{3} = {c},
    }
    \toprule
    State variable & Symmetry group & Product & State action $\phi$ \\
    \midrule
    \SetCell[]{c}{
    Rotational calibration\\
    $\PoseS{}{} \in \mathcal{SO}(3)$
    } & $X \in \SO(3)$ & $\times$ & $A^{\top}\PoseS{}{}X$ \\
    \SetCell[]{c}{
    Pose calibration\\
    $\PoseS{}{} \in \mathcal{SE}(3)$
    } & $X \in \SE(3)$ & $\times$ & $C^{-1}\PoseS{}{}X$ \\
    \SetCell[]{c}{
    Position calibration\\
    of 3D global measurement\\
    $\Vector{}{t}{} \in \R^3$
    } & $X \in \R^3$ & $\ltimes$ & $A^{\top}(\Vector{}{t}{}-X)$ \\
    \SetCell[]{c}{
    Direction state\\
    $\Vector{}{d}{} \in \mathcal{S}^2$
    } & $X \in \SO(3)$ & $\times$ & $\Rot{}{}X^{\top}\Rot{}{}^{\top}\Vector{}{d}{}$ \\
    \SetCell[]{c}{
    Vector state\\
    $\Vector{}{l}{} \in \R^3$
    } & $X \in \SOT(3)$ & $\times$ & $\Pose{}{}X^{-1}\Pose{}{}^{-1}\Vector{}{l}{}$ \\
    \bottomrule
    \end{tblr}
    \label{ms_bins_states_table}
\end{table}
\fi

The following examples will clarify how these base symmetry groups are used and combined to build complex symmetries for generic systems. Note that although we choose $\grpD$ to be the symmetry group of the \ac{ins} state in the following examples, the methodology presented in this chapter is not limited to this choice, and the symmetry group of the \ac{ins} state could be any of the symmetry groups presented in \cref{bins_chp}.

\begin{boxexample}{Self-calibrating position-based navigation}{}
    Consider the \acl{ins} discussed in \cref{bins_chp,bins_bins_se23}. Let ${\xi_{I} \in \torSE_2(3) \times \R^6}$ be an element of the \ac{ins} state space, and $\grpD$ be the symmetry group. Consider the case a global position sensor is attached to the robotic platform with an unknown lever arm. Let $\xi_{C} = \Vector{}{t}{} \in \R^3$ be an extra state variable representing this unknown lever arm. The full system states is defined as ${\xi = (\xi_{I}, \xi_{C}) \in \calM \coloneqq \torSE_2(3) \times \R^6 \times \R^3}$. Define $\grpG \coloneqq \grpD \ltimes \R^3$ to be the symmetry group, and ${A = \Gamma(D) \in \SO(3) \subset \grpD}$, with $\Gamma(\cdot)$ defined in \cref{math_maps_sec}. Then elements of the symmetry group and the state action ${\phi \AtoB{\grpG \times \calM}{\calM}}$ are defined as follows:
    \begin{align*}
        &X = (X_{I}, X_{C}) \in \grpG,\\
        &X^{-1} = (X_{I}^{-1}, -A^{\top}X_{C}) \in \grpG,\\
        &\phi(X, \xi) = (\cdots, A^{\top}(\Vector{}{t}{}-X_{C})) \in \calM.\\
    \end{align*}
    
    Consider the following measurement model:
    \begin{equation*}
        h\left(\xi\right) = \Rot{}{}^\top\left(\Vector{}{\pi}{} - (\Vector{}{p}{} + \Rot{}{}\Vector{}{t}{})\right) \in \R^{3} .
    \end{equation*}
    Let ${y = h\left(\xi\right) \in \R^3}$ be a measurement defined according to the model above, then the output is equivariant under the symmetry group $\grpG$. Define ${\rho \AtoB{\grpG \times \R^{3}}{\R^{3}}}$ as
    \begin{equation*}
        \rho_{X}(y) = A^{\top}(y - b + X_{C}).
    \end{equation*}
    Then, it is straightforward to verify that ${h(\phi_{X}(\xi)) = \rho_{X}(h(\xi))}$.

    With this example, we showed how the $\grpD$ symmetry of the \ac{ins} state is extended to account for the lever arm of a global position sensor and how the resulting symmetry not only preserves the property of the $\grpD$ symmetry discussed in \cref{bins_tg_sec,bins_tg_linan} but also possess equivariance in the output.
\end{boxexample}

\begin{boxexample}[label={ms_bins_sot3_example}]{Camera-based navigation with camera extrinsic self-calibration and single visual landmark}{}
    Consider the \acl{ins} discussed in \cref{bins_chp,bins_bins_se23}. Let ${\xi_{I} \in \torSE_2(3) \times \R^6}$ be an element of the \ac{ins} state space, and $\grpD$ be the symmetry group. Consider the case of a camera mounted on a robotic platform with an unknown roto-translation, observing a single visual landmark in the world. Let ${\xi_{C} = (\PoseS{}{}, \Vector{}{l}{}) \in \torSE(3) \times \R^3}$ represent the unknown camera extrinsic calibration and the visual landmark. The full system states is defined as ${\xi = (\xi_{I}, \xi_{C}) \in \calM \coloneqq \torSE_2(3) \times \R^6 \times \torSE(3) \times \R^3}$. Define $\grpG \coloneqq \grpD \ltimes \SE(3) \times \SOT(3)$ to be the symmetry group, and ${C = \Theta(D) \in \SE(3) \subset \grpD}$, with $\Theta(\cdot)$ defined in \cref{math_maps_sec}. Then elements of the symmetry group and the state action ${\phi \AtoB{\grpG \times \calM}{\calM}}$ are defined as follows:
    \begin{align*}
        &X = (X_{I}, X_{C}, X_{L}) \in \grpG,\\
        &X^{-1} = (X_{I}^{-1}, X_{C}^{-1}, (X_{Q}^{\top}, \frac{1}{X_{q}})) \in \grpG,\\
        &\phi(X, \xi) = (\cdots, C^{-1}\mathbf{S}X_{C}, \Pose{}{}\PoseS{}{}X_{C} * (X_{L}^{-1}(\Pose{}{}\PoseS{}{})^{-1} * \Vector{}{l}{}) \in \calM,
    \end{align*}
    with ${X_{L} \coloneqq (X_{Q}, X_{q}) \in \SOT(3)}$, and $* \AtoB{\torSE(3) \times \R^3}{\R^3}$ defined by ${\PoseP{}{} * \Vector{}{x}{} = \Rot{}{}\Vector{}{x}{} + \Vector{}{p}{}}$  for all ${\PoseP{}{} = \left(\Rot{}{}, \Vector{}{p}{}\right) \in \torSE(3),\; \Vector{}{x}{} \in \R^3}$.

    Consider the following measurement model:
    \begin{equation*}
        h\left(\xi\right) = \pi_{\mathbb{S}^{2}}((\Pose{}{}\PoseS{}{})^{-1} * \Vector{}{l}{}) \in \mathbb{S}^{2},
    \end{equation*}
    where ${\pi_{\mathbb{S}^{2}}(\cdot)}$ represents the projection on the unit sphere.
    Let ${y = h\left(\xi\right) \in \mathbb{S}^{2}}$ be a measurement defined according to the model above, then the output is equivariant under the symmetry group $\grpG$. Define ${\rho \AtoB{\grpG \times S^{2}}{S^{2}}}$ as
    \begin{equation*}
        \rho_{X}(y) = X_{Q}^{\top}y.
    \end{equation*}
    Then, it can be verified that
    \begin{align*}
        h(\phi_{X}(\xi)) &= \pi_{\mathbb{S}^{2}}((\Pose{}{}\PoseS{}{}X_{C})^{-1}\Pose{}{}\PoseS{}{}X_{C} * (X_{L}^{-1}(\Pose{}{}\PoseS{}{})^{-1} * \Vector{}{l}{}))\\
        &= \pi_{\mathbb{S}^{2}}(X_{L}^{-1}(\Pose{}{}\PoseS{}{})^{-1} * \Vector{}{l}{})\\
        &= \frac{\frac{1}{X_{q}}X_{Q}^{-1}(\Pose{}{}\PoseS{}{})^{-1} * \Vector{}{l}{}}{\norm{\frac{1}{X_{q}}X_{Q}^{-1}(\Pose{}{}\PoseS{}{})^{-1} * \Vector{}{l}{}}}\\
        &= \frac{\cancel{\frac{1}{X_{q}}}X_{Q}^{-1}(\Pose{}{}\PoseS{}{})^{-1} * \Vector{}{l}{}}{\cancel{\frac{1}{X_{q}}}\norm{X_{Q}^{-1}(\Pose{}{}\PoseS{}{})^{-1} * \Vector{}{l}{}}}\\
        &= X_{Q}^{-1}\frac{(\Pose{}{}\PoseS{}{})^{-1} * \Vector{}{l}{}}{\norm{(\Pose{}{}\PoseS{}{})^{-1} * \Vector{}{l}{}}}\\
        &= X_{Q}^{-1}\pi_{\mathbb{S}^{2}}((\Pose{}{}\PoseS{}{})^{-1} * \Vector{}{l}{})\\
        &= \rho_{X}(h(\xi)).
    \end{align*}

    With this example, we showed how the $\grpD$ symmetry of the \ac{ins} state is extended to account for extra states representing a fixed visual landmark, and camera extrinsic calibration. Specifically, we showed how multiple extensions of the $\grpD$ symmetry, shown in \cref{ms_bins_states_table}, are combined together to achieve output equivariance for the given measurement.
\end{boxexample}

\section{Symmetry extension for equivariant robust state estimation: the ArduPilot case study}\label{ms_bins_ap_sec}

This section addresses the challenge of designing an \acl{eqf} for the ArduPilot autopilot system, providing a comprehensive case study that combines the results presented so far and applies the developed theoretical framework to real-world scenarios. To this end, we propose a filter that is capable of fusing multiple sources of information from \ac{gnss}, magnetometer, and inertial sensors. Moreover, it incorporates an \emph{equivariant formulation of velocity-type measurements} and a simple \emph{innovation-covariance inflation strategy that seamlessly handles \ac{gnss} outliers and shifts} without requiring coding of a whole set of exception cases. We use the data wealth of the ArduPilot community to identify, highlight, and address the most common real-world challenges in \ac{ins} state estimation, including sensor self-calibration, robustness in static conditions, \ac{gnss} outliers and shifts, and robustness to faulty \acp{imu}. The proposed filter is evaluated with simulated and real-world data from the Ardupilot community to demonstrate its performance in known cases where existing filters fail without careful exception handling or case-specific tuning. 

\subsection{System and symmetry definition}
Consider a robotic platform equipped with an \ac{imu}, a \ac{gnss} receiver, and a magnetometer. Consider the system in \cref{bins_chp,bins_bins} extended with two extra variables modeling the rotational calibration of the magnetometer and the lever arm of the \ac{gnss} receiver. 
\begin{subequations}\label{ms_bins_ap_bins}
    \begin{align}
        &\dot{\Rot{G}{I}} = \Rot{G}{I}\left(\Vector{I}{\bm{\omega}}{} - \Vector{I}{b}{\bm{\omega}}\right)^{\wedge} ,\\
        &\dotVector{G}{v}{I} =  \Rot{G}{I}\left(\Vector{I}{a}{} - \Vector{I}{b}{a}\right) + \Vector{G}{g}{} ,\\
        &\dotVector{G}{p}{I} = \Vector{G}{v}{I} ,\\
        &\dotVector{I}{b_{\bm{\omega}}}{} = \Vector{I}{\tau}{\bm{\omega}} ,\\
        &\dotVector{I}{b_{a}}{} = \Vector{I}{\tau}{a} ,\\
        &\dotVector{I}{t}{} = \Vector{I}{\zeta}{} ,\\
        &\dot{\Rot{I}{M}} = \Rot{I}{M}\Vector{M}{\mu}{}^{\wedge} .
    \end{align}
\end{subequations}
$\Rot{G}{I}$ denotes the rigid body orientation, and $\Vector{G}{p}{I}$ and $\Vector{G}{v}{I}$ denote the rigid body position and velocity expressed in a global frame of reference \frameofref{G}. $\Vector{G}{g}{}$ denotes the gravity vector expressed in \frameofref{G}. $\Vector{I}{t}{}$ and $\Rot{I}{M}$ represent the \ac{gnss} lever arm and the magnetometer rotational calibration respectively. The inputs $\Vector{I}{\zeta}{}$, $\Vector{M}{\mu}{}$ are used to model the calibration states dynamics and are zero for rigid calibrations. The gyroscope and accelerometer measurement are written $\Vector{I}{\bm{\omega}}{}$ and $\Vector{I}{\bm{a}}{}$ respectively. $\Vector{I}{b_{\bm{\omega}}}{}$ and $\Vector{I}{b_{\bm{a}}}{}$ represent their biases. The inputs $\Vector{I}{\tau}{\bm{\omega}}$, $\Vector{I}{\tau}{a}$ are used to model the biases' dynamics and are zero when the biases are modeled as constant quantities. 

An element $\xi$ of the system's state and an element $u$ of the input space are respectively written
\begin{align*}
&\xi =  \left(\Rot{G}{I}, \Vector{G}{v}{I}, \Vector{G}{p}{I}, \Vector{I}{b_{\bm{\omega}}}{}, \Vector{I}{b_{a}}{}, \Vector{I}{t}{}, \Rot{I}{M}\right) \in \calM ,\\
&u = \left(\Vector{I}{\bm{\omega}}{}, \Vector{I}{a}{}, \Vector{I}{\tau}{\bm{\omega}}, \Vector{I}{\tau}{a}, \Vector{I}{\zeta}{}, \Vector{I}{\mu}{}\right) \in \vecL \subset \R^{18} .
\end{align*}

Furthermore, Define the matrices $^I\mathbf{W}, ^I\mathbf{B}, \mathbf{N}, ^G\mathbf{G}$ to be
\begin{equation*}
    \begin{tblr}{ll}
        ^I\mathbf{W} = \begin{bmatrix}
        \Vector{I}{\bm{\omega}}{}^{\wedge} & \Vector{I}{a}{} & \mathbf{0}_{3\times 1}\\
        \mathbf{0}_{1\times 3} & 0 & 0\\
        \mathbf{0}_{1\times 3} & 0 & 0\\
        \end{bmatrix}, & 
        ^I\mathbf{B} = \begin{bmatrix}
        \Vector{I}{b_{\bm{\omega}}}{}^{\wedge} & \Vector{I}{b_{a}}{} & \mathbf{0}_{3\times 1}\\
        \mathbf{0}_{1\times 3} & 0 & 0\\
        \mathbf{0}_{1\times 3} & 0 & 0\\
        \end{bmatrix}, \\
        \mathbf{N} = \begin{bmatrix}
        \mathbf{0}_{3\times 3} & \mathbf{0}_{3\times 1} & \mathbf{0}_{3\times 1}\\
        \mathbf{0}_{1\times 3} & 0 & 1\\
        \mathbf{0}_{1\times 3} & 0 & 0\\
        \end{bmatrix}, &
        ^G\mathbf{G}  = \begin{bmatrix}
        \mathbf{0}_{3\times 3} & \Vector{G}{g}{} & \mathbf{0}_{3\times 1}\\
        \mathbf{0}_{1\times 3} & 0 & 0\\
        \mathbf{0}_{1\times 3} & 0 & 0\\
        \end{bmatrix}.
    \end{tblr}
\end{equation*}
Than the system in \cref{ms_bins_ap_bins} can be written in a compact form similar to \cref{bins_bins_se23}:
\begin{subequations}\label{ms_bins_ap_bins_se23}
    \begin{align}
        &\dot{\Pose{}{}} = \Pose{G}{I}\left(^I\mathbf{W} - ^I\mathbf{B} + \mathbf{N}\right) + \left(^G\mathbf{G} - \mathbf{N}\right)\Pose{G}{I} ,\\
        &\dotVector{}{b}{} = \Vector{I}{\tau}{} ,\\
        &\dotVector{I}{t}{} = \Vector{I}{\zeta}{} ,\\
        &\dot{\Rot{I}{M}} = \Rot{I}{M}\Vector{M}{\mu}{}^{\wedge} .
    \end{align}
\end{subequations}

Define the three measurement models of interest.
First, consider the case where measurements of the known magnetic north direction ${\Vector{G}{m}{}}$ are received in the magnetometer frame \frameofref{M}, thus $\Vector{M}{m}{}$.
The output space associated with such measurements is $\calN \coloneqq \mathbb{S}^{2}$, and the configuration output ${h_m \AtoB{\calM}{\calN_m}}$ is given by
\begin{equation}\label{ms_bins_confout_dir}
h_m\left(\xi\right) = \Rot{I}{M}^\top\Rot{G}{I}^\top\Vector{G}{m}{} \in \calN_m .
\end{equation}
Second, consider the case where position measurements ${\Vector{G}{\pi}{} = \Vector{G}{p}{I} + \Rot{G}{I}\Vector{I}{t}{}}$, are received from a \ac{gnss} receiver. The associated output space is $\calN_{p} \coloneqq \R^{3}$. The configuration output ${h_p \AtoB{\calM}{\calN_p}}$ is defined according to \cref{bins_position_trick_output}
\begin{equation}\label{ms_bins_confout_pos}
     h_p\left(\xi\right) = \Rot{G}{I}^\top\left(\Vector{G}{\pi}{} - \left(\Vector{G}{p}{I} + \Rot{G}{I}\Vector{I}{t}{}\right)\right) \in \calN_p .
\end{equation}
Finally, consider the case where velocity measurements ${\Vector{G}{\nu}{} = \Vector{G}{v}{I} + \Rot{G}{I}\Vector{I}{\omega}{}^{\wedge}\Vector{I}{t}{}}$, are received from a \ac{gnss} receiver.
The associated output space is $\calN_{v} \coloneqq \R^{3}$. To achieve third-order linearization error of the output map through equivariance, we propose the idea of extending the configuration output in \cref{bins_position_trick_output,ms_bins_confout_pos} to the velocity measurements. Thus, we define the configuration output ${h_v \AtoB{\calM \times \vecL}{\calN_v}}$ as
\begin{equation}\label{ms_bins_confout_vel}
    h_v\left(\xi, u\right) = \Rot{G}{I}^\top\left(\Vector{G}{\nu}{} - \left(\Vector{G}{v}{I} + \Rot{G}{I}\Vector{I}{\omega}{}^{\wedge}\Vector{I}{t}{}\right)\right) \in \calN_v .
\end{equation}
In the latter two measurement models, position and velocity measurements for the filter are reformulated by constructing the vectors $\Vector{G}{\pi}{}$, and $\Vector{G}{\nu}{}$ with raw, noisy position and velocity measurements from the sensors. 
For the sake of readability, subscripts and superscripts are omitted in the following sections.

For the \ac{ins} states in \cref{ms_bins_ap_bins}, we have chosen the best-performing symmetry with minimal state representation, that is, the ${\grpB \coloneqq \left(\SE_2(3) \ltimes \gothse(3)\right)}$ symmetry, presented in \cref{bins_sdb_sec,bins_sd_linan}. The $\grpB$ symmetry is extended to account for calibration states as described at the beginning of this chapter in \cref{ms_bins_ext_sec} and in \cref{ms_bins_states_table}. Define $\grpG \coloneqq \grpD \ltimes \R^{3} \times \SO(3)$ to be the symmetry group, and ${A = \Gamma(D) \in \SO(3) \subset \grpB}$, ${B = \chi(D) \in \SE(3) \subset \grpB}$, ${C = \Theta(D) \in \SE(3) \subset \grpB}$, with ${\Gamma(\cdot), \chi(\cdot), \Theta(\cdot)}$ defined in \cref{math_maps_sec}. Elements of the symmetry group and the state action ${\phi \AtoB{\grpG \times \calM}{\calM}}$ are defined as follows:
\begin{align*}
    &X = ((D, \delta), \gamma, E) \in \grpG,\\
    &X^{-1} = ((D^{-1}, -\Adsym{D^{-1}}{\delta}), -A^{\top}\gamma, E^{\top}),\\
    &\phi(X, \xi) = \left(\Pose{}{}D, \AdMsym{B^{-1}}\left(\Vector{}{b}{} - \delta^{\vee}\right), A^\top\left(\Vector{}{t}{} - \gamma\right), A^\top\mathbf{S}E\right),
\end{align*}

The actions of the symmetry group on the output space for the outputs in \cref{ms_bins_confout_dir,ms_bins_confout_pos} are similar to those in \cref{bas_rho,bins_rho_position}, and are written
\begin{align}
    &\rho_m \AtoB{\grpG \times \calN_m}{\calN_m} && \rho_m(X, y_m) = E^\top y_m,\label{ms_bins_rho_m}\\
    &\rho_p \AtoB{\grpG \times \calN_p}{\calN_p} && \rho_p(X, y_p) = A^\top\left(y_p - b + \gamma^{\vee}\right),\label{ms_bins_rho_p}\
\end{align}
For what concerns the configuration output in \cref{ms_bins_confout_vel}, the results presented in \cref{eq_sym_sec}, in particular the output equivariance condition in \cref{eq_out_equi}, need to be extended to measurements model that depends on the input variable $u$. Such configuration output is said to be equivariant if 
\begin{equation}\label{ms_bins_output_equi_extended}
    h(\phi(X,\xi), \theta(X,u)) = \rho(X, u, h(\xi, u)).
\end{equation}
$\forall X \in \grpG$, $\xi \in \calM, u \in \vecL$, and for right-handed actions $\theta \AtoB{\grpG \times \vecL}{\vecL}$, and $\rho \AtoB{\grpG \times \vecL \times \calN}{\calN}$. 
\begin{theorem}
Define $\rho_v \AtoB{\grpG \times \vecL \times \calN_v}{\calN_v}$, and $theta \AtoB{\grpG \times \vecL}{\vecL}$ as
\begin{align}
    &\rho_v \AtoB{\grpG \times \vecL \times \calN_v}{\calN_v} && \rho_v(X, u, y_v) = A^{\top}(y_v - a + \Vector{}{\omega}{}^{\wedge}\gamma),\label{ms_bins_rho_v}\\
    &\theta \AtoB{\grpG \times \vecL}{\vecL} && \theta(X, u) = \left(A^{\top}\Vector{}{\omega}{}, \Vector{}{a}{}, \Vector{}{\tau}{\bm{\omega}}, \Vector{}{a}{}, \Vector{}{\zeta}{}, \Vector{}{\mu}{}\right) .
\end{align}
Then, according to \cref{ms_bins_output_equi_extended}, the configuration output in \cref{ms_bins_confout_vel} is equivariant under actions $\rho_v$ and $\theta$ of the symmetry group onto the output space $\calN_v$ and input space $\vecL$ respectively.
\end{theorem}
\begin{proof}
Expanding $h_v(\phi(X,\xi), \theta(X,u))$ in \cref{ms_bins_output_equi_extended} yields
\begin{align*}
    h_v(\phi(X,\xi), \theta(X,u)) &= A^{\top}\Rot{}{}^\top\left(\Vector{}{\nu}{} - \left(\Vector{}{v}{} + \Rot{}{}a + \Rot{}{}A\left(A^{\top}\Vector{}{\omega}{}\right)^{\wedge}A^{\top}\left(\Vector{}{t}{} - \gamma\right)\right)\right)\\
    &= A^{\top}\left(\Rot{}{}^\top\left(\Vector{}{\nu}{} - \left(\Vector{}{v}{} + \Rot{}{}\left(AA^{\top}\Vector{}{\omega}{}\right)^{\wedge}\left(\Vector{}{t}{} - \gamma\right)\right)\right) - a\right)\\
    &= A^{\top}\left(\Rot{}{}^\top\left(\Vector{}{\nu}{} - \left(\Vector{}{v}{} + \Rot{}{}\Vector{}{\omega}{}^{\wedge}\Vector{}{t}{}\right)\right) + \Vector{}{\omega}{}^{\wedge}\gamma- a\right)\\
    &= A^{\top}\left(h_v\left(\xi, u\right) + \Vector{}{\omega}{}^{\wedge}\gamma- a\right)\\
    &= \rho_v(X, u, h\left(\xi, u\right)) ,
\end{align*}
which demonstrates equivariance of the output map in \cref{ms_bins_confout_vel} with respect to the symmetry group $\grpG$.
\end{proof}

The existence of a transitive group action $\phi$ of the symmetry group $\grpG$ on the state space $\calM$ guarantees the existence of a lift $\Lambda \AtoB{\calM \times \vecL}{\gothg}$.
\begin{theorem}
Define the lift ${\Lambda\left(\xi, u\right)}$ with the four maps ${\Lambda_1\left(\xi, u\right), \Lambda_2\left(\xi, u\right), \Lambda_3\left(\xi, u\right), \Lambda_4\left(\xi, u\right)}$ as follows:
\begin{align}\label{ms_bins_ap_lift}
    &\Lambda_1\left(\xi, u\right) \coloneqq \left(\mathbf{W} - \mathbf{B} + \mathbf{D}\right) + \Pose{}{}^{-1}\left(\mathbf{G} - \mathbf{D}\right)\Pose{}{} ,\\
    &\Lambda_2\left(\xi, u\right) \coloneqq \adsym{\Vector{}{b}{}^{\wedge}}{\Pi\left(\Lambda_1\left(\xi, u\right)\right)},\\
    &\Lambda_3\left(\xi, u\right) \coloneqq \Vector{}{t}{}^{\wedge}\left(\Vector{}{\omega}{} - \Vector{}{b}{\bm{\omega}}\right),\\
    &\Lambda_4\left(\xi, u\right) \coloneqq \mathbf{S}^\top \left(\Vector{}{\omega}{} - \Vector{}{b}{\bm{\omega}}\right).
\end{align}
The map ${\Lambda\left(\xi, u\right)}$ is a lift for the system in \cref{ms_bins_ap_bins,ms_bins_ap_bins_se23} with respect to the symmetry group ${\grpG}$.
\end{theorem}

\subsection{Equivariant filter design}
Consider the state origin $\xizero$ to be the identity of the state space, thus ${\xizero = \left(\eye_3, \Vector{}{0}{3\times1}, \Vector{}{0}{3\times1}, \Vector{}{0}{3\times1}, \Vector{}{0}{3\times1}, \Vector{}{0}{3\times1}, \eye_3\right)}$. Local coordinates of the state space are chosen to be normal coordinates, hence ${e = \vartheta^{-1}(\varepsilon) \coloneqq \phi_{\xizero}(\exp_{\grpG}(\varepsilon^{\wedge})))}$.

The linearized error state matrix ${\mathbf{A}_{t}^{0} \st \dot{\varepsilon} \simeq \mathbf{A}_{t}^{0}\varepsilon}$ is the solution of \cref{eq_A0_alt}, and it is written
\begin{equation}
    \mathbf{A}_{t}^{0} = \begin{bNiceArray}{ccc:cc:c:c}[margin]
        \Block{3-3}{\prescript{}{1}{\mathbf{A}}} & & & \Block{2-2}{\eye_6} & & \Block{2-1}{\mathbf{0}_{6 \times 3}} & \Block{2-1}{\mathbf{0}_{6 \times 3}}\\
        & & & & & &\\
        & & & \hat{b}^{\wedge} & \mathbf{0}_{3 \times 3} & \mathbf{0}_{3 \times 3} & \mathbf{0}_{3 \times 3}\Bstrut\\
        \hdottedline
        \mathbf{0}_{6 \times 3} & \mathbf{0}_{6 \times 3} & \mathbf{0}_{6 \times 3} & \Block{1-2}{\prescript{}{2}{\mathbf{A}}} & & \mathbf{0}_{6 \times 3} & \mathbf{0}_{6 \times 3}\Tstrut\Bstrut\\
        \hdottedline
        \mathbf{0}_{6 \times 3} & \mathbf{0}_{6 \times 3} & \mathbf{0}_{6 \times 3} & \mathbf{0}_{3 \times 3} & \mathbf{0}_{3 \times 3} & \prescript{}{3}{\mathbf{A}} & \mathbf{0}_{3 \times 3}\Tstrut\Bstrut\\
        \hdottedline
        -\prescript{}{3}{\mathbf{A}} & \mathbf{0}_{6 \times 3} & \mathbf{0}_{6 \times 3} & \eye_3 & \mathbf{0}_{3 \times 3} & \mathbf{0}_{3 \times 3} & \prescript{}{3}{\mathbf{A}}\Tstrut
    \end{bNiceArray} \in \R^{21 \times 21}, \label{ms_bins_ap_At0}
\end{equation}
where the upper-left $15 \times 15$ block including $\prescript{}{1}{\mathbf{A}}$, and $\prescript{}{2}{\mathbf{A}}$, is defined as in \cref{bins_At0_SD} and $\prescript{}{3}{\mathbf{A}}$ represents the rotational component of $\prescript{}{2}{\mathbf{A}}$.
\begin{align*}
    &\prescript{}{1}{\mathbf{A}} = \begin{bmatrix}
    \Vector{}{0}{3 \times 3} & \Vector{}{0}{3 \times 3} & \Vector{}{0}{3 \times 3}\\
    \Vector{G}{g}{}^{\wedge} & \Vector{}{0}{3 \times 3} & \Vector{}{0}{3 \times 3}\\
    \Vector{}{0}{3 \times 3} & \eye_3 & \Vector{}{0}{3 \times 3}
    \end{bmatrix} \in \R^{9 \times 9},\\
    &\prescript{}{2}{\mathbf{A}} = \adMsym{\left(\AdMsym{\hat{B}}\Vector{}{w}{} + \hat{\delta}^{\vee} + \mathbf{G}^{\vee}\right)} \in \R^{6 \times 6}\\
    &\prescript{}{3}{\mathbf{A}} = \left(\hat{A}\Vector{}{\omega}{} + \hat{\delta}_{\omega}\right)^{\wedge} \in \R^{3 \times 3}.
\end{align*}

Let $y_m, y_p, y_v$ represent raw, noisy measurements of direction, position, and velocity, respectively. Then, given the configuration output in \cref{ms_bins_confout_dir,ms_bins_confout_pos,ms_bins_confout_vel}, the actions on the output space defined in \cref{ms_bins_rho_m,ms_bins_rho_p,ms_bins_rho_v}, and the use of normal coordinates, the three linearized output matrices are derived as solutions of \cref{eq_C_star}:
\begin{align}
    &\mathbf{C}_m^{\star} = \Vector{G}{m}{}^{\wedge} \begin{bmatrix}
        \mathbf{0}_{3 \times 15} & \frac{1}{2}\left(\Vector{G}{m}{} + \hat{E}y_d\right)^{\wedge} & \mathbf{0}_{3 \times 3}
        \end{bmatrix}, \label{ms_bins_Cstar_dir}\\
    &\mathbf{C}_p^{\star} = \begin{bmatrix}
        \frac{1}{2}\left(y_p + \hat{b} - \hat{d}\right)^{\wedge} & \mathbf{0}_{3 \times 3} & -\eye_3 & \mathbf{0}_{3 \times 6} & \eye_3 & \mathbf{0}_{3 \times 3}
        \end{bmatrix}, \label{ms_bins_Cstar_pos}\\
     &\mathbf{C}_v^{\star} = \begin{bmatrix}
        \frac{1}{2}\left(y_v + \hat{a} - \Vector{}{\omega}{}^{\wedge}\hat{d}\right)^{\wedge} & -\eye_3 & \mathbf{0}_{3 \times 9} &  \Vector{}{\omega}{}^{\wedge} & \mathbf{0}_{3 \times 3}\\
        \end{bmatrix}. \label{ms_bins_Cstar_vel}
\end{align}

\subsection{Uncertain observation handling}\label{ms_bins_unc_obs_handl}
Uncertain observation and signal outages are particularly common in inertial navigation problems. Specifically, \ac{gnss} outliers, commonly referred to as glitches, as well as sudden changes in the \ac{gnss} solution, commonly referred to as \ac{gnss} shifts, distortion of the measured magnetic field, are among the most common signal outages experienced in inertial navigation scenarios.

Typical approaches to overcome such issues rely on $\chi^2$ rejection tests. However, binary rejection strategies are often insufficient and very sensitive to tuning parameters. In the present subsection, we present the strategy for robust uncertain observation handling, based on the concept of generalized covariance union \ac{gcu}~\cite{Reece2010GeneralisedTracking, Ghobadi2018RobustFilter} for which tuning boils down to a simple choice of convergence rate.

Let $y$ represent an observation and ${\mathbf{R}}$ the observation covariance. Let $\mathbf{C}^{\star}$ represent the linearized output matrix, and ${\mathbf{\Sigma}}$ be the estimated state covariance.
Let $\Vector{}{r}{} = \delta(\rho(\hat{X}^{-1}, y))$ be the residual computed via the output action, and $\mathbf{S}$ be the innovation covariance. 
Let ${\alpha \in \left[0,\, 1\right]}$ be a scalar value used to control the convergence rate (the lower, the faster). 
Then, before updating the filter states according to~\cite{vanGoor2022EquivariantEqF}, we compute an inflated innovation covariance $\mathbf{S}^{\prime}$ as follows:
\begin{align*}
    &n = \Vector{}{r}{}^\top\mathbf{S}^{-1}\Vector{}{r}{} = \Vector{}{r}{}^\top\left(\mathbf{C}^{\star}\mathbf{\Sigma}{\mathbf{C}^{\star}}^\top + \mathbf{R}\right)^{-1}\Vector{}{r}{},\\    
    &\beta = \begin{cases}
    \frac{\left(1 + \sqrt{n}\right)^2}{1 + n} & \text{if } n < 1\\
    2 & \text{otherwise}
    \end{cases},\\
    &\mathbf{S}^{\prime} = \beta\left(\mathbf{C}^{\star}\mathbf{\Sigma}{\mathbf{C}^{\star}}^\top + \alpha\Vector{}{r}{}\Vector{}{r}{}^\top\right) + \mathbf{R}.
\end{align*}

The underlying idea is to inflate the innovation covariance in the direction of the innovation and in such a way that after the inflation the quantity ${\Vector{}{r}{}^\top{\mathbf{S}^{\prime}}^{-1}\Vector{}{r}{}}$ is smaller than $1$ \cite{Reece2010GeneralisedTracking, Ghobadi2018RobustFilter}. 

\subsection{Experiments}
In what follows, we present and analyze the performance of the proposed filter methodology under common challenging scenarios often encountered in aircraft navigation. Specifically, the performance of the proposed \ac{eqf} is benchmarked against the ArduPilot's \ac{ekf}3, the most sophisticated \ac{ekf} implementation currently available.

Various experiments are conducted using both simulated data within the ArduPilot's \acf{sitl} environment and real-world flight data. Each experiment includes a flying aircraft equipped with multiple \acp{imu}, \ac{gnss} receivers, magnetometers, and other sensors such as barometers. The ArduPilot \ac{ekf}3 is set to the default configuration, hence to use any available sensor. Accurate truth values of calibration states are given, and the tuning parameters of the filter were left unchanged to their default values. In contrast, the proposed \ac{eqf} is set to fuse the measurements from \acp{imu}, \ac{gnss} receivers, and magnetometers. The \ac{eqf} was initialized with zero \ac{gnss} lever arm and identity magnetometer rotational calibration. These states are self-calibrated online. Furthermore, the proposed \ac{eqf} was set with a similar but less inflated selection of default parameters that were found to work well in practice.

\subsubsection{\acs{sitl}: software in the loop simulations}
We conducted three distinct experiments in the \ac{sitl} environment. The first experiment was targeted at showing the self-calibration capabilities of the proposed \ac{eqf}. To achieve this, we simulated a quadcopter with magnetometer rotational calibrations at angles of $20$\si[per-mode=symbol]{\degree}, $30$\si[per-mode=symbol]{\degree}, and $20$\si[per-mode=symbol]{\degree}, along with a \ac{gnss} lever arm of $-0.4$\si[per-mode=symbol]{\meter}, $0.2$\si[per-mode=symbol]{\meter}, $0.1$\si[per-mode=symbol]{\meter} in the xyz axes, respectively.
The outcomes, depicted in \cref{ms_bins_sitl_cal}, show the estimated positions and attitudes generated by both the proposed \ac{eqf} and the ArduPilot's \ac{ekf}3. Additionally, \cref{ms_bins_sitl_cal} shows the sensor extrinsic calibrations that are estimated online by the \acl{eqf}. The results indicate that the proposed \ac{eqf} achieves effective estimation of all the states, including the sensor's extrinsic parameters, without prior knowledge of these parameters.

\begin{figure}[htb]
\centering
\includegraphics[width=\linewidth]{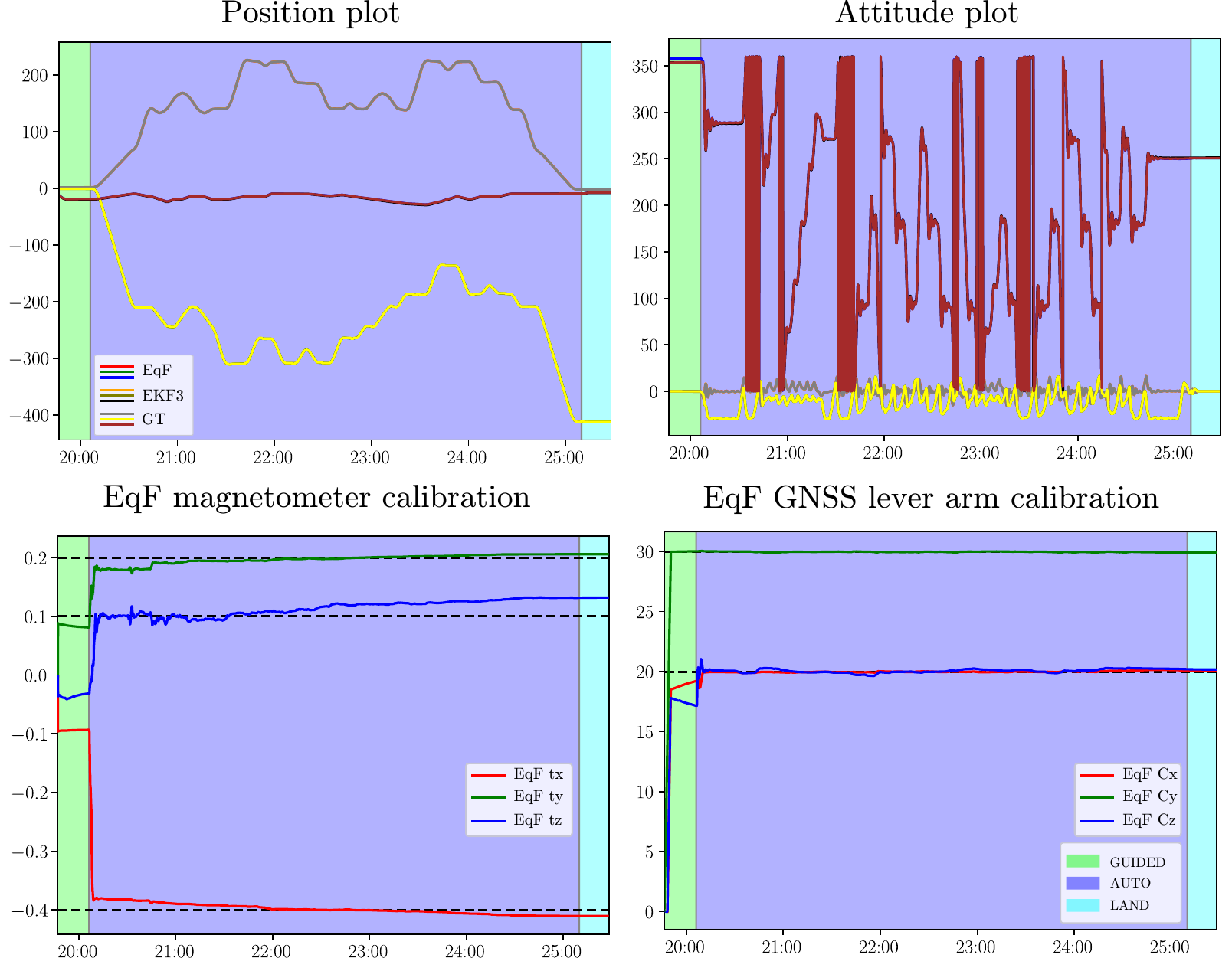}
\caption[Comparison of the proposed \ac{eqf} and the ArduPilot's \ac{ekf}3 for a simulated quadcopter flight, generated in \ac{sitl}.]{Comparison of the proposed \ac{eqf} and the ArduPilot's \ac{ekf}3 for a simulated quadcopter flight, generated in \ac{sitl}. \textbf{Top:} position, and attitude estimates. \textbf{Bottom:} online estimation of \ac{gnss} lever arm $\Vector{I}{t}{}$ (left) and magnetometer rotational calibration ${\Rot{I}{M}}$ (right) using the proposed \ac{eqf}. Dashed lines represent reference values, y-axes are in meters and degrees, x-axes in mm::ss format.}
\label{ms_bins_sitl_cal}
\end{figure}

The second experiment was targeted at showing the consistent nature of the proposed \ac{eqf} under static conditions. \acp{ekf} suffer from spurious information gains when receiving position updates under constant velocity or static conditions, leading to what is commonly termed as \emph{false observability}~\cite{barrau:tel-01247723}. Results in \cref{ms_bins_sitl_idle} highlight the behavior of the proposed \ac{eqf} and the ArduPilot's \ac{ekf}3 in a prolonged simulated static scenario. In this scenario, the \ac{eqf} achieves zero error in position and a constant error in attitude, mainly due to the initial yaw estimate of $0$\si[per-mode=symbol]{\degree}, whereas the actual yaw is approximately $7$\si[per-mode=symbol]{\degree}. On the contrary, the \ac{ekf}3 displays the classical symptoms of the false observability problem. It incorrectly gains spurious information, resulting in an erroneous and non-constant estimation of the yaw. Constant velocity motion and static conditions are common scenarios in autonomous missions, and hence, the ability to handle such situations without the need for exception code is of paramount importance.

The final experiment in the \ac{sitl} environment is targeted at evaluating the behavior of the filters in the presence of \ac{gnss} signal outages. 
In particular, we analyze the binary gating system of the ArduPilot's \ac{ekf}3, and the inflation strategy discussed in \cref{ms_bins_unc_obs_handl} and applied to the proposed \ac{eqf}. \cref{ms_bins_sitl_gps} shows that both filters are unaffected by \ac{gnss} glitches and keep providing a good estimate; thus, the proposed innovation-covariance inflation works as well as an exception-handling binary gating system in the presence of \ac{gnss} glitches. Furthermore, \cref{ms_bins_sitl_gps} shows the behavior of the \ac{ekf}3 and the \ac{eqf} on a scenario where a \ac{gnss} shift happens. In this case, the binary gating system of the ArduPilot's \ac{ekf}3 produces a jump to the new solution after $10$\si[per-mode=symbol]{\second} a shift is observed.
The tuning parameters for the binary gating systems include outlier identification criteria, the period to wait before resetting the state, and modifications for the covariance during the period that the outliers are rejected. 
In contrast, the behavior of the \ac{eqf} depends only on the parameter $\alpha$ that governs the convergence rate. The transition can be tuned to perform a smooth transition ($\alpha = 1$) or a semi-smooth transition ($\alpha=0$). The flexibility of the proposed \ac{eqf} allows it to work properly with different controller implementations.

\begin{figure}[htb]
\centering
\includegraphics[width=\linewidth]{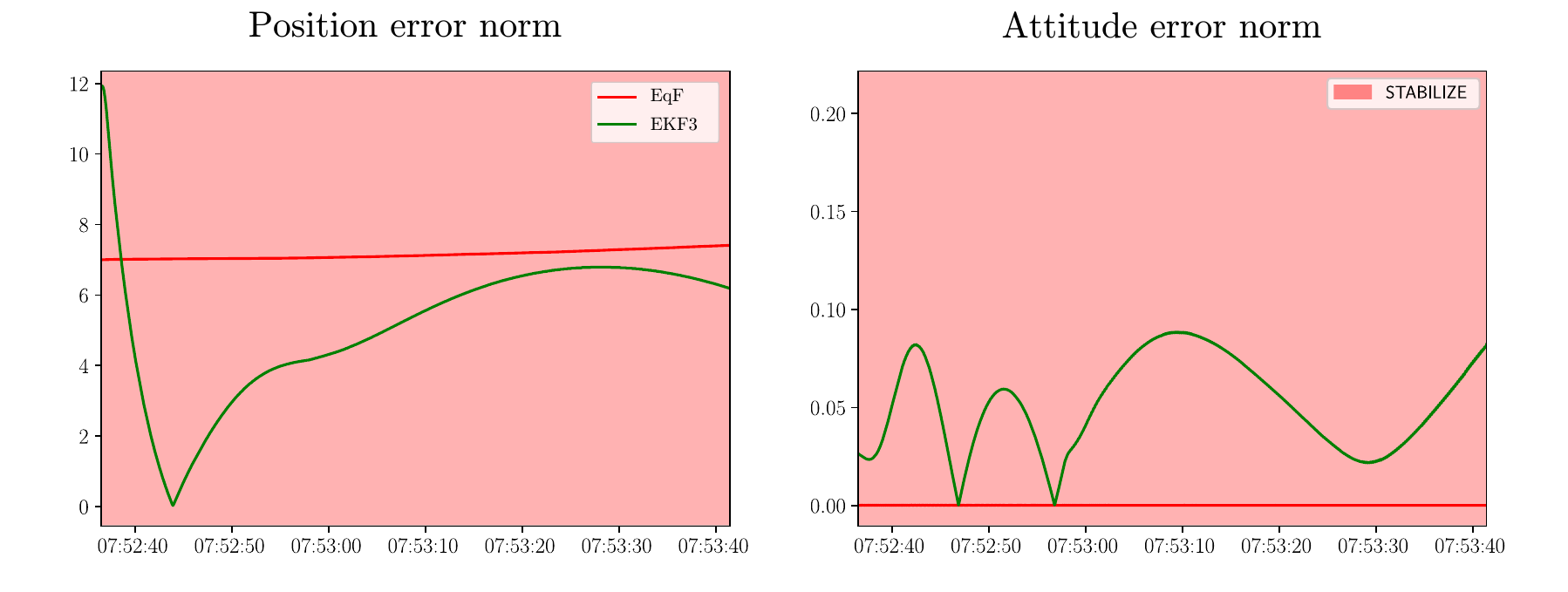}
\caption[Error plots of the proposed \ac{eqf} and the ArduPilot's \ac{ekf}3 in a simulated static scenario.]{Position error norm (left) and attitude error norm (right) using the proposed \ac{eqf} and the ArduPilot's \ac{ekf}3 in a simulated static scenario. The \ac{ekf}3 exhibits a wrong estimate, showing the typical symptoms of inconsistency due to false observability in static conditions. On the contrary, the proposed \ac{eqf} achieves a consistent estimate.}
\label{ms_bins_sitl_idle}
\end{figure}
\begin{figure}[htb]
\centering
\includegraphics[width=\linewidth]{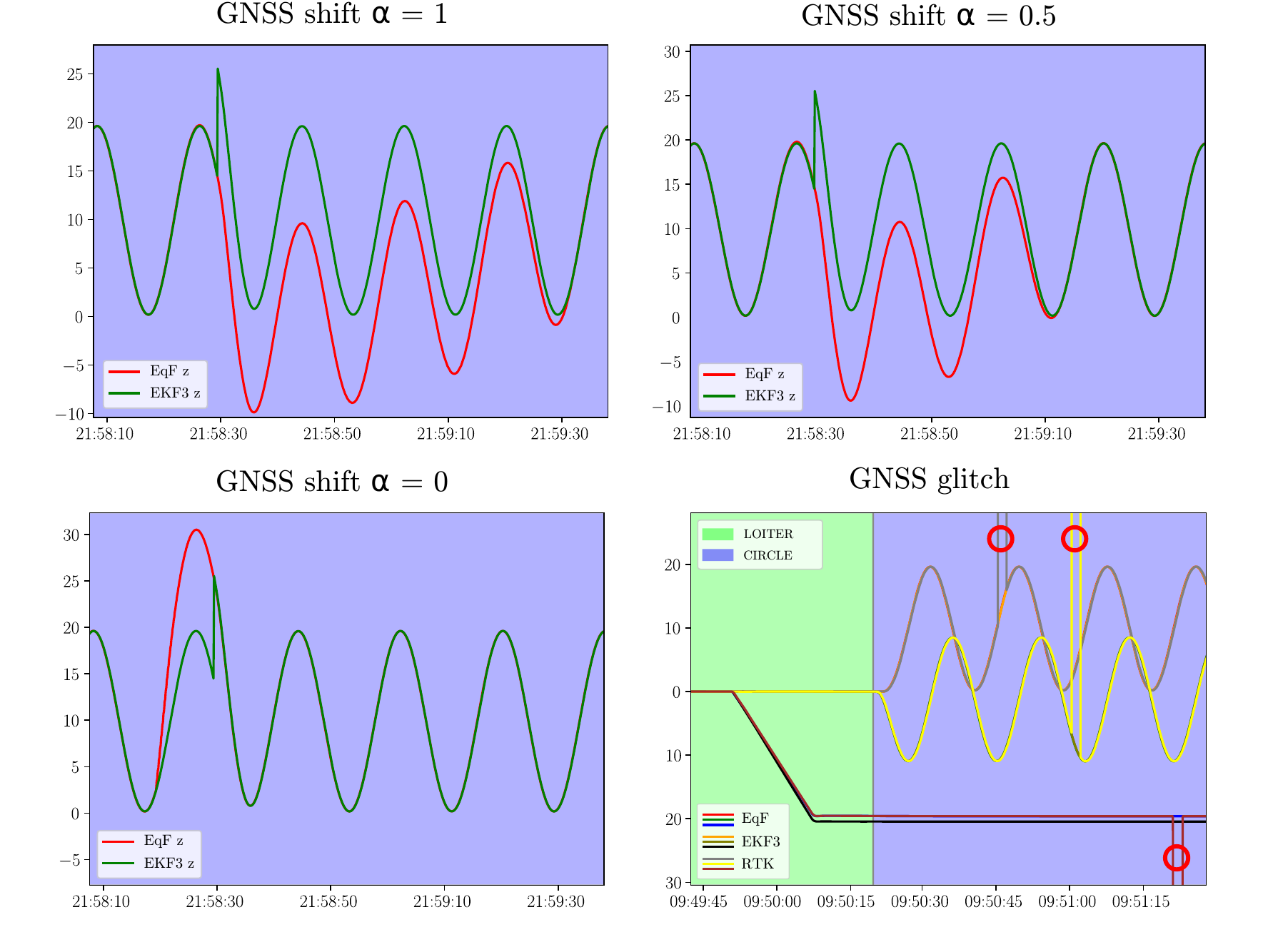}
\caption[Behavior of the ArduPilot's \ac{ekf}3 and the proposed \ac{eqf} in case of \ac{gnss} shifts and glitches.]{Behavior of the ArduPilot's \ac{ekf}3 and the proposed \ac{eqf} in case of \ac{gnss} shifts and glitches. The plots show the behavior of the filters following a \ac{gnss} shift happening at $22:58:20$. \textbf{Top left}, smooth transition of the proposed \ac{eqf} implementing the strategy discussed in \cref{ms_bins_unc_obs_handl} with $\alpha = 1$. \textbf{Top right}, semi-smooth transition of the proposed \ac{eqf} implementing the strategy discussed in \cref{ms_bins_unc_obs_handl} with $\alpha = 0.5$. \textbf{Bottom left}, transition with $\alpha = 1.0$. \textbf{Bottom right}, behavior of the ArduPilot's \ac{ekf}3 and the proposed \ac{eqf} in case of multiple \ac{gnss} glitches highlighted with red circles.}
\label{ms_bins_sitl_gps}
\end{figure}

\subsection{SpringValley: outdoor quadcopter flight}
In this subsequent experiment, real-world flight data collected from a quadcopter was used to validate the previously discussed results in the \acl{sitl} simulation environment. Due to the absence of ground truth data, our analysis focuses on the position and velocity error norms in relation to the measurements obtained from a \ac{rtk} \ac{gnss}. The error plots in \cref{ms_bins_springvalley} depict a comparison between the proposed \ac{eqf} and the ArduPilot's \ac{ekf}3. Notably, the \ac{eqf} achieves lower error compared to the \ac{ekf}3.

\begin{figure}[htb]
\centering
\includegraphics[width=\linewidth]{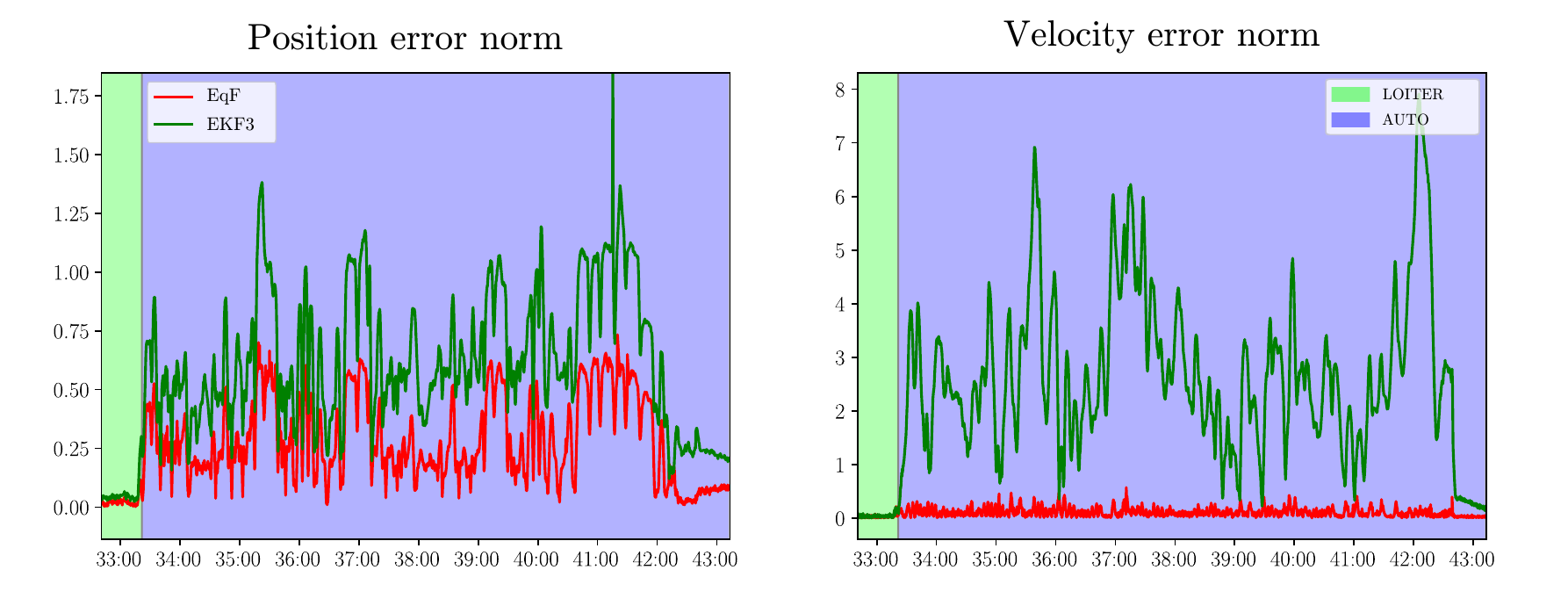}
\caption[Error plots of the proposed \ac{eqf} and the ArduPilot's \ac{ekf}3 on a real world quadcopter flight.]{Position and velocity estimation error when compared to raw \ac{rtk} \ac{gnss} measurements, using the proposed \ac{eqf} and the ArduPilot's \ac{ekf}3 on a real world quadcopter flight.}
\label{ms_bins_springvalley}
\end{figure}

\subsection{BraveHeart: outdoor quadcopter flight with faulty IMU}
In this concluding experiment, we evaluate the ability of the filters to handle a faulty \ac{imu} afflicted by high amplitude, high-frequency components, potentially stemming from excessive vibrations and aliasing effects. This scenario is illustrated in \cref{ms_bins_bh_225}, showing the low-passed filtered signal from two accelerometers recorded from an actual quadcopter flight. The accelerometer from \ac{imu}[0] is faulty and exhibits substantial high-frequency artifacts, setting it apart from the healthy \ac{imu}[1]. 

This experiment demonstrates the superior robustness of the proposed \ac{eqf}. It effectively provides an accurate estimate when used with the faulty \ac{imu}, eliminating the need for specialized parameter adjustments. In particular, we compare the \ac{eqf} run with the faulty \ac{imu} and with standard tuning parameters with two instances of the ArduPilot's \ac{ekf}3, one with the faulty \ac{imu}, and the second one with the healthy \ac{imu}. \cref{ms_bins_bh_225} shows that the \ac{eqf} succeeds in providing an accurate estimate when paired with the faulty \ac{imu} while the \ac{ekf}3 encounters challenges, only succeeding under specific tuning conditions involving high process noise and reduced observation noise.

\begin{figure}[htb]
\centering
\includegraphics[width=\linewidth]{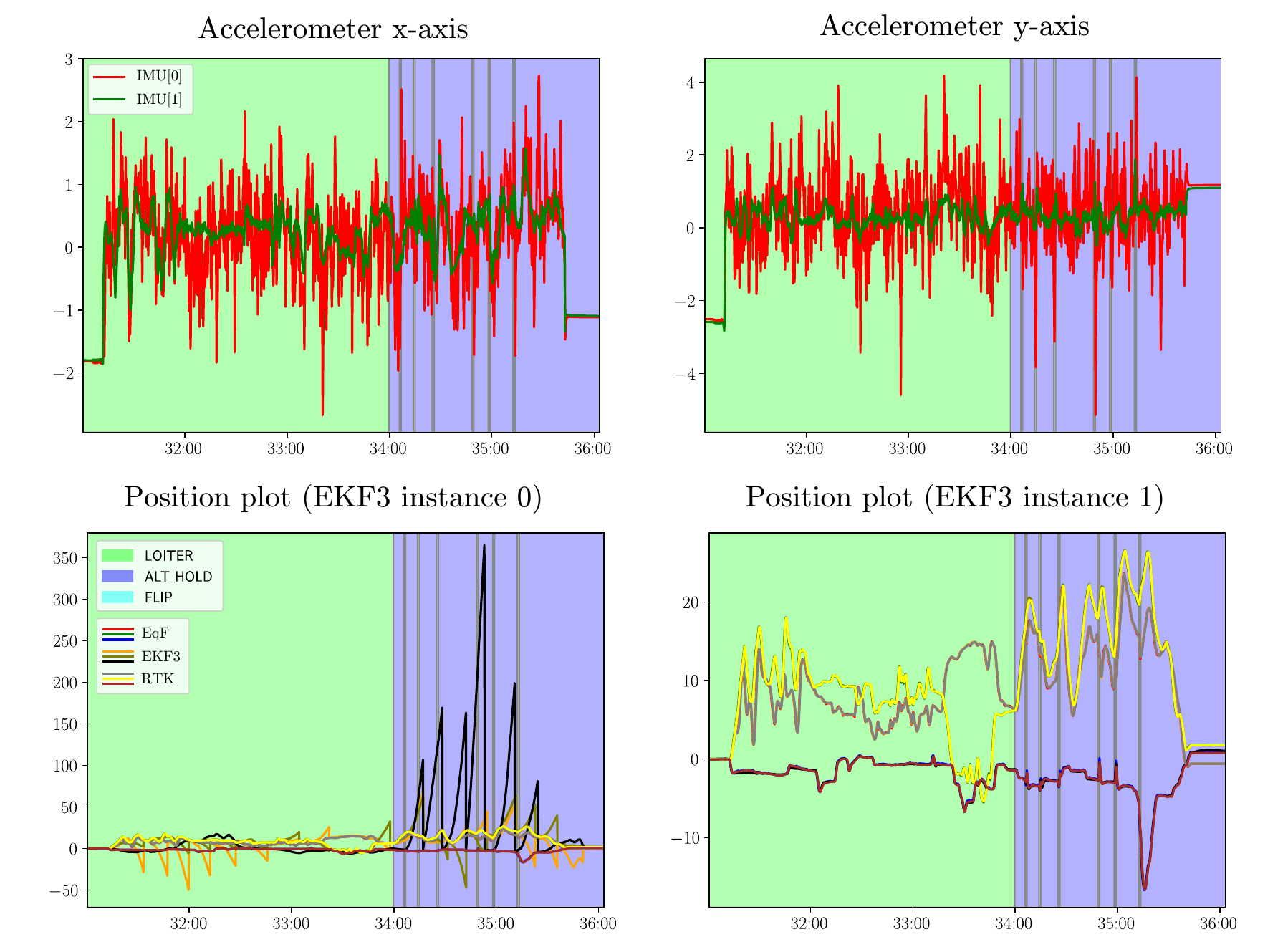}
\caption[Performance comparison of the proposed \ac{eqf} and the \ac{ekf}3 receiving measurements from two different \acp{imu}, one of which is faulty.]{Performance comparison of the proposed \ac{eqf} and the \ac{ekf}3 receiving measurements from two different \acp{imu}, one of which is faulty. \textbf{Top:} low-pass filtered acceleration along the x (left) and y (right) axes for two \acp{imu}. It is observed that \ac{imu}[0] (in red) is faulty and suffers from high amplitude components at high frequency, while \ac{imu}[1] (in green) is healthy. 
\textbf{Bottom left:} estimates obtained from both the proposed \ac{eqf} and an instance of the ArduPilot's \ac{ekf}3 when used with measurements from the faulty \ac{imu}. \textbf{Bottom right:} estimate from the proposed \ac{eqf} when used with measurements from the \emph{faulty} \ac{imu} and a second instance of the ArduPilot's \ac{ekf}3 that relies on measurements from the \emph{healthy} \ac{imu}. 
The \ac{ekf}3 fails to provide accurate estimation when dealing with data from the faulty \ac{imu}, while the proposed \ac{eqf} demonstrates superior robustness by effectively dealing with it.}
\label{ms_bins_bh_225}
\end{figure}

\section{Chapter conclusion}
This chapter introduced a \emph{general methodology to extend the symmetries for \ac{ins} states presented in \cref{bins_chp}, accounting for sensor calibration states and extra states of interest} often encountered in inertial navigation problems. Specifically, \cref{ms_bins_states_table} outlines which Lie groups provide symmetries for common extra state variables and what product rule is exploited in their composition. Additionally, concrete examples of symmetries extensions for self-calibration position-based navigation and camera-based navigation with camera extrinsic self-calibration and single visual landmark are provided. 

Continuing in \cref{ms_bins_chp}, \emph{a robust, high-performance \ac{ins} filter for the ArduPilot autopilot system is designed}, as an application of the results introduced in this and the previous chapters. Specifically, the presented \acl{eqf} can fuse measurements from multiple sensors, self-calibrating the sensor's extrinsic parameters and handling signal outliers and shifts with an innovative innovation covariance inflation technique. Furthermore, motivated by the need to model velocity measurement received from a receiver with a non-negligible lever arm, \emph{the output equivariance property, presented in \cref{eq_chp}, has been extended to observation models that depend on input variables}. The performance and robustness of the proposed filter are benchmarked against the ArduPilot's \ac{ekf}3 on challenging simulated and real-world data from the ArduPilot community.

\chapter[A Multi State Constraint Equivariant Filter for Vision-aided Inertial Navigation][Equivariant VINS]{A Multi State Constraint Equivariant Filter for Vision-aided Inertial Navigation}\label{msceqf_chp}

\emph{The present chapter contains results that have been peer-reviewed and published in the IEEE Robotics and Automation Letters~\cite{Fornasier2023MSCEqF:Navigation}.}
\bigskip

In the preceding chapters, we learned that improved convergence, filter consistency, and robustness are key properties of the theoretical framework introduced in this dissertation. These properties are of paramount importance for modern, robust inertial navigation and are a direct consequence of the choice of symmetry.

Turning our attention to vision-aided \ac{ins} systems, the lack of robustness against unexpected disturbances and the requirement for sophisticated tuning for a given environment and setup remain important limitations. Real-world deployments are typically constrained to precise tuning and highly engineered codebases, where the core \acf{vio} algorithm is encompassed by numerous modules responsible for tasks such as initialization, failure detection, algorithm reset, and more. A \emph{people's \acl{vio}}, that is, an algorithm whose operation requires minimal knowledge, little to no tuning, and yet still functions in many different real-world scenarios, would enable a whole new tranch of real-world applications without the requirement of having highly trained engineers available. 

The present chapter builds upon the established theoretical framework and introduces the \emph{\acf{msceqf}}, an \acl{eqf} approach for real-time vision-aided inertial navigation, as a step toward achieving the aforementioned goal. In this context, we present a performance evaluation that not only considers standard performance measures such as \ac{rmse}, accuracy, and precision, but also provides complimentary evaluations on metrics such as the likelihood of failure for poor initial conditions or poor calibration. Furthermore, the real-time capabilities and performance of the proposed \ac{msceqf} are evaluated in a closed-loop control of a resource-constrained aerial platform.

We will show that the proposed \ac{msceqf} demonstrates consistency naturally without artificial changes of linearization points and very high robustness to poor extrinsic calibration. It not only handles significant absolute (calibration) errors but also addresses the concept of \emph{dealing with ``you don't know what you don't know''}, such as errors exceeding the prior covariance, for example, sudden changes of calibrations states due to a disturbance during the operational phase of the robotic platform, when the state has converged already, and the covariance has shrunk.





\begin{figure}[htp]
\centering
\includegraphics[width=\linewidth]{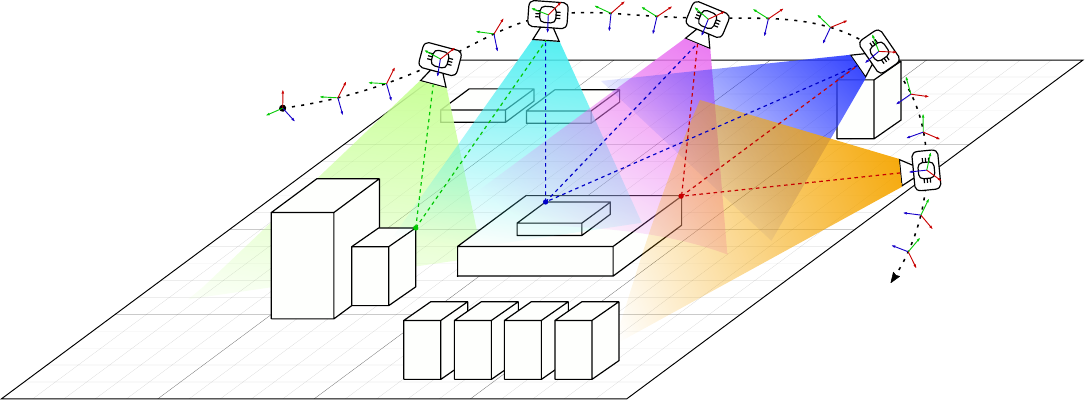}
\caption[Graphical representation of a vision-aided \ac{ins} systems.]{Graphical representation of a vision-aided \ac{ins} systems. A robot carrying a camera and an \ac{imu} moves freely in space and, by fusing information from these sensors, can estimate its own motion.}
\label{msceqf_vio}
\end{figure}

\section{Visual Inertial Navigation System}\label{msceqf_vins_sec}

Consider a mobile platform equipped with a camera observing global visual features and an \ac{imu} providing biased acceleration and angular velocity measurements, as depicted in \cref{msceqf_vio}. Consider the system described in \cref{bins_chp,bins_bins} extended with two extra state variables, $\PoseS{I}{C}$ and $\mathbf{K}$, representing respectively the camera extrinsic and intrinsic calibration parameters. Dropping superscripts and subscripts, the visual-inertial navigation system may then be written in compact form as in \cref{bins_bins_se23}, hence
\begin{subequations}\label{msceqf_vins}
    \begin{align}
        &\dotPose{}{} = \Pose{}{}\left(\mathbf{W} - \mathbf{B} + \mathbf{N}\right) + \left(\mathbf{G} - \mathbf{N}\right)\Pose{}{} ,\\
        &\dotVector{}{b}{} =  \Vector{}{\tau}{} ,\\
        &\dot{\mathbf{S}} = \mathbf{S}\Vector{}{\mu}{}^{\wedge} ,\\
        &\dot{\mathbf{K}} = \mathbf{K}\Vector{}{\zeta}{}^{\wedge} ,
    \end{align}
\end{subequations}
where $\bm{\tau}, \bm{\mu}, \bm{\zeta}$ are used to model the deterministic dynamics of the bias and calibration states and are zero when these states are modeled as constants, as they are in our formulation.

Define $\xi_{I} = \left(\Pose{}{}, \Vector{}{b}{}\right) \in \torSE_2(3) \times \R^{6}$ to be the inertial navigation state.
Define $\xi_{S} = \left(\mathbf{S}, \mathbf{K}\right) \in \torSE(3) \times \torIN(3)$ to be the camera calibration state.
Then the full system state is defined as $\xi = \left(\xi_{I}, \xi_{S}\right) \in \calM \coloneqq \torSE_2(3) \times \R^{6} \times \torSE(3) \times \torIN(3)$.
Define $u = \left(\Vector{}{w}{}, \Vector{}{\tau}{}, \Vector{}{\mu}{}, \Vector{}{\zeta}{}\right) \in \mathbb{L} \subset \R^{18}$ to be the system's input.
Note that in this work, \deleted{as in other works on multi-state constraint VIO~\cite{Wu2017AnConsistency, Sun2018RobustFlight, li2013high},} visual features are not considered as part of the state since the dependency of measurement on features is removed through nullspace projection.

Without loss of generality, let us consider the case of a single feature $\Vector{}{p}{f}$.
The camera measurement is modeled as the measurement of the bearing of the feature $\Vector{}{p}{f}$ seen from the camera. 
\begin{equation}\label{msceqf_h}
    h\left(\xi, \Vector{}{p}{f}\right) = \mathbf{K}\pi\left(\left(\mathbf{P}\mathbf{S}\right)^{-1} * \Vector{}{p}{f}\right) ,
\end{equation}
where the operation $* \AtoB{\torSE(3) \times \R^3}{\R^3}$ is defined by ${\PoseP{}{} * \Vector{}{x}{} = \Rot{}{}\Vector{}{x}{} + \Vector{}{p}{}}$  for all ${ \PoseP{}{} = \left(\Rot{}{}, \Vector{}{p}{}\right) \in \torSE(3),\; \Vector{}{x}{} \in \R^3}$. $\pi$ represents the projection method, and it can either be the projection on the unit plane $\pi_{\mathbb{Z}_1}$ in \cref{eq_piz1}, or the projection on the unit sphere $\pi_{\mathbb{S}^2}$ in \cref{eq_pis2}.

\subsection{Symmetry of the visual-inertial navigation system}

The symmetry for the \ac{ins} state $\xi_{I}$ is given by the semi-direct bias symmetry group $\grpB \coloneqq \left(\SE_2(3) \ltimes \gothse(3)\right)$ presented in \cref{bins_sdb_sec,bins_sd_linan}. The semi-direct bias symmetry group is extended with an instance of the special Euclidean group $\SE(3)$ to account for the camera extrinsic calibration states and an instance of the intrinsic group $\IN$ of dimension $4$ to account for the camera intrinsic calibration states without the skew parameter. The complete symmetry for the visual-inertial navigation system is thus defined by the product group $\grpG \coloneqq \grpB \times \SE(3) \times \IN$.

Let ${X = \left(\left(D,\delta\right), E, L\right) \in \grpG}$, with ${D = \left(A, a, b\right) \in \SE_2(3)}$ such that ${A \in \SO(3),\; a,b \in \R^{3}}$. Define the subgroups ${B = \chi\left(D\right) \in \SE(3)}$, and ${C = \Theta\left(D\right) \in \SE(3)}$, with ${\chi(\cdot),\; \Theta(\cdot)}$ defined in \cref{math_maps_sec}. Finally, define ${E \in \SE(3)}$, and ${L \in \IN}$.
\begin{lemma}
Define ${\phi \AtoB{\grpG \times \calM}{\calM}}$ as
\begin{equation}\label{msceqf_phi}
    \phi\left(X, \xi\right) \coloneqq \left(\Pose{}{}D, \AdMsym{B^{-1}}\left(\Vector{}{b}{} - \delta^{\vee}\right), C^{-1}\mathbf{S}E, \mathbf{K}L\right) \in \calM .
\end{equation}
Then, $\phi$ is a transitive right group action of $\grpG$ on $\calM$.
\end{lemma}

\subsection{Lifted system}

To define a lifted system on the symmetry group $\grpG$ that projects down to the original system dynamics via the proposed group action $\phi$, a lift ${\Lambda \AtoB{\calM \times \vecL}{\gothg}}$ is required. The transitivity of $\phi$ guarantees the existence of such a lift, and the following theorem provides an explicit form for a lift of the system studied in this chapter.

\begin{theorem}
Define the map ${\Lambda \AtoB{\calM \times \vecL}{\gothg}}$ by
\begin{align*}
    \Lambda\left(\xi, u\right) &\coloneqq \left(\left(\Lambda_1\left(\xi, u\right), \Lambda_2\left(\xi, u\right)\right), \Lambda_3\left(\xi, u\right), \Lambda_4\left(\xi, u\right)\right),
\end{align*}
where ${\Lambda_1 \AtoB{\calM \times \vecL} \se_2(3)}$, ${\Lambda_2 \AtoB{\calM \times \vecL} \se(3)}$, ${\Lambda_3 \AtoB{\calM \times \vecL} \se(3)}$, and ${\Lambda_4 \AtoB{\calM \times \vecL} \mathfrak{in}}$ are given by
\begin{subequations}\label{msceqf_lift}
    \begin{align}
        &\Lambda_1\left(\xi, u\right) \coloneqq \left(\mathbf{W} - \mathbf{B} + \mathbf{N}\right) + \Pose{}{}^{-1}\left(\mathbf{G} - \mathbf{N}\right)\Pose{}{} ,\\
        &\Lambda_2\left(\xi, u\right) \coloneqq \left(\adMsym{\Vector{}{b}{}^{\wedge}}\left(\Pi\left(\Lambda_1\left(\xi, u\right)\right)^{\vee}\right) - \Vector{}{\tau}{}\right)^{\wedge},\\
        &\Lambda_3\left(\xi, u\right) \coloneqq \left(\AdMsym{\mathbf{S}^{-1}}\left(\Upsilon\left(\Lambda_1\left(\xi, u\right)\right)^\vee\right) + \Vector{}{\mu}{}\right)^{\wedge},\\
        &\Lambda_4\left(\xi, u\right) \coloneqq \Vector{}{\zeta}{}^{\wedge},
    \end{align}
\end{subequations}
Then ${\Lambda}$ is a lift for the system in~\cref{msceqf_vins} with respect to the symmetry group $\grpG$.
\end{theorem}

\section{Multi state constraint equivariant filter}\label{msceqf_msceqf_sec}
\subsection{Filter state definition}
Define ${\hat{X} = \left(\left(\left(\hat{D}, \hat{\delta}\right), \hat{E}, \hat{L}\right), \hat{E}_1, \cdots, \hat{E}_k\right) \in \grpG \times \SE(3)^k}$ to be the filter's state evolving on the symmetry group. Similarly to the original formulation~\cite{Mourikis2007ANavigation} we maintain a sliding window of $k$ past $\hat{E}$ elements in the state of the filter, corresponding to the different times a camera measurement was collected.

\subsection{Error dynamics state matrix and input matrix}
Let ${e = \phi_{\hat{X}^{-1}}\left(\xi\right)}$ denote the equivariant error. Normal coordinates of the state space $\calM$ in a neighborhood of the origin $\xizero$ are defined as ${\varepsilon = \vartheta\left(e\right) \coloneqq \log\left(\phi_{\xizero}^{-1}\left(e\right)\right)^{\vee} \in \R^{25}}$, where ${\log \AtoB{\grpG}{\gothg}}$ is the logarithm of the symmetry group.

Recall the linearized error dynamics in~\cref{eq_A0_alt,eq_A0_alt_normal}
\begin{align*}
    &\dot{\varepsilon} \approx \mathbf{A}_{t}^{0}\varepsilon ,\\
    &\mathbf{A}_{t}^{0} = \Fr{e}{\xizero}\vartheta\left(e\right)\Fr{\xi}{\hat{\xi}}\phi_{\hat{X}^{-1}}\left(\xi\right)\Fr{E}{I}\phi_{\hat{\xi}}\left(E\right)\;\cdot\\
    & \quad\;\;\;\cdot \Fr{\xi}{\phi_{\hat{X}}\left(\xizero\right)}\Lambda\left(\xi, u\right)\Fr{e}{\xizero}
    \phi_{\hat{X}}\left(e\right)\Fr{\varepsilon}{\mathbf{0}}\vartheta^{-1}\left(\varepsilon\right)\\
    &\quad\; = \Fr{e}{\xizero}\vartheta\left(e\right)
    \Fr{E}{I}\phi_{\xizero}\left(E\right)\mathrm{Ad}_{\hat{X}}\;\cdot\\
    &\quad\;\;\;\cdot\Fr{\xi}{\phi_{\hat{X}}\left(\xizero\right)}\Lambda\left(\xi, u\right)
    \Fr{e}{\xizero}\phi_{\hat{X}}\left(e\right)
    \Fr{\varepsilon}{\mathbf{0}}\vartheta^{-1}\left(\varepsilon\right).
\end{align*}
The state matrix $\mathbf{A}_{t}^{0}$ is given by
\begin{equation}\label{msceqf_A}
    \mathbf{A}_{t}^{0} = 
    \begin{bmatrix}
        \prescript{}{1}{\mathbf{A}} & \prescript{}{2}{\mathbf{A}} & \Vector{}{0}{9 \times 6} & \Vector{}{0}{9 \times 4}\\
        \prescript{}{3}{\mathbf{A}} & \prescript{}{4}{\mathbf{A}} & \Vector{}{0}{6 \times 6} & \Vector{}{0}{6 \times 4}\\
        \prescript{}{5}{\mathbf{A}} & \prescript{}{6}{\mathbf{A}} & \prescript{}{7}{\mathbf{A}} & \Vector{}{0}{6 \times 4}\\
        \Vector{}{0}{4 \times 9} & \Vector{}{0}{4 \times 6} & \Vector{}{0}{4 \times 6} & \Vector{}{0}{4 \times 4}
    \end{bmatrix} \in \R^{25 \times 25},
\end{equation}
where
\begin{align*}
    & \prescript{}{1}{\mathbf{A}} =
    \begin{bNiceMatrix}[margin]
        \mathbf{\Psi} - \adMsym{\mathring{\Vector{}{b}{}}} & \Block{1-2}{\Vector{}{0}{6 \times 3}}\\
        \left(\mathring{\Rot{}{}}^T\mathring{\Vector{}{v}{}}\right)^{\wedge} - \hat{b}^{\wedge}\mathring{\Vector{}{b}{\omega}}^{\wedge} & \eye_3 & \Vector{}{0}{3 \times 3}
    \end{bNiceMatrix} \in \R^{9 \times 9},\\
    & \prescript{}{2}{\mathbf{A}} =
    \begin{bmatrix}
        \eye_3 & \Vector{}{0}{3 \times 3} \\
        \Vector{}{0}{3 \times 3} & \eye_3 \\
        \hat{b}^{\wedge} & \Vector{}{0}{3 \times 3}
    \end{bmatrix} \in \R^{9 \times 6}, \\
    & \prescript{}{3}{\mathbf{A}} = \begin{bmatrix}
        \adMsym{\mathring{\Vector{}{b}{}}}\mathbf{\Psi} - \adMsym{\left(\AdMsym{\hat{B}}\Vector{}{w}{} +  \hat{\delta}^{\vee} + \Vector{}{\theta}{}\right)}\adMsym{\mathring{\Vector{}{b}{}}} & \mathbf{0}_{6 \times 3}
    \end{bmatrix} \in \R^{6 \times 9},\\
    & \prescript{}{4}{\mathbf{A}} = \adMsym{\left( \AdMsym{\hat{B}}\Vector{}{w}{} + \hat{\delta}^{\vee} + \Vector{}{\theta}{}\right)}\in \R^{6 \times 6},\\
    & \prescript{}{5}{\mathbf{A}} = \AdMsym{\mathring{\mathbf{S}}^{-1}}
    \begin{bmatrix}
    -\Vector{}{\psi}{1}^{\wedge} & \Vector{}{0}{3 \times 3} & \Vector{}{0}{3 \times 3}\\
    -\Vector{}{\psi}{3}^{\wedge} - \mathring{\Vector{}{b}{\omega}}^{\wedge}\hat{b}^{\wedge} & \eye_3 & -\Vector{}{\psi}{2}^{\wedge}
    \end{bmatrix} \in \R^{6 \times 9},\\
    & \prescript{}{6}{\mathbf{A}} = \AdMsym{\mathring{\mathbf{S}}^{-1}}
    \begin{bmatrix}
        \eye_3 & \Vector{}{0}{3 \times 3}\\
        \hat{b}^{\wedge} & \Vector{}{0}{3 \times 3}
    \end{bmatrix} \in \R^{6 \times 6},\\
    & \prescript{}{7}{\mathbf{A}} = \adMsym{\left(\AdMsym{\mathring{\mathbf{S}}^{-1}}\Vector{}{\varrho}{}\right)} \in \R^{6 \times 6} ,
\end{align*}
with
\begin{align*}{}
    &\Vector{}{\psi}{1} = \hat{A}\Vector{}{\omega}{} + \delta_{\omega}^{\vee} \in \R^{3}, && \Vector{}{\theta}{} = \left(\Vector{}{0}{3 \times 1}, \mathring{\Rot{}{}}^T\Vector{}{g}{}\right) \in \R^6,\\
    &\Vector{}{\psi}{2} = \Vector{}{\psi}{1} - \mathring{\Vector{}{b}{\omega}} \in \R^{3}, && \mathbf{\Psi} =
    \begin{bmatrix}
        \Vector{}{0}{3 \times 3} & \Vector{}{0}{3 \times 3} \\
        \left(\mathring{\Rot{}{}}^T\Vector{}{g}{}\right)^{\wedge} & \Vector{}{0}{3 \times 3}
    \end{bmatrix} \in \R^{6 \times 6} ,\\
    &\Vector{}{\psi}{3} = \hat{a} - \Vector{}{\psi}{1}^{\wedge}\hat{b} \in \R^{3}, && \Vector{}{\varrho}{} = \left(\Vector{}{\psi}{2}, \Vector{}{\psi}{4}\right) \in \R^{3}. \\
    &\Vector{}{\psi}{4} = \hat{a} + \mathring{\Rot{}{}}^T\mathring{\Vector{}{v}{}} - \Vector{}{\psi}{2}^{\wedge}\hat{b} \in \R^{3}.
\end{align*}
The discrete-time state transition matrix is defined by ${\mathbf{\Phi} = \exp\left(\mathbf{A}_{t}^{0} \Delta T\right)}$ for time steps $\Delta T$.

Similarly, the matrix $\mathbf{B}_t$ is derived solving \cref{eq_B}:
\begin{equation}\label{msceqf_B}
    \mathbf{B}_t = \begin{bmatrix}
        \prescript{}{1}{\mathbf{B}} & \Vector{}{0}{9 \times 6}\\
        \prescript{}{2}{\mathbf{B}} & \prescript{}{3}{\mathbf{B}}\\
        \prescript{}{4}{\mathbf{B}} & \Vector{}{0}{6 \times 6}\\
        \Vector{}{0}{6 \times 6} & \Vector{}{0}{6 \times 6}
    \end{bmatrix} \in \R^{25 \times 12}
\end{equation}
with
\begin{align*}
    &\prescript{}{1}{\mathbf{B}} = \AdMsym{\hat{D}}\begin{bmatrix}
        \eye_3 & \Vector{}{0}{3 \times 3}\\
        \Vector{}{0}{3 \times 3} & \eye_3\\
        \Vector{}{0}{3 \times 3} & \Vector{}{0}{3 \times 3}
    \end{bmatrix} \in \R^{9 \times 6}\\
    &\prescript{}{2}{\mathbf{B}} = \adMsym{\mathring{\Vector{}{b}{}}}\AdMsym{\hat{B}} \in \R^{6 \times 6}\\
    &\prescript{}{3}{\mathbf{B}} = - \AdMsym{\hat{B}} \in \R^{6 \times 6}\\
    &\prescript{}{4}{\mathbf{B}} = \AdMsym{\mathring{\mathbf{S}}^{-1}\hat{C}}
    \begin{bmatrix}
        \eye_3 & \Vector{}{0}{3 \times 3}\\
        \Vector{}{0}{3 \times 3} & \Vector{}{0}{3 \times 3}
    \end{bmatrix} \in \R^{6 \times 6}
\end{align*}

\subsection{Multi state constraint}
Consider the measurement model in~\cref{msceqf_h}; applying the action of the symmetry group to the state space in~\cref{msceqf_phi} yields
\begin{equation}
    h\left(\phi_{X}\left(\xi\right)\right) = \mathbf{K}L \pi\left(E^{-1}\left(\mathbf{P}\mathbf{S}\right)^{-1} * \Vector{}{p}{f}\right)
\end{equation}
Recall the equivariant error ${e = \phi_{\hat{X}^{-1}}\left(\xi\right) =  \vartheta^{-1}\left(\varepsilon\right)}$.
Define the feature error ${\tilde{y} = \varsigma\left(\Vector{}{p}{f}\right) - \varsigma\left(\hatVector{}{p}{f}\right)}$, where $\varsigma\left(\cdot\right)$ represents the feature parametrization.
The true feature can then be written as ${\Vector{}{p}{f} = \varsigma^{-1}\left(\varsigma\left(\hatVector{}{p}{f}\right) + \tilde{y}\right)}$. 
Therefore, the measurement model in~\cref{msceqf_h} can be linearized at ${\varepsilon = \mathbf{0}}$, and ${\tilde{y} = \mathbf{0}}$ as follows:
\begin{equation}\label{msceqf_hlin}
\begin{split}
    h\left(\xi, \Vector{}{p}{f}\right) &= h\left(\phi_{\hat{X}}\left(\vartheta^{-1}\left(\varepsilon\right)\right), \varsigma^{-1}\left(\varsigma\left(\hatVector{}{p}{f}\right) + \tilde{y}\right)\right)\\
    &= h\left(\hat{\xi}, \hatVector{}{p}{f}\right) + \mathbf{C}_t\varepsilon + \mathbf{C}^{f}_t\tilde{y} + \cdots .
\end{split}
\end{equation}

Let ${^A\mathbf{P} ^A\mathbf{S}}$ be the pose of the anchor, defined as the pose of the camera where we want the feature to be parametrized in, often chosen to be the pose of the camera where the feature has been first seen. Define the feature in the anchor frame as ${\Vector{}{a}{f} = \left(^A\mathbf{P} ^A\mathbf{S}\right)^{-1} * \Vector{}{p}{f}}$, with ${\Vector{}{a}{f} = \left(a_{f_x}, a_{f_y}, a_{f_z}\right) \in \R^3}$.
Then, the matrix ${\mathbf{C}^{f}_t}$ is derived for different feature parametrization.

\paragraph{Anchored Euclidean.}
The anchored Euclidean parametrization is written
\begin{align}\label{msceqf_ae_param}
    &\Vector{}{z}{} = \varsigma\left(\Vector{}{p}{f}\right) = \left(^A\mathbf{P} ^A\mathbf{S}\right)^{-1} * \Vector{}{p}{f} = \Vector{}{a}{f} = \left(a_{f_x}, a_{f_y}, a_{f_z}\right) ,\\
    &\Vector{}{p}{f} = \varsigma^{-1}\left(\Vector{}{z}{}\right) = \left(^A\mathbf{P} ^A\mathbf{S}\right) * \Vector{}{a}{f} .
\end{align}
Then the matrix ${\mathbf{C}^{f}_t}$ is written
\begin{equation}\label{msceqf_ae_Cf}
    \begin{split}
        \mathbf{C}^{f}_t\tilde{y} &= 
        \mathring{\mathbf{K}}\hat{L}d_{\pi}
        \Gamma\left(\left(\hat{\mathbf{P}}\hat{\mathbf{S}}\right)^{-1} {^A\hat{\mathbf{P}}} {^A\hat{\mathbf{S}}}\right)\tilde{y} \\
        &= \mathring{\mathbf{K}}\hat{L}d_{\pi}
        \Gamma\left({\hat{E}^{-1}} {^A\hat{E}}\right)\tilde{y},
    \end{split}
\end{equation}
where we have used ${\hat{\xi} \coloneqq \phi_{\hat{X}}\left(\xizero\right)}$ to map between the estimated state in the homogeneous space $\hat{\xi}$, and the estimated state in the symmetry group $\hat{X}$. Therefore
\begin{equation*}
    \left(\hat{\mathbf{P}}\hat{\mathbf{S}}\right)^{-1} {^A\hat{\mathbf{P}}} {^A\hat{\mathbf{S}}} = \hat{E}^{-1}\mathring{\PoseS{}{}}^{-1}\hat{C}\hat{C}^{-1}\mathring{\PoseP{}{}}^{-1}\mathring{\PoseP{}{}}^A\hat{C}^A\hat{C}^{-1}\mathring{\PoseS{}{}}^A\hat{E} = \hat{E}^{-1}{^A\hat{E}} .
\end{equation*}
Note that we can choose to parametrize the feature in the global fame, hence choosing $^A\mathbf{P} ^A\mathbf{S} = \eye$. In this case, it is trivial to see that $^A\hat{E} = \left(\mathring{\PoseP{}{}}\mathring{\PoseS{}{}}\right)^{-1}$.

\paragraph{Anchored inverse depth.} 
The anchored inverse depth parametrization~\cite{Civera2008InverseSLAM, Mourikis2007ANavigation} is written
\begin{align}\label{msceqf_aid_param}
    &\Vector{}{z}{} = \varsigma\left(\Vector{}{p}{f}\right) = \left(\Vector{}{z}{1}, z_2\right) = \left(\left(\frac{a_{f_x}}{a_{f_z}}, \frac{a_{f_y}}{a_{f_z}}\right), \frac{1}{a_{f_z}}\right) ,\\
    &\Vector{}{p}{f} = \varsigma^{-1}\left(\Vector{}{z}{}\right) = \left(^A\mathbf{P} ^A\mathbf{S}\right) * \begin{bmatrix}
        \frac{\Vector{}{z}{1}}{z_2}\\
        \frac{1}{z_2}
    \end{bmatrix} .
\end{align}
Then the matrix ${\mathbf{C}^{f}_t}$ is written
\begin{equation}\label{msceqf_aid_Cf}
    \begin{split}
        \mathbf{C}^{f}_t\tilde{y} &= 
        \mathring{\mathbf{K}}\hat{L}d_{\pi}
        \Gamma\left(\left(\hat{\mathbf{P}}\hat{\mathbf{S}}\right)^{-1} {^A\hat{\mathbf{P}}} {^A\hat{\mathbf{S}}}\right)\frac{1}{\hat{z}_2}
        \begin{bmatrix}
            \eye_2 & -\frac{\hatVector{}{z}{1}}{\hat{z}_2}\\
            \Vector{}{0}{1 \times 2} & -\frac{1}{\hat{z}_2}
        \end{bmatrix}
        \tilde{y} \\
        &= \mathring{\mathbf{K}}\hat{L}d_{\pi}
        \Gamma\left({\hat{E}^{-1}} {^A\hat{E}}\right)\frac{1}{\hat{z}_2}
        \begin{bmatrix}
            \eye_2 & -\frac{\hatVector{}{z}{1}}{\hat{z}_2}\\
            \Vector{}{0}{1 \times 2} & -\frac{1}{\hat{z}_2}
        \end{bmatrix}
        \tilde{y}.
    \end{split}
\end{equation}

\paragraph{Anchored polar parametrization.}
The anchored polar parametrization, is the anchored version of the polar parametrization recently introduced in~\cite{vanGoor2023EqVIO:Odometry}. In particular, it is defined by normal coordinates of $\SOT(3)$ about ${\mathring{\Vector{}{c}{f}} = \left(\mathring{\mathbf{P}}\mathring{\mathbf{S}}\right)^{-1} * \mathring{\Vector{}{p}{f}}}$ as follows
\begin{align}
    &\Vector{}{z}{} = \varsigma\left(\Vector{}{p}{f}\right) = \left(z_1, \Vector{}{z}{2}\right) = \left(\log\left(\frac{\norm{\mathring{\Vector{}{c}{f}}}}{\norm{\Vector{}{a}{f}}}\right), -\arccos\left(\frac{\mathring{\Vector{}{c}{f}}^T\Vector{}{a}{f}}{\norm{\mathring{\Vector{}{c}{f}}}\norm{\Vector{}{a}{f}}}\right)\frac{\mathring{\Vector{}{c}{f}} \times \Vector{}{a}{f}}{\norm{\mathring{\Vector{}{c}{f}} \times \Vector{}{a}{f}}}\right) ,\\
    &\Vector{}{p}{f} = \varsigma^{-1}\left(\Vector{}{z}{}\right) = \left(^A\mathbf{P} ^A\mathbf{S}\right) * \exp_{\SOT(3)}\left(\Vector{}{z}{}^{\wedge}\right)^{-1}\mathring{\Vector{}{c}{f}} .
\end{align}
Then the matrix ${\mathbf{C}^{f}_t}$ is written
\begin{equation}\label{msceqf_ap_Cf}
    \begin{split}
        \mathbf{C}^{f}_t\tilde{y} &= 
        \mathring{\mathbf{K}}\hat{L}d_{\pi}
        \Gamma\left(\left(\hat{\mathbf{P}}\hat{\mathbf{S}}\right)^{-1} {^A\hat{\mathbf{P}}} {^A\hat{\mathbf{S}}}\right)
        \begin{bmatrix}
            \Vector{}{a}{f}^{\wedge}\mathbf{J}_{L}\left(-\Vector{}{z}{2}\right) & -\Vector{}{a}{f}
        \end{bmatrix}\tilde{y}\\
        &= 
        \mathring{\mathbf{K}}\hat{L}d_{\pi}
        \Gamma\left({\hat{E}^{-1}} {^A\hat{E}}\right)
        \begin{bmatrix}
           \Vector{}{a}{f}^{\wedge}\mathbf{J}_{L}\left(-\Vector{}{z}{2}\right) & -\Vector{}{a}{f}
        \end{bmatrix}\tilde{y}
    \end{split}
\end{equation}

According to \cref{msceqf_hlin}, the matrix ${\mathbf{C}_t}$ is defined, for any given parametrization, by
\begin{equation}\label{msceqf_Ct}
    \begin{split}
        \mathbf{C}_t\varepsilon &= \Fr{\xi}{\hat{\xi}}h\left(\xi\right)\Fr{e}{\mathring{\xi}}\phi_{\hat{X}}\Fr{\varepsilon}{\mathbf{0}}\vartheta^{-1}\left(\varepsilon\right)\left[\varepsilon\right]\\
        &= \mathring{\mathbf{K}}
        \hat{L}
        d_{\pi}\Gamma\left(\hat{E}^{-1}\right)
        \begin{bmatrix}
        \left({^AE}\hatVector{}{a}{f}\right)^{\wedge} & -\eye_3
        \end{bmatrix}
        \varepsilon_{E}\; -\\
        &- \mathring{\mathbf{K}}
        \hat{L}
        d_{\pi}\Gamma\left(\hat{E}^{-1}\right)
        \begin{bmatrix}
        \left({^AE}\hatVector{}{a}{f}\right)^{\wedge} & -\eye_3
        \end{bmatrix}
        \varepsilon_{^AE}\; +\\
        &+
        \mathring{\mathbf{K}}\Xi\left(\hat{L}\pi\left(\hat{E}^{-1}{^AE}\hatVector{}{a}{f}\right)\right)
        \varepsilon_L ,
    \end{split}
\end{equation}
where $\varepsilon_{E}$, and $\varepsilon_{^AE}$ represent respectively the error in normal coordinates for the element $E$ of the symmetry group corresponding to the most recent clone and to the anchor clone, whereas $\varepsilon_{L}$ represent the error in normal coordinates that is related to the camera intrinsics.

To compute the matrices ${\mathbf{C}^{f}_t}$, and  ${\mathbf{C}_t}$ above, an estimate of the feature position in the anchor frame is required. To this end, when a feature has been seen from multiple views, a linear-nonlinear least square problem is solved~\cite{Mourikis2007ANavigation, Geneva2020OpenVINS:Estimation}.

\paragraph{Linear feature triangulation.}

Consider the following relation between the feature expressed in the camera frame $\Vector{}{c}{f}$ and the feature expressed in the anchor frame $\Vector{}{a}{f}$:
\begin{equation}\label{msceqf_lintri_cf_af}
\begin{split}
    \Vector{}{c}{f} 
    &= \left({\mathbf{P}}{\mathbf{S}}\right)^{-1}{^A\mathbf{P}}{^A\mathbf{S}} * \Vector{}{a}{f}\\
    &= \left({E^{-1}}{^AE}\right) * \Vector{}{a}{f} ,
\end{split}
\end{equation}
The feature ${\Vector{}{c}{f}}$ can be expressed as ${\Vector{}{c}{f} = ^{C}z_f\Vector{C}{b}{f}}$, where ${\Vector{C}{b}{f}}$ is the bearing on the unit plane (the homogeneous vector of normalized ${u, v}$ coordinates) and $^{C}z_f$ its depth; thus
\begin{equation*}
    \Vector{}{c}{f} = \left(^{C}z_f\right)\left(\Vector{C}{b}{f}\right) 
    = \left(^{C}z_f\right)
    \begin{bmatrix}
        u\\
        v\\
        1\\
    \end{bmatrix} .
\end{equation*}
Substituting $\Vector{}{c}{f}$ in \cref{msceqf_lintri_cf_af}, and inverting it yields
\begin{align*}
    \Vector{}{a}{f} 
    &= \left({^AE^{-1}}{E}\right) * \Vector{}{c}{f} \\
    &= \left({^AE^{-1}}{E}\right) * \left(^{C}z_f\Vector{C}{b}{f}\right)\\
    &= \Gamma\left({^AE^{-1}}{E}\right)\left(^{C}z_f\Vector{C}{b}{f}\right) + \left({^AE^{-1}}{E}\right) * \Vector{}{0}{}\\
    &= ^{C}z_f\Gamma\left({^AE^{-1}}{E}\right)\Vector{C}{b}{f} + \left({^AE^{-1}}{E}\right) * \Vector{}{0}{}\\
    &= ^{C}z_f\Vector{A}{b}{f} + \left({^AE^{-1}}{E}\right) * \Vector{}{0}{} ,
\end{align*}
then, pre-multiplying both sides by ${\left(\Vector{A}{b}{f}\right)^{\wedge}}$ yields
\begin{equation*}
    \left(\Vector{A}{b}{f}\right)^{\wedge}\Vector{}{a}{f} = \left(\Vector{A}{b}{f}\right)^{\wedge}\left(\left({^AE^{-1}}{E}\right) * \Vector{}{0}{}\right) .
\end{equation*}
Therefore, a linear system in the form ${\mathbf{A}\Vector{}{a}{f} = \mathbf{b}}$ can be built with all the measurements of the same feature at different times:
\begin{equation*}
    \begin{bmatrix}
        \left(\Vector{A}{b}{f}^1\right)^{\wedge}\\
        \vdots\\
        \left(\Vector{A}{b}{f}^k\right)^{\wedge}
    \end{bmatrix}
    \Vector{}{a}{f} = 
    \begin{bmatrix}
        \left(\Vector{A}{b}{f}^1\right)^{\wedge}\left(\left({^AE^{-1}}{E_1}\right) * \Vector{}{0}{}\right)\\
        \vdots\\
        \left(\Vector{A}{b}{f}^n\right)^{\wedge}\left(\left({^AE^{-1}}{E_k}\right) * \Vector{}{0}{}\right)
    \end{bmatrix} .
\end{equation*}

\paragraph{Nonlinear feature triangulation.}
The solution of the linear system above can be used as an initial condition in a nonlinear least square refinement. Specifically, \cref{msceqf_lintri_cf_af} can be written as follows:
\begin{align*}
    \Vector{}{c}{f} 
    &= \left({E^{-1}}{^AE}\right) * \Vector{}{a}{f}\\
    &= \left({E^{-1}}{^AE}\right) * \varsigma^{-1}\left(\Vector{}{z}{}\right) ,
\end{align*}
Then, a nonlinear least square problem can be formulated:
\begin{equation*}
    \sum_{j = 1}^{k} \norm{\pi\left(\Vector{C}{b}{f}^j\right) - \pi\left(\left({E_j^{-1}}{^AE}\right) * \varsigma^{-1}\left(\Vector{}{z}{}\right)\right)}^2 ,
\end{equation*}
where ${\Vector{C}{b}{f}^j}$ is again the direct measurement of the feature in the $j$-th camera frame expressed in normalized coordinates.

\paragraph{Two-stage feature triangulation: Wahba's problem and ray intersection.}
The previously introduced polar parameterization can be exploited together with the projection function on the two-sphere to define a novel feature triangulation methodology based on a decomposition into two linear problems. Specifically, the triangulation problems seek to find a scaling term $^A\lambda$ first and a rotation matrix $^A\Rot{}{}$ after such that $^A\lambda ^A\Rot{}{} \mathring{\Vector{}{c}{f}} = \Vector{}{a}{f}$. Form \cref{msceqf_lintri_cf_af} we can write
\begin{align*}
    {^AE} * \Vector{}{a}{f}
    &= {E} * \Vector{}{c}{f} \\
    &= {E} * \lambda\pi_{\mathbb{S}^2}\left(\Vector{C}{b}{f}\right)\\
    {^AE} * ^A\lambda\pi_{\mathbb{S}^2}\left(\Vector{A}{b}{f}\right)
    &= {E} * \lambda\pi_{\mathbb{S}^2}\left(\Vector{C}{b}{f}\right)\\
    ^A\lambda\Gamma\left({^AE}\right)\pi_{\mathbb{S}^2}\left(\Vector{A}{b}{f}\right) + {^AE}*\Vector{}{0}{}    
    &= \lambda\Gamma\left({E}\right)\pi_{\mathbb{S}^2}\left(\Vector{C}{b}{f}\right) + {E}*\Vector{}{0}{}\\
    ^A\lambda\Vector{A}{x}{f} + {^AE}*\Vector{}{0}{} 
    &=\lambda\Vector{C}{x}{f} + {E}*\Vector{}{0}{} 
\end{align*}
Then, the following linear least square problem is built
\begin{equation*}
    \begin{bmatrix}
        \Vector{A}{x}{f}  & -\Vector{C}{x}{f}^1  & \cdots & \Vector{}{0}{}\\
        \vdots & \vdots & & \vdots\\
        \Vector{A}{x}{f} & \Vector{}{0}{} & \cdots & -\Vector{C}{x}{f}^k
    \end{bmatrix}
    \begin{bmatrix}
        ^A\lambda\\
        \lambda_1\\
        \vdots\\
        \lambda_k
    \end{bmatrix} =
    \begin{bmatrix}
        {E_1}*\Vector{}{0}{} - {^AE}*\Vector{}{0}{}\\
        {E_2}*\Vector{}{0}{} - {^AE}*\Vector{}{0}{}\\
        \vdots\\
        {E_k}*\Vector{}{0}{} - {^AE}*\Vector{}{0}{}
    \end{bmatrix}
\end{equation*}
Once a solution for ${^A\lambda,\lambda_1,\cdots,\lambda_n}$ is found, these can be used to build a Wahba's problem~\cite{Wahba2006AAttitude}.
\begin{align*}
    \Vector{}{a}{f} 
    &= \left({^AE^{-1}}{E}\right) * \Vector{}{c}{f} \\
    ^A\lambda^A\Rot{}{}\mathring{\Vector{}{c}{f}}
    &= \left({^AE^{-1}}{E}\right) * \Vector{}{c}{f}\\
    ^A\Rot{}{}\pi_{\mathbb{S}^2}\left(\mathring{\Vector{}{c}{f}}\right)
    &= \pi_{\mathbb{S}^2}\left(\left({^AE^{-1}}{E}\right) * \Vector{}{c}{f}\right) .
\end{align*}
Therefore, ${\Vector{}{z}{} = \left(\log\left(\frac{1}{^A\lambda}\right), \log\left(^A\Rot{}{}^{\top}\right)^{\vee}\right)}$

Finally, to remove the dependency of the features in \cref{msceqf_hlin}, and hence perform a filter update, we employ nullspace marginalization of the matrix ${\mathbf{C}^{f}_t}$, according to the original \acs{msckf} formulation~\cite{Mourikis2007ANavigation}.

\subsection{Equivariant zero velocity update}
For multi state constraint filter types, different strategies are proposed to better handle zero motion scenarios. Common approaches include optimized keyframe section~\cite{Allak2018Key-FrameFeatures, Kottas2013DetectingSystems} and \acf{zvu}~\cite{QIU2020LightweightUpdate}. While the former seeks to optimize the selection of cloned poses in the sliding window to always ensure sufficient parallax and hence a successful feature triangulation, the latter seeks to identify zero motion scenarios through heuristics and perform a zero velocity filter update. 

Here, we exploit an extended pose measurement with constant position and rotation and zero velocity to define a novel \emph{equivariant \acl{zvu}} formulation.
\begin{lemma}
    Define the \acl{zvu} measurement model $h\left(\xi\right)$ as follows:
    \begin{equation}\label{msceqf_zvu_h}
         h\left(\xi\right) = \Pose{}{} = (\Rot{}{}, \Vector{}{v}{}, \Vector{}{p}{}) \in \calN \coloneqq \torSE_2(3) .
    \end{equation}
    Let ${y = (\hatRot{}{}, \Vector{}{0}{}, \hatVector{}{p}{}) \in \calN}$ be a measurement defined according to the model in \cref{msceqf_zvu_h}, define ${\rho \AtoB{\grpG \times \calN}{\calN}}$ as
    \begin{equation}\label{msceqf_zvu_rho}
        \rho\left(X,y\right) \coloneqq yD.
    \end{equation}
    Then, the configuration output defined in \cref{msceqf_zvu_h} is equivariant.
\end{lemma}
Local coordinates for the output are chosen to be logarithmic coordinates of $\SE_2(3)$, hence define
\begin{equation}
    \delta(y) = \log_{\SE_2(3)}(\mathring{\Pose{}{}}^{-1}y).
\end{equation}
Exploiting the equivariance of the \ac{zvu} output we can derive the $\mathbf{C}^{\star}$ matrix for the \acl{zvu} solving \cref{eq_C_star}
\begin{align*}
    \mathbf{C}^{\star}\varepsilon &= \frac{1}{2}\left(\left(\eye_9 + \mathring{\Pose{}{}}^{-1}y\hat{D}^{-1}\right){\varepsilon_D}^{\wedge}\right)^{\vee} = \varepsilon_D\\
    &= \begin{bmatrix} \eye_9 & \Vector{}{0}{9 \times 6} & \Vector{}{0}{9 \times 6} & \Vector{}{0}{9 \times 4} & \Vector{}{0}{9 \times 6} & \cdots & \Vector{}{0}{9 \times 6} \end{bmatrix}\varepsilon .
\end{align*}

In our approach, stationary detection is performed based on a threshold of the image disparity over a rolling window of image data. This approach can be extended with robust methodologies such as statistical stationarity test~\cite{Solin2018InertialSmartphones} or frequency-domain tecniques~\cite{Ramanandan2012InertialUpdates} on a rolling window of \ac{imu} data.

\subsection{Implementation details}
The proposed framework is implemented as a stand-alone C++ library, and it is source-available to the community\footnotemark. 
Wrappers for the standard middle-ware (e.g., ROS1, ROS2, etc.) are provided such that code is available for direct use and comparison against other approaches. 

\footnotetext{\href{https://github.com/aau-cns/MSCEqF}{https://github.com/aau-cns/MSCEqF}}

The vision frontend is implemented utilizing the OpenCV library~\cite{opencv_library}. The feature detector can be chosen between \emph{Good Feature to Track (GFTT)} and \emph{Fast}. The input image is divided into a $N \times M$ grid, and the feature extraction is executed in parallel. A minimum pixel distance between detected features is imposed; furthermore, sub-pixel refinement is performed. Detected features are subsequently tracked between consecutive images using the pyramidal Lucas-Kanade algorithm~\cite{lucas1981iterative} with the previous flow as the initial value, and RANSAC is used to remove outliers.

A High-level overview of the filter backend logic is shown in the flowchart in \cref{msceqf_flowchart}. At the beginning, the filter tries to initialize the state origin $\xizero$. This is done by collecting a window of \ac{imu} measurements during a still phase and computing the roll and pitch of the platform. Origin position and velocity are set to zero, and origin calibration is set to the provided values. The origin can also be set with given values, such as those provided by an external dynamic initialization module. Once the state origin is initialized, the feature extraction and matching, as well as the propagation, are executed in parallel on different threads. Once the propagation is performed, either a stochastic cloning or a \acl{zvu} is performed. Finally, when the \acl{zvu} is not active or the platform is moving, the multi state constraint update is performed.

\begin{figure}[htp]
\centering
\includegraphics[width=\linewidth]{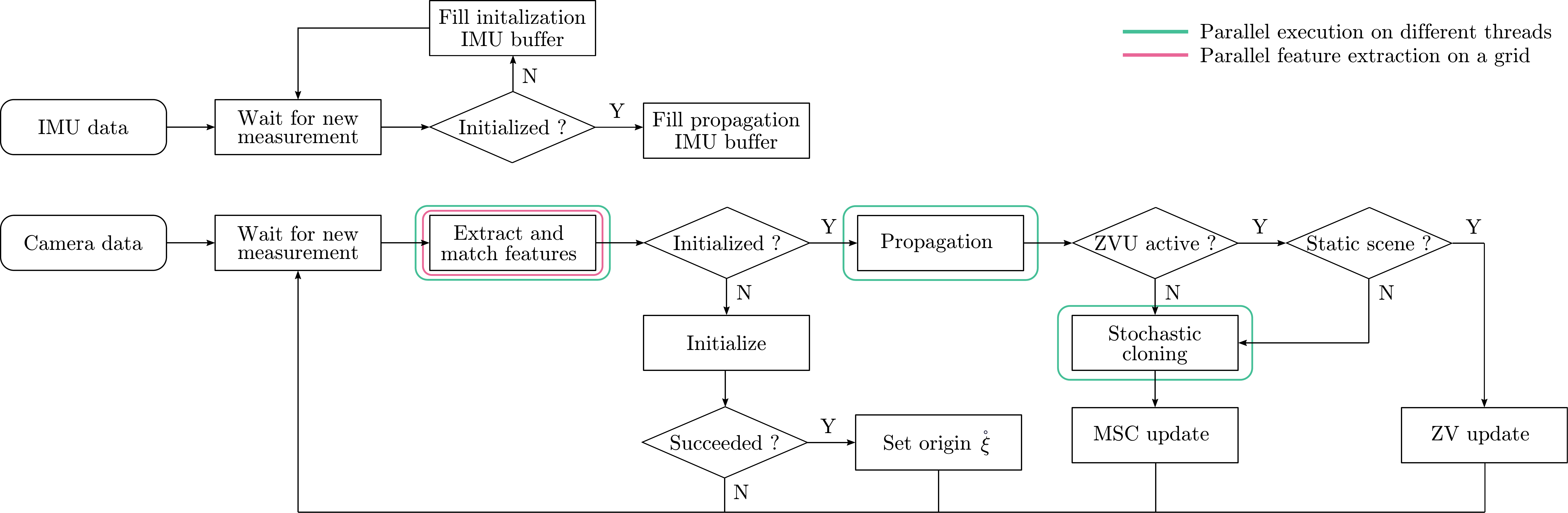}
\caption{\ac{msceqf} algorithm flowchart.}
\label{msceqf_flowchart}
\end{figure}


\section{Consistency properties of the MSCEqF}
An estimator is said to be consistent if the estimated covariance of the error reflects its real distribution; in other words, an estimator is consistent if the error is unbiased and within the sigma bounds of the estimated covariance. 
\emph{The \ac{msceqf} presented in \cref{msceqf_msceqf_sec} is a consistent estimator}. Its consistency property is proven by compatibility of the group action $\phi$ in \cref{msceqf_phi}, and invariance of the lift $\Lambda$ in \cref{msceqf_lift}, to reference frame transormations~\cite{vanGoor2023EqVIO:Odometry, Wu2017AnConsistency}. This ensures that the filter does not gain spurious information along the unobservable directions.

\begin{theorem}
Define $H{ \coloneqq (R_{H}, 0, p_{H}) \in \SE_2(3)}$, where ${R_{H} \in \SE_{\ethree}(3)}$ represent a anti-clockwise rotation about the vertical axis $\ethree$, and $p_{H}$ represent the a translation. Define the right group action ${\alpha \AtoB{\SE_2(3) \times \calM}{\calM}}$ such that ${\alpha(H, \xi) \coloneqq (H^{-1}\Pose{}{}, \Vector{}{b}{}, \mathbf{S}, \mathbf{K})}$ represents a change of reference, from \frameofref{G} to \frameofref{H} that leaves the direction of gravity unchanged.

Then the action of the symmetry group on the state space $\phi$ and the lift $\Lambda$ are respectively compatible and invariant with respect to change of reference, that is
\begin{align*}
    &\alpha(H,\phi(X,\xi)) = \phi(X, \alpha(H, \xi)),\\
    &\Lambda(\alpha(H,\xi), u) = \Lambda(\xi, u) .
\end{align*}
\end{theorem}
\begin{proof}
\begin{align*}
    \phi(X, \alpha(H, \xi)) &= ((H^{-1}\Pose{}{})D, \AdMsym{B^{-1}}\left(\Vector{}{b}{} - \delta^{\vee}\right), C^{-1}\mathbf{S}E, \mathbf{K}L)\\
    &= (H^{-1}\Pose{}{}D, \AdMsym{B^{-1}}\left(\Vector{}{b}{} - \delta^{\vee}\right), C^{-1}\mathbf{S}E, \mathbf{K}L)\\
    &= \alpha(H,\phi(X,\xi)) ,
\end{align*}
as required.

To prove the invariance of $\Lambda$ to the action $\alpha$, it is sufficient to show that ${\Lambda_1(\alpha(H,\xi), u) = \Lambda_1(\xi, u)}$.
\begin{align*}
\Lambda_1(\alpha(H,\xi), u) &= (\mathbf{W} - \mathbf{B} + \mathbf{N}) + (\Pose{}{}^{-1}H)(\mathbf{G} - \mathbf{N})(H^{-1}\Pose{}{}) \\
&= (\mathbf{W} - \mathbf{B} + \mathbf{N}) + \Pose{}{}^{-1}(H(\mathbf{G} - \mathbf{N})H^{-1})\Pose{}{} \\
&= (\mathbf{W} - \mathbf{B} + \mathbf{N}) + \Pose{}{}^{-1}(H(\mathbf{G} - \mathbf{N})H^{-1})\Pose{}{} \\
&= (\mathbf{W} - \mathbf{B} + \mathbf{N}) + \Pose{}{}^{-1}(\mathbf{G} - \mathbf{N})\Pose{}{} \\
&= \Lambda_1(\xi, u) ,
\end{align*}
where we have used the fact that ${H(\mathbf{G-N})H^{-1} = \mathbf{G-N}}$. Specifically, the gravity vector $\Vector{}{g}{}$ is aligned along the $\ethree$ direction, thus $\Vector{}{g}{} = g\ethree$, where $g$ is the magnitude of gravity. Therefore, for any rotation $R_{H}$ about the $\ethree$ axis we have that ${R_{H}\Vector{}{g}{} = R_{H}g\ethree = g\ethree = \Vector{}{g}{}}$, and hence 
\begin{equation*}
    H(\mathbf{G-N})H^{-1} = \begin{bmatrix}
    \mathbf{0}_{3\times 3} & R_{H}\Vector{}{g}{} & \mathbf{0}_{3\times 1}\\
    \mathbf{0}_{1\times 3} & 0 & -1\\
    \mathbf{0}_{1\times 3} & 0 & 0\\
    \end{bmatrix} = \begin{bmatrix}
    \mathbf{0}_{3\times 3} & \Vector{}{g}{} & \mathbf{0}_{3\times 1}\\
    \mathbf{0}_{1\times 3} & 0 & -1\\
    \mathbf{0}_{1\times 3} & 0 & 0\\
    \end{bmatrix} = \mathbf{G-N} .
\end{equation*}
This completes the proof.
\end{proof}

\section{Experimental evaluation}
In this section, we perform a series of experiments to evaluate the accuracy, consistency, and, more importantly, robustness of the proposed \ac{msceqf}. We perform many experiments on real-world data to evaluate robustness to expected and unexpected errors in the camera extrinsic calibration. In all these experiments, we limit our comparison to filter-based \ac{msckf} algorithms for \ac{vio}, and in particular, to the best available one we believe represents the state-of-the-art, that is OpenVINS~\cite{Geneva2020OpenVINS:Estimation}. For a fair comparison, we turned off OpenVINS's persistent features (\acs{slam} features), and only compare against its pure \ac{msckf} part. Furthermore, in all the experiments, OpenVINS's \ac{msckf} parameters were tuned for each dataset according to the authors' suggested parameters. In contrast, the proposed \ac{msceqf} shares the same tuning parameters, shown in \cref{msceqf_param}, across all the experiments and datasets.

\ifdefined\includetblr
\begin{table}[htp]
    \setlength\tabcolsep{5pt}
    \centering
    \captiontitlefont{\scshape\small}
    \captionnamefont{\scshape\small}
    \captiondelim{}
    \captionstyle{\centering\\}
    \caption{\ac{msceqf} parameters across all the experiments on the Euroc, TUM-VI and UZH-FPV datasets}
    \begin{tblr}{
        rows = {m},
        colspec={ccc},
    }
    \toprule
    \textsc{Features} & \textsc{Detector} & \textsc{Detection grid}\\
    \hline
    $200$ & GFTT & $6 \times 6$\\
    \toprule
    \textsc{Parametrization} & \textsc{Projection} & \textsc{Clones} \\
    \hline
    anchored Euclidean & Unit plane & $11$\\
    \bottomrule
    \end{tblr}
    \label{msceqf_param}
\end{table}
\fi

\subsection{Robustness}
Robustness is an important property of a modern filter-based \acl{vio} algorithm. It is the ability to function with significant yet known errors, as well as the ability to deal with unknown unknowns. In simpler terms, it refers to how well an algorithm performs under non-ideal conditions, such as imperfect tuning parameters, poor calibration, or unexpected changes in the sensor's extrinsic parameters during field operations.

To assess the robustness of the proposed \ac{msceqf} and OpenVINS's \ac{msckf}, we ran a series of experiments using widely-known datasets for evaluating \ac{vio} algorithms. Specifically, the Euroc dataset~\cite{Burri2016TheDatasets}, the TUM-VI dataset~\cite{Schubert2018TheOdometry}, and the UZH-FPV dataset~\cite{Delmerico19icra}. For each dataset, we selected two sequences and ran each estimator $6\times6\times6 = 216$ times (for a total number of runs of $2592$). In these experiments, we intentionally initialized the filters with incorrect camera extrinsic parameters, introducing errors in six steps ranging from ($15$\si[per-mode = symbol]{\degree}, $0.05$\si[per-mode = symbol]{\meter}), to ($90$\si[per-mode = symbol]{\degree}, $0.3$\si[per-mode = symbol]{\meter}). For each error step, we ran the estimators with six different priors (initial covariance) accounting for initial calibration errors in the range of the six error steps. We run each estimator for each pair (prior, error) six times. Finally, for each individual run, we classified an estimator as converged or diverged based on a position error threshold.

\begin{figure}[htp]
\centering
\includegraphics[width=0.77\linewidth]{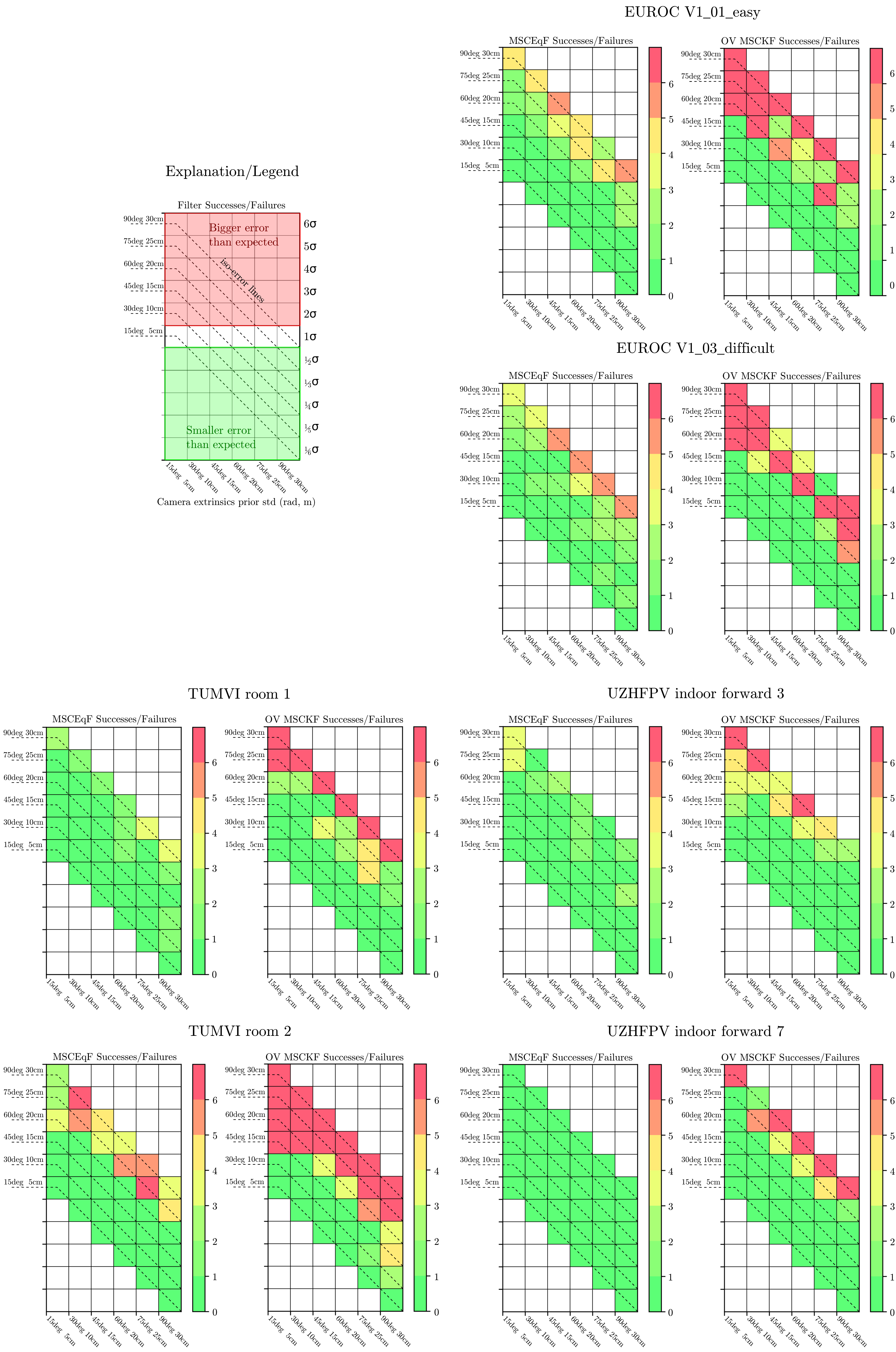}
\figrefcaption{\cite{Fornasier2023MSCEqF:Navigation}}
\caption[Results of the experiment evaluating the robustness of the proposed \ac{msceqf} and OpenVINS's \ac{msckf}.]{Results of the experiment evaluating the robustness of the proposed \ac{msceqf} and OpenVINS's \ac{msckf}. In these grid plots, the x-axis is the prior standard deviation the estimators are set with. The y-axis is how many $\sigma$-levels that error corresponds to. Labeled diagonal dashed lines represent iso-error lines (lines along with the error is constant). The bottom part of each grid represents expected errors, thus errors falling within $\sfrac{1}{6} \sigma$-$\sfrac{1}{2} \sigma$, whereas the top part of each grid represents unexpected errors, thus errors falling within $2 \sigma$-$6 \sigma$. According to the colorbar, the color of each cell shows the number of failures.}
\label{msceqf_grid}
\end{figure}

Based on the results of the experiment in \cref{msceqf_grid}, we derive the following noteworthy observations. In absolute terms, there seems to be an upper limit of absolute error that, no matter the prior, makes the estimators diverge. Although this limit highly depends on the dataset, for each of the tested sequences, the proposed \ac{msceqf} possesses a higher error limit and hence improved robustness to known absolute error.
In relative terms, the proposed \ac{msceqf} seems to deal better with unknown errors since the line at which the estimator fails is straight and does not bend towards the left side as it appears to happen for the \ac{msckf}. Encouraged by these results, we ran an additional experiment on the \emph{V1\_01\_easy} sequence of the Euroc dataset, introducing new, smaller priors and errors to effectively evaluate whether the estimators are able to manage errors that are smaller in absolute terms but outside the prior covariance. \cref{msceqf_grid_extension} clearly shows that the \ac{msceqf} is indeed a more robust filter, able to deal with unexpected errors. 
Finally, \cref{msceqf_extrinsics} shows the convergence of the camera extrinsic parameters for both filters evaluated on the Euroc \emph{V1\_01\_easy} sequence, with an initial error of ($30$\si[per-mode = symbol]{\degree}, $0.1$\si[per-mode = symbol]{\meter}) and an initial covariance to match the error. The error plots clearly show that the proposed \ac{msceqf} not only is a more robust filter, but it also converges faster. 

For the sake of completeness, the self-calibration capabilities, including camera extrinsic and intrinsic parameters, of the proposed \ac{msceqf} are re-evaluated on the Euroc \emph{V1\_01\_easy} sequence, with an initial error of ($30$\si[per-mode = symbol]{\degree}, $0.1$\si[per-mode = symbol]{\meter}) in the camera extrinsic parameters, and an initial error of $50$ units in each of the camera intrinsic parameters ($f_x,\;f_y,\;c_x,\;c_y$). Results are shown in \cref{msceqf_intrinsics}.

\begin{figure}[htp]
\centering
\includegraphics[width=\linewidth]{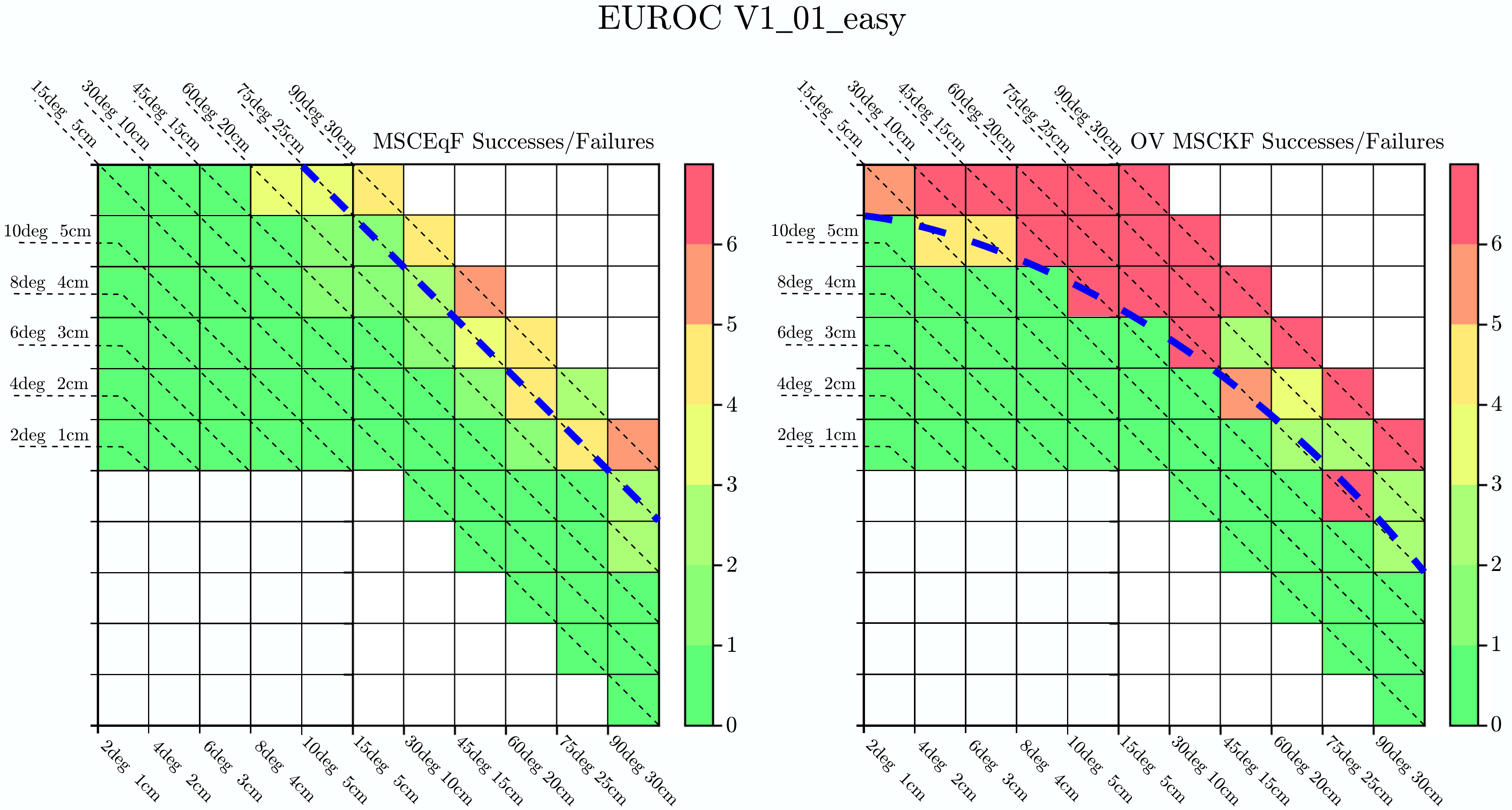}
\figrefcaption{\cite{Fornasier2023MSCEqF:Navigation}}
\caption[Results of the experiment evaluating the robustness of the proposed \ac{msceqf} compared to OpenVINS's \ac{msckf} for unexpected errors.]{Grid plot showing the robustness of the proposed \ac{msceqf} compared to OpenVINS's \ac{msckf} for unexpected errors, thus the ability to deal with \emph{you don't know what you don't know}. The x-axis is the prior standard deviation the estimators are set with. The y-axis is how many $\sigma$-levels that error corresponds to. Diagonal dashed lines represent iso-error lines. The blue bold dashed line is the limit at which each estimator fails. According to the colorbar, the color of each cell represents the number of failures.}
\label{msceqf_grid_extension}
\end{figure}

\begin{figure}[htp]
\centering
\includegraphics[width=\linewidth]{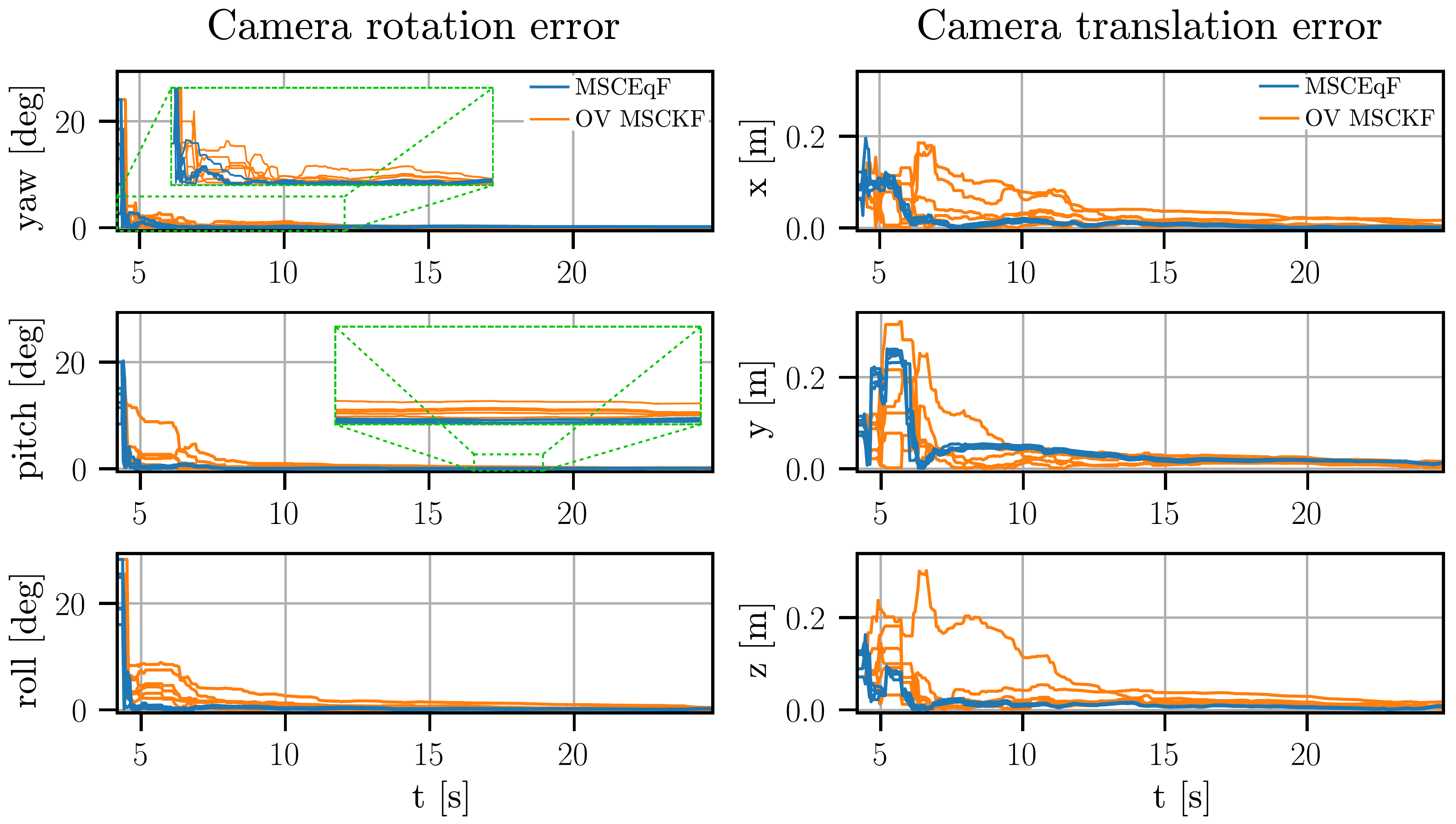}
\figrefcaption{\cite{Fornasier2023MSCEqF:Navigation}}
\caption[Absolute errors of camera extrinsic parameters for the proposed \ac{msceqf}, and OpenVINS's \ac{msckf}.]{Absolute errors of camera extrinsic parameters for the proposed \ac{msceqf}, and OpenVINS's \ac{msckf}. The plots show the convergence performance of the filters evaluated on the Euroc \emph{V1\_01\_easy} sequence, for $6$ runs, with an initial error of ($30$\si[per-mode = symbol]{\degree}, $0.1$\si[per-mode = symbol]{\meter}).}
\label{msceqf_extrinsics}
\end{figure}

\begin{figure}[htp]
\centering
\includegraphics[width=\linewidth]{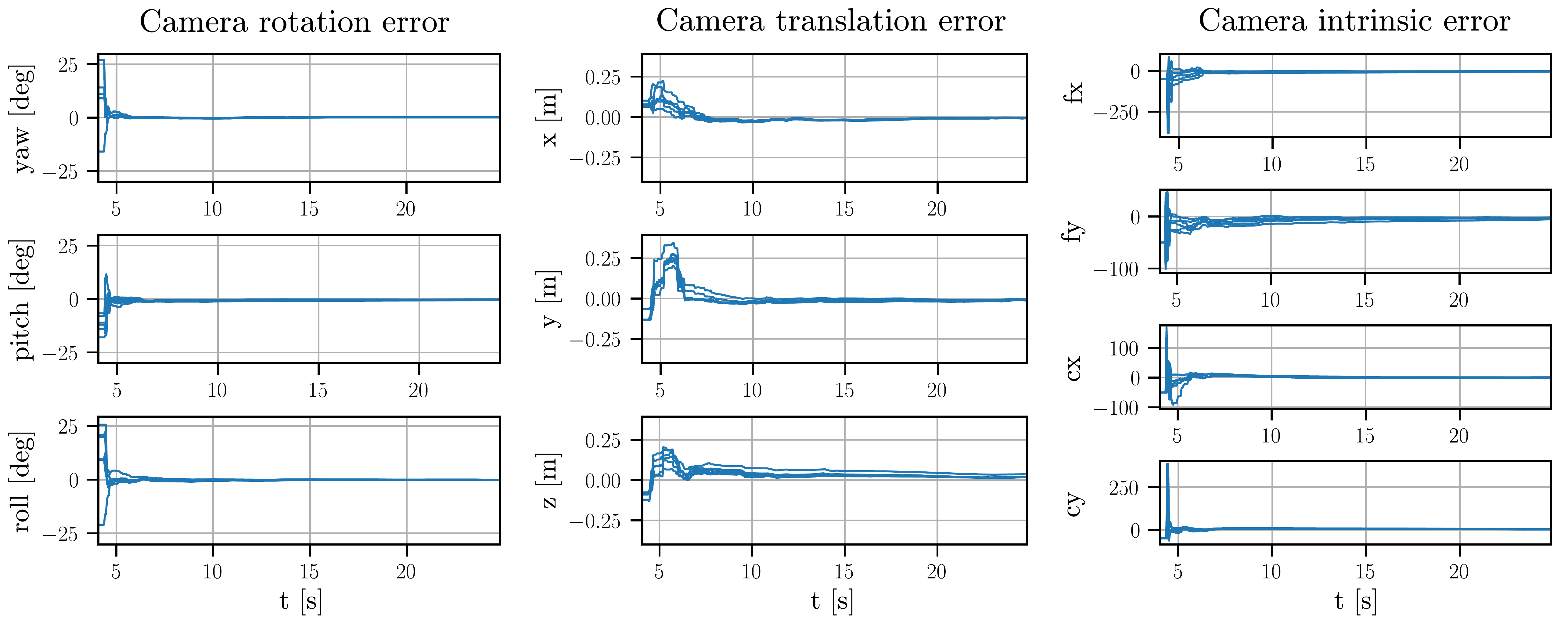}
\caption[Errors of camera extrinsic and intrinsic parameters for the proposed \ac{msceqf}.]{Errors of camera extrinsic and intrinsic parameters for the proposed \ac{msceqf}. The plots show the convergence performance of the filter evaluated on the Euroc \emph{V1\_01\_easy} sequence, for $6$ runs, with an initial error of ($30$\si[per-mode = symbol]{\degree}, $0.1$\si[per-mode = symbol]{\meter}) in the camera extrinsic parameters, and an initial error of $50$ units in each of the camera intrinsic parameters ($f_x,\;f_y,\;c_x,\;c_y$).}
\label{msceqf_intrinsics}
\end{figure}

Quantifying robustness in robotics, however, remains an ongoing challenge. 
In the presented evaluation, we have chosen the camera extrinsic calibration as a state subjected to error. Even though static and dynamic initialization approach exists~\cite{Dong-Si2012EstimatorCalibration, Campos2019FastSLAM} for such a problem, in our formulation, extrinsic parameters are treated as regular state variables, and our proposed algorithm showcases inherent robustness by successfully attaining reliable estimation, for both expected and unexpected errors, eliminating the need of any auxiliary module.
This characteristic sets our algorithm apart from conventional \ac{vio} algorithms, emphasizing its superior robustness.

\subsection{Accuracy}
Our next experiment focuses on the classical and widely-used metric for evaluating the performance of \acl{vio} algorithms~\cite{Delmerico2018ARobots}, namely the \ac{rmse} of the \ac{ate}. For this experiment, we ran the proposed \ac{msceqf} and OpenVINS's \ac{msckf} on all Euroc sequences~\cite{Burri2016TheDatasets}. The results presented in \cref{msceqf_euroc_rmse} demonstrate that the proposed \ac{msceqf} achieves state-of-the-art accuracy comparable to \changed{OpenVINS}{the \ac{msckf}}. It should be noted that in our evaluation, we aligned each estimate with the groundtruth using the initial state rather than finding the optimal alignment that minimizes the error throughout the entire trajectory.

\ifdefined\includetblr
\begin{table}[htp]
    \setlength\tabcolsep{5pt}
    \centering
    \tabrefcaption{\cite{Fornasier2023MSCEqF:Navigation}}
    \captiontitlefont{\scshape\small}
    \captionnamefont{\scshape\small}
    \captiondelim{}
    \captionstyle{\centering\\}
    \caption{Attitude (A), and position (P) Absolute Trajectory Error (ATE) RMSE on Euroc dataset.}
    \begin{tblr}{
        rows = {m},
        colspec={c|cc|cc},
    }
    \toprule
    \textsc{Sequence} & \SetCell[c=2]{c}{\textsc{MSCEqF}} & & \SetCell[c=2]{c}{\textsc{OV MSCKF}} & \\
    \hline
    & A [\si[per-mode = symbol]{\radian}] & P [\si[per-mode = symbol]{\meter}] & A [\si[per-mode = symbol]{\radian}] & P [\si[per-mode = symbol]{\meter}] \\
    \hline
    V1\_01\_easy & $0.07$ & \color{Blue1}{$\mathbf{0.24}$} & \color{Blue1}{$\mathbf{0.05}$} & $0.36$\\
    V1\_02\_medium & $0.03$ & \color{Blue1}{$\mathbf{0.20}$} & \color{Blue1}{$\mathbf{0.02}$} & $0.22$\\
    V1\_03\_difficult & $0.05$ & $0.30$ & \color{Blue1}{$\mathbf{0.02}$} & \color{Blue1}{$\mathbf{0.18}$} \\
    V2\_01\_easy & \color{Blue1}{$\mathbf{0.02}$} & \color{Blue1}{$\mathbf{0.13}$} & $0.05$ & $0.18$ \\
    V2\_02\_medium & $0.08$ & $0.55$ & \color{Blue1}{$\mathbf{0.03}$} & \color{Blue1}{$\mathbf{0.17}$}\\
    V2\_03\_difficult\footnotemark{} & \color{Blue1}{$\mathbf{0.03}$} & $0.39$ & \color{Blue1}{$\mathbf{0.03}$} & \color{Blue1}{$\mathbf{0.28}$}\\
    MH\_01\_easy & \color{Blue1}{$\mathbf{0.05}$} & \color{Blue1}{$\mathbf{0.29}$} & \color{Blue1}{$\mathbf{0.05}$} & $0.42$ \\
    MH\_02\_easy & \color{Blue1}{$\mathbf{0.01}$} & \color{Blue1}{$\mathbf{0.38}$} & $0.03$ & $0.54$\\
    MH\_03\_medium & $0.02$ & \color{Blue1}{$\mathbf{0.34}$} & \color{Blue1}{$\mathbf{0.01}$} & $0.41$\\
    MH\_04\_difficult & \color{Blue1}{$\mathbf{0.03}$} & \color{Blue1}{$\mathbf{0.53}$} & $0.04$ & $0.61$\\
    MH\_05\_difficult & $0.03$ & \color{Blue1}{$\mathbf{0.70}$} & \color{Blue1}{$\mathbf{0.02}$} & $0.78$\\
    \bottomrule
    \end{tblr}
    \label{msceqf_euroc_rmse}
\end{table}
\footnotetext{Due to non-deterministic results with varying in accuracy, the best result out of $5$ runs was reported}
\fi

\subsection{Consistency}
In this final experiment, we employed the pose (orientation and position) \ac{anees} as a metric to analyze the consistency of the proposed \ac{msceqf}. In particular, we used the VINSEval framework~\cite{Fornasier2021} to generate a photorealistic synthetic dataset of 25 runs of the same trajectory, with the same noise statistics but different noise realizations.

The \ac{anees} for the \ac{msceqf} was computed according to the following formula
\begin{equation*}
    \text{ANEES} = \frac{1}{Mn}\sum_{i=1}^{M}\varepsilon_i^T\bm{\Sigma}_i^{-1}\varepsilon_i ,
\end{equation*}
where $M$ is the number of runs, $n = dim\left(\varepsilon\right)$ is the dimension of the error $\varepsilon$, and $\bm{\Sigma}$ is the covariance of the error. The error $\varepsilon = \log_{\SE(3)}\left(\mathring{\PoseP{}{}}^{-1}\PoseP{}{}\hatPoseP{}{}^{-1}\mathring{\PoseP{}{}}\right)^{\vee}$ is the pose components of the equivariant error defined in \cref{msceqf_msceqf_sec}.

The resulting \ac{anees} shown in \cref{msceqf_anees} fluctuates around a computed average of $1.0$ and is not increasing or decreasing over time. The results obtained are similar to \ac{fej} estiamtors~\cite{Geneva2020OpenVINS:Estimation, Chen2022FEJ2:Design}, but without requiring artificial modification of the linearization points to achieve consistency.

\begin{figure}[htp]
\centering
\includegraphics[width=\linewidth]{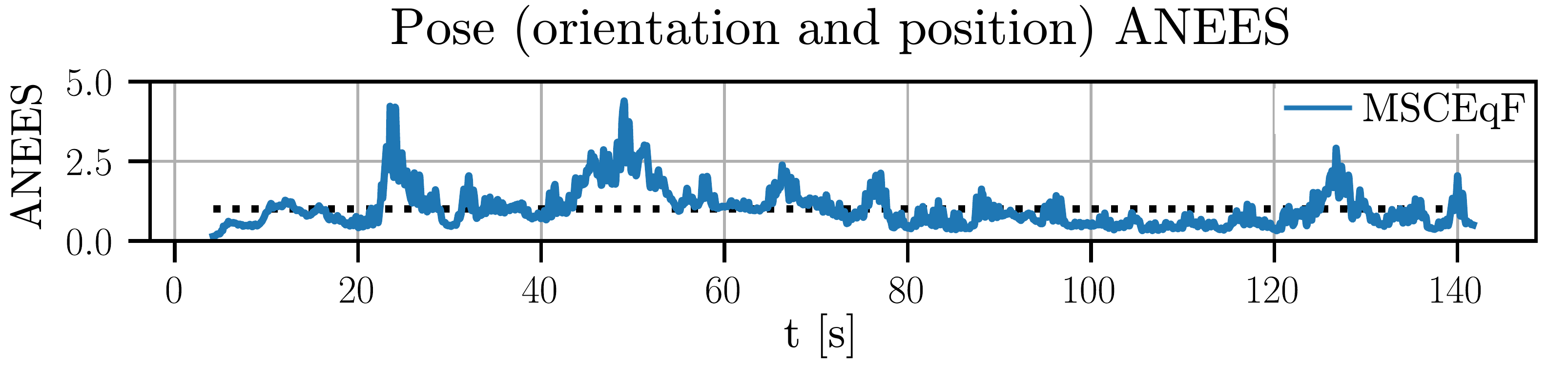}
\figrefcaption{\cite{Fornasier2023MSCEqF:Navigation}}
\caption[Pose \acs{anees} of the proposed \ac{msceqf}]{Pose (orientation and position) \acs{anees} of the proposed \ac{msceqf} executed for 25 runs on a custom photorealistic dataset generated with the VINSEval framework. The plot shows that the \acs{anees} is fluctuating around an average of $1.0$ (dashed line) and does not increase or decrease over time.}
\label{msceqf_anees}
\end{figure}

\section{Real-time performance evaluation -- closed-loop control of resource-constrained aerial platforms}

In this section, we evaluate the performance obtained by deploying the proposed \ac{msceqf} on a resource-constrained aerial platform for closed-loop control. Specifically, the \ac{msceqf} was run on a Raspberry Pi 4 with 4GB of RAM mounted on a Twins twinFOLD SCIENCE quadcopter. The \ac{imu} data was collected from the quadcopter autopilot, a Holybro Pixhawk 4 at $200$\si[per-mode = symbol]{\hertz}, whereas the image data was collected at $15$\si[per-mode = symbol]{\hertz} from a $60$\si[per-mode = symbol]{\degree} downward-looking 
Matrix Vision BlueFox MLC200WG global-shutter camera with $752 \times 480$ CMOS grayscale sensor.

With this experiment, we do not claim completeness in comparative evaluations but rather demonstrate the real-time capabilities of the proposed framework and its applicability to real-world vision-based robotic navigation.

The experiment consists of flying $38$ times a $2.5$\si[per-mode = symbol]{\meter}$\times 1.5$\si[per-mode = symbol]{\meter} square trajectory for a total length of $288$\si[per-mode = symbol]{\meter}. For this experiment, the \ac{msceqf} was set with the parameters shown in \cref{msceqf_param_cl}. Moreover, curvature correction was disabled.

\ifdefined\includetblr
\begin{table}[htp]
    \setlength\tabcolsep{5pt}
    \centering
    \captiontitlefont{\scshape\small}
    \captionnamefont{\scshape\small}
    \captiondelim{}
    \captionstyle{\centering\\}
    \caption{\ac{msceqf} parameters for closed-loop control running on a Raspberry Pi 4 with 4GB RAM}
    \begin{tblr}{
        rows = {m},
        colspec={ccc},
    }
    \toprule
    \textsc{Features} & \textsc{Detector} & \textsc{Detection grid}\\
    \hline
    $60 - 90$ & Fast & $4 \times 4$\\
    \toprule
    \textsc{Parametrization} & \textsc{Projection} & \textsc{Clones} \\
    \hline
    anchored Euclidean & Unit plane & $9$\\
    \bottomrule
    \end{tblr}
    \label{msceqf_param_cl}
\end{table}
\fi

Trajectory, position, and orientation plots for the closed-loop control experiment are reported in \cref{msceqf_closedloop_plots}. The \ac{rmse} of the \ac{ate} for this experiment sets at $1.79$\si[per-mode = symbol]{\degree} and $0.38$\si[per-mode = symbol]{\meter}. Moreover, we evaluated the position and angular error every $135$ steps, to evaluate the accumulated drift and obtained a final angular error of $4.25$\si[per-mode = symbol]{\degree} and a final position error of $0.64$\si[per-mode = symbol]{\meter} which corresponds to a drift of $0.2\%$ of the total trajectory length. Accumulated angular and position errors are reported in \cref{msceqf_closedloop_ape}.

This experiment has demonstrated the real-time capabilities of the \ac{msceqf} for closed-loop control of resource-constrained platforms. Although many modules of the CNS Flight stack~\cite{Scheiber2022} modules were running on the Raspberry Pi during the execution of the experiment, the framerate of the estimate provided by the \ac{msceqf} was constant at the camera rate, and no computational spikes or high load was registered on the CPU. For this experiment, we kept a conservative tuning (\cref{msceqf_param_cl}); the accuracy of the \ac{msceqf} could be further improved by increasing the number of tracked features and the number of stochastic clones. Furthermore, additional techniques specifically tailored for resource-constrained platforms, such as dynamic number of stochastic clones as in~\cite{Li2012Vision-aidedSystems}, could be implemented to set an upper bound of the time taken for each iteration of the filter.

\begin{figure}[htp]
\centering
\includegraphics[width=0.825\linewidth]{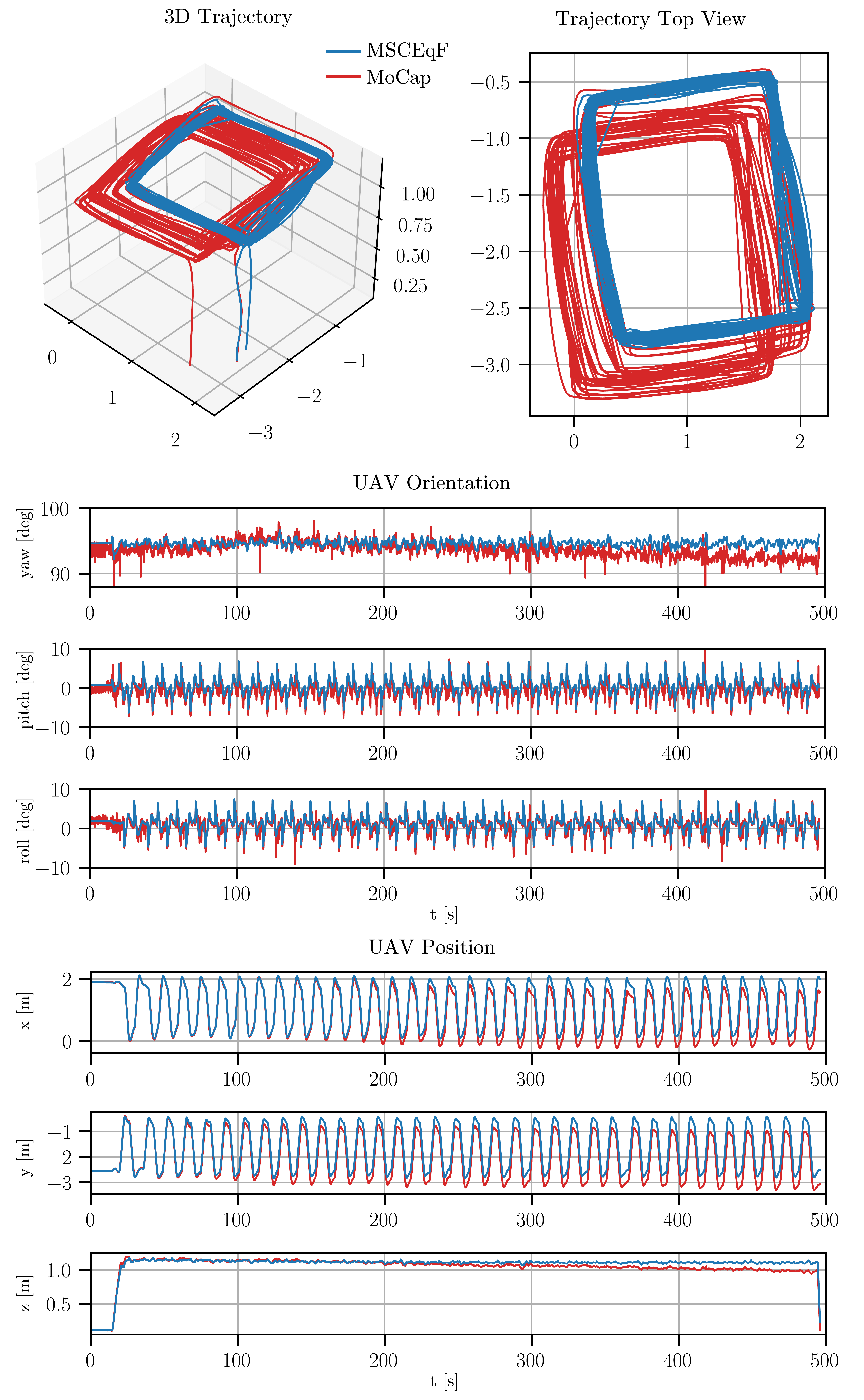}
\caption[\ac{msceqf} estimate and motion capture groundtruth comparison for the closed-loop control of a computationally-constrained aerial platform experiment.]{\ac{msceqf} estimate and motion capture groundtruth comparison for the closed-loop control of a computationally-constrained aerial platform experiment. \textbf{Top:} 3D Trajectory and top view. \textbf{Middle:} \ac{uav} orientation in degrees. \textbf{Bottom:} \ac{uav} position in meters}
\label{msceqf_closedloop_plots}
\end{figure}

\begin{figure}[htp]
\centering
\includegraphics[width=\linewidth]{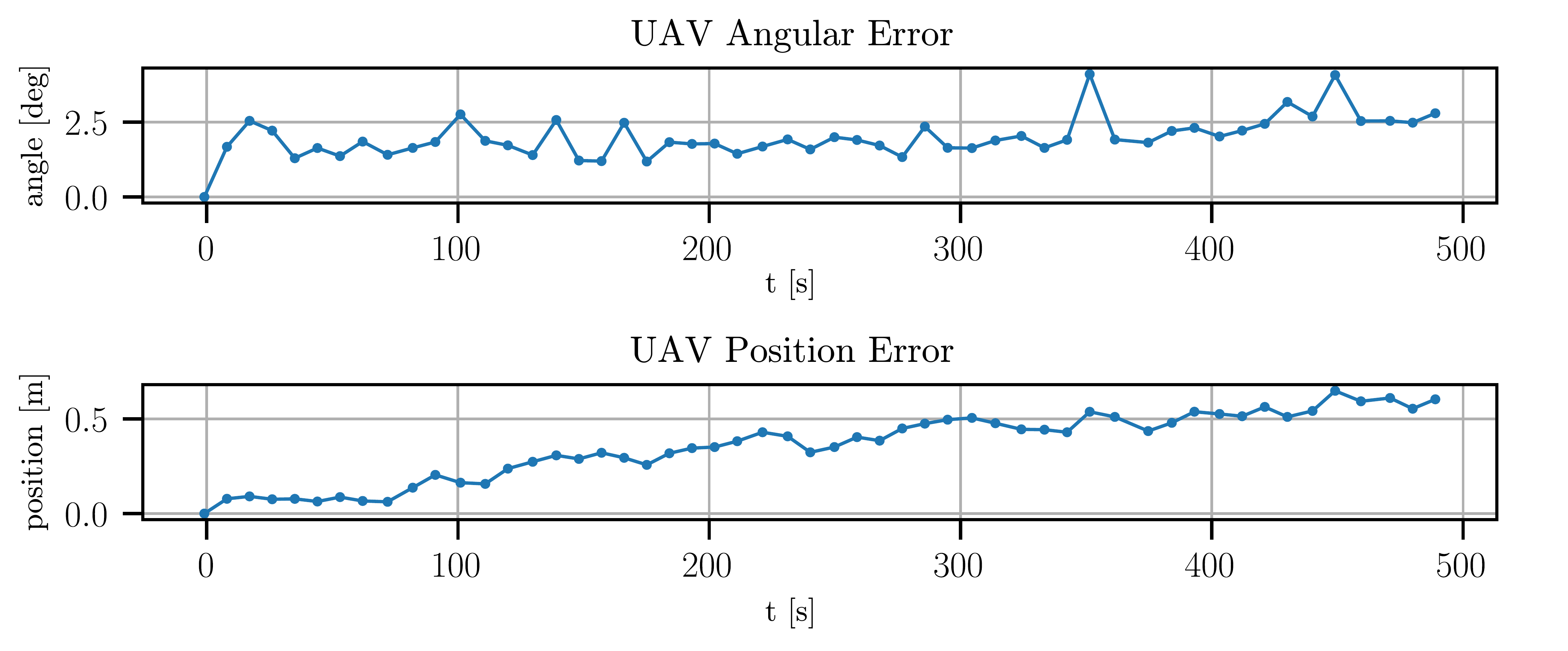}
\caption[Error plots for the closed-loop control of a computationally-constrained aerial platform experiment.]{\ac{uav} angular and position errors evaluated every $135$ steps for the closed-loop control of a computationally-constrained aerial platform experiment. The final errors are respectively at $4.25$\si[per-mode = symbol]{\degree} and $0.64$\si[per-mode = symbol]{\meter}, which correspond to a drift of $0.2\%$ of the total trajectory length.}
\label{msceqf_closedloop_ape}
\end{figure}

\section{Hybrid equivariant visual-inertial navigation system -- MSCEqF with persistent features}

Multi state constraint filters are specifically designed to minimize computational complexity. They do not retain visual landmarks as part of their state, resulting in high sensitivity to accumulated drift. In contrast, classical filter-based algorithms incorporate a subset of visual landmarks as part of their state, with the benefit of limiting the accumulated drift, allowing for loop-closures~\cite{Durrant-Whyte2006SimultaneousI, Geneva2019AClosures}, albeit at the expense of increased computational complexity.

In this section, we exploit the polar symmetry of fixed landmark measurements introduced by van Goor \etal~\cite{Goor2019AMapping, VanGoor2020AnEquivariance, vanGoor2021ConstructiveMapping, vanGoor2021AnOdometry, vanGoor2023EqVIO:Odometry, VanGoor2023EquivariantAwareness}, and propose an extension of the \ac{msceqf} that includes a subset of visual landmark as part of the state. 

\subsection{Extended symmetry}

Recall the procedure described in \cref{ms_bins_ext_sec}. We propose the following extension of the symmetry introduced in \cref{msceqf_vins_sec}. Consider the state $(\xi_{I}, \xi_{S})$ of the visual inertial navigation system discussed in \cref{msceqf_vins_sec}. Without loss of generality, extend the system state with the position of a fixed visual feature in the world; therefore, define ${\xi_{f} = \Vector{}{p}{f} \in \R^3}$. The full system states is defined as ${\xi = (\xi_{I}, \xi_{S}, \xi_{f}) \in \calM \coloneqq \torSE_2(3) \times \R^{6} \times \torSE(3) \times \torIN(3) \times \R^3}$ . Define a novel group $\grpG \coloneqq \grpB \times \SE(3) \times \IN \times \SOT(3)$ termed $\mathbf{VINS}$ group. Elements of the $\mathbf{VINS}$ group $\grpG$ and the state action ${\phi \AtoB{\grpG \times \calM}{\calM}}$ are defined as follows:
\begin{align}
    &X = \left(\left(D,\delta\right), E, L, Q\right) \in \grpG,\\
    &X^{-1} = ((D^{-1}, -\Adsym{D^{-1}}{\delta}), E^{-1}, L^{-1}, Q^{-1}) \in \grpG,\\
    &\phi\left(X, \xi\right) \coloneqq \left(\Pose{}{}D, \AdMsym{B^{-1}}\left(\Vector{}{b}{} - \delta^{\vee}\right), C^{-1}\mathbf{S}E, \mathbf{K}L, \right.\\
    &\qquad\qquad\quad\left.\PoseP{}{}\PoseS{}{}E * Q^{-1}(\PoseP{}{}\PoseS{}{})^{-1} * \Vector{}{p}{f}\right) \in \calM \label{msceqf_phi_pf},
\end{align}
where ${Q = (\mathrm{Q}, \mathrm{q}), Q^{-1} = (\mathrm{Q}^{\top}, \frac{1}{\mathrm{q}}) \in \SOT(3)}$. $\PoseP{}{}\PoseS{}{}E * Q^{-1}(\PoseP{}{}\PoseS{}{})^{-1} * \Vector{}{p}{f}$ is a concatenation of roto-translation actions of $\SE(3)$, and rotation and scaling actions of $\SOT(3)$ on $\R^{3}$. $* \AtoB{\torSE(3) \times \R^3}{\R^3}$ is defined by ${\PoseP{}{} * \Vector{}{v}{} = \Rot{}{}\Vector{}{v}{} + \Vector{}{p}{}}$  for all ${ \PoseP{}{} = \left(\Rot{}{}, \Vector{}{p}{}\right) \in \torSE(3),\; \Vector{}{v}{} \in \R^3}$. The last expression can also be written as $(\PoseP{}{}\PoseS{}{}) * \frac{1}{\mathrm{q}}\mathrm{Q}^{\top}((\PoseP{}{}\PoseS{}{})^{-1} * \Vector{}{p}{f})$.

\subsection{Extended lift}

The lift presented in \cref{msceqf_lift} is extended as follows.
\begin{theorem}
Let $\Lambda_3 = (\mathbf{\Omega}, \Vector{}{\omega}{}) \in \se(3)$, with $\mathbf{\Omega} \in \so(3)$ and $\Vector{}{\omega}{} \in \R^3$. Let $\Vector{}{c}{f} = (\PoseP{}{}\PoseS{}{})^{-1} * \Vector{}{p}{f}$. Define the map ${\Lambda_5 \AtoB{\calM \times \vecL}{\sot(3)}}$ by
\begin{equation}\label{msceqf_ext_lift}
    \Lambda_5 = (\mathbf{\Omega} + \frac{\Vector{}{c}{f}^{\wedge}\Vector{}{\omega}{}}{\norm{\Vector{}{c}{f}}^2}, \frac{\Vector{}{c}{f}^{\top}\Vector{}{\omega}{}}{\norm{\Vector{}{c}{f}}^2})
\end{equation}
Then ${\Lambda \AtoB{\calM \times \vecL}{\gothg}}$ given by
\begin{align*}
    \Lambda\left(\xi, u\right) &\coloneqq \left(\left(\Lambda_1\left(\xi, u\right), \Lambda_2\left(\xi, u\right)\right), \Lambda_3\left(\xi, u\right), \Lambda_4\left(\xi, u\right), \Lambda_5\left(\xi, u\right)\right),
\end{align*}
is a lift for the system in~\cref{msceqf_vins}, extended with a single visual feature $\Vector{}{p}{f}$, with respect to the symmetry group $\grpG$.
\end{theorem}

\subsection{Extended state and input matrices}
The state matrix $\mathbf{A}_{t}^{0}$ in~\cref{msceqf_A}, is extended as follows:
\begin{equation}\label{msceqf_A_pf}
    \mathbf{A}_{t}^{0} = \begin{bNiceArray}{ccccc}[margin]
        \prescript{}{1}{\mathbf{A}} & \prescript{}{2}{\mathbf{A}} & \Vector{}{0}{9 \times 6} & \Vector{}{0}{9 \times 4} & \Vector{}{0}{9 \times 4}\\
        \prescript{}{3}{\mathbf{A}} & \prescript{}{4}{\mathbf{A}} & \Vector{}{0}{6 \times 6} & \Vector{}{0}{6 \times 4} & \Vector{}{0}{6 \times 4}\\
        \prescript{}{5}{\mathbf{A}} & \prescript{}{6}{\mathbf{A}} & \prescript{}{7}{\mathbf{A}} & \Vector{}{0}{6 \times 4} & \Vector{}{0}{6 \times 4}\\
        \Vector{}{0}{4 \times 9} & \Vector{}{0}{4 \times 6} & \Vector{}{0}{4 \times 6} & \Vector{}{0}{4 \times 4} & \Vector{}{0}{4 \times 4} \Bstrut\\
        \hdottedline
        \Block{1-5}{\prescript{}{8}{\mathbf{A}}} &&&& \Bstrut\\
        \hdottedline
        \Block{1-5}{\prescript{}{9}{\mathbf{A}}} &&&& \Tstrut
    \end{bNiceArray} \in \R^{29 \times 29},
\end{equation}
where
\begin{align*}
    &\begin{bmatrix}
        \mathbf{\Upsilon}\\
        \mathbf{\Gamma}
    \end{bmatrix} = \AdMsym{\hat{E}^{-1}}
    \begin{bmatrix}
        \prescript{}{5}{\mathbf{A}} & \prescript{}{6}{\mathbf{A}} & \prescript{}{7}{\mathbf{A}} & \Vector{}{0}{6 \times 4} & \Vector{}{0}{9 \times 4}
    \end{bmatrix}  \in \R^{6 \times 29},\\
    &\Vector{}{\gamma}{} = \frac{1}{\hat{\mathrm{q}}}\hat{\mathrm{Q}}^{\top}(\mathring{\PoseP{}{}}\mathring{\PoseS{}{}})^{-1}\mathring{\Vector{}{p}{f}} = \frac{1}{\hat{\mathrm{q}}}\hat{\mathrm{Q}}^{\top}\mathring{\Vector{}{c}{f}} \in \R^{3}\\
    &\Vector{}{\zeta}{} = \frac{1}{\mathrm{q}}\hat{\mathrm{Q}}^{\top}
    \begin{bmatrix}
        \mathring{\Vector{}{c}{f}}^{\wedge} & -\mathring{\Vector{}{c}{f}}
    \end{bmatrix} \in \R^{3 \times 4},\\
    & \mathbf{\Theta} = 
    \begin{bmatrix}
        \Vector{}{0}{3 \times 15} & \Vector{}{0}{3 \times 6} & \Vector{}{0}{3 \times 6} & \Vector{}{0}{3 \times 4} & \left(\left(\eye_3 - \frac{2\Vector{}{\gamma}{}\Vector{}{\gamma}{}^{\top}}{\norm{\Vector{}{\gamma}{}}^2}\right)\Vector{}{\zeta}{}\right)
    \end{bmatrix} \in \R^{3 \times 29},\\    
    &\prescript{}{8}{\mathbf{A}} = \hat{\mathrm{Q}}\left(\mathbf{\Upsilon} + \frac{1}{\norm{\Vector{}{\gamma}{}}^2}\left(
    \Vector{}{\gamma}{}^{\wedge}\mathbf{\Gamma} - \Vector{}{\omega}{}^{\wedge}\mathbf{\Theta}
    \right)\right) \in \R^{3 \times 29},\\
    &\prescript{}{9}{\mathbf{A}} = \frac{1}{\norm{\Vector{}{\gamma}{}}^2}\left(\Vector{}{\gamma}{}^{\top}\mathbf{\Gamma} + \Vector{}{\omega}{}^{\top}\mathbf{\Theta}\right) \in \R^{1 \times 29}.
\end{align*}
Similarly, the input matrix $\mathbf{B}_{t}$ in~\cref{msceqf_B} is extended and it is written
\begin{equation}\label{msceqf_B_pf}
    \mathbf{B}_{t} = \begin{bNiceArray}{cc}[margin]
        \prescript{}{1}{\mathbf{B}} & \Vector{}{0}{9 \times 6}\\
        \prescript{}{2}{\mathbf{B}} & \prescript{}{3}{\mathbf{B}}\\
        \prescript{}{4}{\mathbf{B}} & \Vector{}{0}{6 \times 6}\\
        \Vector{}{0}{6 \times 6} & \Vector{}{0}{6 \times 6} \Bstrut\\
        \hdottedline
        \prescript{}{5}{\mathbf{B}} & \Vector{}{0}{3 \times 6} \Bstrut\\
        \hdottedline
        \prescript{}{6}{\mathbf{B}} & \Vector{}{0}{1 \times 6} \Tstrut
    \end{bNiceArray} \in \R^{29 \times 12},
\end{equation}
with
\begin{align*}
    &\prescript{}{5}{\mathbf{B}} = \hat{\mathrm{Q}} 
    \begin{bmatrix}
        \eye_3 & \frac{\Vector{}{\gamma}{}^{\wedge}}{\norm{\Vector{}{\gamma}{}}^2}
    \end{bmatrix}
    \prescript{}{4}{\mathbf{B}} \in \R^{3 \times 6}\\
    &\prescript{}{6}{\mathbf{B}} = \frac{\Vector{}{\gamma}{}^{\top}}{\norm{\Vector{}{\gamma}{}}^2}
    \begin{bmatrix}
        \Vector{}{0}{3 \times 3} & \eye_3 
    \end{bmatrix}
    \prescript{}{4}{\mathbf{B}} \in \R^{1 \times 6}
\end{align*}

\subsection{Equivariant persistent features update}

When camera intrinsic parameters are explicitly considered, the configuration output in \cref{msceqf_h} is not equivariant with respect to the defined symmetry. However, as demonstrated in~\cite{vanGoor2023EqVIO:Odometry} and in \cref{ms_bins_sot3_example}, the output is equivariant when no camera intrinsic parameters are considered, and the projection onto the unit-sphere is chosen. We will leverage this observation to define an update routine of the hybrid-\ac{msceqf}, where camera intrinsic parameters are excluded from the configuration output. This modification leads to an update that exploits the equivariance of the output and achieves third-order linearization error. 

\begin{lemma}
    Consider the configuration output
    \begin{equation}\label{msceqf_h_pf}
         h(\xi) = \pi_{\mathbb{S}^2}\left(\left(\mathbf{P}\mathbf{S}\right)^{-1} * \Vector{}{p}{f}\right) \in \calN \coloneqq \mathbb{S}^2.
    \end{equation}
    Similar to~\cref{ms_bins_sot3_example}, applying the action $\phi$ of the symmetry group in~\cref{msceqf_phi_pf} yields
    \begin{align*}
        h(\phi_X(\xi)) &= \pi_{\mathbb{S}^2}\left(\left(\PoseP{}{}\PoseS{}{}E\right)^{-1}\left(\PoseP{}{}\PoseS{}{}\right)E * \left(Q^{-1}\left(\left(\PoseP{}{}\PoseS{}{}\right)^{-1} * \Vector{}{p}{f}\right)\right)\right)\\
        &= \pi_{\mathbb{S}^2}\left(Q^{-1}\left(\left(\PoseP{}{}\PoseS{}{}\right)^{-1} * \Vector{}{p}{f}\right)\right)\\
        &= \pi_{\mathbb{S}^2}\left(\frac{1}{\mathrm{q}}\mathrm{Q}^{\top}\left(\left(\PoseP{}{}\PoseS{}{}\right)^{-1} * \Vector{}{p}{f}\right)\right)\\
        &= \mathrm{Q}^{\top}\pi_{\mathbb{S}^2}\left(\left(\PoseP{}{}\PoseS{}{}\right)^{-1} * \Vector{}{p}{f}\right)\\
        &= \mathrm{Q}^{\top}h(\xi) .\\
    \end{align*}
    Let $y \in \calN$ be a measurement defined according to the model in~\cref{msceqf_h_pf}.
    Define ${\rho \AtoB{\grpG \times \calN}{\calN}}$ as
    \begin{equation}\label{msceqf_pf_rho}
        \rho\left(X,y\right) \coloneqq \mathrm{Q}^{\top}y.
    \end{equation}
    Then, the configuration output defined in \cref{msceqf_h_pf} is equivariant.
\end{lemma}

Let $\yzero \coloneqq h(\xizero)$, then local coordinates for the output are defined as follows:
\begin{equation}
    \delta(y) = \yzero^{\wedge}y \in \R^3 .
\end{equation}

The equivariance of the output in \cref{msceqf_h_pf} allows us to achieve third-order linearization error through the equivariant output approximation in~\cref{eq_eq_out_app,eq_C_star}; hence the $\mathbf{C}^{\star}$ matrix is written
\begin{align*}
    \mathbf{C}^{\star}\varepsilon &= \frac{1}{2}\yzero^{\wedge}\left(\yzero + \hat{\mathrm{Q}}y\right)^{\wedge}\varepsilon_Q\\
    &= \begin{bmatrix} \Vector{}{0}{3 \times 9} & \Vector{}{0}{3 \times 6} & \Vector{}{0}{3 \times 6} & \Vector{}{0}{3 \times 4} & \frac{1}{2}\yzero^{\wedge}\left(\yzero + \hat{\mathrm{Q}}y\right)^{\wedge} & \Vector{}{0}{3 \times 1} & \Vector{}{0}{3 \times 6} & \cdots & \Vector{}{0}{3 \times 6} \end{bmatrix}\varepsilon .
\end{align*}

\section{Chapter conclusion}

This chapter exploited the theoretical framework established in \cref{bins_chp,ms_bins_chp} and \emph{presented a novel symmetry group for \aclp{vins}, along with a novel \acl{eqf} formulation for the \acl{vio} problem; the \acf{msceqf}}. The proposed system integrates the concept of multi state constraint into an \ac{eqf} based on the proposed symmetry, resulting in a filter that achieves \emph{state-of-the-art accuracy and beyond state-of-the-art robustness}. Furthermore, the proposed symmetry is compatible with reference frame transformations that preserve the direction of gravity. This ensures that \emph{the \ac{msceqf} does not gain spurious information along the unobservable direction and is a consistent estimator} 

A known issue of multi state constraint filter formulations is their inability to handle zero motion scenarios without either a custom keyframe selection strategy or \acl{zvu}. To address this issue, \emph{a novel equivariant formulation of the \acl{zvu} that exploits the symmetry of \acl{vins} is presented}.

Through a series of experiments, it is demonstrated that the \ac{msceqf} exhibits robustness against both high absolute errors and unexpected errors that exceed the prior covariance. It achieves accuracy in the estimation comparable to a state-of-the-art \ac{msckf} while naturally maintaining consistency. Unlike \ac{fej}~\cite{huang2009first} and \ac{oc}~\cite{Hesch2014ConsistencyNavigation} methodologies, the \ac{msceqf} does not require artificial manipulation of the linearization point to achieve consistent estimation. Moreover, the proposed system's real-time capabilities and applicability to real-world vision-based robotic navigation scenarios are demonstrated by deploying the \ac{msceqf} to a Raspberry Pi 4 for closed-loop control of a resource-constrained aerial platform.

Finally, this chapter is concluded by proposing an extension of the \acl{vins} symmetry that includes fixed visual features. Such extension exploits the polar symmetry for fixed landmarks proposed by van Goor \etal~\cite{Goor2019AMapping, VanGoor2020AnEquivariance, vanGoor2021ConstructiveMapping, Mahony2020EquivariantWild, vanGoor2021AnOdometry, vanGoor2023EqVIO:Odometry, VanGoor2023EquivariantAwareness} to build a \emph{hybrid-\ac{msceqf}}

\chapter{Discussion and Conclusion}\label{concl_chp}

\section{Discussion}

The recognition that current theories of state estimation for biased \aclp{ins} are incomplete, along with the rise of a novel, powerful theory of equivariant systems, has motivated the research presented in this thesis. Specifically, the search for new equivariant symmetries of \aclp{ins} that properly account for \ac{imu} biases and the design of novel \aclp{eqf} for such problems.

The introduction of the tangent group for biased \aclp{ins}, a Lie group symmetry accounting for \ac{imu} biases, and the subsequent development of the first \acl{eqf} for biased \aclp{ins} with autonomous navigation error dynamics represent significant contributions of this dissertation. The theoretical assessment to identify symmetries for filter design, conducted in \cref{bins_eqf_sec} through linearization error analysis, has demonstrated that modeling the coupling between bias and navigation states using a semi-direct product symmetry improves the linearization error of the filter’s navigation states, shifting the nonlinearities toward the bias states. This dissertation formalized the concept that every filter is equivariant and demonstrated that symmetry-based \acl{eqf} design represents a unified design methodology. Every modern filter, in fact, is derived as an \ac{eqf} for a specific choice of symmetry. Furthermore, experimental evaluations have demonstrated that a good choice of symmetry leads to better convergence, accuracy, and consistency than state-of-the-art solutions.

\Aclp{ins} are core components of modern autonomous robotic platforms. Consequently, the theoretical framework presented finds application in many modern robotic systems. The results presented in this dissertation extend beyond theory to practical application. First, it was shown that casting global-referenced measurements of state variables into body-referenced measurements that are compatible with the underlying symmetry allows for handling many measurement types with equivariant output. Additionally, \cref{ms_bins_ext_sec} introduces the concept of symmetry extension to account for sensors' calibration states and variables of interest in modern multi-sensor fusion problems. These findings are combined in a case study for the ArduPilot autopilot system, discussing the design and implementation of a robust, self-calibrating multi-sensor fusion \ac{eqf}-based algorithm.

The algorithm proposed in \cref{ms_bins_ap_sec} demonstrates robustness and low sensitivity to tuning parameters. It provides accurate state estimation even in the presence of measurement outliers, shifts, and faulty sensors, without the need for extensive fine-tuning and a whole set of exception cases. Therefore, it stands as a viable option for future iterations of commercial and open-source autopilot systems for \acl{uav} targeted at a broad spectrum of users. 

Finally, this dissertation presented a novel Lie group symmetry for \acl{vins} as well as the \acs{msceqf}, a \acl{msceqf} for vision-aided inertial navigation. This filter demonstrated state-of-the-art accuracy and out-of-the-box consistency properties stemming from the compatibility of the proposed symmetry with the reference frame invariance of the \acl{vio} problem. More importantly, the proposed system achieves remarkable robustness against significant absolute errors and showcases the ability to deal with ``you don't know what you don't know'', such as unexpected errors during operation. Additionally, the real-time capabilities of the \ac{msceqf} were tested in an experimental setting for closed-loop control of a resource-constrained aerial platform. 

Multi state constraint filters are generally excellent choices for high-performance vision-aided inertial navigation systems. However, most implementations incorporate \acl{zvu} routines, custom keyframe selection strategies, or fixed visual landmarks as part of their state. This addresses the filter's limitation in handling zero motion scenarios. In this context, an equivariant \acl{zvu} routine that exploits the underlying Lie group symmetry, resulting in no linearization error, is presented. In conclusion of \cref{msceqf_chp}, an extension of the proposed algorithm is presented, wherein fixed visual features are considered as part of the \ac{eqf} state, resulting in a hybrid-\ac{msceqf}.

The results presented in \cref{msceqf_chp} differentiate the \ac{msceqf} from conventional \ac{vio} algorithms, emphasizing its superior robustness. This sets the stage for future algorithms with minimal setup time, able to provide accurate estimation without requiring extensive fine-tuning and only requiring minimal sensor calibration procedures.

\section{Future work}

The work presented in this dissertation serves as a foundation for equivariant \aclp{ins}, opening the way for future research directions.

\subsection{Data-driven discovery of new symmetries and symmetries in data-driven methods}
Discovering symmetries of kinematic systems is indeed very hard, and, as of today, no golden recipe exists to find new symmetries. Recent works have exploited data-driven techniques to discover governing equations of dynamical systems~\cite{Camps-Valls2023DiscoveringData}. Deep symbolic regression~\cite{Petersen2019DeepGradients} and sparse identification of nonlinear dynamics~\cite{Brunton2016DiscoveringSystems} are promising technologies, and exploring data-driven techniques for symmetries discovery seems to be a promising research direction.

On the other hand, the concepts of symmetry and equivariance of \aclp{ins} introduced in this dissertation might find application in data-driven methodologies for robotic systems. Specifically, an interesting research direction is to explore whether encoding the equivariance and the symmetries of a kinematic system and, hence, the geometric structure associated with the law of motion within data-driven algorithms for robotics play a role in sample efficiency and generalization of those algorithms. 

\subsection{Equivariant symmetries for discrete-time inertial navigation systems and preintegration}
In \cref{math_G3_sec} of this dissertation, the inhomogeneous Galilean group $\grpG(3)$ and its algebraic properties are introduced. $\grpG(3)$ is the group of 3D rotations, translations in space and time, and transformations between frames of reference that differ only by constant relative motion. The inhomogeneous Galilean group is an overgroup of the extended special Euclidean group $\SE_2(3)$, which includes a time component. Due to this characteristic, it finds application in discrete-time \aclp{ins}~\cite{vanGoor2023ConstructiveNavigation, Wang2023MAVIS:Pre-integration}. Exploiting the inhomogeneous Galilean group and adapting the symmetries presented in this thesis for discrete-time systems and discrete-time \acl{eqf}~\cite{Ge2022EquivariantSystems} is thus an exciting research direction. Moreover, exploiting the tangent group of $\grpG(3)$ may lead to a novel equivariant theory of \ac{imu} preintegration, wherein \ac{imu} biases are explicitly considered, and the preintegrated navigation error is autonomous, in contrast to standard methodologies~\cite{Forster2017On-ManifoldOdometry, Barrau2020APreintegration, Eckenhoff2019Closed-formNavigation, Yang2020AnalyticNavigation, Brossard2021AssociatingEarth, Wang2023MAVIS:Pre-integration, Tsao2023AnalyticSystems}.

\subsection{Modular multi-sensor fusion and temporal calibration}
This dissertation introduced novel equivariant symmetries for \aclp{ins} and their extension for multi-sensor fusion problems. The low computational complexity, the natural consistency, and the robustness of the \acl{eqf} algorithm make it a perfect candidate for such applications. There are many different research problems in the world of multi-sensor fusion; however, the next logical step is represented by extending the results presented in this dissertation towards the problem of modular multi-sensor fusion~\cite{Brommer2021}, where the sensor suite, and hence the underlying symmetry, is extended or pruned at run-time in a modular fashion. 

A particularly interesting aspect of multi-sensor fusion problems is the time synchronization of the various sensors' readings and clocks and, hence, the ability of online temporal calibration of such sensors~\cite{Liu2022AFusion, Furgale2013, Li2014OnlineAlgorithms, Li2013Real-timeCamera, Guo2014EfficientTimestamps., Yang2019DegenerateCalibration}. The inhomogeneous Galilean group $\grpG(3)$ introduced in \cref{math_G3_sec} might be suitable for providing symmetry for systems where temporal calibration parameters are considered. Therefore, its exploitation as Lie group symmetry for \aclp{ins} with online sensors temporal calibration represents a fascinating research direction.

\subsection{Lidar-, radar-, dense-, and event-aided inertial odometry}
Potential research directions include the application and the extension of the symmetries presented in this dissertation for inertial navigation based on different sensor modalities, such as lidars~\cite{Xu2021FAST-LIO:Filter, Xu2022FAST-LIO2:Odometry, Wu2023LIO-EKF:Filters} and radars~\cite{Michalczyk2022Radar-InertialManoeuvres, Michalczyk2022Tightly-CoupledOdometry, Michalczyk2023Multi-StateLandmarks}. Lidar sensors directly provide scans of 3D features. Hence, the symmetries discussed in this dissertation should be directly applicable and lead to equivariance of the output. 

Radars are cost-effective sensor modalities, providing accurate range and doppler velocity measurements of 3D features in the environment while being insensitive to environmental conditions and privacy-preserving. It is certainly interesting to evaluate how the symmetries presented in this thesis apply and extend to the radar-inertial navigation problem. 

Additionally, interesting research directions are represented by investigating novel Lie group symmetries of the underlying nature of dense and event measurements for dense \acl{vio}~\cite{Hardt-Stremayr2019, Hardt-Stremayr2020} and event-based \acl{vio}~\cite{Zhu2017Event-basedOdometry} respectively.

\section{Conclusion}

This dissertation started with a clear goal: \emph{discover novel symmetries for \aclp{ins} that properly include a geometric treatment of the \ac{imu} biases and exploit them to design filtering algorithms that respect the geometry of the problem and outperform the state-of-the-art.}

In conclusion, this dissertation successfully achieved the overall goal of advancing the field of \aclp{eqf} for \aclp{ins}. 
By discovering novel equivariant symmetries, presenting a unified framework for modern inertial navigation and multi-sensor fusion problems, bridging theoretical results with practical applications, case study, and implementation, and finally showcasing the potential for \aclp{vins}, this work represents a leap forward in autonomous robotics and \aclp{ins}.

\appendix
\chapter{Derivation of \texorpdfstring{${\mathbf{C}^{\star}}$}{C star}}\label{appendix_B_chp}

Recall the definition of the equivariant error provided in \cref{eq_eqf_design}, ${e = \phi(E, \xizero) = \phi_{\xizero}(E)}$.
Let $h(e) \in \calN$ denote the configuration output computed at the error $e$. Assuming equivariance of the output, hence the existence of a valid right group action ${\rho \AtoB{\grpG \times \calN}{\calN}}$, the following relation holds:
\begin{equation*}
    h(e) = h(\phi(E, \xizero)) = \rho(E, h(\xizero)) = \rho(E, \yzero).
\end{equation*}

Normal coordinates are defined fixing ${\varepsilon = \log_{\grpG}(E)^{\vee}}$, thus ${E = X\hat{X}^{-1} = \exp_{\grpG}(\varepsilon^{\wedge})}$, and hence ${e = \phi_{\xizero}(\exp_{\grpG}(\varepsilon^{\wedge})}$. Choosing normal coordinates, one has
\begin{align*}
    \delta(h(e)) &= \delta(h(\vartheta^{-1}(\varepsilon)))\\
    &=\delta(\rho(E, \yzero)) = \delta(\rho(\exp_{\grpG}(\varepsilon^{\wedge}), \yzero)).
\end{align*}

The previous relation can be exploited to obtain two linearization points for $\delta(h(e))$. One at ${\varepsilon = 0}$ and a second one at $\varepsilon = \vartheta(e)$ in the direction given by $\varepsilon$~\cite{vanGoor2022EquivariantEqF, VanGoor2023EquivariantAwareness}.
Linearizing ${\delta(\rho(\exp_{\grpG}(\varepsilon^{\wedge}), \yzero))}$ at ${\varepsilon = 0}$ yields
\begin{equation*}
    \Fr{y}(\yzero)\delta{y}\Fr{E}{I}\rho(E, \yzero).
\end{equation*}
Linearizing ${\delta(\rho(\exp_{\grpG}(\varepsilon^{\wedge}), \yzero))}$ at $\varepsilon = \vartheta(e)$ in the direction given by $\varepsilon$ yields
\begin{align*}
    &\Fr{y}{\yzero}\delta(y) \circ \ddt\Bigr|_{t = 0}h(\phi(\exp_{\grpG}((1+t)\varepsilon^{\wedge}), \xizero))\\
    &\Fr{y}{\yzero}\delta(y) \circ \ddt\Bigr|_{t = 0}h(\phi(\exp_{\grpG}(\varepsilon^{\wedge})\exp_{\grpG}(t\varepsilon^{\wedge}), \xizero))\\
    &\Fr{y}{\yzero}\delta(y) \circ \ddt\Bigr|_{t = 0}h(\phi(E\exp_{\grpG}(t\varepsilon^{\wedge}), \xizero))\\
    &\Fr{y}{\yzero}\delta(y) \circ \ddt\Bigr|_{t = 0}h(\phi(\exp_{\grpG}(t\varepsilon^{\wedge}), \phi(E, \xizero)))\\
    &\Fr{y}{\yzero}\delta(y) \circ \ddt\Bigr|_{t = 0}h(\phi(\exp_{\grpG}(t\varepsilon^{\wedge}), e))\\
    &\Fr{y}{\yzero}\delta(y) \circ \ddt\Bigr|_{t = 0}\rho(\exp_{\grpG}(t\varepsilon^{\wedge}), h(e))\\
    &\Fr{y}{\yzero}\delta(y) \circ \ddt\Bigr|_{t = 0}\rho(\exp_{\grpG}(t\varepsilon^{\wedge}), \rho_{\hat{X}^{-1}}(y))\\
    &\Fr{y}{\yzero}\delta(y) \circ \Fr{E}{I}\rho(E, \rho_{\hat{X}^{-1}}(y)).
\end{align*}
Combining the two results yields
\begin{equation*}
    \mathbf{C}^{\star}\varepsilon = \frac{1}{2}\Fr{y}{\yzero}\delta\left(y\right)\left(\Fr{E}{I}\rho_E\left(\yzero\right) + \Fr{E}{I}\rho_E\left(\rho_{\hat{X}^{-1}}\left(y\right)\right)\right)\varepsilon^{\wedge}.
\end{equation*}

\chapter{EqF Matrices with Generic \texorpdfstring{${\xizero}$}{State Origin}}\label{appendix_A_chp}
Recall the $\grpD$ symmetry introduced in \cref{bins_tg_sec}. Here, we extend the filter state matrix for an \ac{eqf} based on the $\grpD$ symmetry presented in \cref{bins_At0_tg}, for an arbitrary choice of $\xizero$.

Recall the formula in \cref{eq_A0}. By choosing normal coordinates, the two most left differentials cancel out, yielding
\begin{equation*}
    \mathbf{A}_{t}^{0} = \Fr{e}{\xizero}\Lambda\left(e, \mathring{u}\right)\Fr{\varepsilon}{\mathbf{0}}\vartheta^{-1}\left(\varepsilon\right).
\end{equation*}
The derivation of the $\mathbf{A}_{t}^{0}$ is carried out by solving each differential separately and then combining the results. The overall derivation is tedious and lengthy, and therefore, we report the final result:
\begin{equation*}
    \mathbf{A}_{t}^{0} = \begin{bNiceArray}{c:c}[margin]
        \mathbf{\Psi} - \adMsym{\mathring{\Vector{}{b}{}}} & \eye_9 \Bstrut\\
        \hdottedline
        \adMsym{\mathring{\Vector{}{b}{}}}\mathbf{\Psi} - \adMsym{\left(\mathring{\Vector{}{w}{}} + \left(\mathbf{N} + \mathring{\Pose{}{}}^{-1}\left(\mathbf{G - N}\right)\right)^{\vee}\right)}\adMsym{\mathring{\Vector{}{b}{}}} & \adMsym{\left(\mathring{\Vector{}{w}{}} + \left(\mathbf{N} + \mathring{\Pose{}{}}^{-1}\left(\mathbf{G - N}\right)\right)^{\vee}\right)} \Tstrut
    \end{bNiceArray} \in \R^{18 \times 18},
\end{equation*}
with
\begin{align*}
    &\mathbf{\Psi} = \begin{bmatrix}
        \Vector{}{0}{3 \times 3} & \Vector{}{0}{3 \times 3} & \Vector{}{0}{3 \times 3}\\
        \left(\mathring{\mathbf{R}}^\top\Vector{}{g}{}\right)^{\wedge} & \Vector{}{0}{3 \times 3} & \Vector{}{0}{3 \times 3}\\
        \left(\mathring{\mathbf{R}}^\top\mathring{\mathbf{v}}\right)^{\wedge} & \eye_3 & \Vector{}{0}{3 \times 3}
    \end{bmatrix} \in \R^{9 \times 9},\\
    &\mathring{u} = (\mathring{\Vector{}{w}{}}, \mathring{\Vector{}{\tau}{}}) \coloneqq\psi_{\hat{X}^{-1}}(u) \in \vecL.
\end{align*}

Something worth discussing is how we can certify the correctness of the matrices we derive. One approach that was extensively used in the work presented in this dissertation is based on comparing the analytical form of the matrix with its \emph{numerical differentiation} counterpart. The following snippet shows a straightforward Python implementation of numerical differentiation for the $\mathbf{A}_{t}^{0}$ matrix.

\begin{boxsnippet}{Numerical filter state matrix}{}
\begin{minted}{python}
def numericalDifferential(f, x) -> np.ndarray:
    if isinstance(x, float):
        x = np.reshape([x], (1, 1))
    h = 1e-6
    fx = f(x)
    n = fx.shape[0]
    m = x.shape[0]
    Df = np.zeros((n, m))
    for j in range(m):
        ej = np.zeros((m, 1))
        ej[j, 0] = 1.0
        Df[:, j:j+1] = f(x + h * ej) - f(x - h * ej)
        Df[:, j:j+1] /= (2 * h)
    return Df


@dataclass
class Symmetry:
    # Definition of the Lie group symmetry $X \in \grpG$

     def __mul__(self, other) -> 'Symmetry':
        # Implementation of the group product rule

    @staticmethod
    def exp(x : np.ndarray) -> 'Symmetry':
        # Implementation of the group exponential $\exp(x^{\wedge})$

    @staticmethod
    def log(X : 'Symmetry') -> np.ndarray:
        # Implementation of the group logarithm $\log(X)^{\vee}$


@dataclass
class State:
    # Definition of the state $\xi \in \calM$

@dataclass
class Input:
    # Definition of the input $u \in \vecL$
    

def stateAction(X : Symmetry, xi : State) -> State:
    # Implementation of the $\phi$ action


def localCoords(e : State) -> np.ndarray:
    # Implementation of local coordinates $\varepsilon = \vartheta(e)$


def localCoordsInv(epsilon : np.ndarray) -> "State":
    # Implementation of the inverse local coordinates $e = \vartheta^{-1}(\varepsilon)$


def lift(xi : State, u : Input) -> np.ndarray:
    # Implementation of the lift $\Lambda$


def stateMatrixA(X_hat: Symmetry, u: Input) -> np.ndarray:
    xi_hat = stateAction(X_hat, xi_0)

    hdlr = lambda e: localCoords(stateAction(Symmetry.exp(
        lift(stateAction(X_hat, localCoordsInv(e)), u)), xi_0))
    A = numericalDifferential(hdlr, np.zeros((dof, 1)))

    return A
\end{minted}
\end{boxsnippet}

\backmatter
\bibliographystyle{plain}
\bibliography{bib/b042115, bib/eqf, bib/vio, bib/preinit}

\end{document}